\documentclass[11pt,letterpaper]{article} 
\usepackage{amsmath,amsthm,amssymb,graphicx,mathtools,setspace,mdframed,verbatim, dsfont,pifont,soul,wrapfig}
\usepackage{algorithm,algorithmicx,algpseudocode}
\usepackage{enumerate}
\usepackage[shortlabels]{enumitem}
\usepackage[title]{appendix}

\usepackage{tikz}
\usepackage{caption}
\usepackage{subcaption}

\usepackage[square,numbers]{natbib}
\usepackage{geometry}
\usepackage{bold-extra}
\usepackage{nicefrac}
\allowdisplaybreaks[4]
\usepackage[skins]{tcolorbox}

\usepackage[colorlinks=true, linkcolor=black, urlcolor=black, citecolor=black]{hyperref}

\usetikzlibrary{positioning}

\newenvironment{Msg}[1]
  {\mdfsetup{
    frametitle={\colorbox{white}{\space \large #1\space}},
    shadow=true,shadowsize=1pt,
    skipabove=2pt,
    innertopmargin=-3pt,
    innerbottommargin=7pt,
    innerrightmargin=7pt,
    innerleftmargin=7pt,
    frametitleaboveskip=-\ht\strutbox,
    frametitlealignment=\center,
    linewidth=0.5pt
    }
  \begin{mdframed}%
  }
{\end{mdframed}}

\newcommand{\R}{\mathbb{R}}
\newcommand{\E}{\mathbb{E}}
\newcommand{\N}{\mathcal{N}}

\newcommand{\vbrack}[1]{\langle #1\rangle}

\renewcommand{\S}{\mathcal{S}}
\newcommand{\D}{\mathcal{D}}

\newcommand{\Ecal}{\mathcal{E}}

\newcommand{\1}{\mathds{1}}
\newcommand{\sign}{\mathrm{sign}}

\renewcommand{\P}{\mathcal{P}}

\newcommand{\poly}{\mathsf{poly}}
\newcommand{\polylog}{\mathsf{polylog}}

\newcommand{\Sim}{\mathbf{sim}}

\newcommand{\StopGrad}{\mathsf{StopGrad}}

\newcommand{\BN}{\mathsf{BN}}
\newcommand{\dbrack}[1]{\langle #1\rangle}

\newcommand{\myref}[3][]{\hyperref[#2]{#3 \ref*{#2}#1}}
\newcommand{\newref}[2]{\hyperref[#1]{\textit{#2}}}

\algnewcommand{\INPUT}{\textbf{Input:} }

\newcounter{main}
\numberwithin{main}{section}
\newtheorem{theorem}[main]{Theorem}
\newtheorem{proposition}[main]{Proposition}
\newtheorem{lemma}[main]{Lemma}
\newtheorem{claim}[main]{Claim}
\newtheorem{induct}[main]{Inductions}
\newtheorem{corollary}[main]{Corollary}
\theoremstyle{definition}
\newtheorem{assumption}[main]{Assumption}
\newtheorem{definition}[main]{Definition}

\newtheorem{fact}[main]{Fact}

\theoremstyle{remark}
\newtheorem{remark}[main]{Remark}

\numberwithin{equation}{section}

\tikzstyle{arrow} = [thick,->,>=stealth]

\begin{document}

\title{The Mechanism of Prediction Head in Non-contrastive Self-supervised Learning}

\author{
    Zixin Wen\\
    \texttt{\href{mailto:zixinw@andrew.cmu.edu}{\color{black}zixinw@andrew.cmu.edu}}\\
    Carnegie Mellon University
    \and 
    Yuanzhi Li\\
    \texttt{\href{mailto:yuanzhil@andrew.cmu.edu}{\color{black}yuanzhil@andrew.cmu.edu}}\\
    Carnegie Mellon University
}

\date{May 13, 2022}

\maketitle

\begin{abstract}
    Recently the surprising discovery of \textit{Bootstrap Your Own Latent} (\texttt{BYOL}) method by \citeauthor{grill2020bootstrap} shows the negative term in contrastive loss can be removed if we add the so-called \textit{prediction head} to the network architecture, which breaks the symmetry between the positive pairs. This initiated the research of \textit{non-contrastive self-supervised learning}. It is mysterious why even when trivial collapsed \textit{global optimal} solutions exist, neural networks trained by (stochastic) gradient descent can still learn competitive representations and avoid collapsed solutions. This phenomenon is one of the most typical examples of implicit bias in deep learning optimization, and its underlying mechanism remains little understood to this day. 

    In this work, we present our empirical and theoretical discoveries about the mechanism of prediction head in non-contrastive self-supervised learning methods. Empirically, we find that when \textbf{the prediction head is initialized as an identity matrix with only its off-diagonal entries being trained}, the network can learn competitive representations even though the trivial optima still exist in the training objective. Moreover, we observe a consistent rise and fall trajectory of off-diagonal entries during training. Our evidence suggests that understanding the identity-initialized prediction head is a good starting point for understanding the mechanism of the trainable prediction head. 

    Theoretically, we present a framework to understand the behavior of the trainable, but identity-initialized prediction head. Under a simple setting, we characterized the \textbf{substitution effect} and \textbf{acceleration effect} of the prediction head during the training process. The substitution effect happens when {learning the stronger features in some neurons can substitute for learning these features in other neurons through updating the prediction head}. And the acceleration effect happens when {the substituted features can accelerate the learning of other weaker features to prevent them from being ignored}. These two effects together enable the neural networks to learn all the features rather than focus only on learning the stronger features, which is likely the cause of the dimensional collapse phenomenon. To the best of our knowledge, this is also the first end-to-end optimization guarantee for non-contrastive methods using nonlinear neural networks with a trainable prediction head and normalization.
\end{abstract}

\thispagestyle{empty}
\clearpage
\renewcommand{\baselinestretch}{0.9975}\normalsize
\tableofcontents
\renewcommand{\baselinestretch}{1.0}\normalsize

\thispagestyle{empty}
\clearpage
\setcounter{page}{1}

\section{Introduction}

Self-supervised learning is about learning representations of real-world vision or language data without human supervision, and contrastive learning \cite{oord2018representation,hjelm2018learning,he2020momentum,Chen2020,caron2020unsupervised,gao2021simcse} is one of the most successful self-supervised learning approaches. It has been known that the behavior of contrastive learning depends critically on the minimization of the \textit{negative term}, which corresponds to contrasting the representations of \textit{negative pairs}, i.e., pairs of different data points. However, the surprising finding of the \textit{Bootstrap Your Own Latent} (\texttt{BYOL}) method by \citet{grill2020bootstrap} initiated the research of \textit{non-contrastive self-supervised learning}, which refers to contrastive learning methods without using the negative pairs. \texttt{BYOL} achieved state-of-the-art results in various computer vision benchmarks and there are plenty of follow-up works \cite{grill2020bootstrap,chen2021exploring,caron2021emerging,bardes2021vicreg,ermolov2021whitening,zbontar2021barlow,hua2021feature,niizumi2021byol} making improvements in this direction.

\begin{wrapfigure}{r}{0.32\textwidth}
    \vspace*{-10pt}\caption{\footnotesize \textbf{Dimensional Collapse.} Network trained without prediction head will learn extremely correlated neurons.}\centering
    \vspace*{-10pt}
    {\includegraphics[width=0.31\textwidth]{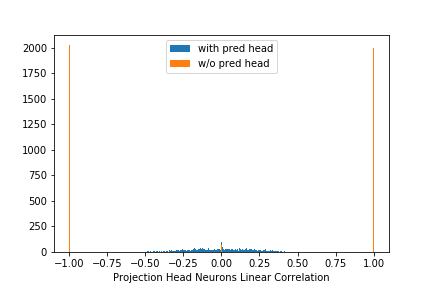}\vspace*{-5pt}
    \subcaption{\footnotesize Histograms of the correlations of projection head neurons.}}
    {\includegraphics[width=0.31\textwidth]{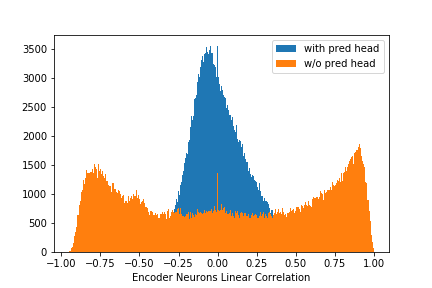}\vspace*{-5pt}
    \subcaption{\footnotesize Histograms of the correlations of encoder network neurons (before projection head).}}\vspace*{-5pt}
    \label{wrap-fig:1}
\end{wrapfigure} 

On a high level, in non-contrastive self-supervised learning, one wishes to learn a network \(\phi\) such that \(\phi(x)\) aligns in direction with \(\phi(x')\), where \(x\) and \(x'\) are called the \textit{positive pair}, generated by random augmentations from the same sample. Without contrasting the negative pairs, it is extremely easy for neural networks to cheat the learning task by learning certain inferior representations. One trivial solution known as the \textit{complete collapse} is when \(\phi(\cdot)\) is a constant vector whose variance is zero. Another trivial \emph{global optimal} solution, typically learned by the neural network after training, is when all the coordinates $\phi_i(\cdot)$ are exactly aligned, which is named as \textbf{dimensional collapse} by \citet{hua2021feature}. Nevertheless, adding a trainable prediction head on top of (one branch of) $\phi(x)$ magically avoids learning such solutions, \textbf{even though the prediction head \emph{can} possibly learn the identity mapping and render itself useless}. It is mysterious why even if the network can minimize the training objective by learning an identity prediction head and a collapsed encoder network \(\phi(\cdot)\), it still optimizes for a non-collapsed state-of-the-art representation instead when trained by (variants of) stochastic gradient descent (SGD).

Since the proposition of \texttt{BYOL}, there have been lots of empirical studies trying to understand non-contrastive learning. The \texttt{SimSiam} method by \citet{chen2021exploring} shows the exponential moving average (EMA) is not necessary for avoiding collapsed solutions while \textbf{stop-gradient} is necessary. \citet{richemond_byol_2020} empirically disproved the conjecture that information leakage from batch normalization (BN) is the reason why \texttt{BYOL} can avoid collapse. \texttt{DINO} \cite{caron2021emerging} further explored replacing the normalized \(\ell_2\)-loss by a cross-entropy loss. \citet{zhang2022does} gives empirical evidence that using a single bias layer as a prediction head is capable of avoiding collapsed solutions. {All the methods above use loss functions that are asymmetric with respect to the positive pair.} If one wishes to work without both asymmetry and the negative pairs, one must add extra diversity-enforcing structures say neuron-wise regularization in \textit{Barlow Twins} \cite{zbontar2021barlow} or a more complicated output normalization scheme than BN \cite{ermolov2021whitening,hua2021feature}. The seminal works~\cite{zbontar2021barlow,hua2021feature} provide empirical evidence that the prediction head encourages the network to learn more diversified features. But in theory, the question of {how the prediction head helps in learning those diverse features} is still unanswered.

Despite the great empirical effort put to investigate these non-contrastive learning methods, there is very little theoretical progress towards explaining them. Most of existing theories focus on contrastive learning, especially from the statistical learning perspective \cite{tosh2020contrastive,tsai2020self,bansal2020self,tosh2021contrastive,haochen2021provable,von2021self,ash2021investigating,bao2021sharp,ji2021power,huang2021towards,luo2022one}. The theoretical tools used in these paper rely heavily on the properties of the minima of loss function. However, due to the existence of trivial dimensional collapsed \emph{global optimal} solutions (even with the prediction head) of the non-contrastive methods, to the best of our knowledge, \textit{there is no well-established statistical framework for those methods yet}. To explain the non-contrastive learning, it is inevitable to study how the solutions are chosen during the optimization. Therefore, we consider understanding the optimization process to be crucial for understanding these methods. Our research questions are:

\begin{Msg}{Our theoretical questions: the role of prediction head}
    Why do most non-contrastive self-supervised methods learn collapsed solutions when the so-called prediction head is absent in the network architecture? How does the \emph{trainable} prediction head help \textbf{optimize} the neural network to learn more diversified representations in non-contrastive self-supervised learning?
\end{Msg}

\paragraph{Theoretical challenges of our questions.} Due to the existence of trivial collapsed optimal solutions of the non-contrastive learning objective, we need to understand the \textbf{implicit bias in optimization} posed by the prediction head. However, to the best of our knowledge, all of the previous implicit biases theories focus only on the supervised learning tasks, and thus cannot be applied to our question. Even though \cite{wen2021toward} has characterized the training trajectory of contrastive learning, its analysis cannot incorporate the training of the prediction head. In theory, the optimization of nonlinear neural networks with at least two trainable layers in self-supervised learning is still intractable. A detailed explanation of our challenges will be given in \myref{sec:prelim}{Section}.

There are already some theoretical papers \cite{tian2021understanding,wang2021towards,pokle2022contrasting} that try to address similar questions. While none of these papers studied the training process of the prediction head, our results provide a completely different perspective: \textbf{We explain why \emph{training} the prediction head can encourage the network to learn diversified features and avoid dimensional collapses}, \ul{even when the trivial collapsed optima still exist in the training objective}, which is not covered by the prior works. We defer the detailed comparison of similar works to \myref{sec:comparison}{Section}. On a high level, the results in this paper are summarized as follows:
\vspace{-4pt}
\paragraph{Our empirical contributions.}
In non-contrastive self-supervised learning, we obtain the following experimental results:
\vspace{-1pt}
\begin{itemize}
    \item We discover empirically that even when the prediction head is \textbf{linear} and initialized as an identity matrix with only off-diagonal entries being trainable, the performance of learned representation is comparable to using the usual non-linear two-layer MLP or randomly initialized (trainable) linear prediction head. This disproves the belief that non-symmetric initialization of the online and target network is needed. See \myref{fig:identity-init-performance}{Figure}.
    \vspace*{-4pt}
    \item We empirically verified that even when the prediction head is an identity-initialized matrix, it does not always converge to a symmetric matrix during training. This proves the trainable prediction head does not need to behave like a symmetric matrix during most of the training process. Therefore the theories based on symmetric prediction head \cite{tian2021understanding,wang2021towards} cannot fully explain the behaviors of the trainable prediction head. See \myref{fig:pred-head-trajectory-1}{Figure} and \myref{fig:pred-head-trajectory-2}{Figure}.
\end{itemize}
\vspace{-12pt}
\paragraph{Our theoretical contributions.} 
We based our theory on a very simple setting, where the data consist of two features: the strong feature and the weak feature. Intuitively, we can think of the strong features in a dataset are the ones that show up more frequently or with large magnitude, and weak features as those that show up rarely or with small magnitude. We consider learning with a \textbf{two-layer non-linear neural network with output normalization} using (stochastic) gradient descent. Under this setting, we obtain the following results. 
\vspace{-4pt}
\begin{itemize}
    \item We prove that without a prediction head, even with BN on the output to avoid complete collapse, the networks will still converge to dimensional collapsed solutions, which provides a theoretical explanation to the dimensional collapse phenomenon observed in \cite{hua2021feature}.
    \vspace{-4pt}
    \item We prove that the trainable prediction head, combined with suitable output normalization and stop-gradient operation, can learn diversified features to avoid the dimensional collapse problem. We characterize two effects of prediction head: the \textbf{substitution effect} and the \textbf{acceleration effect}. The intuitions of these two effects are summarized below:
\end{itemize}

\begin{figure*}[t!]\centering
    \begin{subfigure}[1]{0.32\textwidth}
        \centering
        \includegraphics[width=\textwidth]{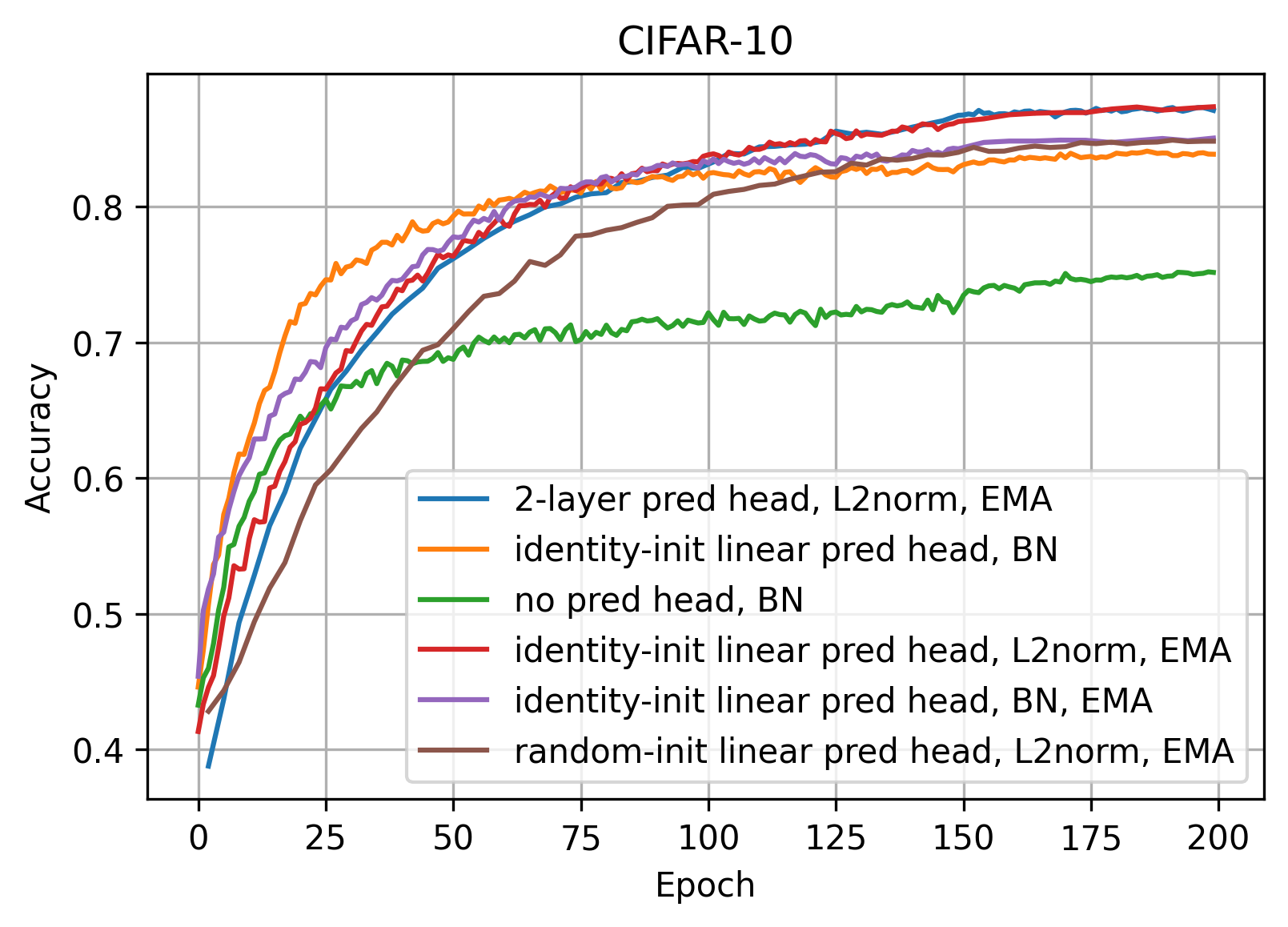}
        \caption{\small CIFAR-10 Accuracy}
    \end{subfigure}
    \hfill
    \begin{subfigure}[2]{0.32\textwidth}
        \centering
        \includegraphics[width=\textwidth]{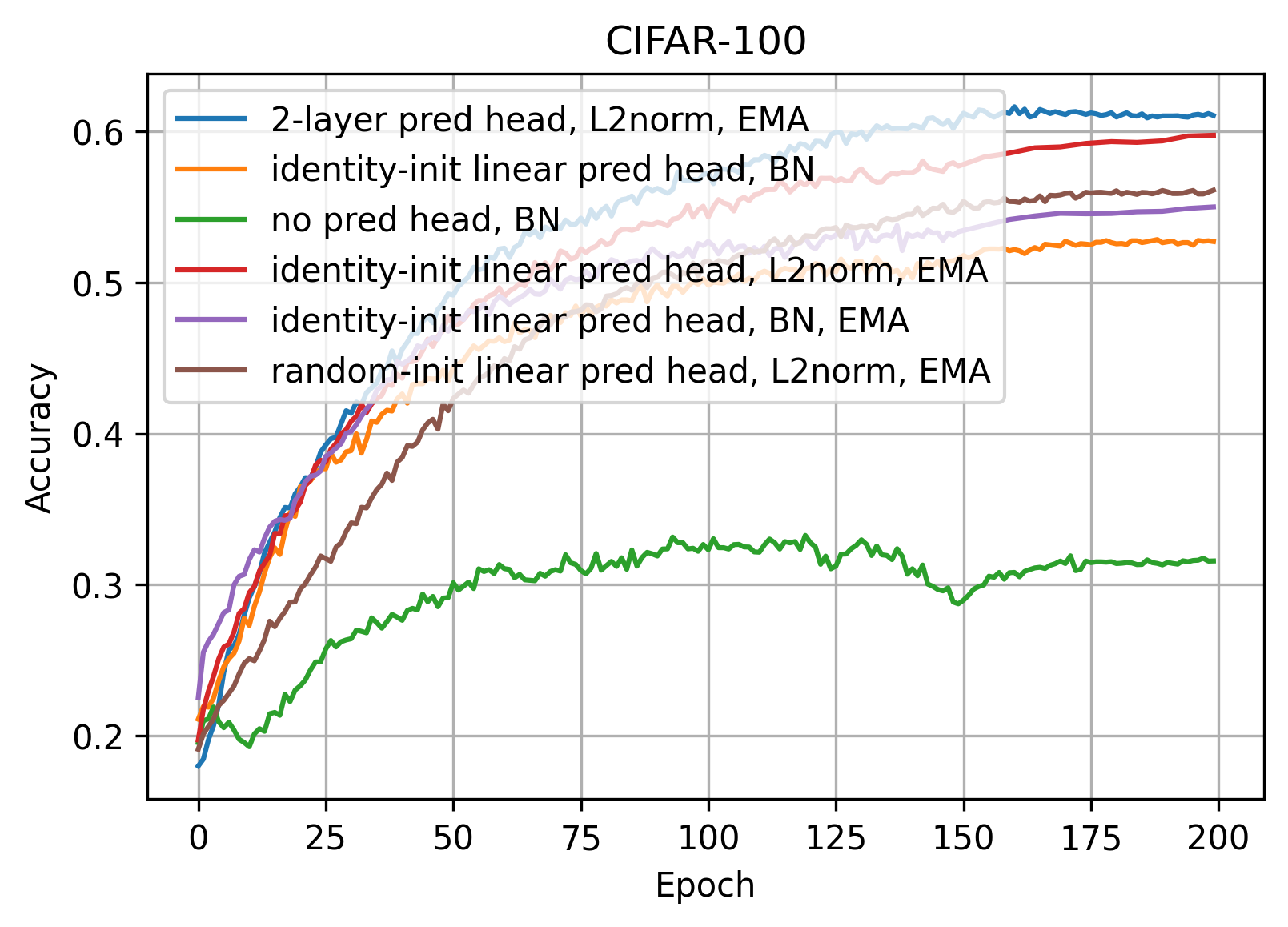}
        \caption{\small CIFAR-100 Accuracy}
    \end{subfigure}
    \hfill
    \begin{subfigure}[3]{0.32\textwidth}
        \centering
        \includegraphics[width=\textwidth]{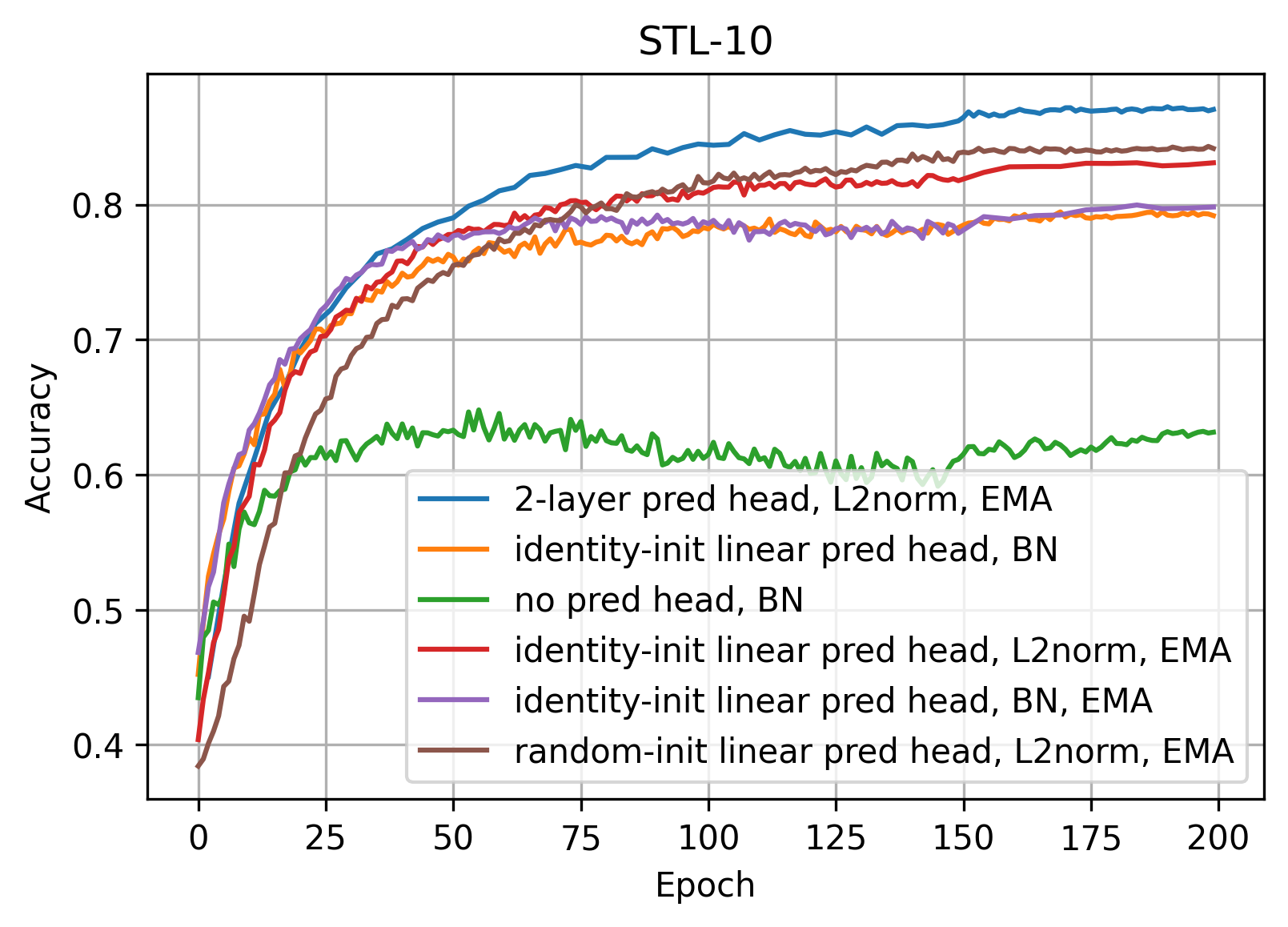}
        \caption{{\small STL-10 Accuracy}}
    \end{subfigure}\\
    \vspace{8pt}
    \begin{subfigure}[4]{0.32\textwidth}
        \centering
        \includegraphics[width=\textwidth]{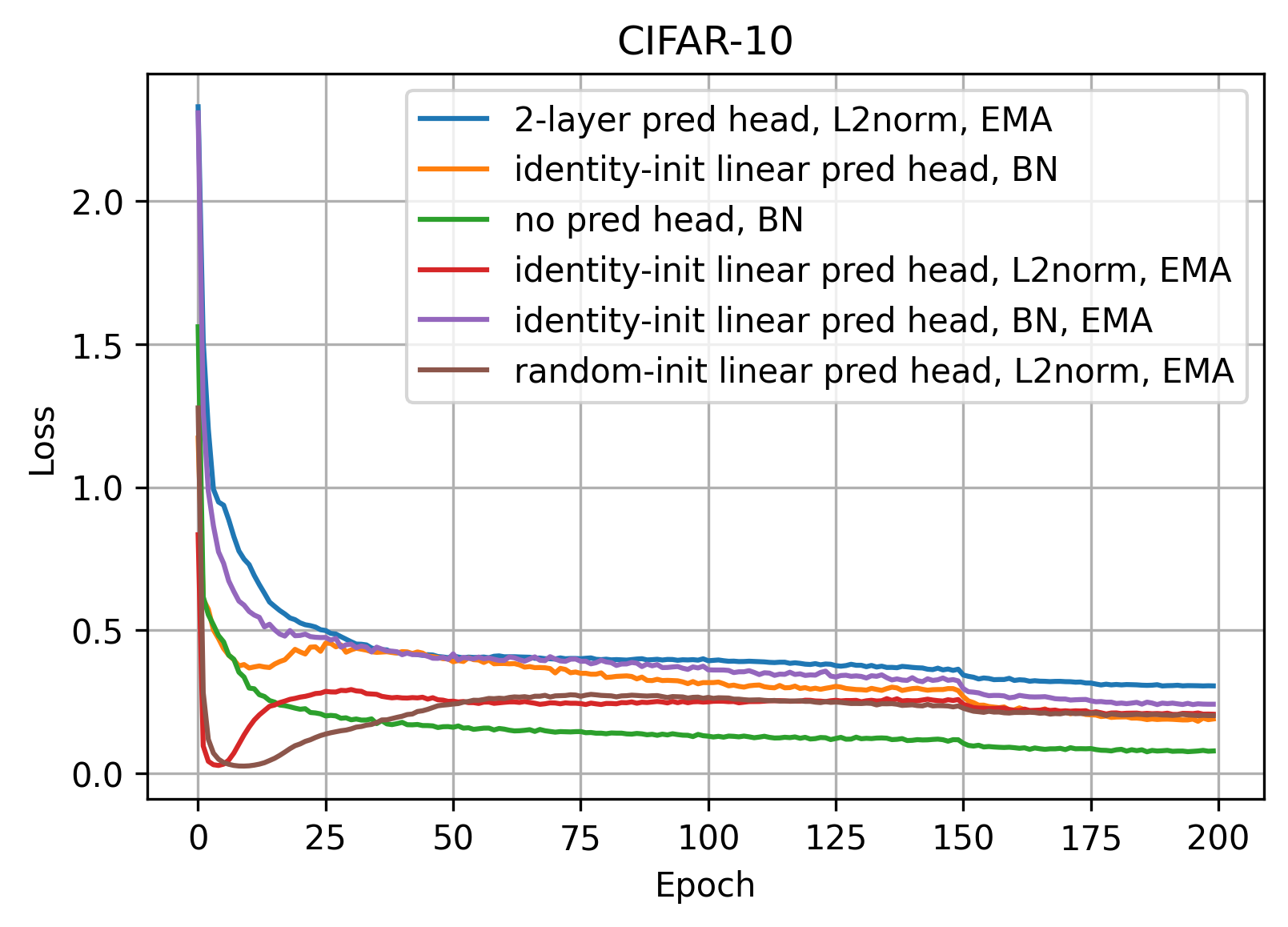}
        \caption{\small CIFAR-10 Loss}
    \end{subfigure}
    \hfill
    \begin{subfigure}[5]{0.32\textwidth}
        \centering
        \includegraphics[width=\textwidth]{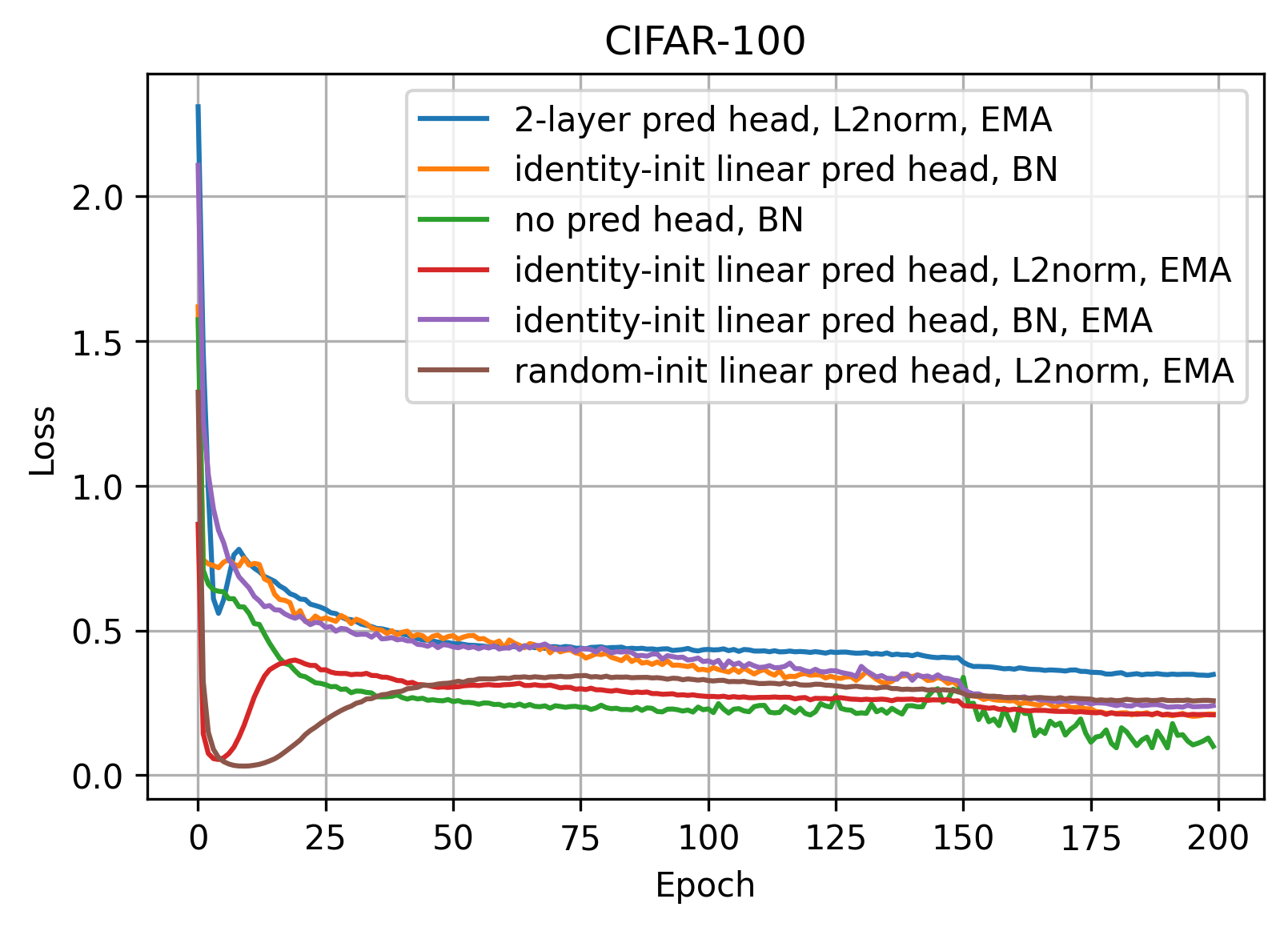}
        \caption{\small CIFAR-100 Loss}
    \end{subfigure}
    \hfill
    \begin{subfigure}[6]{0.32\textwidth}
        \centering
        \includegraphics[width=\textwidth]{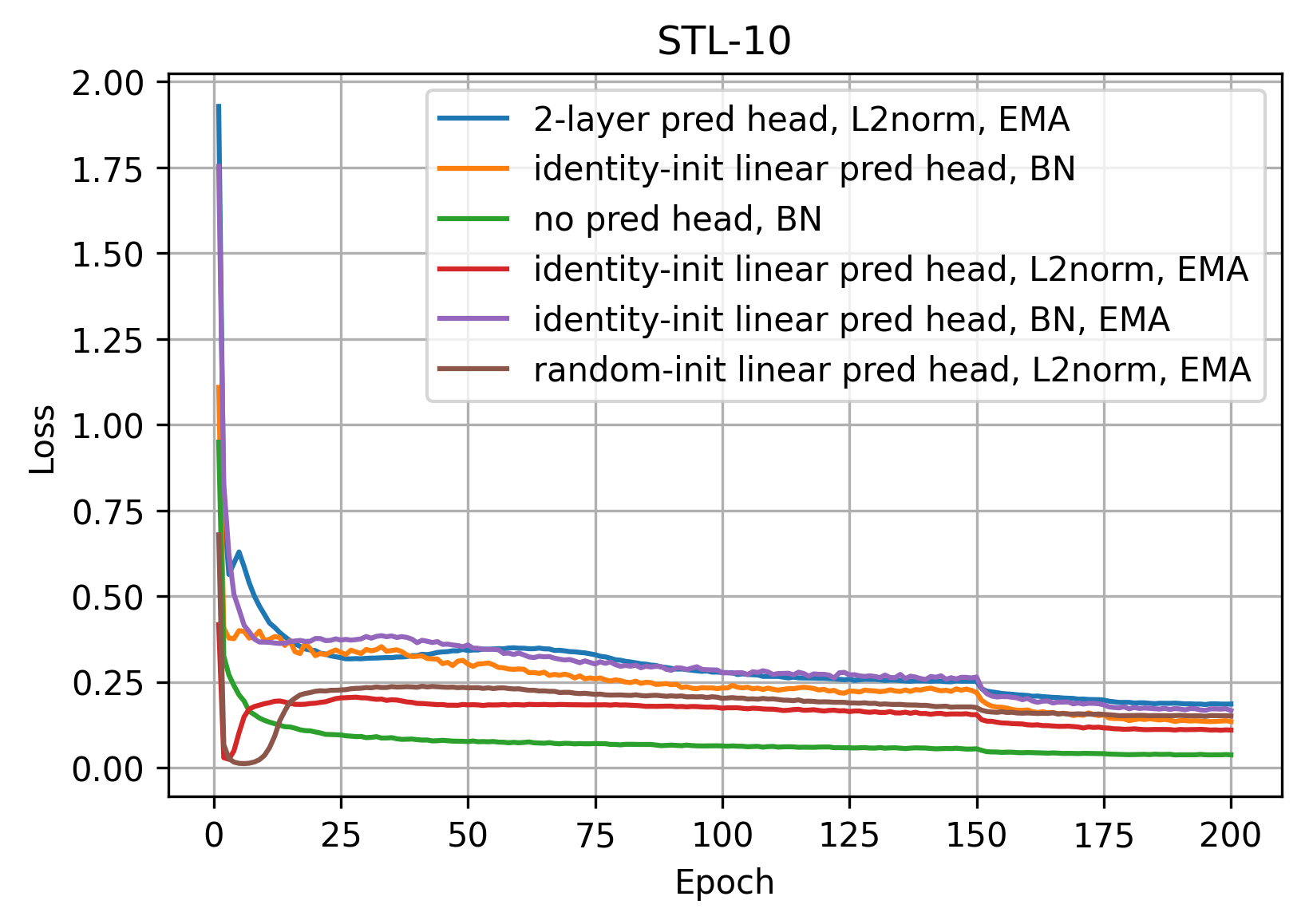}
        \caption{\small STL-10 Loss}
    \end{subfigure}
    \caption{\small Performances of using different prediction heads. Here in CIFAR-10, CIFAR-100 and STL-10, identity-initialized linear prediction head can achieve good accuracies comparable to commonly used two-layer non-linear MLP or randomly-initialized linear head. All the prediction heads are trainable, while for identity-initialized prediction head only the off-diagonal entries are trainable. Here BN or L2norm represents the output normalization, and EMA represents using exponential moving average to update the target network as in \texttt{BYOL} \cite{grill2020bootstrap}. More details of these experiments can be seen in \myref{sec:experiment-detail}{Section}.}
    \label{fig:identity-init-performance}
\end{figure*}

\begin{Msg}{The mechanism of the trainable prediction head}\label{msg:main message}
   In our setting, we prove that (1) without the prediction head, all the neurons will only learn the strongest feature in the data set thus causing dimensional collapses; (2) the trainable prediction head can help to learn weak features by leveraging two effects: the \textbf{substitution effect} and the \textbf{acceleration effect}. The substitution effect happens when by learning the prediction head, the learned stronger features in some neurons can substitute for learning the same features in other neurons, which decreases the learning speed of strong features in those neurons. And the acceleration effect happens when the strong features substituted via the prediction head can further accelerate the learning of weaker features in those substituted neurons.
\end{Msg}

Besides the above effects, we also explain, in our setting, how the two common components in non-contrastive learning: \emph{stop-gradient} operation and \emph{output normalization}, can assist the prediction head in creating those effects during the training process. We point out it is the \underline{interactions} between these components, rather than their individual effects, that ensure the success of training. We shall discuss this in more detail in \myref{sec:accelerate}{Section}.

\vspace*{-5pt}

\subsection{Comparison to Similar Studies}\label{sec:comparison}

In this section, we will clarify the differences between our results and some similar studies. Especially the theoretical papers by \citet{tian2021understanding} and \citet{wang2021towards}. \citet{pokle2022contrasting} compared the landscapes between contrastive and non-contrastive learning and points out the existence of non-collapsed bad minima for non-contrastive learning without a prediction head. 

We point out that all the claims below are derived \textbf{only in our theoretical setting} and are partially verified in experiments over datasets such as CIFAR-10, CIFAR-100, and STL-10.
\begin{figure}[t!]\centering
    \caption{\small Trajectories of the identity-initialized prediction head. \(\textrm{off-diag}(E)\) is obtained by setting the diagonal of \(E\) to be zero. In (a), we discover that over all three datasets considered here, the Frobenius norm of our identity-initialized prediction head's off-diagonal matrix clearly display a two stage separation, more precisely, a rise and fall pattern; In (b), The off-diagonal matrix of the prediction head is not symmetric in CIFAR-10 and CIFAR-100. Since the diagonal entries are fixed to one, our measure is more accurate in measuring the symmetricity of the prediction head matrix.}
    \begin{subfigure}[1]{0.45\textwidth}
        \centering
        \includegraphics[width=\textwidth]{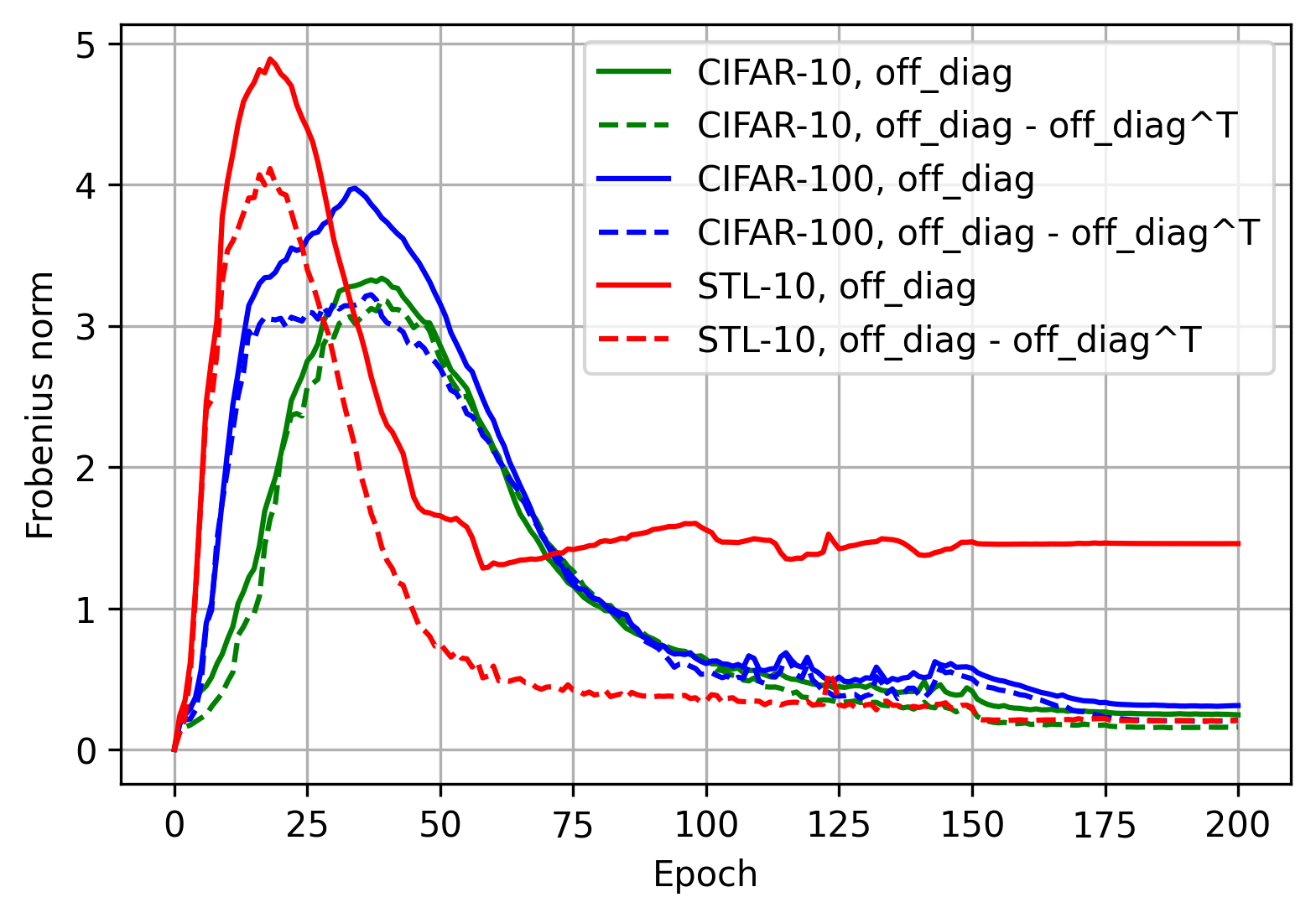}
        \caption{\small \(\|\textrm{off-diag}(E^{(t)})\|_F\) and \(\|E^{(t)} - (E^{(t)})^{\top}\|_F\)}
    \end{subfigure}
    \hfill
    \begin{subfigure}[2]{0.45\textwidth}
        \centering
        \includegraphics[width=\textwidth]{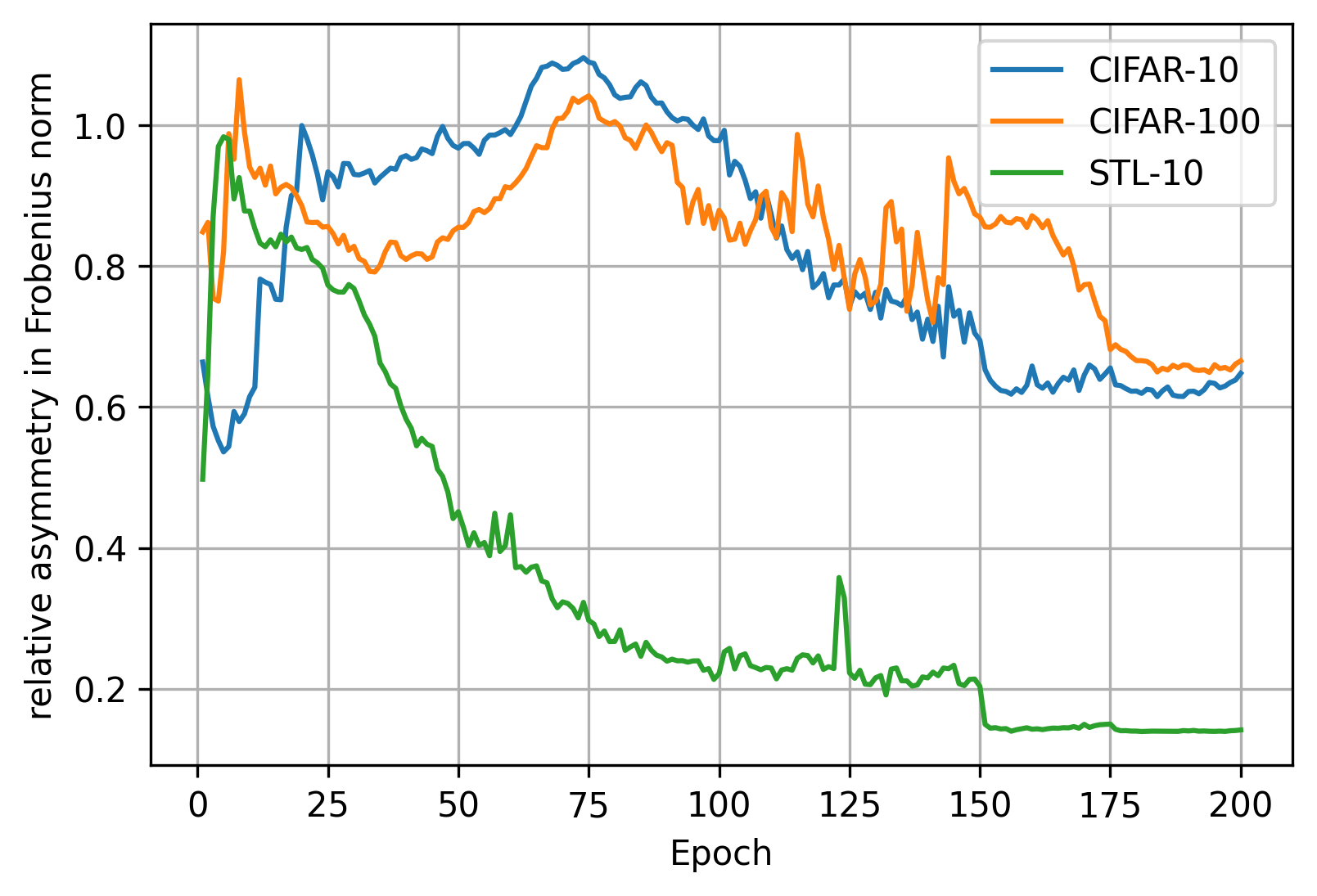}
        \caption{\small \(\|E^{(t)} - (E^{(t)})^{\top}\|_F/ \|\textrm{off-diag}(E^{(t)} )\|_F\)}
    \end{subfigure}\hspace{15pt}\
    \label{fig:pred-head-trajectory-1}
\end{figure}
\vspace*{-5pt}
\paragraph{Can eigenspace alignment explain the effects of training the prediction head?}
The paper \cite{tian2021understanding} presented a theoretical statement that (symmetric) linear prediction head will \emph{converge} to a matrix that commutes with the covariance matrix of linear representations \emph{at the end of training}, and they provided experiments to support their theory. However, our theory suggests that \textbf{the intermediate stage of training the prediction head matters more to the feature learning of the base network than the convergence stage}. Indeed, as shown in Figure~\ref{fig:pred-head-trajectory-1}, in many cases, the trainable projection head will \textbf{converge back to identity} after training, which commutes with any covariance matrix. However, simply setting the prediction head to identity without training leads to significantly worse results. Therefore, we believe that it is critical to study the entire learning process to understand the role of the prediction head. We prove that in our setting, the substitution effect and the acceleration effect happen during the stage when the networks are trying to learn the weaker features, and after that, the prediction head will converge back to the identity matrix at the end of training (see \myref{prop:5-pred-head-convergence}{Proposition}). Again, we emphasize that our characterization of the prediction head trajectory is partially verified by the experiments in \myref[a]{fig:pred-head-trajectory-1}{Figure}: \textbf{the training trajectory of the prediction head displays a clear two-stage separation}, which demonstrates that the convergence result (e.g., the eigenspace alignment result in \cite{tian2021understanding}) is not sufficient to characterize the training process of prediction head. We conjecture the result in \cite{tian2021understanding} on the prediction head is due to a similar convergence result we obtain at the end of training. 
\vspace*{-5pt}
\paragraph{Can the symmetric prediction head explain the trainable prediction head?} 
In the paper \cite{tian2021understanding}, experiments over the STL-10 dataset showed that the linear prediction head tends to converge to a symmetric matrix during training. And the follow-up paper \cite{wang2021towards} established a theory under the symmetric prediction head (which is not trained but manually set at each iteration). However, similar to the reason why eigenspace alignment cannot fully explain the effects of the prediction head, the symmetric prediction head given in \cite{wang2021towards} might not explain the trainable prediction head as well. Under their linear network setting, where \(W\) is the weight matrix of the base encoder, they manually set the prediction head \(W_p\) at iteration \(t\) to be
\begin{align}\label{eqdef:symmetric-pred-head}
    W_p^{(t)} \gets  W^{(t)} \E_{x_1} x_1x_1^{\top}(W^{(t)})^{\top} 
\end{align}
and the outputs of both online and target network are not normalized. Under this manual update rule of the prediction head, they proved a subspace learning result under gaussian data setting.

Nevertheless, our experiments in \myref{fig:identity-init-performance}{Figure} and \myref[b]{fig:pred-head-trajectory-1}{Figure} show that even if we initialize the prediction head using a symmetric matrix (identity), \textbf{the trainable prediction head can be very asymmetric at the early training stage when the encoder network learn most of its features}. Moreover, \myref[b]{fig:pred-head-trajectory-1}{Figure} demonstrates that the prediction heads in CIFAR-10 and CIFAR-100 experiments do not converge to a symmetric matrix. In accord with these experiments, our theory suggests that the prediction head cannot converge to a symmetric matrix before the encoder network has successfully learned all the features. Moreover, the theory in \cite{wang2021towards} cannot distinguish between learning complete collapsed (zero) solutions and learning dimensional collapsed ones, therefore cannot explain why the prediction can help avoid the dimensional collapse. Actually, in the presence of feature imbalance (e.g.,  \(\E_{x_1}x_1x_1^{\top}\) has huge eigen-gap), the symmetric prediction head in \eqref{eqdef:symmetric-pred-head} is also likely to collapse into a rank-one matrix where \(W\) focus on learning the largest eigenvector of the covariance \(\E_{x_1}x_1x_1^{\top}\).

The differences between our results and \cite{wang2021towards}‘s are in that we are based on nonlinear network architecture and a trainable prediction head. Indeed, our theory and experiments in \myref{fig:synthetic-experiment}{Figure} show that when feature imbalance happens (which is very common in vision datasets \cite{chen2021intriguing}), training a nonlinear network would cause discrepancies in the learning pace between different neurons. We proved that {by becoming asymmetric}, the trainable prediction head can leverage such discrepancies and creates the substitution effect (see \myref{lem:3-substitute}{Lemma}) and the acceleration effect (see \myref{lem:4-accelerate-effect}{Theorem}) to help feature learning. We believe this proves that {asymmetry is the key to explaining the implicit bias of the trainable prediction head} and our results establish the {symmetry-breaking mechanism of the prediction head} in non-contrastive learning. 

\begin{figure}[t!]\centering
    \begin{subfigure}[1]{0.33\textwidth}
        \centering
        \includegraphics[width=\textwidth]{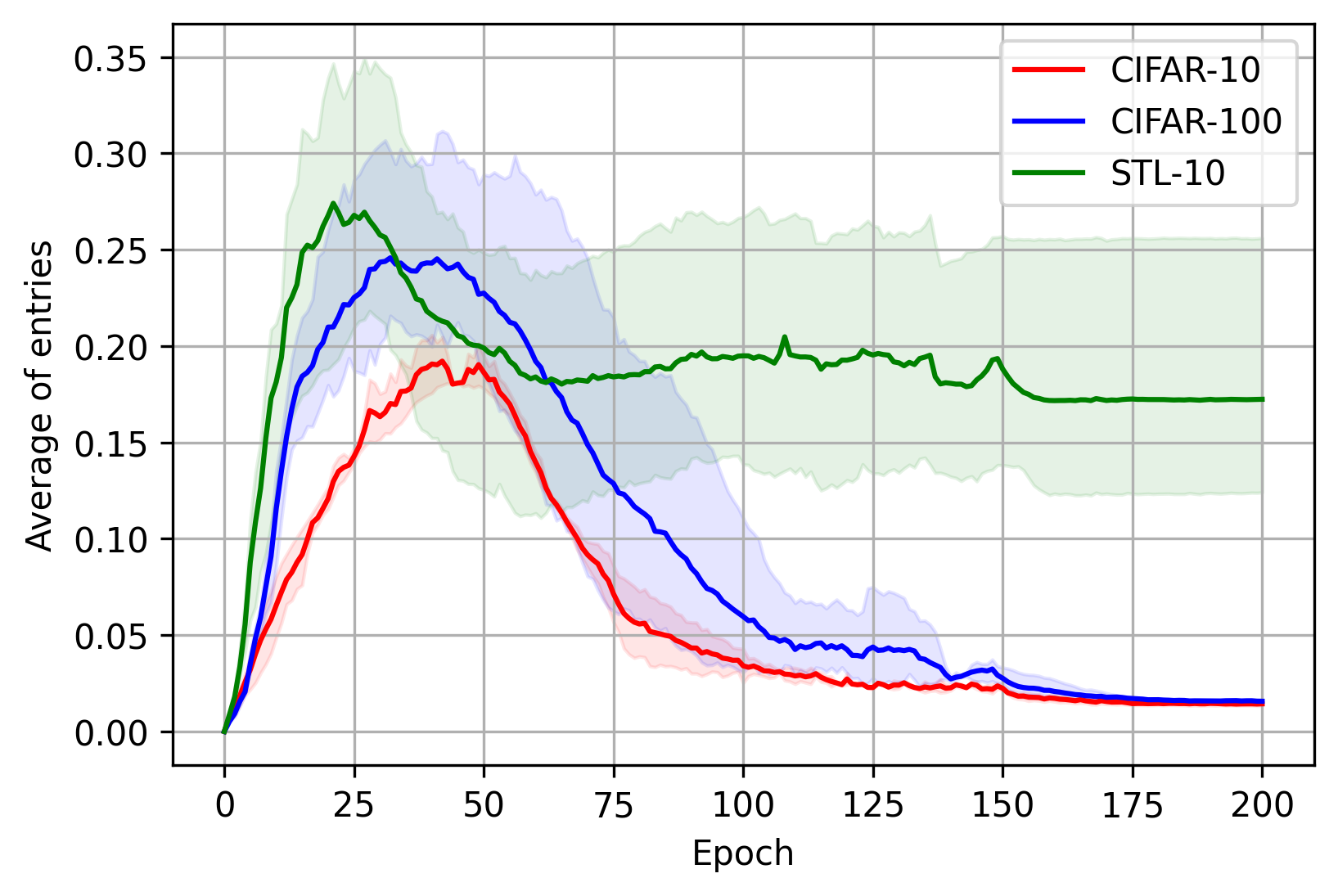}
        \caption{\small Average of off-diag entries}
    \end{subfigure}
    \hfill
    \begin{subfigure}[2]{0.32\textwidth}
        \centering
        \includegraphics[width=\textwidth]{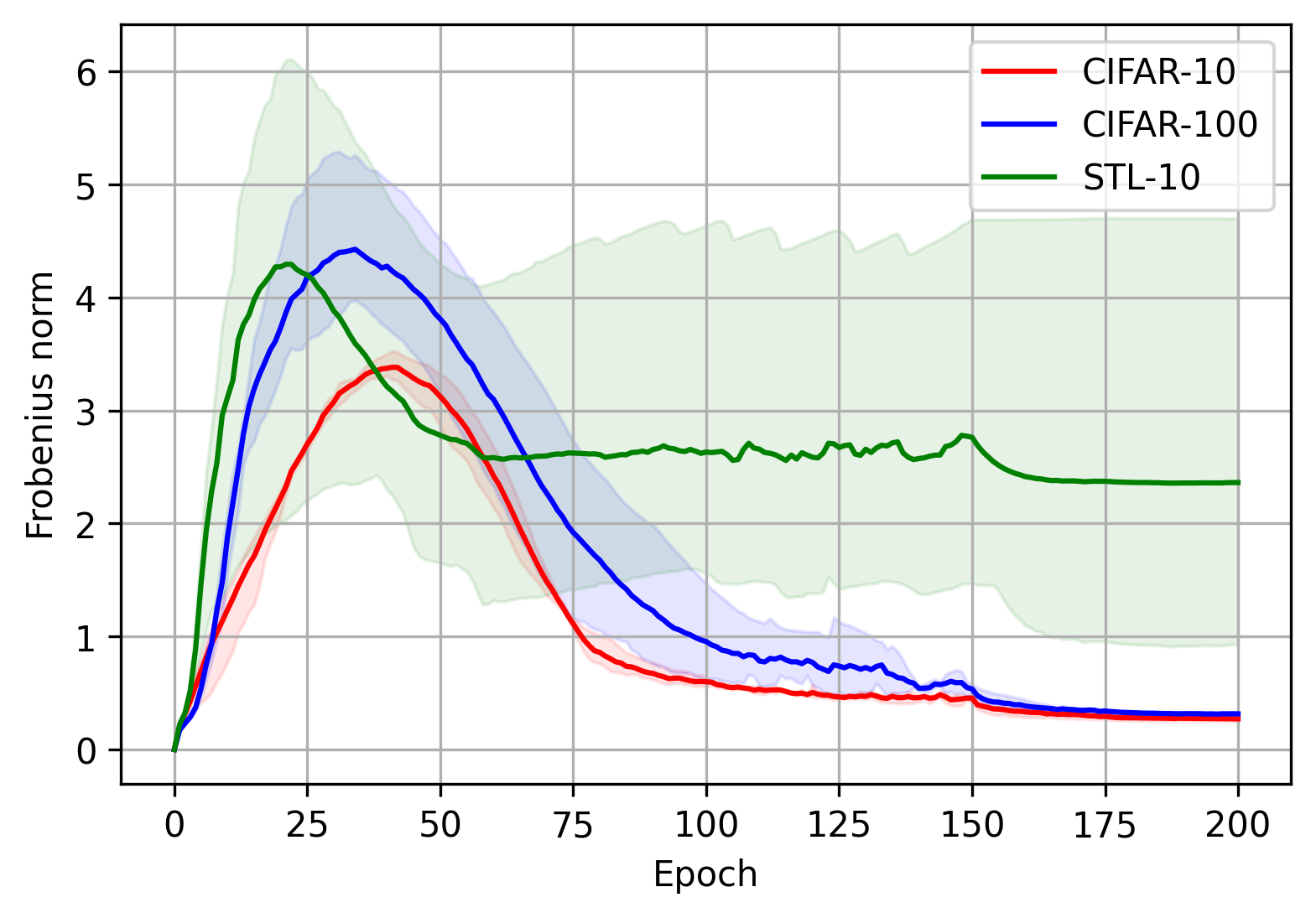}
        \caption{\small F-norm of off-diag matrix}
    \end{subfigure}
    \hfill
    \begin{subfigure}[2]{0.33\textwidth}
        \centering
        \includegraphics[width=\textwidth]{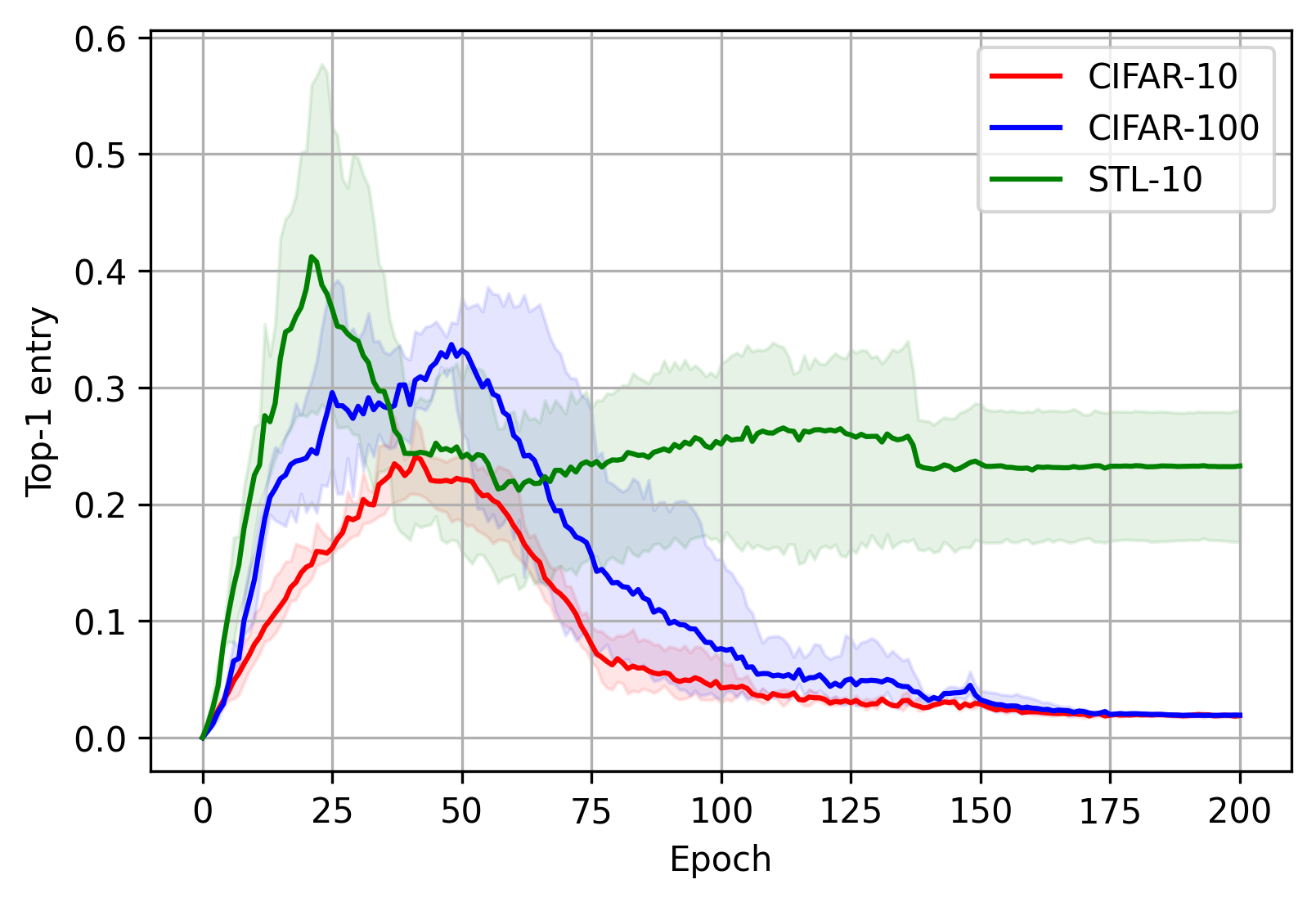}
        \caption{\small Maximum of off-diag entries}
    \end{subfigure}
    \caption{\small Trajectories of the identity-initialized prediction head with a \((\min,\max)\) confidence band, average over 3 runs. In all three datasets, we observe a consistent rise and fall trajectory pattern.}
    \label{fig:pred-head-trajectory-2}
\end{figure}
\vspace*{-5pt}
\paragraph{The role of stop-gradient and output-normalization.} The seminal work \cite{chen2021exploring} gave empirical results showing that stop-gradient operation is essential for avoiding the collapsed solutions. It is discussed in the theory of \citet{tian2021understanding} that without the stop-gradient, the linear network will learn the zero (constant) solution. \cite{wang2021towards} also incorporated the stop-gradient into their theory, but they did not explain why stop-gradient is necessary for their setting. We provide a different perspective about why stop-gradient and output normalization (together) are necessary for non-contrastive learning. We proved in our setting, that {the stop-gradient and output-normalization together can turn the features substituted via the prediction head into a factor in the gradient of the slower learning neurons, thereby creating the acceleration effect}. If either one of these components is missing, the acceleration effect of the prediction head will not happen and all neurons in the network will focus on learning the strongest feature. Formal arguments will be given in \myref{sec:accelerate}{Section}. 

In contrast, \cite{tian2021understanding,wang2021towards} did not incorporate the output normalization into their theory, even though their experiments have used certain forms of normalizations. We believe their method is closely related to the whitening method in \cite{ermolov2021whitening}. To the best of our knowledge, our paper is the first to explain the effects of output-normalization in optimizing nonlinear neural networks in self-supervised learning.
\vspace*{-5pt}
\paragraph{Dimensional collapse} Currently the only theoretical investigation on the dimensional collapse is by \citet{jing2021understanding}, where they focus on the contrastive learning setting. We believe their result on the role of the projection head is meaningful to understanding non-contrastive learning. But we emphasize that the objective \eqref{eqdef:obj-simsiam-sym} suffer from much more extreme dimensional collapse compared to the one in \cite{jing2021understanding}, as shown in \myref{wrap-fig:1}{Figure}. Thus the causes described in \citet{jing2021understanding} such as strong data augmentations cannot fully explain the dimensional collapse in the non-contrastive setting.
\vspace*{-5pt}
\section{Preliminaries on Non-contrastive Learning}\label{sec:prelim}

\vspace*{-3pt}
In this section, we formally define what is non-contrastive self-supervised learning. To do this, we first introduce contrastive learning following \cite{Chen2020,wen2021toward} as background. We use \([N]\) as a shorthand for the index set \(\{1,\dots, N\}\).
\vspace*{-5pt}
\paragraph{Background on contrastive learning.} Letting \(\phi_W(\cdot)\) be the neural networks, contrastive learning aims to learn good representations \(\phi_W\) via contrasting representations of similar data samples to those of dissimilar ones. Usually we are given a batch of data points \(\{X_i\}_{i \in [N]}\), and we construct for each \(i \in [N]\) a positive pair \((X_i^{(1)}, X_i^{(2)})\) (which are assumed to be simmilar) by applying random data augmentations to \(X_i\), and collect negative pairs \((X_{i}^{(1)}, X_{j}^{(2)})\) for \(i\neq j \in [N]\) (which are assumed to be dissimilar). Now given the representations \(z_i = \phi_W(X_i^{(1)}),\ z'_i = \phi_W(X_i^{(2)}),\, i\in [N]\), we train the network \(\phi_W\) to minimize the following contrastive loss:
\begin{align}\label{eqdef:contrastive-loss}
    L_{\mathrm{contrastive}}(\phi_W)  := \frac{1}{N}\sum_{i\in[N]} \underbrace{-\Sim(z_i, z'_i)/\tau}_{\text{positive term}} + \underbrace{\log \left[\sum_{j\in[N]} \exp\big(\Sim(z_i,z'_j)/\tau\big)\right]}_{\text{negative term}}
\end{align}
where \(\Sim(\cdot,\cdot)\) is the similarity metric, often defined as the cosine similarity, and \(\tau\) is the so-called temperature hyper-parameter. Intuitively, minimizing the contrastive loss can be roughly viewed as trying to classify the representation \(z_i\) as \(z'_i\) instead of \(z'_j, j\neq i \). It is a common belief that in order for the network \(\phi_W\) to be able to “distinguish” data points \(X_i\) from \(\{X_j\}_{ j\neq i}\), merely minimizing the positive term of contrastive loss is not sufficient. 

As shown by the papers \cite{chen2021intriguing,wen2021toward}, the performance of contrastive learning depends critically on the negative term. But the \texttt{BYOL} method \cite{grill2020bootstrap} managed to remove the negative term without harm by adding a trainable prediction head to the network architecture, which opened the new direction of non-contrastive self-supervised learning.

\begin{figure}[t!]\centering
    \begin{subfigure}[1]{0.48\textwidth}
        \centering
        \includegraphics[width=\textwidth]{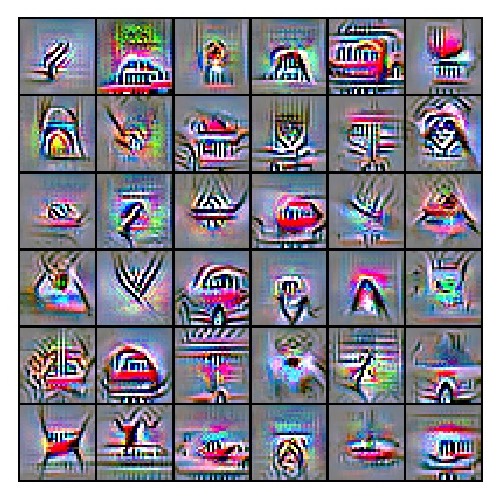}
        \caption{\small Features learned with prediction head}
    \end{subfigure}
    \hfill
    \begin{subfigure}[2]{0.48\textwidth}
        \centering
        \includegraphics[width=\textwidth]{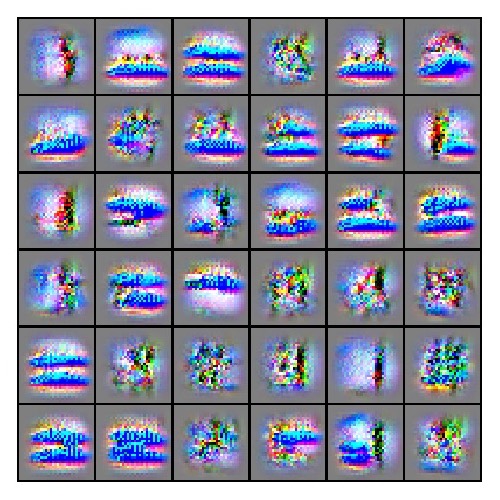}
        \caption{\small Features learned without prediction head}
    \end{subfigure}\hspace{2pt}\
    \caption{\small Feature visualization of deep neural network. We visualized the features of an Wide-ResNet-16x5 following the \texttt{BYORL} method by \citet{gowal2020self}, a adversarial robust version of \texttt{BYOL}. Features learned with prediction head obviously have more variety than features learned without the prediction head. Our feature visualization technique follows from \cite{allenzhu2021feature}.}
    \label{fig:visual-features}
\end{figure}

\paragraph{Non-contrastive self-supervised learning.} We choose the \texttt{SimSiam} method \cite{chen2021exploring} as our primary framework, whose differerence with \texttt{BYOL} is a EMA component that is proven inessential in \cite{chen2021exploring}. Following the same notations as above, except that \(z'_i = \StopGrad[\phi_W(X_i^{(2)})]\) is detached from gradient computation, the loss objective become: (the symmetric network version)
\begin{align}\label{eqdef:obj-simsiam-sym}
    L'_{\texttt{SimSiam}} = \frac{1}{N}\sum_{i\in[N]} - \Sim(z_i, z'_i)
\end{align}
which is just the positive term in contrastive loss \eqref{eqdef:contrastive-loss} (not divided by \(\tau\)). Removing the negative term results in the existence of plenty trivial \textbf{global optimal} solutions. For example, the \emph{complete collapse} refers to when \(\phi_W(\cdot)\) is some constant vector function with zero variance. Another trivial solution is called \textbf{dimensional collapse} \cite{hua2021feature}, which is when all the coordinates \([\phi_W(\cdot)]_i\) has correlation \(\pm 1\), meaning \(\phi_W(\cdot)\) lies in a one-dimensional subspace of the representation space. The dimensional collapsed solution can minimize the objective \eqref{eqdef:obj-simsiam-sym} even when the network output \(\phi_W(\cdot)\) is normalized by BN to avoid converging to a constant vector \cite{hua2021feature,zhang2022does}.

However, by adding a \emph{trainable prediction head} on top of \(z_i\), the training miraculously succeeds and outputs a state-of-the-art feature extractor. Let \(g(\cdot)\) be a shallow feed-forward network (often one or two-layer, or even simply linear), we train \(g\) and \(\phi_W\) simultaneously on the following objective: 
\begin{align}\label{eqdef:obj-simsiam}
     L_{\texttt{SimSiam}} = \frac{1}{N}\sum\limits_{i\in[N]} - \Sim(g(z_i), z'_i)
\end{align}
where \(z'_i\) is still detached from gradient computation. The \(g(z_i) = g\circ\phi_W(X_i^{(1)})\) and the detached part \(z'_i = \StopGrad[\phi_W(X_i^{(2)})]\) are often called the \textit{online} network and the \textit{target} network respectively following \cite{grill2020bootstrap}, known as two branches of non-contrastive learning. Even when such a trainable prediction head is able to represent identity function, the network can still avoid the common collapsed solutions, which presents challenges in understanding their training process and the underlying mechanism of trainable prediction head.

\paragraph{Challenges of understanding non-contrastive learning.} The success of non-contrastive methods like \texttt{BYOL} or \texttt{SimSiam} is one of the most typical examples of \textit{implicit bias} of optimization in deep learning. Even though the non-contrastive losses \eqref{eqdef:obj-simsiam-sym} and \eqref{eqdef:obj-simsiam} seem like just the positve term of the contrastive loss \eqref{eqdef:contrastive-loss}, their behaviors are vastly different. Without the negative term, the learner has no explicit incentive to learn all the discriminative features from the objective \eqref{eqdef:obj-simsiam}, especially when the trainable prediction head \(g(\cdot)\) can be an identity map and has the same trivial collapsed global optima in the objective. 

Empirically, the seminal paper \cite{chen2021exploring} discovered that \textit{even with trainable linear prediction head which can possibly learn identity mapping}, neural networks trained by SGD still avoid such collapsed solutions. Moreover, as we show in \myref{fig:identity-init-performance}{Figure}, even with an identity-initialized linear prediction head, as long as we train the prediction head via SGD, it still produces results comparable to when using other types of prediction head. Our empirical evidence suggests that {understanding the asymmetry provided by the off-diagonal entries in the identity-initialized linear prediction head suffices to explain (most of) the mechanisms of the prediction head}. This observation significantly simplifies the theoretical problem and makes the complete characterization of the training dynamics of the prediction head possible. 

Nevertheless, understanding the trainable prediction head urges us to go beyond the traditional statistical framework and optimization landscape analysis. The recent development of the \textit{feature learning theory} of neural networks \cite{jelassi2022towards,allenzhu2021feature,allen2020towards,wen2021toward,jelassi2022adam} showed it is possible to directly analyze the training dynamics of neural networks in various supervised or self-supervised tasks. Inspired by this line of research and our observations, we consider understanding the optimization of identity-initialized prediction head the key to understanding the underlying mechanism of these methods, and the characterization of the training dynamics of the full network the major technical challenges.

\section{Problem Setup}
In this section, we present the setting of our theoretical results. We first define the data distribution. 

\paragraph{Notations.} We use \(O,\Omega,\Theta\) notations to hide universal constants with respect to \(d\) and \(\widetilde{O},\widetilde{\Omega},\widetilde{\Theta}\) notations to hide polynomial factors of \(\log d\). We denote \(a = o(1)\) if \(a \to 0\) when \(d \to \infty\). We use the notations \(\poly(d),\ \polylog(d)\) to represent large constant degree polynomials of \(d\) or \(\log d\). We use \(\N(\mu,\Sigma)\) to denote standard normal distribution in with mean \(\mu\) and covariance matrix \(\Sigma\). We use the bracket \(\dbrack{\cdot,\cdot}\) to denote the inner product and \(\|\cdot\|_2\) the \(\ell_2\)-norm in Euclidean space. And for a subspace \(V \subset \R^d\), we denote \(V^{\perp}\) as its orthogonal complement. We use \(\1_B\) to denote the indicator function of event \(B\).

Following the standard structure of image datasets, we consider data divided into patches, where each patch can contain either features or noises.

\begin{definition}[data distribution and features]\label{def:data}
    Let \(X \sim \D\) be \(X = (X_1,\dots,X_P) \in \R^{d\times P}\) where each \(X_i \in \R^d\) is a patch. We assume that there are two feature vectors \(v_1, v_2\) such that \(\|v_\ell\|_2 = 1, \ell=1,2\) and are orthogonal to each other. To generate a sample \(X\), we uniformly sampled \(\ell \in [2]\) and generate for each \(p \in [P]\):
    \begin{align*}
        X_p = z_p(X) v_{\ell} + \xi_p \1_{z_p=0},\quad\E_{X\sim\D}[z_p(X)] = 0,\quad \forall p \in [P]
    \end{align*}
    We denote \(\S(X) = \{p:z_p(X) \neq 0\} \subseteq [P]\) as the set of feature patches and assume \(z_p(X) = z_{p'}(X) \in \{0,\pm\alpha_{\ell}\}, \forall p, p'\in [P]\), i.e., all feature patches have the same direction of \(v_{\ell}\) within the same \(X\). We assume \(P = \polylog(d)\), \(S(X) \equiv P_0 = \Theta(\log d)\) for every \(X\). The assumption of \(\xi_p\) will be given in \myref{assump-1}{Assumption}. An intuitive illustration is given in \myref{fig:patch-data}{Figure}.
\end{definition}
\paragraph{Strong and weak features.} We pick $\alpha_1 = 2^{\textsf{polyloglog}(d)}$ and $\alpha_2 =\alpha_1/\polylog(d)$. Hence \(v_1\) is the \textit{strong feature} and \(v_2\) is the \textit{weak feature}, and we want the learner network to learn both \(v_1,v_2\) (but by different neurons) as their learning goal. This is a simplification of the real scenario where features show up in multiple patches of the images, while noises are local and roughly independent across different patches. Intuitively, we can think of the strong features in a dataset are the ones that show up more frequently or with larger magnitude, and weak features as those that show up rarely or with smaller magnitude, which is the common case in any practical dataset.

\begin{figure}[t!]\centering
    \includegraphics[width=0.95\textwidth]{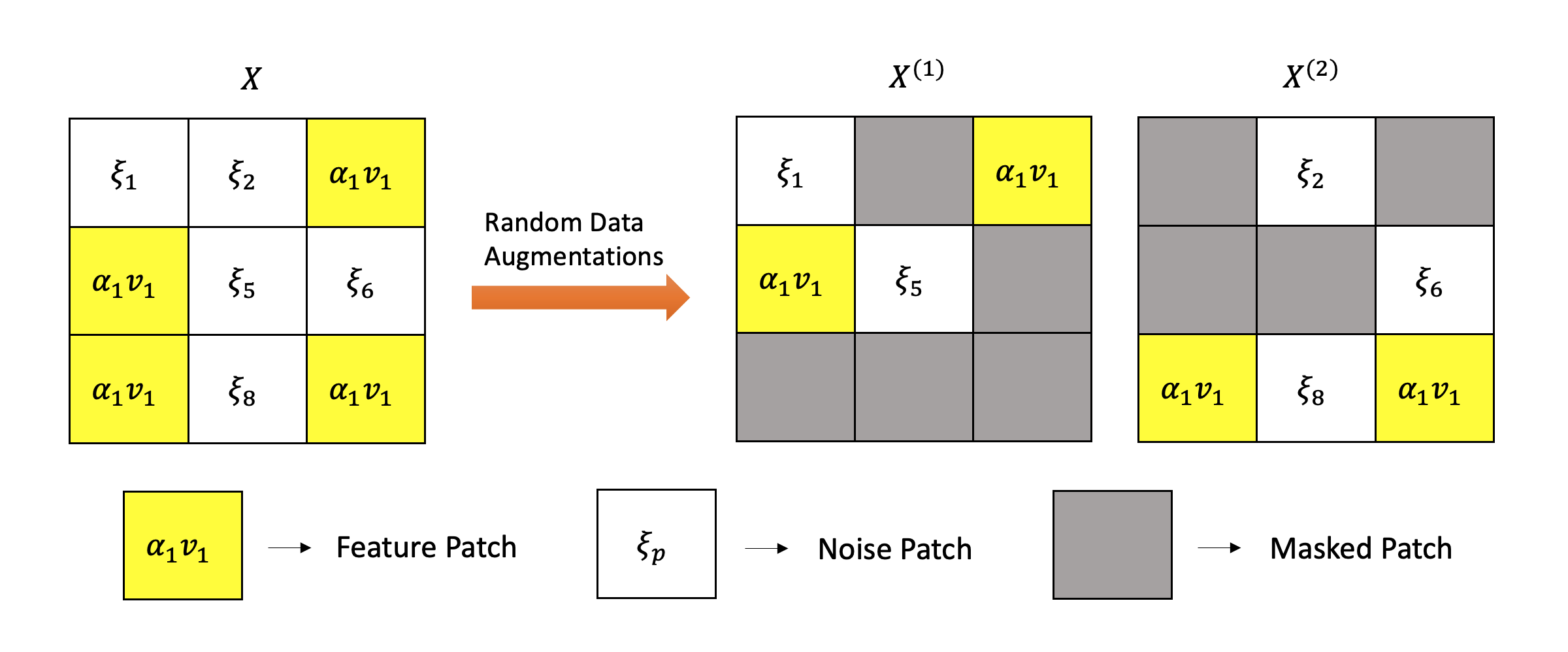}
    \caption{\small Illustration of the data distribution and data augmentations. Each data is equipped with a feature, either \(v_1\) or \(v_2\), and contains a lot of noise patches. After the data augmentations, the positive pair \((X^{(1)},\, X^{(2)})\) is constructed by randomly masking out half of non-overlapping patches for each positive sample. The reason for constructing positive pair with non-overlapping patches is because of the strong noise assumption we made in \myref{assump-1}{Assumption} and the \textit{feature decoupling} principle in \cite{wen2021toward}.}
    \label{fig:patch-data}
\end{figure}

\begin{remark}
    Our analysis can be easily generalized to settings of either (1) when \(\alpha_1=\alpha_2\) but the sampling of \(\ell\in[2]\) is of non-equal probability (i.e., dataset imbalance setting); or (2) when the two features always co-occur in the same sample but not of the same strength. But we still require \(\alpha_1,\alpha_2 \gg \polylog(d)\) to simplify the analysis.
\end{remark}

\begin{assumption}[noise]\label{assump-1}
    Denoting \(V = \mathrm{span}(v_1,v_2)\), we assume \(\xi_p\in V^{\perp}\) is independent for each \(p \in [P]\setminus S(X)\), where \(X = (X_p)_{p\in[P]} \sim\D\), and:
    \begin{enumerate}[(a)]
        \item For any unit vector \(u \in V^{\perp}\), \(\E[\dbrack{\xi_p,u}] = 0\), and \(\E[\dbrack{\xi_p,u}^6] = \sigma^6\) for some \(\sigma=\Theta(1)\);
        \item It holds for some \(\varrho \in [0,\frac{1}{d^{\Omega(1)}}]\) it holds \(|\E[\dbrack{u_1,\xi_p}^3\dbrack{u_2,\xi_p}^3]|\leq \varrho\) and \(|\E[\dbrack{u_1,\xi_p}^5\dbrack{u_2,\xi_p}]|\leq \varrho\) for any two vectors \(u_1,u_2 \in \R^d\) that are orthogonal to each other.
    \end{enumerate}
\end{assumption}

\begin{remark}
    A simple example of our noise \(\xi_p\) is the spherical Gaussian noise in \(V^{\perp}\). Our \myref[b]{assump-1}{Assumption} ensures that the prediction head cannot be used to cancel the noise correlation between different neurons. We point out that the features in our data can be learned via clustering, but we emphasize that {we do not intend to compare our algorithm with any clustering method in this setting since our goal is to study how the prediction head helps in learning the features.}
\end{remark}

\subsection{Learner Network}
Following the \texttt{SimSiam} framework, the online and target network share the same encoder network in our setting, as explained in \myref{sec:prelim}{Section}. We consider the base encoder network \(f\) as a simple convolutional neural network: Let \(W = (w_1, \dots, w_m) \in \R^{d\times m}\) be the weight matrix, where \(w_i \in \R^d\), the \textbf{encoder network} \(f\) is defined by
\begin{align*}
    \textstyle f_j(X) := \sum_{p\in[P]}\sigma(\vbrack{w_j,X_p}),  \qquad \forall j \in [m]
\end{align*}
Here we use the cubic activation function \(\sigma(z) = z^3\), as polynomial activations are standard in literatures of deep learning theory \cite{andoni2014learning,gautier2016globally,kileel2019expressive,allen2020backward,li2020making,chen2022learning} and also has comparable performance in practice \cite{allen2020backward}. The (identity initialized) prediction head is defined as a matrix \(E = [E_{i,j}]_{(i,j)\in[m]^2}\) with \(E_{i,i} \equiv 1, i\in[m]\), where only the the off-diagonals \(E_{i,j}, i\neq j\) are trainable parameters. The \textbf{online network} \(\widetilde{F}\) is defined by: given \(j \in [m]\), we let \(F_j(X) := f_j(X) + \sum_{r\neq j}E_{j,r}f_r(X)\), and 
\begin{align*}
    \textstyle \widetilde{F}_j(X) & := \BN\left( F_j(X) \right) = \BN\Bigg[ \sum_{p\in[P]}\Big(\sigma(\dbrack{w_j,X_p}) + \sum_{r \neq j} E_{j,r}\sigma(\vbrack{w_r,X_p}\Big)  \Bigg]
\end{align*}
where the batch normalization \(\BN\) here\footnote{We use batch normalization as a output-normalization method, rather than for the supposed implicit negative term effects as disproved in \citet{richemond_byol_2020}.} is defined as follows: Given a batch of inputs \(\{z_i\}_{i\in[N]}\), 
\begin{align}\label{eqdef:BN}
    \BN(z_i) := \frac{z_i - \frac{1}{N}\sum_{i\in[N]}z_i}{\sqrt{\frac{1}{N}\sum_{i\in[N]}z_i^2 - \left(\frac{1}{N}\sum_{i\in[N]}z_i\right)^2}}
\end{align}
And the \textbf{target network} \(G\) is defined as follows: Given \(j \in [m]\)
\begin{align*}
    \widetilde{G}_j(X) := \BN\left( G_j(X) \right) = \BN\Bigg[ \sum_{p\in[P]}\sigma(\dbrack{w_j,X_p}) \Bigg]
\end{align*}

\subsection{Training Algorithm}

\begin{algorithm}[!ht]
    \caption{Training Algorithm} \label{alg}
    \begin{algorithmic}[1]
    \Require data distribution \(\D\), objective \(L_{\S}\) \eqref{eqdef:objective}, networks \(\widetilde{F}, \widetilde{G}\), hyper-parameters \(T, N, \eta, \eta_E, m\), and a bool variable \(\mathsf{TrainPredHead} = \mathsf{True}\).
    \State Initialize \(w_j^{(0)} \sim \N(0,I_d/d)\) \(\forall j\in[m]\) i.i.d., and \(E^{(0)} = I_m\);
    \For{$t \in \{0, 1, 2, \cdots, T - 1\}$}
        \State Sample \(X^{(t,i)} \gets (X_p^{(t,i)})_{p\in[P]} \sim \D,\forall i\in[N]\) i.i.d.;
        \State Sample \(\{\P^{(t,i)} \}_{i\in[N]}\) i.i.d., and obtain \(\S_t \gets \{X^{(t,i,1)}, X^{(t,i,2)}\}_{i\in[N]}\) via data augmentations
        \begin{align*}
            X^{(t,i,1)} \gets (X_p^{(t,i)} \1_{p\in\P^{(t,i)}})_{p\in[P]},\qquad X^{(t,i,2)} \gets (X_p^{(t,i)} \1_{p\notin\P^{(t,i)}})_{p\in[P]};
        \end{align*}
        \State Perform stochastic gradient descent step to \(W^{(t)} = (w_j^{(t)})_{j\in[m]}\) by
        \begin{align*}
            w_j^{(t+1)} &\gets w_j^{(t)} - \eta \nabla_{w_j}L_{\S_t}(W^{(t)},E^{(t)});
        \end{align*}
        \If{\(\mathsf{TrainPredHead}=\mathsf{True}\)} update the off diagonal of prediction head \(E^{(t)}\) by
        \begin{align*}
            E_{i,i}^{(t+1)} \gets 1,\quad E_{i,j}^{(t+1)} \gets E_{i,j}^{(t)} - \eta_E \nabla_{E_{i,j}}L_{\S_t}(W^{(t)},E^{(t)}),\quad \forall j\neq i,\ i,j\in[m];
        \end{align*}  
        \Else{} keep \(E^{(t+1)}  = I_m\).
        \EndIf
    \EndFor
    \end{algorithmic}
\end{algorithm}

\paragraph{Data augmentation.}
We use a very simple data augmentation: for each data \(X = (X_p)_{p\in[P]}\), we randomly and uniformly sample half of the patches \(\P \subseteq [P]\) to generate two samples (which is the so-called \textit{positive pair} in contrastive learning):
\begin{align}\label{eqdef:data-aug}
    X^{(1)} = (X_p\1_{p\in\P})_{p\in[P]},\quad X^{(2)} = (X_p\1_{p\notin\P})_{p\in[P]}
\end{align}
An intuitive illustration is given in \myref{fig:patch-data}{Figure}. Our data augmentation approach is similar to the common cropping augmentation used in contrastive learning \cite{chen2020a,tian2020what} and the patch masking strategy in generative pretraining \cite{bao2021beit,he2021masked} and NLP pretraining \cite{devlin2019bert}. It is also analogous to the data augmentations being studied in theoretical literatures \cite{wen2021toward,ji2021power,liu2022masked} of self-supervised learning, especially the \textsf{RandomMask} augmentation in \cite{wen2021toward}. 
\vspace*{-5pt}
\paragraph{Non-contrastive loss function.}
Now we define the loss function as follows: we sample \(N\) data points \(\{X_i\}_{i\in [N]}, X_i \stackrel{\mathrm{i.i.d.}}\sim \D\) and apply our data augmentation \eqref{eqdef:data-aug} to obtain \(\mathcal{S} = \{X^{(i,1)}, X^{(i,2)}\}_{i\in[N]}\). Now we define 
\begin{align}\label{eqdef:objective}
    L_{\mathcal{S}}(W,E) &:= \frac{1}{N}\sum_{i\in[N]} \left\| \widetilde{F}(X^{(i,1)}) - \StopGrad[\widetilde{G}(X^{(i,2)}])\right\|_2^2 \\
    & = 2 - \frac{1}{N}\sum_{i\in[N]}\dbrack{\widetilde{F}(X^{(i,1)}), \StopGrad[\widetilde{G}(X^{(i,2)})]} \nonumber
\end{align}
where the \(\StopGrad\) operator detach gradient computation of the target network \(\widetilde{G}(\cdot)\). This form of objective \eqref{eqdef:objective} is first defined in \citet{grill2020bootstrap} and is equivalent to \eqref{eqdef:obj-simsiam} in \citet{chen2021exploring} when \(\widetilde{F}\) and \(\widetilde{G}\) share the same encoder network \(f(\cdot)\) and their outputs are normalized.

\paragraph{Intuition of the data augmentation.} Our data augmentation is an analog of the the standard cropping data augmentation. In \myref{def:data}{Definition}, the features $v_1, v_2$ appear in multiple patches, but the noises are independent across different patches (see Figure~\ref{fig:patch-data}). As our data augmentation produces positive pairs with non-overlapping patches, learning to emphasize noises cannot align the representations of the positive pair, but learning \textbf{either one of} the features $\phi(X) = \sum_p \sigma(\langle v_1, X_p\rangle)$ or $ \phi(X) = \sum_p\sigma(\langle v_2, X_p \rangle)$ is sufficient. \textbf{We consider learning  the same feature $v_i$ in \emph{all the neurons $f_j$ in the encoder network $f$} as the dimensional collapsed solution.}

\paragraph{Initialization and hyper-parameters.} At \(t = 0\), we initialize \(W\) and \(E\) as \(W_{i,j}^{(0)}\sim \N(0,\frac{1}{d})\) and \(E^{(0)} = I_m\) and we only train the off-diagonal entries of \(E^{(t)}\). For the simplicity of analysis, \textbf{we let \(m=2\), which suffices to illustrate our main message.} For the learning rates, we let \(\eta \in (0, \frac{1}{\poly(d)}]\) be sufficiently small and \(\eta_E \in [\frac{\eta}{\alpha_1^{O(1)}},\frac{\eta}{\polylog(d)} ]\), which is smaller than \(\eta\)\footnote{We conjecture that by modifying certain assumptions for the noise (especially by allowing the noise to span the feature subspace \(V\)), one can prove a similar result for the case \(\eta_E = \eta\).}.

\paragraph{Optimization algorithm} Given the data augmentation and the loss function, we perform (stochastic) gradient descent on the training objective \eqref{eqdef:objective} as follows: at each iteration \(t =0,\dots, T-1\), we sample a new batch of augmented data \(\S_t = \{X^{(t,i,1)},X^{(t,i,2)} \}_{i\in[N]}\) and update
\begin{align*}
    W^{(t+1)} = W^{(t)} - \eta \nabla_{W}L_{\S_t}(W^{(t)},E^{(t)}),\quad E_{i,j}^{(t+1)} = E_{i,j}^{(t)} - \eta_E \nabla_{E_{i,j}}L_{\S_t}(W^{(t)},E^{(t)}),\ \ \forall i\neq j,\, i,j\in[m]
\end{align*}
If we do not train the prediction head, we just simply keep \(E^{(t)}\equiv I_m\). We summarize our algorithm in \myref{alg}{Algorithm}.

\section{Statements of Main Results}
In this section, we shall present our main theoretical results on the mechanism of learning the prediction head in non-contrastive learning. To measure the correlation between neurons, we introduce the following notion: letting 
\begin{align*}
    \mathbf{Var}(\psi(X)) := \E_{X\sim \D} [(\psi(X)-\E[\psi(X)])^2]
\end{align*}
be the variance of any function \(\psi\) of \(X\sim \D\), we denote the correlation \(\mathbf{Corr}(\psi(X),\psi'(X))\) of any two function \(\psi, \psi'\) over \(\D\) as 
\begin{align*}
    \mathbf{Corr}(\psi(X),\psi'(X)) := \frac{\E[(\psi(X) - \E[\psi(X)])(\psi'(X) - \E[\psi'(X)])]}{\sqrt{\mathbf{Var}(\psi(X))}\sqrt{\mathbf{Var}(\psi'(X))}}
\end{align*}

Now we present the main theorem of training with a prediction head, and set \(m=2\). 

\begin{theorem}[learning with prediction head and BN, see \myref{thm:end-phase}{Theorem}]\label{thm:1-w-head}
    For every \(d > 2\), let \(N \geq \poly(d)\), \(\eta \in  (0, \frac{1}{\poly(d)}]\) be sufficiently small, and \(\eta_E \in [\frac{\eta}{\alpha_1^{O(1)}},\frac{\eta}{\polylog(d)} ]\). Then with probability \(1-o(1)\), after runing \myref{alg}{Algorithm} for \(T = \poly(d)/\eta\) many iterations, we shall have for some \(\ell\in[2]\):
    \begin{align*}
         w_1^{(T)} = \beta_1 v_\ell + \varepsilon_1,\quad w_2^{(T)} = \beta_2 v_{3-\ell} + \varepsilon_2 \qquad \text{with}\quad |\beta_1|,|\beta_2| = \Theta(1),\  \|\varepsilon_1\|_2, \|\varepsilon_2\|_2\leq \widetilde{O}(\frac{1}{\sqrt{d}})
    \end{align*}
    Furthermore, the objective converges: \(\E_{\S\sim\D^N}[L_{\S}(W^{(T)},E^{(T)})] \leq \mathsf{OPT} + \frac{1}{\poly(d)} \leq O(\frac{1}{\log d})\). Here \(\mathsf{OPT}\) stands for the global minimum of the objective\footnote{Under our data model \myref{def:data}{Definition}, non-overlapping data augmentation \eqref{eqdef:data-aug} and learner network definition, the \textit{global minimum} of our objective \eqref{eqdef:objective} in population is the following quantity: 
    \begin{align*}
        \mathsf{OPT} := \min_{W,E} \E_{\S\sim \D^N} [L_{\mathcal{S}}(W,E)] = 2 - 2\frac{\E[|S(X)\cap \P|\cdot|\S(X)\setminus \P|]}{\E[|S(X)\cap \P|^2]} = \Theta(\frac{1}{\log d})
    \end{align*}}.
\end{theorem}
\myref{thm:1-w-head}{Theorem} clearly shows the network learn all the desired features, even under huge imbalance between \(v_1\) and \(v_2\). This leads to the following corollary.

\begin{corollary}
    Under the same hyper-parameter in \myref{thm:1-w-head}{Theorem}, with probability \(1-o(1)\), after runing \myref{alg}{Algorithm} for \(T = \poly(d)/\eta\) many iterations, we shall have that the learning \textbf{avoids dimensional collapse:}
    \begin{align*}
        |\mathbf{Corr}(f_1(X),f_2(X))| \leq O(\frac{1}{\sqrt{d}}).
    \end{align*}
\end{corollary}

In contrast, learning without the prediction head will result in learning only the strong feature \(v_1\) in both neurons, which creates strong correlations between any two neurons. Learning \(v_2\) in this case cost at least \(\Omega(\alpha_1/\alpha_2) \gg \polylog(d)\) many neurons, as shown below.

\begin{theorem}[learning without prediction head but with BN, see \myref{thm:without-pred-head}{Theorem}]\label{thm:2-w/o-head}
    Let \(N \geq \poly(d)\), \(\eta = o(1)\) and the number of neurons \(m  = o(\alpha_1/\alpha_2)\) be any positive integer. Then with probability \(1-o(1)\), after runing \myref{alg}{Algorithm} with \(\mathsf{TrainPredHead} = \mathsf{False}\) for \(T = \poly(d)/\eta\) many iterations, we shall have:
    \begin{align*}
         w_j^{(T)} =  \beta_j v_1 + \varepsilon_j \qquad \text{with}\quad|\beta_j| = \Theta(1),\ \|\varepsilon_j\|_2 \leq \widetilde{O}(\frac{1}{\sqrt{d}})  \tag*{for all \(j\in[m]\)}
    \end{align*}
    Furthermore, the objective converges: \(\E_{\S\sim\D^N}[L_{\S}(W^{(T)},E^{(T)})] \leq \mathsf{OPT} + \frac{1}{\poly(d)} \leq O(\frac{1}{\log d})\). This means the collapsed solution also reaches the global minimum of the objective. 
\end{theorem}

Note that since we have used BN as our output normalization instead of \(\ell_2\)-norm, the learner is immune to complete collapse and must have a certain variance in the outputs.
Immediately, we have the following corollary.

\begin{corollary}
    Under the same hyper-parameter in \myref{thm:2-w/o-head}{Theorem}, with probability \(1-o(1)\), after runing \myref{alg}{Algorithm} with \(\mathsf{TrainPredHead} = \mathsf{False}\) for \(T = \poly(d)/\eta\) many iterations, we shall have \textbf{dimensional collapse}:
    \begin{align*}
        |\mathbf{Corr}(f_i(X),f_j(X))| \geq 1 - O(\frac{1}{\sqrt{d}}),\qquad \text{for all \(i,j\in [m]\).}
    \end{align*}
\end{corollary}


\begin{remark}
    Note that since we have used BN as our output normalization instead of \(\ell_2\)-norm, the learner is regularized to avoid complete collapse and must have a certain variance in its neurons. It is easier to obtain a complete collapse result when the network has \(\ell_2\)-normalized outputs and there is a low-variance feature (but not of smaller magnitude) in the data set, which we refrain from proving here.
\end{remark}

How does using the prediction head or not create such a difference in features learned by the non-contrastive methods? We shall give some intuitions by digging through the training process and separately discuss the four phases of the training process.

\section{The Four Phases of the Learning Process}

\begin{figure}[t!]\centering
    \begin{subfigure}[1]{0.49\textwidth}
        \centering
        \includegraphics[width=\textwidth]{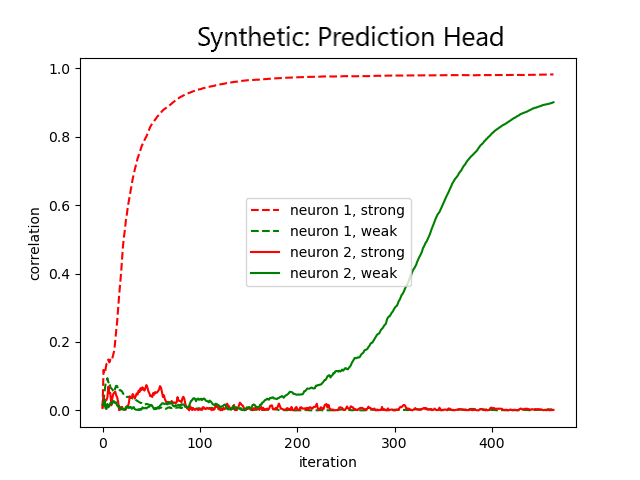}
        \caption{\small Identity-initialized (trainable) prediction head}
    \end{subfigure}
    \hfill
    \begin{subfigure}[2]{0.49\textwidth}
        \centering
        \includegraphics[width=\textwidth]{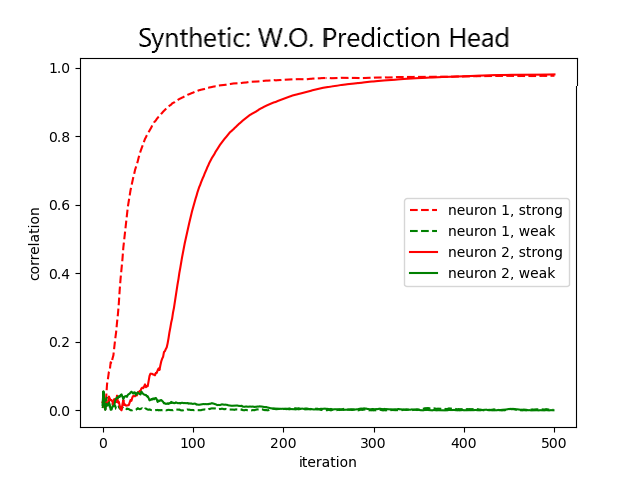}
        \caption{Learning without prediction head}
    \end{subfigure}
    \caption{\small The feature learning process over synthetic data. When trained with the prediction head, after the strong feature is learned in the faster learning neuron, the weak feature can be learned in the slower learning neuron. When trained without the prediction head, both neurons will learn the strong feature and ignore the weak feature.}
    \label{fig:synthetic-experiment}
\end{figure}

We divide the complete training process into four phases: phase I for learning the stronger feature, phase II for the substitution effect, phase III for the acceleration effect, and the end phase for convergence. The first three phases explain how the prediction head can help learn the base encoder network, and the last phase of the training explains why the off-diagonal entries often shrink in the later stage of training. 

\subsection{Phase I: Learning the Stronger Feature}

At the beginning of training, the stronger feature \(v_1\) enjoys a much larger gradient as opposed to the weaker feature \(v_2\), so naturally, \(v_1\) will be learned first. However, if for both neurons \(f_1, f_2\) the speed of learning \(v_1\) is the same, then we cannot argue the difference between them and will not be able to show the substitution from either one to another. Indeed, let us assume at initialization, the neuron \(f_1(\cdot))\) won the jackpot of having larger signal-to-noise ratio of feature \(\dbrack{w_j^{(0)},v_1}\), then we can show the following result under our setting.

\begin{lemma}[learning the stronger feature, formal statement see \myref{lem:phase-1}{Lemma}]\label{lem-0}
    After \(T_1=d^{2+o(1)}/\eta\) many iterations, the feature \(v_1\) in neuron \(f_1\) will be learn to \(\dbrack{w_1^{(T_1)} ,v_1} = \Omega(1)\), while all other features \(\dbrack{w_j^{(T_1)} ,v_\ell} = o(1)\) for \((j,\ell)\neq (1,1)\) are small. And the prediction head \(\|E^{(T_1)} - I_2\|_2\leq d^{-\Omega(1)}\) is still close to the initialization.
\end{lemma}

In this phase, the prediction head has not come into play. The substitution effect can only happen after the feature \(v_1\) in neuron \(f_1\) is learned to a certain degree, and neuron \(f_2\) remains largely unlearned.

\subsection{Phase II: The Substitution Effect}

To illustrate the substitution effect, let us keep assuming that neuron \(w_1^{(t)}\) has already learned some significant amount of the strong feature \(v_1\), say \(w_1^{(t)} = \beta_1 v_1 + \mathit{residual} \) with \(|\beta_1| = \Omega(\|\mathit{residual}\|)\). When this happens, we have the following result: (recall \(f_j(\cdot), j\in [2]\) are the neurons of the base encoder network)

\begin{lemma}[substitution effect, formal statement see \myref{lem:phase-2}{Lemma}]\label{lem:3-substitute}
    After \(|\dbrack{w_1^{(t)}, v_1}| = \Omega(1)\) in \(O(d^{2+o(1)}/\eta)\) iterations (as shown by \myref{lem:phase-1}{Lemma}), for much shorter time than learning \(\dbrack{w_1^{(t)}, v_1}\), we shall have \(|E_{2,1}^{(t)}|\) increasing until \(|E_{2,1}^{(t)} f_1(X^{(1)})| \gg |f_2(X^{(1)})| \) when \(X\) is equipped with feature \(v_1\). In other words, \(E_{2,1}^{(t)} f_1(X^{(1)})\) is a substitute for the feature \(v_1\) that should be learned by \(f_2\).
\end{lemma}

\paragraph{Intuition of the substitution effect.} After the stronger feature is learned in neuron \(f_1\), the optimal way to align two positive representations \(F_2(X^{(1)}), G_2(X^{(2)})\) is no longer learning features in weight \(w_2\), but use the prediction head to ``borrow'' the features in $f_1$ and incorporate them into $F_2$. This is how the substitution effect happens when trained with a prediction head.

\paragraph{Proof sketch for \myref{lem:3-substitute}{Lemma}.} Indeed, let us look at the learning of \(E_{2,1}^{(t)}\). In this phase, \(w_2^{(t)}\) and \(E_{2,1}^{(t)}\) are roughly learned to maximize the following quantity: 
\begin{align*}
    \textstyle \widetilde{F}_2(X^{(1)})\cdot \widetilde{G}_2(X^{(2)}) & \propto \big(f_2(X^{(1)}) + E_{2,1}^{(t)}f_1(X^{(1)})\big)\times f_2(X^{(2)}) \\
    &\approx \textstyle\sum_{\ell \in [2]}\alpha_\ell^6 \Big( \dbrack{w_2^{(t)},v_{\ell}}^6 + E_{2,1}^{(t)}\cdot \dbrack{w_1^{(t)}, v_\ell}^3 \cdot \dbrack{w_2^{(t)},v_{\ell}}^3\Big) 
\end{align*}
As the neuron \(f_1(\cdot)\) is already learned with feature \(v_1\), in order to maximize the RHS, we can either try to maximize \(\sum_{\ell \in [2]}\dbrack{w_2^{(t)},v_{\ell}}^6 \), or to maximize \(E_{2,1}^{(t)}\dbrack{w_1^{(t)}, v_1}^3 \cdot \dbrack{w_2^{(t)},v_{1}}^3 \approx E_{2,1}^{(t)} \dbrack{w_2^{(t)},v_{1}}^3 \). In this case, the more efficient choice is to learn \(|E_{2,1}^{(t)}|\) to substitute for maximizing \(\dbrack{w_2^{(t)}, v_\ell}^3\). Actually, because of the high signal-to-noise ratio of learning \(w_2^{(t)}\) than \(E_{2,1}^{(t)}\), feature \(\dbrack{w_2^{(t)},v_\ell}\) is learned with slower pace than \(E_{2,1}^{(t)}\), so that \myref{lem:3-substitute}{Lemma} can be shown.

\subsection{Phase III: The Acceleration Effect}\label{sec:accelerate}

After the substitution of \(v_1\) in \(F_2\), our concern is, whether or not \(w_2^{(t)}\) will learn \(v_2\) and only \(v_2\) eventually, so that we can obtain a diverse representation? The answer is yes, as we summarize in the following lemma.

\begin{lemma}[acceleration effect, formal statement see \myref{lem:phase-3}{Lemma}]\label{lem:4-accelerate-effect}
    After \(E_{2,1}^{(t)}\) is learned in \myref{lem:3-substitute}{Lemma}, learning \(v_2\) in \(w_2^{(t)}\) will be much faster than \(v_1\), until \(\|w_2^{(t)} - \beta_2 v_2\| \leq o(1)\) for some \(\beta_2 = \Theta(1)\).
\end{lemma}

The acceleration effect is caused by the interactions between the prediction head, the stop gradient operation, and the normalization method (which in this case is the batch normalization). We shall explain these interactions with insights from our theoretical analyses below.

\paragraph{What is the role of the stop-gradient?} Thanks to the \textsf{StopGrad} operation, when we compute the gradient \(-\nabla_{w_2} F_2(X^{(1)})\cdot \StopGrad [G_2(X^{(2)})]\) to learn \(f_2\), this negative gradient will only try to maximize \(f_2(X^{(1)})\cdot f_2(X^{(2)})\), rather than to maximize \(f_2(X^{(2)})\cdot F_2(X^{(1)})\). This is because the stop-gradient is on $G$ not on $F$: while $F_2$ has a large component of $v_1$ borrowed from $f_1$ using $E$, $G_2$ does not have this component. So the gradient of $F_2$ is to align with the features in $G_2$ that does not contain many $v_1$, while the gradient of $G_2$ is to aligned with the features in $F_2$ that contains a lot of $v_1$. Thus the stop gradient on \(G\) help ignore the feature borrowed from $f_1$ using prediction head $E$ and ensures the slower learning neuron \(f_2\) will focus on learning feature \(v_2\). 

\paragraph{What is the role of the output normalization?} Again due to the \(\StopGrad\) operation, the gradient of \(\widetilde{F}_2\) is taken with respect to the ratio \(f_2(X^{(1)})/\sqrt{\mathbf{Var}[F_2(X^{(1)})]} \). As gradient descent tries to maximize this ratio, a direct computation gives 
\begin{align*}
    \nabla_{w_2} \frac{f_2(X^{(1)})}{\sqrt{\mathbf{Var}(F_2(X^{(1)}))}} = \frac{\nabla_{w_2}f_2(X^{(1)})\cdot\mathbf{Var}(F_2(X^{(1)})) -  f_2(X^{(1)})\cdot\nabla_{w_2}\mathbf{Var}(F_2(X^{(1)}))}{\mathbf{Var}(F_2(X^{(1)}))^{3/2}}
\end{align*}
From some calculation, we can obtain the above gradient is proportional to
\begin{align*}
    \textstyle\sum_{\ell\in[2]}\Big([E_{2,1}^{(t)} \dbrack{w_1^{(t)},v_{3-\ell}}^3]^2 + \mathbf{Var}[f_2(X^{(1)})] \Big) \dbrack{\nabla_{w_2}f_2(X^{(1)}),v_{\ell}}v_{\ell}
\end{align*}
which borrow the \textit{substituted feature} \(v_{3-\ell}\) from \(f_1(\cdot)\) to adjust the gradient of \(v_{\ell}\) in \(f_2(\cdot)\), via the prediction head \(E_{2,1}^{(t)}\). Without the output normalization, the learning of \(v_{1}\) will dominate that of \(v_2\) even when we train the prediction head.

\paragraph{Proof sketch for \myref{lem:4-accelerate-effect}{Lemma}.} At this stage, when we are updating the weights of \(w_2^{(t)}\), we are simultaneuously maximizing \( f_2(X^{(1)})\cdot f_2(X^{(2)})\) and also minimizing the normalizing constants \(\sqrt{\mathbf{Var}[F_2(X^{(1)})]}\). This two goals are in slight conflict because of the normalization, and by careful calculation the gradients are roughly given by (interpreting the expectation as empirical)
\begin{align*}
    \textstyle \dbrack{-\nabla_{w_2}L_{\S}, v_\ell} & \propto \E  \left[ \Big([E_{2,1}^{(t)} \dbrack{w_1^{(t)},v_{3-\ell}}^3]^2 + \mathbf{Var}[f_2(X^{(1)})] \Big) \cdot f_2(X^{(2)}) \dbrack{-\nabla_{w_2}f_2(X^{(1)}), v_\ell} \right]
\end{align*}
Because of the learning of \(f_1\) and the substitution effect, we now knows \([E_{2,1}^{(t)} \dbrack{w_1^{(t)},v_{3-\ell}}^3]^2\) is much larger when \(\ell=2\), which accelerates the learning of \(v_2\) in \(w_2^{(t)}\) to surpass that of \(v_1\) and leads to \myref{lem:4-accelerate-effect}{Lemma}.

\subsection{The End Phase: Convergence}

As the weak features are learned, we have already obtained a good encoder network \(f(\cdot)\) as shown in \myref{thm:1-w-head}{Theorem}. The rest of our analysis is to understand what the prediction head converges to in polynomial time. Actually, our \myref{thm:end-phase}{Theorem} also contains the following result:

\begin{proposition}[convergence of the prediction head, see \myref{thm:end-phase}{Theorem}c]\label{prop:5-pred-head-convergence}
    After some \(t\geq T = \poly(d)/\eta\) iterations, we shall have \(\|E^{(t)} - I_2\|_F \leq \frac{1}{\poly(d)}\).
\end{proposition}

This result also implies that after learning the weak feature \(v_2\) is complete, the off-diagonal entries of the prediction head will reverse their trajectory and converge to zero at the end of training. While we admit that only some of our real-world experiments show the convergence to zero for the off-diagonal entries of the prediction head, most of the experiments do display a rise and fall trajectory pattern of off-diagonal entries consistently. 

\section{Additional Related Work}

\paragraph{Self-supervised learning} The area of self-supervised learning has evolved at a tremendous speed in recent years. It has created huge success in natural language processing \cite{devlin2019bert,yang2019xlnet,brown2020language} and established a paradigm where the networks are first trained on an unsupervised pretext task and then be finetuned in downstream applications. In vision, supervised pretraining had been the go-to choice until representations learned by contrastive learning \cite{tian2020contrastive,he2020momentum,Chen2020,caron2020unsupervised,chen2020improved,chen2021empirical,gao2021simcse,radford2021learning,ermolov2021whitening} became dominant in many downstream tasks. Another type of self-supervised learning is the generative learning \cite{ramesh2021zero,bao2021beit,he2021masked}, which also gives promising results in downstream adaptations. Interesting applications such as \cite{radford2021learning,ramesh2022hierarchical} also illustrate the power of contrastive learning in multiple domains. 

\paragraph{Theory of self-supervised learning} The theoretical side of self-supervised learning developed quickly due to the success of contrastive learning, which is closely related to the methods we are studying. Since \citet{arora2019theoretical}, lots of papers have studied the properties of contrastive learning, as mentioned in the introduction. \cite{chen2021intriguing,robinson2021can} discussed many interesting phenomena associated with the negative term in contrastive learning. \citet{saunshi2022understanding} provided pieces of evidence that contrastive loss is function class-specific rather than agnostic. \citet{wen2021toward} took a feature learning view to understand contrastive learning with neural networks, which inspired our analysis in the non-contrastive setting.  For generative self-supervised learning, \cite{lee2021predicting,teng2021can} provides downstream performance guarantees for generative pretrained models. \cite{saunshi2020mathematical,wei2021pretrained} studied the natural language tasks, where the data are sequentially structured. \citet{liu2022masked} gave a recovery guarantee for tensors in generative learning under hidden Markov models. \cite{allen2021forward} analyzed multi-layer generative adversarial networks and provided an optimization guarantee for their stochastic gradient descent ascent algorithm. \nocite{liu2021analyzing}

\paragraph{Feature learning theory of deep learning} Our theoretical results are also inspired by the recent progress of the feature learning theory of neural networks \cite{li2019towards,li2020learning,allenzhu2021feature,allen2020towards,karp2021local,zou2021understanding,jelassi2022towards}. \citet{li2019towards} initiate the study of the speed difference in learning different types of features. \cite{li2020learning} developed theory for learning two-layer neural networks over Gaussian distribution beyond the \textit{neural tangent kernel} (NTK) \cite{allen-zhu2019a,allen-zhu2019on,allen-zhu2019learning,du2019gradient,arora2019fine}. \citet{allenzhu2021feature} studied the origin of adversarial examples and how adversarial training help in robustify the networks. \cite{allen2020towards} tried to explain ensemble and knowledge distillation under multi-view assumptions. Techniques in this paper are built on this line of research, as the non-convex nature of these analyses allows us to describe the interaction between neural networks, optimization algorithms, and the structures of data. \cite{allen2019can,allen2020backward} also obtained results separating deep neural networks and shallow models such as kernel methods. Before this recent progress, \cite{tian2017analytical,zhong2017recovery,brutzkus2017globally,soltanolkotabi2017learning,du2018convolutional,li2017convergence,li2018algorithmic} also studied how shallow neural networks can learn on certain simple data distributions, but all of them focus on the supervised learning. There are also plenty of studies \cite{soudry2018implicit,gunasekar2018implicit,arora2019implicit,lyu2019gradient,ji2019implicit,razin2020implicit,chizat2020implicit} on the implicit bias of optimization in deep learning, but none of their techniques can be applied to the setting of self-supervised learning.

\section{Conclusion and Discussion}

In this paper, we showed how the prediction head can ensure the neural network learns all the features in non-contrastive learning through theoretical investigation. Our key observation is that the prediction head can leverage two effects called substitution effect and acceleration effect during the training process. We also explained how the necessary components such as output normalization and stop-gradient operation are involved and how they interact during training. Furthermore, we proved that without the prediction head, all neurons of the neural network would focus on learning the strongest feature and result in a collapsed representation. We believe our theory, although based on a very simple setup, can provide some insights into the inner workings of non-contrastive self-supervised learning. We also believe our theoretical framework can be extended to understanding other phenomena in the practice of deep learning. 

On the other hand, our results are still very preliminary, we point out the following open problems that are not addressed by this paper:
\begin{itemize}
    \item When the output normalization is \(\ell_2\)-norm instead of BN. Experiments in \myref{fig:identity-init-performance}{Figure} seem to suggest that there is still a gap between using \(\ell_2\)-norm and BN as output normalization methods. In this case, the acceleration effect may not happen in exactly the same way as in the BN case, but we believe they share the same underlying mechanism and can be proven in theory.
    \item The mystery of the projection head. As our experiments in \myref{wrap-fig:1}{Figure} showed, the outputs of the projection head in the symmetric case (without the prediction head) suffer an extremely strong correlation even with batch normalization used. However, the impact on the base encoder is milder and thus the network can avoid complete collapse, shown in \myref{wrap-fig:1}{Figure} and \myref{fig:identity-init-performance}{Figure}. It is mysterious how the projection head works in non-contrastive learning, and also how it compares to the case of contrastive learning, which has been studied by \cite{Chen2020,jing2021understanding}.
    \item Learning non-linearly features. For the simplicity of analysis, we have assumed the features in the data set are linear. It is of interest to study whether neural networks trained by non-contrastive self-supervised learning can learn non-linear representations better than traditional learning methods such as linear regression or kernel methods, as there has been a series of papers \cite{allen2019can,GhorbaniMMM19,GhorbaniMMM20,allen2020backward,karp2021local} trying to understand it in the supervised setting. 
\end{itemize}

In the end, we also point out that theories based on a one-hidden-layer neural network and linear data composition assumption obviously cannot explain all the phenomena in deep learning. In supervised learning, the \textit{backward feature correction} \cite{allen2020backward} process is observed and theoretically proven as a mechanism for learning hierarchical feature extractors. It is an important open direction to understand how a multi-layer network can learn the complicated features in non-contrastive self-supervised learning. 

\section{Experiment Details}\label{sec:experiment-detail}

\begin{wrapfigure}{r}{0.25\textwidth}
    \caption{\small Framework.}\centering
    
    \includegraphics[width=0.25\textwidth]{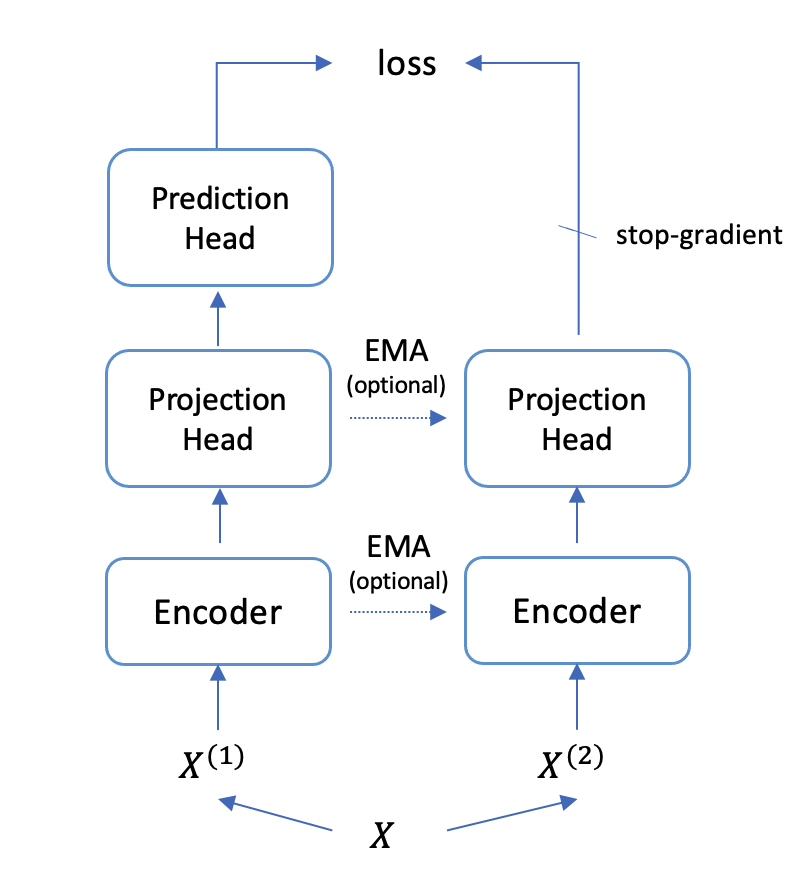}
    \label{wrap-fig:2-experiments}
\end{wrapfigure} 

The framework we use in our experiments is shown in \myref{wrap-fig:2-experiments}{Figure}. We use a modified version of the codebase shared by the authors of \cite{ermolov2021whitening}, and we use the same data augmentation in their implementation. All our experiments (except for \myref{fig:visual-features}{Figure} and \myref{fig:synthetic-experiment}{Figure}) use the following architecture and hyper-parameters: we choose standard ResNet-18 as base encoder architecture, \(0.003\) as the learning rate for Adam optimizer, a two-layer MLP with ReLU activation and \(512\) hidden neurons as the projection head, an identity-initialized but diagonally froze linear matrix (with shape (64x64)) as the prediction head and a non-tracking-stats, non-affine, non-momentum BN layer as the output normalization. Our experiments in \myref{fig:pred-head-trajectory-1}{Figure} use the same architecture and hyper-parameters, but some runs are trained with EMA with momentum \(0.99\), with output BN replaced by \(\ell_2\)-norm or using different prediction heads (such as a two-layer MLP or a linear head, with Pytorch default initialization). Evaluation in \myref{fig:identity-init-performance}{Figure} is by training a linear classifier on top of frozen encoder with no data augmentation.

\clearpage 
\appendix

\bigskip
\begin{center}
{\Huge
\textsc{Appendix: The Proofs}
}
\end{center}
\bigskip

We will be working with population gradients throughout the entire appendix. Indeed, since our algorithms use fresh random samples at each iteration, one can easily obtain from standard concentration inequalities an empirical estimate of population gradients up to \(\frac{1}{\poly(d)}\) error with \(N = \poly(d)\) samples. So we can obtain the same proofs in finite sample case as long as the training ends before some \(T=\poly(d)/\eta\). Now we give some notations and warm-up calculations.

\section{Notations and Gradients}
In this section, we will give some useful notations and warm-up computations for the technical proofs in subsequent sections. We summarize here the notations that will also be defined in later sections:

\paragraph{Notations.} 
We denote \(\Ecal_j = \E[\dbrack{w_j,\xi_p}^6]\), \(\Ecal_{j,3-j} = \E\left[(\dbrack{w_j,\xi_p}^3 + E_{j,3-j}\dbrack{w_{3-j},\xi_p}^3)^2\right]\), and
\begin{align*}
    C_0 & = \frac{\E[|S(X)\cap \P|\cdot|\S(X)\setminus \P|]}{2},& C_1 & = \frac{\E \left[|S(X)\cap \P |^2\right]}{2}, & C_2 & = P - |S(X)|, \\
    \bar{B}_{j,\ell}^3 & = \StopGrad[\dbrack{w_j,v_{\ell}}^3], &B_{j,\ell} &= \dbrack{w_j,v_{\ell}}, & Q_j &= (\E[\StopGrad[G_j^2(X^{(2)})]])^{-1/2}.
\end{align*}
and 
\begin{align*}
    U_j &:= \E[F_j^2(X^{(1)})] = \textstyle\sum_{\ell \in [2]}C_1 \alpha_{\ell}^6(B_{j,\ell}^3 + E_{j,3-j}B_{3-j,\ell}^3)^2  + C_2\Ecal_{j,3-j}\\
    H_{j,\ell} &:= C_1 \alpha_{\ell}^6(B_{j,\ell}^3 + E_{j,3-j}B_{3-j,\ell}^3)^2  + C_2\Ecal_{j,3-j},\\
    K_{j,\ell} &:= C_1\alpha_{\ell}^6(B_{j,\ell}^3 + E_{j,3-j}B_{3-j,\ell}^3)(B_{j,3-\ell}^3 + E_{j,3-j}B_{3-j,3-\ell}^3)
\end{align*}
Moreover, we denote \(\Phi_j := Q_j/U_j^{3/2}\), and (recall \(V := \mathrm{span}(v_1,v_2)\))
\begin{align*}
    R_j := \dbrack{\Pi_{V^{\perp}}w_j,w_j}\qquad R_{1,2} :=\dbrack{\Pi_{V^{\perp}}w_1,w_2}\qquad \overline{R}_{1,2} :=\frac{\dbrack{\Pi_{V^{\perp}}w_1,w_2}}{\|\Pi_{V^{\perp}}w_1\|_2\|\Pi_{V^{\perp}}w_2\|_2}
\end{align*}
For any \(j\in [2]\), the gradient \(-\nabla_{w_j}L(W,E)\) can be decomposed as 
\begin{align*}
    -\nabla_{w_j}L(W,E) &= \sum_{\ell\in[2]}(\Lambda_{j,\ell} + \Gamma_{j,\ell} - \Upsilon_{j,\ell})v_{\ell} - \sum_{(j',\ell) \in [2]\times [2]}\Sigma_{j',\ell}\nabla_{w_j}\Ecal_{j',3-j'} \\
    \Lambda_{j,\ell} & := C_0\Phi_j \alpha_{\ell}^6B_{j,\ell}^5H_{j,3-\ell} \\
    \Gamma_{j,\ell} &:= C_0\Phi_{3-j} E_{3-j,j} \alpha_{\ell}^6B_{3-j,\ell}^3B_{j,\ell}^2H_{3-j,3-\ell} \\
    \Upsilon_{j,\ell} & := C_0\alpha_{3-\ell}^6\left(\Phi_jB_{j,3-\ell}^3B_{j,\ell}^2K_{j,\ell} + \Phi_{3-j}E_{3-j,j}B_{3-j,3-\ell}^3B_{j,\ell}^2K_{3-j,\ell}\right) \\
    \Sigma_{j,\ell} & := C_0C_2\Phi_{j}\alpha_{\ell}^6B_{j,\ell}^3(B_{j,\ell}^3 + E_{j,3-j}B_{3-j,\ell}^3)
\end{align*}
Sometimes we need to decompose \(\Upsilon_{j,\ell} = \Upsilon_{j,\ell,1} + \Upsilon_{j,\ell,2}\) which is straightforward from its expression. In \myref{sec:appendix-phase-3}{Section}, we further define 
\begin{align*}
    \Xi_j^{(t)} & = C_0C_1\alpha_1^6\alpha_2^6\Phi_j^{(t)} \Big((B_{1,1}^{(t)})^6(B_{2,2}^{(t)})^6 + (B_{2,1}^{(t)})^6(B_{1,2}^{(t)})^6\Big) \\
    \Delta_{j,\ell}^{(t)} & = C_0\Phi_j^{(t)}\alpha_{\ell}^6(B_{j,\ell}^{(t)})^3(B_{3-j,\ell}^{(t)})^3C_2\Ecal_{j,3-j}^{(t)}
\end{align*}
for the gradients of the prediction head.
\subsection{Gradient Computation}

Let us \(L(W,E)\) to be the population version of the objective. Because \(\E[F_j(X^{(1)})]\) and \(\E[G_j(X^{(2)})]\) are both zero (which can be verified easily from the zero-mean assumptions of \(z_p(X)\) and \(\xi_p\)), a direct computation gives:
\begin{align*}
    L(W,E) = 2 - \sum_{j\in[2]}\frac{\E[F_j(X^{(1)})\cdot \StopGrad[G_j(X^{(2)})]]}{\sqrt{\E[F_j^2(X^{(1)})]}\sqrt{\E[\StopGrad[G_j^2(X^{(2)})]]}}
\end{align*}
We first calculate the normalizing quantity \(\E[F_j^2(X^{(1)})]\):
\begin{align*}
    \E[F_j^2(X^{(1)})] & =  \E\left[\left(\sum_{p\in[P]}\sigma(\dbrack{w_j,X_p^{(1)}}) + E_{j,3-j}\sigma(\dbrack{w_{3-j},X_p^{(1)}})  \right)^2\right] \\
    &= \frac{1}{2}\sum_{\ell \in [2]}\E \left[|S(X)\cap \P |^2 \alpha_{\ell}^6(\dbrack{w_j,v_{\ell}}^3 + E_{j,3-j}\dbrack{w_{3-j},v_{\ell}}^3)^2 \right] \tag{Because all signal patches has the same sign within the same data} \\
    & \quad + \E\left[|\P \setminus S(X)|(\dbrack{w_j,\xi_p}^3 + E_{j,3-j}\dbrack{w_{3-j},\xi_p}^3)^2\right] \tag{Because noise patches are independent and have mean zero}\\
    &= \sum_{\ell \in [2]}\alpha_{\ell}^6(\dbrack{w_j,v_{\ell}}^3 + E_{j,3-j}\dbrack{w_{3-j},v_{\ell}}^3)^2 \frac{\E \left[|S(X)\cap \P |^2\right]}{2} + (P - |S(X)|) \Ecal_{j,3-j}
\end{align*} 
where we let
\begin{align*}
    \Ecal_{j,3-j} & \stackrel{\text{def}}= \E\left[(\dbrack{w_j,\xi_p}^3 + E_{j,3-j}\dbrack{w_{3-j},\xi_p}^3)^2\right] \\
    &=\E\left[\dbrack{w_j,\xi_p}^6 + 2E_{j,3-j}\dbrack{w_j,\xi_p}^3\dbrack{w_{3-j},\xi_p}^3 + E_{j,3-j}^2\dbrack{w_{3-j},\xi_p}^6\right] 
\end{align*}
On the other hand, we have 
\begin{align*}
    &\E[F_j(X^{(1)})\cdot \StopGrad[G_j(X^{(2)})]] \\
    = \ & \E\left[\left(\sum_{p\in[P]}\sigma(\dbrack{w_j,X_p^{(1)}}) + E_{j,3-j}\sigma(\dbrack{w_{3-j},X_p^{(1)}})\right)\times \left(\sum_{p\in[P]}\sigma(\dbrack{w_j,X_p^{(2)}})\right)\right] \\
    = \ & \frac{1}{2}\sum_{\ell \in [2]}\E \left[\sum_{p \in S(X)\cap \P}\alpha_{\ell}^3(\dbrack{w_j,v_{\ell}}^3 + E_{j,3-j}\dbrack{w_{3-j},v_{\ell}}^3)\times\sum_{p\in S(X)\setminus \P} \alpha_{\ell}^3\StopGrad[\dbrack{w_j,v_{\ell}}^3]\right] \\
    = \ & \sum_{\ell \in [2]}\alpha_{\ell}^6(\dbrack{w_j,v_{\ell}}^3 + E_{j,3-j}\dbrack{w_{3-j},v_{\ell}}^3)\cdot\StopGrad[\dbrack{w_j,v_{\ell}}^3]\cdot\frac{\E[|S(X)\cap \P|\cdot|\S(X)\setminus \P|]}{2}
\end{align*}
Now, by denoting 
\begin{align*}
    C_0 & = \frac{\E[|S(X)\cap \P|\cdot|\S(X)\setminus \P|]}{2},& C_1 & = \frac{\E \left[|S(X)\cap \P |^2\right]}{2}, & C_2 & = P - |S(X)|, \\
    \bar{B}_{j,\ell}^3 & = \StopGrad[\dbrack{w_j,v_{\ell}}^3], &B_{j,\ell} &= \dbrack{w_j,v_{\ell}}, & Q_j &= (\E[\StopGrad[G_j^2(X^{(2)})]])^{-1/2}.
\end{align*}
we denote \(U_j:=\E[F_j^2(X^{(1)})]\), where the expanded expression is
\begin{align*}
    &U_j = \E[F_j^2(X^{(1)})] = \sum_{\ell \in [2]}C_1 \alpha_{\ell}^6(B_{j,\ell}^3 + E_{j,3-j}B_{3-j,\ell}^3)^2  + C_2\Ecal_{j,3-j}
\end{align*}
and we can rewrite the objective as follows
\begin{align}\label{eqdef:loss-obj-expression}
     L(W,E) = 2 - \sum_{j \in [2]}\sum_{\ell \in [2]}\frac{Q_jC_0\alpha_{\ell}^6\bar{B}_{j,\ell}^3(B_{j,\ell}^3 + E_{j,3-j}B_{3-j,\ell}^3)}{U_j^{1/2}}
\end{align}
Now denote
\begin{align*}
    H_{j,\ell} &= C_1 \alpha_{\ell}^6(B_{j,\ell}^3 + E_{j,3-j}B_{3-j,\ell}^3)^2  + C_2\Ecal_{j,3-j},\\
    K_{j,\ell} &= C_1\alpha_{\ell}^6(B_{j,\ell}^3 + E_{j,3-j}B_{3-j,\ell}^3)(B_{j,3-\ell}^3 + E_{j,3-j}B_{3-j,3-\ell}^3)
\end{align*}
It is easy to calculate 
\begin{align*}
    Q_j^{-2} &= \E[\StopGrad[G_j^2(X^{(2)})]] \\
    & = \E\left[\Bigg(\sum_{p\in[P]}\sigma(\dbrack{w_j,X_p^{(2)}})\Bigg)^2\right] \\
    & = \frac{1}{2}\sum_{\ell\in[2]}\alpha_{\ell}^6\dbrack{w_j,v_{\ell}}^6\E\left[|S(X)\cap \P|^2\right] + \E\left[|\P\setminus S(X)|\dbrack{w_j,\xi_p}^6\right] \\ 
    & = \sum_{\ell\in[2]}C_1\alpha_{\ell}^6B_{j,\ell}^6 + C_2\Ecal_j
\end{align*}
where \(\Ecal_j = \E[\dbrack{w_j,\xi_p}^6]\). And thus the gradient can be computed as (notice \(\bar{B}_{j,\ell}^3 = B_{j,\ell}^3\))
\begin{align}
    - \nabla_{w_j} L(W,E) &= \sum_{\ell\in[2]}\left(\frac{C_0Q_j \alpha_{\ell}^6H_{j,3-\ell}B_{j,\ell}^5}{U_j^{3/2}}\right)v_{\ell} + \sum_{\ell\in[2]}\left(\frac{C_0Q_{3-j} E_{3-j,j} \alpha_{\ell}^6B_{3-j,\ell}^3B_{j,\ell}^2H_{3-j,3-\ell}}{U_{3-j}^{3/2}}\right)v_{\ell} \nonumber\\
    & \quad - \sum_{\ell\in[2]}\left(\frac{C_0Q_j\alpha_{3-\ell}^6B_{j,3-\ell}^3B_{j,\ell}^2K_{j,\ell}}{U_j^{3/2 }} + \frac{C_0Q_{3-j}E_{3-j,j}\alpha_{3-\ell}^6B_{3-j,3-\ell}^3B_{j,\ell}^2K_{3-j,\ell}}{U_{3-j}^{3/2 }}\right)v_{\ell} \nonumber\\
    & \quad - \sum_{j'\in[2]}\sum_{\ell\in[2]}\frac{C_0C_2Q_{j'}\alpha_{\ell}^6B_{j',\ell}^3(B_{j',\ell}^3 + E_{j',3-j'}B_{3-j',\ell}^3)}{U_{j'}^{3/2}}\nabla_{w_j} \Ecal_{j',3-j'} \nonumber\\
    & = \sum_{\ell\in[2]}(\Lambda_{j,\ell} + \Gamma_{j,\ell} - \Upsilon_{j,\ell})v_{\ell} - \sum_{(j',\ell) \in [2]\times [2]}\Sigma_{j',\ell}\nabla_{w_j}\Ecal_{j',3-j'} \label{eqdef:weight-grad}
\end{align}
where 
\begin{align*}
    \nabla_{w_j} \Ecal_{j,3-j} &= 6\E[\dbrack{w_{j},\xi_p}^5\xi_p + E_{j,3-j}\dbrack{w_{j},\xi_p}^2\dbrack{w_{3-j},\xi_p}^3\xi_p] \\
    \nabla_{w_j} \Ecal_{3-j,j} &= 6\E[E_{3-j,j}^2\dbrack{w_{j},\xi_p}^5\xi_p + E_{3-j,j}\dbrack{w_{3-j},\xi_p}^3\dbrack{w_{j},\xi_p}^2\xi_p]
\end{align*}
As for the gradient of the prediction head, we can calculate 
\begin{align*}
     -\nabla_{E_{j,3-j}} L(W,E) & = \sum_{\ell \in [2]}\frac{C_0Q_j\alpha_{\ell}^6B_{j,\ell}^3B_{3-j,\ell}^3U_j}{U_j^{3/2}} \\
    & \quad - \sum_{\ell \in [2]}\frac{C_0Q_j\alpha_{\ell}^6B_{j,\ell}^3(B_{j,\ell}^3 + E_{j,3-j}B_{3-j,\ell}^3)\sum_{\ell'\in[2]}C_1\alpha_{\ell'}^6(B_{j,\ell'}^3+E_{j,3-j}B_{3-j,\ell'}^3)B_{3-j,\ell'}^3}{U_j^{3/2}} \\
    & \quad - \sum_{\ell \in [2]}\frac{C_0C_2Q_{j}\alpha_{\ell}^6B_{j,\ell}^3(B_{j,\ell}^3 + E_{j,3-j}B_{3-j,\ell}^3)}{U_{j}^{3/2}}\nabla_{E_{j,3-j}} \Ecal_{j,3-j} \\
    & = \sum_{\ell \in [2]}\frac{C_0Q_j\alpha_{\ell}^6B_{j,\ell}^3(B_{3-j,\ell}^3H_{j,3-\ell} - B_{3-j,3-\ell}^3K_{j,3-\ell})}{U_j^{3/2}}\\
    & \quad - \sum_{\ell \in [2]}\Sigma_{j,\ell}\E\left[2\dbrack{w_j,\xi_p}^3\dbrack{w_{3-j},\xi_p}^3 + 2E_{j,3-j}\dbrack{w_{3-j},\xi_p}^6\right]
\end{align*}
where \(\Sigma_{j,\ell}\) is defined in \eqref{eqdef:weight-grad}. In fact, all the above gradient expressions can be simplified by letting \(\Phi_j := Q_j/U_j^{3/2}\) for \(j\in[2]\), which is what we shall do in later sections.

\paragraph{Summarizing the notations.} We shall define some useful notations to simplify the proof. We define \(V = \mathrm{span}(v_1,v_2)\). Let \(\Pi_A\) be the projection operator to subspace \(A\subset \R^d\), then
\begin{align*}
    R_j := \dbrack{\Pi_{V^{\perp}}w_j,w_j}\qquad R_{1,2} :=\dbrack{\Pi_{V^{\perp}}w_1,w_2}\qquad \overline{R}_{1,2} :=\frac{\dbrack{\Pi_{V^{\perp}}w_1,w_2}}{\|\Pi_{V^{\perp}}w_1\|_2\|\Pi_{V^{\perp}}w_2\|_2}
\end{align*}

\subsection{Some Useful Bounds for Gradients}

In this section we use the superscript \(^{(t)}\) to denote the iteration \(t\) during training. Below we present a claim which comes from direct calculations of \(\Sigma_{j,\ell}^{(t)}\) and \(\nabla_{w_j}\Ecal_{j',3-j'}^{(t)}\), which is very useful in the following sections.

\begin{claim}[on \(\Sigma_{j,\ell}^{(t)}\) and \(\nabla_{w_j}\Ecal_{j',3-j'}^{(t)}\)]\label{claim:noise}
    Let \(R_j, R_{1,2}^{(t)}\) be defined as above, then we have
    \begin{enumerate}[(a)]
        \item \(\Sigma_{j,\ell}^{(t)} = O(\Sigma_{1,1}^{(t)})\frac{(B_{j,\ell}^{(t)})^6 + E_{j,3-j}^{(t)}(B_{3-j,\ell}^{(t)})^3(B_{j,\ell}^{(t)})^3 }{(B_{1,1}^{(t)})^6} \frac{\Phi_j^{(t)}}{\Phi_1^{(t)}}\);
        \item \(\dbrack{\nabla_{w_j} \Ecal_{j,3-j}^{(t)} ,\Pi_{V^{\top}}w_j^{(t)}} = \Theta([R_{j}^{(t)}]^3) \pm \Theta(E_{j,3-j}^{(t)})(\overline{R}_{1,2}^{(t)} + \varrho)[R_{1}^{(t)}]^{3/2}[R_{2}^{(t)}]^{3/2} \);
        \item \(\dbrack{\nabla_{w_j} \Ecal_{3-j,j}^{(t)} ,w_j^{(t)}} = \Theta((E_{3-j,j}^{(t)})^2)[R_{j}^{(t)}]^3 \pm O(E_{3-j,j}^{(t)})(\overline{R}_{1,2}^{(t)} + \varrho)[R_{1}^{(t)}]^{3/2}[R_{2}^{(t)}]^{3/2}\)
        \item \(\dbrack{\nabla_{w_j} \Ecal_{j,3-j}^{(t)},w_{3-j}^{(t)}} = (\Theta(\overline{R}_{1,2}^{(t)}) \pm \varrho )[R_{j}^{(t)}]^{5/2}[R_{3-j}^{(t)}]^{1/2} + O(E_{j,3-j}^{(t)})R_{j}^{(t)}[R_{3-j}^{(t)}]^{2}\);
        \item \(\dbrack{\nabla_{w_j} \Ecal_{3-j,j}^{(t)},w_{3-j}^{(t)}} = ((E_{3-j,j}^{(t)})^2(\Theta(\overline{R}_{1,2}^{(t)}) \pm \varrho )[R_{j}^{(t)}]^{5/2}[R_{3-j}^{(t)}]^{1/2} + O(E_{3-j,j}^{(t)})R_{j}^{(t)}[R_{3-j}^{(t)}]^{2}) \)
    \end{enumerate}
\end{claim} 

\begin{proof}
    The part on \(\Sigma_{j,\ell}^{(t)}\) is trivial from its expression, we shall focus on proving (b) -- (d).\\
    \textbf{On \(\dbrack{\nabla_{w_j} \Ecal_{j',3-j'}^{(t)},w_j^{(t)}} \):} If \(j = j'\), then 
    \begin{align*}
        \dbrack{\nabla_{w_j} \Ecal_{j,3-j}^{(t)},w_j^{(t)}} &= \Theta(1)\E[\dbrack{w_j^{(t)},\xi_p}^6 + E_{j,3-j}^{(t)}\dbrack{w_j^{(t)},\xi_p}^3\dbrack{w_{3-j}^{(t)},\xi_p}^3] \\
        &= \Theta(1)\E[\dbrack{w_j^{(t)},\xi_p}^6] + O(E_{j,3-j}^{(t)})\E[\dbrack{w_j^{(t)},\xi_p}^3(\dbrack{w_{3-j}^{(t)},\xi_p}^3 \\
        &\quad - \dbrack{(I - \bar{w}_{j,t}\bar{w}_{j,t}^{\top}) w_{3-j}^{(t)},\xi_p}^3)] \\
        &\quad  + O(E_{j,3-j}^{(t)})\E[\dbrack{w_j^{(t)},\xi_p}^3\dbrack{(I - \bar{w}_{j,t}\bar{w}_{j,t}^{\top}) w_{3-j}^{(t)},\xi_p}^3]
    \end{align*}
    Write \(\bar{w}_{j,t} = \frac{\Pi_{V^{\perp}}w_j^{(t)}}{\|\Pi_{V^{\perp}}w_j^{(t)}\|_2}\), we can derive 
    \begin{align*}
        &\quad \E[\dbrack{w_j^{(t)},\xi_p}^3(\dbrack{w_{3-j}^{(t)},\xi_p}^3 - \dbrack{(I - \bar{w}_{j,t}\bar{w}_{j,t}^{\top}) w_{3-j}^{(t)},\xi_p}^3)] \\
        &= \E[\dbrack{w_j^{(t)},\xi_p}^3\dbrack{\bar{w}_{j,t}\bar{w}_{j,t}^{\top}w_{3-j}^{(t)},\xi_p} O(\dbrack{w_{3-j}^{(t)},\xi_p}^2)] \\
        & = O(\frac{R_{1,2}^{(t)}}{\|\Pi_{V^{\perp}}w_{j}^{(t)}\|_2^2})\E[\dbrack{w_j^{(t)},\xi_p}^4 \dbrack{w_{3-j}^{(t)},\xi_p}^2]\\
        & \leq O(\frac{R_{1,2}^{(t)}}{\|\Pi_{V^{\perp}}w_{j}^{(t)}\|_2^2})\E[\dbrack{w_j^{(t)},\xi_p}^6]^{\frac{2}{3}} \E[\dbrack{w_{3-j}^{(t)},\xi_p}^6]^{\frac{1}{3}} \tag{by H\"older's inequality} \\
        & \leq O(\overline{R}_{1,2}^{(t)})\|\Pi_{V^{\perp}}w_j^{(t)}\|_2^3\|\Pi_{V^{\perp}}w_{3-j}^{(t)}\|_2^3
    \end{align*}
    and by our assumption on noise \(\xi_p\), we also have
    \begin{displaymath}
        \E[\dbrack{w_j^{(t)},\xi_p}^3\dbrack{(I - \bar{w}_{j,t}\bar{w}_{j,t}^{\top}) w_{3-j}^{(t)},\xi_p}^3] \leq O(\varrho)\|\Pi_{V^{\perp}}w_j^{(t)}\|_2^3\|\Pi_{V^{\perp}}w_{3-j}^{(t)}\|_2^3
    \end{displaymath}
    Combined with the fact that \(\E[\dbrack{w_j^{(t)},\xi_p}^6] = O(\|\Pi_{V^{\perp}} w_j^{(t)}\|_2^3)\), we can get 
    \begin{align*}
        \dbrack{\nabla_{w_j} \Ecal_{j,3-j}^{(t)} ,w_j^{(t)}} = O(\|\Pi_{V^{\perp}}w_j^{(t)}\|_2^6) \pm O(E_{j,3-j}^{(t)})(R_{1,2}^{(t)} + \varrho)\|\Pi_{V^{\perp}}w_j^{(t)}\|_2^3\|\Pi_{V^{\perp}}w_{3-j}^{(t)}\|_2^3
    \end{align*}
    when \(j' = 3-j\), we also have 
    \begin{align*}
        \dbrack{\nabla_{w_j} \Ecal_{3-j,j}^{(t)} ,w_j^{(t)}} &= \Theta(1)\E[(E_{3-j,j}^{(t)})^2\dbrack{w_j^{(t)},\xi_p}^6 + E_{3-j,j}^{(t)}\dbrack{w_j^{(t)},\xi_p}^3\dbrack{w_{3-j}^{(t)},\xi_p}^3] \\
        & = O((E_{3-j,j}^{(t)})^2)\|\Pi_{V^{\perp}}w_j^{(t)}\|_2^6 \pm O(E_{3-j,j}^{(t)})(R_{1,2}^{(t)} + \varrho)\|\Pi_{V^{\perp}}w_j^{(t)}\|_2^3\|\Pi_{V^{\perp}}w_{3-j}^{(t)}\|_2^3
    \end{align*}
    \newline
    \textbf{On \(\dbrack{\nabla_{w_j} \Ecal_{j',3-j'}^{(t)},w_{3-j}^{(t)}} \):} when \(j' = j\), we have 
    \begin{align}\label{eqdef:claim-noise-phase2-1}
        \begin{split}
           \dbrack{\nabla_{w_j} \Ecal_{j,3-j}^{(t)},w_{3-j}^{(t)}} & = O(1)\E[\dbrack{w_j^{(t)},\xi_p}^5\dbrack{w_{3-j}^{(t)},\xi_p} + E_{j,3-j}^{(t)}\dbrack{w_j^{(t)},\xi_p}^2\dbrack{w_{3-j}^{(t)},\xi_p}^4] \\
            & = O(1) \E[\dbrack{w_j^{(t)},\xi_p}^5\dbrack{ (I-\bar{w}_{j,t}\bar{w}_{j,t}^{\top} + \bar{w}_{j,t}\bar{w}_{j,t}^{\top}) w_{3-j}^{(t)},\xi_p}]  \\
            &\quad + O(1)\E[E_{j,3-j}^{(t)}\dbrack{w_j^{(t)},\xi_p}^2\dbrack{w_{3-j}^{(t)},\xi_p}^4] 
        \end{split}
    \end{align}
    Using H\"older's inequality and our assumpsion on \(\xi_p\), we have 
    \begin{align*}
        \E[\dbrack{w_j^{(t)},\xi_p}^5\dbrack{(I-\bar{w}_{j,t}\bar{w}_{j,t}^{\top})w_{3-j}^{(t)},\xi_p}] &\lesssim \varrho\|\Pi_{V^{\perp}}w_j^{(t)}\|_2^5\|\Pi_{V^{\perp}}w_{3-j}^{(t)}\|_2
    \end{align*}
    In the meantime, we also have 
    \begin{align*}
        \E[\dbrack{w_j^{(t)},\xi_p}^5\dbrack{\bar{w}_{j,t}\bar{w}_{j,t}^{\top}w_{3-j}^{(t)},\xi_p}] = \Theta(\overline{R}_{1,2}^{(t)})\E[\dbrack{w_j^{(t)},\xi_p}^6][R_{j}^{(t)}]^{-1/2}[R_{3-j}^{(t)}]^{1/2} = \Theta(\overline{R}_{1,2}^{(t)})[R_{j}^{(t)}]^{5/2}[R_{3-j}^{(t)}]^{1/2}
    \end{align*} 
    for the last term in \eqref{eqdef:claim-noise-phase2-1}, we can also use H\"older's inequality to get 
    \begin{align*}
        E_{j,3-j}^{(t)}\E[\dbrack{w_j^{(t)},\xi_p}^2\dbrack{w_{3-j}^{(t)},\xi_p}^4] & \lesssim E_{j,3-j}^{(t)}\E[\dbrack{w_j^{(t)},\xi_p}^6]^{1/3} \E[\dbrack{w_{3-j}^{(t)},\xi_p}^6]^{2/3} \lesssim E_{j,3-j}^{(t)} R_j^{(t)}[R_{3-j}^{(t)}]^{2} 
    \end{align*}
    Therefore, we can combine above analysis to get
    \begin{align*}
        \dbrack{\nabla_{w_j} \Ecal_{j,3-j}^{(t)},w_{3-j}^{(t)}} = (\Theta(\overline{R}_{1,2}^{(t)}) \pm \varrho )[R_{j}^{(t)}]^{5/2}[R_{3-j}^{(t)}]^{1/2} + O(E_{j,3-j}^{(t)})R_{j}^{(t)}[R_{3-j}^{(t)}]^{2} 
    \end{align*}
    When \(j' = 3-j\), we also have 
    \begin{align*}
        \dbrack{\nabla_{w_j} \Ecal_{3-j,j}^{(t)},w_{3-j}^{(t)}} & = 6\E[(E_{3-j,j}^{(t)})^2\dbrack{w_j^{(t)},\xi_p}^5\dbrack{w_{3-j}^{(t)},\xi_p} + E_{3-j,j}^{(t)}\dbrack{w_j^{(t)},\xi_p}^2\dbrack{w_{3-j}^{(t)},\xi_p}^4] \\
        & = 6(E_{3-j,j}^{(t)})^2(\Theta(\overline{R}_{1,2}^{(t)}) \pm \varrho )[R_{j}^{(t)}]^{5/2}[R_{3-j}^{(t)}]^{1/2} + E_{3-j,j}^{(t)}R_{j}^{(t)}[R_{3-j}^{(t)}]^{2} 
    \end{align*}
    which proves the claim.
\end{proof}

\section{Phase I: Learning the Stronger Feature}

In this section, we shall discuss the initial phase of learning the stronger feature. Firstly, we establish some properties at the initialization for our induction afterwards.

\paragraph{Initialization properties.}
We prove the following properties for our network at initialization. Recall our initialization is \(w_{j}^{(0)} \sim \N(0,I_d/d), \forall j\in[2]\) and \(E^{(0)} = I_2\).
\begin{lemma}[properties at initialization]\label{property-init} Recall that without loss of generality we let \(|B_{1,1}^{(0)}| = \max_{j\in[2]}|B_{j,1}^{(0)}|\). With probability \(1 - o(1)\), the following holds:
    \begin{enumerate}[(a)]
        \item \(\|w_j^{(0)}\|_2^2 = 1\pm\widetilde{O}(\frac{1}{\sqrt{d}})\) for all \(j\in[2]\), and \(|\dbrack{w_1^{(0)},w_2^{(0)}}| \leq \widetilde{O}(\frac{1}{\sqrt{d}})\);
        \item \(\max_{j,\ell}|B_{j,\ell}^{(0)}| \leq O(\sqrt{\log d/d})\) and \(\min_{j,\ell}|B_{j,\ell}^{(0)}| \geq \Omega(\frac{1}{\log d}) \max_{j,\ell }|B_{j,\ell}^{(0)}|\);
        \item \(|B_{1,1}^{(0)}| \geq |B_{2,1}^{(0)}|(1 + \frac{1}{\log d})\);
        \item \(\Ecal_{j}^{(0)} = (1- O(\frac{1}{d^3}))\sigma^6\|w_j^{(0)}\|_2^6 = \Theta(1)\) for all \(j\in[2]\);
        \item \(H_{j,\ell}^{(0)} = C_2\Ecal_{j}^{(0)}(1 + \widetilde{O}(\frac{1}{\sqrt{d}}))\) for all \((j,\ell)\in [2]\times[2]\);
        \item \(U_j^{(0)} = C_2\Ecal_{j}^{(0)}(1 + \widetilde{O}(\frac{\alpha_1^6}{\sqrt{d}}))\) for all \(j \in [2]\);
        \item \((Q_j^{(0)})^{-2} = C_2\Ecal_{j}^{(0)}(1 + \widetilde{O}(\frac{\alpha_1^6}{\sqrt{d}}))\) for all \(j \in [2]\);
        \item \(K_{j,\ell}^{(0)} \leq \widetilde{O}(\alpha_{\ell}^6/d^{3})\) for all \((j,\ell)\in [2]\times[2]\).
    \end{enumerate}
\end{lemma}

Let us first introduce a fact about Gaussian ratio distribution without proof.

\begin{fact}[Gaussian ratio distribution]\label{fact:gaussian-ratio}
    If \(X\) and \(Y\) are two independent standard Gaussian variables, then the probability density of \(Z = X/Y\) is \(p(z) = \frac{1}{\pi(1 +z^2)}, z\in (-\infty,\infty)\).
\end{fact}

\begin{proof}[Proof of \myref{property-init}{Lemma}]
    \begin{enumerate}[a.]
        \item Norm bound comes from simple \(\chi^2\) concentration inequality and our initialization \(w_j^{(0)}\sim\N(0,\frac{I_d}{d})\). The inner product bound comes from Gaussian concentration.
        \item It is from a direct calculation under our initialization, and some application of Gaussian c.d.f. and a union bound.
        \item It is from a probability distribution of Gaussian ratio distribution from \myref{fact:gaussian-ratio}{Fact} to bound the probability of \(|B_{1,1}^{(0)}| / |B_{2,1}^{(0)}| \leq (1+\frac{1}{\log d})\) (WLOG we let \(|B_{1,1}^{(0)}| = \max_{j\in[2]}|B_{j,1}^{(0)}|\)).
        \item It can be directly proven from our assumption on noise \(\xi_p\) in the subspace \(V^{\perp}\) and (a).
        \item Since at the initialization we have \(B_{j,\ell}^{(0)} = \widetilde{O}(\frac{1}{\sqrt{d}}), j,\ell \in [2]\) and \(E_{j,3-j}^{(0)}=0\), it is easy to directly upper bound the errors. 
        \item Again from \(B_{j,\ell}^{(0)} = \widetilde{O}(\frac{1}{\sqrt{d}}), \forall j,\ell \in [2]\) at initialization and a direct upper bound.
        \item Proof is similar to (e).
        \item Directly from a naive upper bound using (b).
    \end{enumerate}
\end{proof}

\subsection{Induction in Phase I}

We define phase I as all iterations \(t \leq T_1\), where \(T_1:=\min\{t: B_{1,1}^{(t)}\geq 0.01\}\), we will prove the existence of \(T_1\) at the end of this section. We state the following induction hypotheses, which will hold throughout the phase I: 
\begin{induct}\label{induct:phase-1}
    For each \(t\leq T_1\), all of the followings hold:
    \begin{enumerate}[(a).]
        \item \(\|w_j^{(t)}\|_2 = \|w_j^{(0)}\|_2 \pm \widetilde{O}(\varrho+\frac{1}{\sqrt{d}}) \) for each \(j \in [2]\);
        \item \(|B_{1,2}^{(t)}|, |B_{2,1}^{(t)}|, |B_{2,2}^{(t)}| = \widetilde{\Theta}(\frac{1}{\sqrt{d}})\);
        \item \(|B_{1,1}^{(t)}|\geq \Omega(\frac{1}{\log d}) \max(|B_{1,2}^{(t)}|, |B_{2,2}^{(t)}|, |B_{2,1}^{(t)}|)\);
        \item \(|E_{1,2}^{(t)}|\leq \widetilde{O}(\varrho+\frac{1}{\sqrt{d}})\frac{\eta_E}{\eta}|B_{1,1}^{(t)}|\) and \(|E_{2,1}^{(t)}| \leq \widetilde{O}(\frac{1}{d})\);
        \item \(R_1^{(t)}, R_2^{(t)} = \Theta(1)\), \(|R_{1,2}^{(t)}| \leq \widetilde{O}(\varrho + \frac{1}{\sqrt{d}})\)
    \end{enumerate}
\end{induct}
\begin{remark}
    Since we have chosen \(\eta_E \leq \eta\) and \(\varrho \leq \frac{1}{d^{\Omega(1)}}\), \myref[d]{induct:phase-1}{Induction} implies \(|E_{j,3-j}^{(t)}| = o(1)\) throughout \(t\leq T_1\). 
\end{remark}
We shall prove the above induction holds in later sections, but first we need some useful claims assuming our induction holds in this phase.

\subsection{Computing Variables at Phase I}

Firstly we establish a claim controlling the noise terms \(\Ecal_j, \Ecal_{j,3-j}\) during this phase.
\begin{claim}\label{claim:Ecal-phase1}
    At each iteration \(t \leq T_1\), if \myref{induct:phase-1}{Induction} holds, then
    \begin{enumerate}[(a)]
        \item \(\Ecal_1^{(t)} = \Ecal_2^{(t)} \pm O(\sum_{\ell \in [2]}|B_{j,\ell}^{(t)}|  + \widetilde{O}(\varrho+\frac{1}{\sqrt{d}}))\)
        \item \(\Ecal_j^{(t)} = \Ecal_j^{(0)}\pm O(\sum_{\ell \in [2]}|B_{j,\ell}^{(t)}|  + \widetilde{O}(\varrho+\frac{1}{\sqrt{d}}))\)
        \item \(\Ecal_{j,3-j}^{(t)} = \Ecal_j^{(t)} \pm \widetilde{O}(E_{j,3-j}^{(t)}(\varrho+\frac{1}{\sqrt{d}})+ (E_{j,3-j}^{(t)})^2) \);
    \end{enumerate}
\end{claim}

\begin{proof}
    For (a), we can simply write down 
    \begin{align*}
        \Ecal_j^{(t)} = \E[\dbrack{w_j,\xi_p}^6] = \sigma^6\|\Pi_{V^{\perp}}w_j^{(t)}\|_2^6
    \end{align*}
    Note that by \myref[a]{induct:phase-1}{Induction} we always have \(\|w_j^{(t)}\|_2 = \|w_j^{(0)}\|_2 \pm \widetilde{O}(\varrho+\frac{1}{\sqrt{d}})\), and by \myref[a]{property-init}{Lemma} we also have \(\|w_j^{(0)}\|_2 = (1 \pm \widetilde{O}(\frac{1}{\sqrt{d}}))\|w_j^{(0)}\|_2\), which implies
    \begin{align*}
        \|\Pi_{V^{\perp}}w_j^{(t)}\|_2 - \|\Pi_{V^{\perp}}w_{3-j}^{(t)}\|_2 & = \|w_j^{(t)}\|_2 - \|w_{3-j}^{(t)}\|_2 \pm  O(\sum_{j,\ell\in [2]^2}B_{j,\ell}^{(t)}) \\
        &= \|w_j^{(0)}\|_2 - \|w_{3- j}^{(0)}\|_2 \pm O(\sum_{j,\ell\in [2]^2}B_{j,\ell}^{(t)}) \pm \widetilde{O}(\varrho+\frac{1}{\sqrt{d}}) \\
        & = \widetilde{O}(\frac{1}{\sqrt{d}}) \pm O(\sum_{j,\ell\in [2]^2}B_{j,\ell}^{(t)}) \pm \widetilde{O}(\varrho+\frac{1}{\sqrt{d}})
    \end{align*}
    By the elementary equality \(x^n - y^n = (x-y)\sum_{0\leq i\leq n-1}x^{i}y^{n-1-i}\), we can obtain (a). The proof of (b) is almost the same as (a), and the proof of (c) is just direct calculation.
\end{proof}

Equipped with \myref{claim:Ecal-phase1}{Claim}, we can establish the following lemma, which will be frequently applied to bound the gradient in our induction argument.

\begin{lemma}[variables control in phase I]\label{lem:phase1-variables}
    Suppose Induction~\ref{induct:phase-1} holds at some iteration \(t\leq T_{1}\) , then we have: 
    \begin{enumerate}[(a)]
        \item if \(\forall \ell \in [2],\alpha_\ell |B_{j,\ell}^{(t)}|\leq O(1)\), then \(\Phi_j^{(t)} = (C_2\Ecal_j^{(t)})^{-2}(1\pm\frac{1}{\polylog(d)})\);
        \item if \(\exists \ell \in [2], |B_{j,\ell}^{(t)}| \geq \Omega(\frac{1}{\alpha_{\ell}})\), then \(\Phi_j^{(t)} = O((C_2\Ecal_j^{(t)} + \sum_{\ell\in[2]}C_1\alpha_{\ell}^6(B_{j,\ell}^{(t)})^6 )^{-2})\);
        \item if \(\alpha_{\ell}|B_{j,\ell}^{(t)}|\leq O(1)\), \(H_{j,\ell}^{(t)} = C_2\Ecal_j^{(t)}(1+ \frac{1}{\polylog(d)}) = \Theta(C_2)\), otherwise \( H_{j,\ell}^{(t)} \in [\Omega(C_2), \widetilde{O}(\alpha_{\ell}^6)]\)
        \item \(|K_{j,\ell}^{(t)}| \leq \widetilde{O}(\alpha_{\ell}^6/d^{3/2})\)
    \end{enumerate}
\end{lemma}

\begin{proof}
    \begin{enumerate}[(a)]
        \item From our assumptions that \(|B_{1,2}^{(t)}|, |B_{2,1}^{(t)}|,  |B_{2,2}^{(t)}| \leq \widetilde{O}(\frac{1}{\sqrt{d}})\) and \(\alpha_1B_{1,1}^{(t)}\leq O(1)\), and also the fact that \(\Ecal_j^{(t)} = \Omega(\sigma^6) = \Omega(1)\), \(C_2 = \Theta(\polylog(d))\gg C_1\), we can calculate 
        \begin{align*}
            U_j^{(t)} &= \sum_{\ell \in [2]}C_1 \alpha_{\ell}^6((B_{j,\ell}^{(t)})^3 + E_{j,3-j}^{(t)}(B_{3-j,\ell}^{(t)})^3)^2  + C_2\Ecal_{j,3-j}^{(t)} \\
            &= O(C_1) + C_2\Ecal_{j}^{(t)} + \widetilde{O}(\varrho + \frac{1}{\sqrt{d}}) \\
            &= C_2\Ecal_{j}^{(t)}(1 \pm \frac{1}{\polylog(d)})
        \end{align*}
        Meanwhile, we can also compute similarly 
        \begin{align*}
            Q_j^{(t)} = \sum_{\ell\in[2]}C_1\alpha_{\ell}^6(B_{j,\ell}^{(t)})^6 + C_2\Ecal_j = C_2\Ecal_{j}^{(t)}(1 \pm \frac{1}{\polylog(d)})
        \end{align*}
        Therefore \(\Phi_j^{(t)} = Q_j^{(t)}/(U_j^{(t)})^{3/2} = (C_2\Ecal_{j}^{(t)}(1 \pm \frac{1}{\polylog(d)}))^{-2}\) as desired.
        \item The proof is similar to that of (a).
        \item when \(\alpha_1B_{1,1}^{(t)}\leq O(1)\), the proof is similar to (a). When \(\alpha_1B_{1,1}^{(t)}\geq O(1)\), we have from \myref[a]{induct:phase-1}{Induction} and \(H_{j,\ell}^{(t)}\)'s expression that
        \begin{align*}
            H_{j,\ell}^{(t)} &= C_1 \alpha_{\ell}^6((B_{j,\ell}^{(t)})^3 + E_{j,3-j}^{(t)}(B_{3-j,\ell}^{(t)})^3)^2  + C_2\Ecal_{j,3-j}^{(t)} \leq \widetilde{O}(\alpha_{\ell}^6)
        \end{align*}
        And since \(T_1 := \min\{t: B_{1,1}^{(t)} \geq 0.01\}\), so for \(t\leq T_1\), we have 
        \begin{align*}
            H_{j,\ell}^{(t)} \geq C_2\Ecal_{j,3-j}^{(t)} \stackrel{\text{\ding{172}}}{\geq} C_2\Ecal_{j}^{(t)} - |E_{j,3-j}^{(t)}| \stackrel{\text{\ding{173}}}\geq \Omega(C_2)
        \end{align*}
        where \ding{172} is from \myref[b]{claim:Ecal-phase1}{Claim} and \ding{173} is from \myref[d]{induct:phase-1}{Induction}.
        \item Since we have assumed \(|B_{1,2}^{(t)}|, |B_{2,1}^{(t)}|, |B_{2,2}^{(t)}| \leq \widetilde{O}(\frac{1}{\sqrt{d}})\), it is direct to bound \(|K_{j,\ell}^{(t)}| \leq \widetilde{O}(\alpha_{\ell}^6/d^{1.5})\).
    \end{enumerate}
\end{proof}

\begin{claim}[about \(\Sigma_{j,\ell}^{(t)}\) and \(\nabla_{w_j}\Ecal_{j',3-j'}^{(t)}\)]\label{claim:noise-phase1}
    If \myref{induct:phase-1}{Induction} holds at iteration \(t \leq T_1\), then 
    \begin{enumerate}[(a)]
        \item \(\Sigma_{j,\ell}^{(t)} = O(\Lambda_{1,1}^{(t)}B_{1,1}^{(t)})\frac{(B_{j,\ell}^{(t)})^6 + E_{j,3-j}^{(t)}(B_{3-j,\ell}^{(t)})^3(B_{j,\ell}^{(t)})^3 }{(B_{1,1}^{(t)})^6} \frac{\Phi_j^{(t)}}{\Phi_1^{(t)}}\);
        \item \(\dbrack{\nabla_{w_j} \Ecal_{j,3-j}^{(t)} ,w_j^{(t)}} = O(1) \pm O(E_{j,3-j}^{(t)})(R_{1,2}^{(t)} + \varrho)\);
        \item \(\dbrack{\nabla_{w_j} \Ecal_{3-j,j}^{(t)} ,w_j^{(t)}} = O((E_{3-j,j}^{(t)})^2) \pm O(E_{3-j,j}^{(t)})(R_{1,2}^{(t)} + \varrho)\)
        \item \(|\dbrack{\nabla_{w_j} \Ecal_{j,3-j}^{(t)},w_{3-j}^{(t)}}| = O(R_{1,2}^{(t)} + \varrho ) + O(E_{j,3-j}^{(t)})\);
        \item \(|\dbrack{\nabla_{w_j} \Ecal_{3-j,j}^{(t)},w_{3-j}^{(t)}}| = O(R_{1,2}^{(t)} + \varrho)(E_{3-j,j}^{(t)})^2 + O(E_{3-j,j}^{(t)})\)
    \end{enumerate}
\end{claim} 

\begin{proof}
    Notice that \(\|\Pi_{V^{\perp}} w_j^{(t)}\|_2 = \Theta(1), \forall j\in [2] \) for \(t\leq T_1\), which is because of \(\|w_j^{(t)}\|_2 = \sqrt{2}\pm o(1)\) from \myref[a]{induct:phase-1}{Induction} and \(\max_{j,\ell}|B_{j,\ell}^{(t)}|<0.02\)\footnote{due to our choice of \(\eta = \frac{1}{\poly(d)}\) is small, we can make sure when \(T_1 = \min\{t:B_{1,1}^{(t)}\geq 0.01\}\), \(B_{1,1}^{(T_1)}<0.02\).}. Now we can apply \myref{claim:noise}{Claim} to obtain the bounds.
\end{proof}

\subsection{Gradient Lemmas for Phase I}

We first present an interesting lemma regarding the effects of Batch-Normalization on the gradients of weights. The following lemma allow us maintain the norm of weights to above a constant throughout phase I.  

\begin{lemma}[effects of BN on gradients]\label{lem:bn-grad}
    For any \(W = (w_1,w_2)\) and \(E\), it holds
    \begin{enumerate}
        \item[(a)] \(\sum_{j\in[2]}\dbrack{\nabla_{w_j}L(W,E), w_j} = 0 \);
    \end{enumerate}
    Further, if \myref{induct:phase-1}{Induction} holds for each \(t \leq T_1\), we have
    \begin{enumerate}
        \item[(b)] \(|\dbrack{\nabla_{w_j} L(W^{(t)} ,E^{(t)} ), w_j^{(t)} }| \leq \widetilde{O}(\varrho + \frac{1}{\sqrt{d}})|\Lambda_{1,1}|\sum_{j\in[2]}|E_{j,3-j}^{(t)}|\) for each \(j\in[2]\).
    \end{enumerate}
\end{lemma}

\begin{proof}
    \textbf{Proof of (a):} We first calculate the gradient term as follows:
    \begin{align*}
        \nabla_{W} L(W,E) &= \nabla_{W}\sum_{j\in[2]}\frac{\E[F_j(X^{(1)})\cdot \StopGrad[G_j(X^{(2)})]]}{\sqrt{\E[F_j^2(X^{(1)})]}\sqrt{\E[\StopGrad[G_j^2(X^{(2)})]]}} \\
        & = \sum_{j\in[2]}\frac{\E[(\nabla_WF_j(X^{(1)})) \cdot [G(X^{(2)})]_j]\cdot\E[F_j^2(X^{(1)})] }{(\E[F_j^2(X^{(1)})])^{3/2}\sqrt{\E[G_j^2(X^{(2)})]}} \\
        &\quad - \sum_{j\in[2]}\frac{\E[(\nabla_WF_j(X^{(1)}))\cdot F_j(X^{(1)})]\cdot\E[[F_j(X^{(1)})\cdot [G(X^{(2)})]_j]}{(\E[F_j^2(X^{(1)})])^{3/2}\sqrt{\E[G_j^2(X^{(2)})]}}
    \end{align*}
    Since by our definition \(\dbrack{\nabla_WF_j(X^{(1)}), W} = \sum_{i\in[2]}\dbrack{\nabla_{w_i}[F_j(X^{(1)}), w_i} = 3[F_j(X^{(1)})\), we immediately have \(\sum_{j\in[2]}\dbrack{\nabla_{w_j}L(W,E), w_j} = 0\). \\
    \newline
    \textbf{Proof of (b):} Firstly we define a new notion 
    \begin{align*}
        \nabla_{i,j} = \nabla_{w_i}\frac{\E[F_j(X^{(1)})\cdot \StopGrad[G_j(X^{(2)})]]}{\sqrt{\E[F_j^2(X^{(1)})]}\sqrt{\E[\StopGrad[G_j^2(X^{(2)})]]}}
    \end{align*}
    Then it is straghtforward to verify that \(\sum_{i\in[2]}\dbrack{\nabla_{i,j},w_i}= 0 \) for any \(j\in[2]\), which implies that \(|\dbrack{\nabla_{j',j},w_{j'}}| = |\dbrack{\nabla_{3-j',j},w_{3-j'}}|\). So in order to obtain an upper bound for \(|\dbrack{\nabla_{w_j}L(W,E), w_j}| = |\sum_{j'\in[2]}\dbrack{\nabla_{j,j'},w_j}|\), we only need to upper bound \(|\dbrack{\nabla_{j,j'},w_{3-j'}}|\), each of which can be calculated as (ignoring all time superscript \(^{(t)}\))
    \begin{align*}
        |\dbrack{\nabla_{3-j,j},w_{3-j}}| & = \frac{\E\left[\sum_{p\in[P]\cap \P}E_{j,3-j}\sigma(\dbrack{w_{3-j},X_p})\cdot [G(X^{(2)})]_j\right]\cdot\E[F_j^2(X^{(1)})] }{(\E[F_j^2(X^{(1)})])^{3/2}\sqrt{\E[G_j^2(X^{(2)})]}} \\
        &\quad - \frac{\E\left[\sum_{p\in[P]\cap \P}E_{j,3-j}\sigma(\dbrack{w_{3-j},X_p})\cdot F_j(X^{(1)})\right]\cdot\E[[F_j(X^{(1)})\cdot [G(X^{(2)})]_j]}{(\E[F_j^2(X^{(1)})])^{3/2}\sqrt{\E[G_j^2(X^{(2)})]}} 
    \end{align*}
    Now we compute
    \begin{align*}
        \E\left[\sum_{p\in[P]\cap \P}E_{j,3-j}\sigma(\dbrack{w_{3-j},X_p}) [G(X^{(2)})]_j\right] &= \E\left[\sum_{p\in[P]\cap \P}E_{j,3-j}\sigma(\dbrack{w_{3-j},X_p})\sum_{p\in[P]\setminus \P}\sigma(\dbrack{w_j,X_p})\right] \\
        & = \sum_{\ell\in[2]}E_{j,3-j}C_0\alpha_{\ell}^6B_{3-j,\ell}^3B_{j,\ell}^3 
    \end{align*}
    and
    \begin{align*}
        & \quad \E\left[\sum_{p\in[P]\cap \P}E_{j,3-j}\sigma(\dbrack{w_{3-j},X_p})\cdot F_j(X^{(1)})\right] \\
        &= \E\left[\sum_{p\in[P]\cap \P}E_{j,3-j}\sigma(\dbrack{w_{3-j},X_p})\cdot\sum_{p\in[P]\cap \P}\left(\sigma(\dbrack{w_j,X_p}) + E_{j,3-j}\sigma(\dbrack{w_{3-j},X_p})\right)\right] \\
        & = \sum_{\ell\in[2]}E_{j,3-j}C_1\alpha_{\ell}^6B_{3-j,\ell}^3(B_{j,\ell}^3 + E_{j,3-j}B_{3-j,\ell}^3) + C_2E_{j,3-j}\E[\dbrack{w_j,\xi_p}^3\dbrack{w_{3-j},\xi_p}^3 + E_{j,3-j}\dbrack{w_{3-j},\xi_p}^6]
    \end{align*}
    So we can further obtain the nominator in the expression of \(|\dbrack{\nabla_{3-j,j},w_{3-j}}|\) as 
    \begin{align*}
        &\quad \E\left[\sum_{p\in[P]\cap \P}E_{j,3-j}\sigma(\dbrack{w_{3-j},X_p})\cdot [G(X^{(2)})]_j\right]\cdot\E[F_j^2(X^{(1)})] \\
        &\quad - \E\left[\sum_{p\in[P]\cap \P}E_{j,3-j}\sigma(\dbrack{w_{3-j},X_p})\cdot F_j(X^{(1)})\right]\cdot\E[[F_j(X^{(1)})\cdot [G(X^{(2)})]_j] \\
        & = \left(\sum_{\ell\in[2]}E_{j,3-j}C_0\alpha_{\ell}^6B_{3-j,\ell}^3B_{j,\ell}^3\right) \cdot \left(\sum_{\ell \in [2]}C_1 \alpha_{\ell}^6(B_{j,\ell}^3 + E_{j,3-j}B_{3-j,\ell}^3)^2  + C_2\Ecal_{j,3-j}\right) \\
        &\quad - \left(\sum_{\ell\in[2]}E_{j,3-j}C_1\alpha_{\ell}^6B_{3-j,\ell}^3(B_{j,\ell}^3 + E_{j,3-j}B_{3-j,\ell}^3)\right)\cdot \left(\sum_{\ell\in[2]}C_0\alpha_{\ell}^6B_{j,\ell}^3(B_{j,\ell}^3 + E_{j,3-j}B_{3-j,\ell}^3)\right) \\
        &\quad - C_2E_{j,3-j}\E[\dbrack{w_j,\xi_p}^3\dbrack{w_{3-j},\xi_p}^3 + E_{j,3-j}\dbrack{w_{3-j},\xi_p}^6]\cdot \left(\sum_{\ell\in[2]}C_0\alpha_{\ell}^6B_{j,\ell}^3(B_{j,\ell}^3 + E_{j,3-j}B_{3-j,\ell}^3)\right) \\
        & = E_{j,3-j}\sum_{\ell\in[2]}C_0\alpha_{\ell}^6B_{3-j,\ell}^3(B_{j,\ell}^3H_{j,3-\ell} - B_{j,3-\ell}^3K_{j,3-\ell}) \\
        & \quad - C_2E_{j,3-j}\E[\dbrack{w_j,\xi_p}^3\dbrack{w_{3-j},\xi_p}^3 + E_{j,3-j}\dbrack{w_{3-j},\xi_p}^6]\cdot\left(\sum_{\ell\in[2]}C_0\alpha_{\ell}^6B_{j,\ell}^3(B_{j,\ell}^3 + E_{j,3-j}B_{3-j,\ell}^3)\right)
    \end{align*}
    Now can sum over \(j'\in[2]\) to get
    \begin{align*}
        &\quad |\dbrack{\nabla_{w_j} L(W,E), w_j}| \\
        &\leq \sum_{j\in[2]}\sum_{\ell\in[2]}C_0E_{j,3-j}\left|\Phi_j\alpha_{\ell}^6B_{3-j,\ell}^3B_{j,\ell}^3H_{j,3-\ell} \right| + \sum_{j\in[2]}\sum_{\ell\in[2]}\left|C_0E_{j,3-j}\Phi_j\alpha_{\ell}^3B_{3-j,\ell}^3 B_{j,3-\ell}^3K_{j,3-\ell}\right| \\
        & \quad + \sum_{j\in[2]}\sum_{\ell\in[2]}\left|C_2E_{j,3-j}\Phi_j\E[\dbrack{w_j,\xi_p}^3\dbrack{w_{3-j},\xi_p}^3 + E_{j,3-j}\dbrack{w_{3-j},\xi_p}^6] C_0\alpha_{\ell}^6B_{j,\ell}^3(B_{j,\ell}^3 + E_{j,3-j}B_{3-j,\ell}^3)\right|
    \end{align*}
    Next we are going to bound each term, for the first term of LHS we have 
    \begin{align*}
        \sum_{j\in[2]}\sum_{\ell\in[2]}\left|C_0E_{j,3-j}\Phi_j\alpha_{\ell}^6B_{3-j,\ell}^3B_{j,\ell}^3H_{j,3-\ell} \right| &\leq \sum_{j\in[2]}\sum_{\ell\in[2]}|E_{j,3-j}||\Lambda_{j,\ell}|\left|\frac{B_{3-j,\ell}^3}{B_{j,\ell}^2}\right| \\
        &\leq |\Lambda_{1,1}|\sum_{j\in[2]}|E_{j,3-j}||\left|\frac{B_{3-j,\ell}^3B_{j,\ell}^3\Phi_j}{B_{1,1}^5\Phi_1}\right| \\
        &\leq \widetilde{O}(\frac{d^{o(1)}}{\sqrt{d}})|\Lambda_{1,1}|\sum_{j\in[2]}|E_{j,3-j}|
    \end{align*}
    where the last inequality is because
    \begin{itemize}
        \item By \myref[a,b]{lem:phase1-variables}{Lemma}, we have \(\Phi_j^{(t)}/\Phi_1^{(t)} \leq O(\alpha_1^O(1)) \leq d^{o(1)}\) during \(t\leq T_1\).
        \item \((B_{3-j,\ell}^{(t)})^3(B_{j,\ell}^{(t)})^3 \leq \widetilde{O}(\frac{1}{\sqrt{d}})(B_{1,1}^{(t)})^5\) from \myref[b,c]{induct:phase-1}{Induction}.
    \end{itemize}
    Similarly, we can also compute
    \begin{align*}
        \sum_{j\in[2]}\sum_{\ell\in[2]}\left|C_0E_{j,3-j}\Phi_j\alpha_{\ell}^3B_{3-j,\ell}^3 B_{j,3-\ell}^3K_{j,3-\ell}\right| &\leq \sum_{j\in[2]}\sum_{\ell\in[2]}E_{j,3-j}|\Lambda_{1,1}|\left|\frac{B_{3-j,\ell}^3B_{j,3-\ell}^3K_{j,3-\ell}}{B_{1,1}^5H_{j,3-\ell}}\right| \\
        & \leq \widetilde{O}(\frac{d^{o(1)}}{d^2})|\Lambda_{1,1}|\sum_{j\in[2]}|E_{j,3-j}|
    \end{align*}
    and 
    \begin{align*}
        &\sum_{j\in[2]}\sum_{\ell\in[2]}\left|C_2E_{j,3-j}\Phi_j\E[\dbrack{w_j,\xi_p}^3\dbrack{w_{3-j},\xi_p}^3 + E_{j,3-j}\dbrack{w_{3-j},\xi_p}^6] C_0\alpha_{\ell}^6B_{j,\ell}^3(B_{j,\ell}^3 + E_{j,3-j}B_{3-j,\ell}^3)\right| \\
        \stackrel{\text{\ding{172}}}{\leq} \ &\sum_{j\in[2]}\sum_{\ell\in[2]}|E_{j,3-j}\Lambda_{j,\ell}|\left|\frac{B_{j,\ell}^3+E_{j,3-j}B_{3-j,\ell}^3}{B_{j,\ell}^2}\right| \left|\E[\dbrack{w_j,\xi_p}^3\dbrack{w_{3-j},\xi_p}^3 + E_{j,3-j}\dbrack{w_{3-j},\xi_p}^6]\right| \\
        \stackrel{\text{\ding{173}}}{\leq} \ &\sum_{j\in[2]}\sum_{\ell\in[2]}|E_{j,3-j}\Lambda_{j,\ell}|\left|\frac{B_{j,\ell}^3+E_{j,3-j}B_{3-j,\ell}^3}{B_{j,\ell}^2}\right|(O(R_{1,2}+\varrho) + O(E_{j,3-j})) \\
        \leq \ & \widetilde{O}(R_{1,2}+\varrho)|\Lambda_{1,1}|\sum_{j\in[2]}|E_{j,3-j}|
    \end{align*}
    where \ding{172} is due to \myref[c]{lem:phase1-variables}{Lemma}, \ding{173} is from the same calculation in \myref{claim:noise-phase1}{Claim} for \(\E[\dbrack{w_j,\xi_p}^3\dbrack{w_{3-j},\xi_p}^3]\) and \myref[a]{induct:phase-1}{Induction}. Now combining the above and \myref[e]{induct:phase-1}{Induction} together we have 
    \begin{align*}
        |\dbrack{\nabla_{w_j} L(W,E), w_j}| &\leq \widetilde{O}(\varrho + \frac{1}{\sqrt{d}})|\Lambda_{1,1}|\sum_{j\in[2]}|E_{j,3-j}|
    \end{align*}
    which gives the desired bound.
\end{proof}

Next we give a lemma characterizing the gradient of feature \(v_1\) in this phase.

\begin{lemma}[learning feature \(v_1\) in phase I]\label{lem:learning-v1-phase1}
    For each \(t\leq T_1\), if Induction~\ref{induct:phase-1} holds at iteration \(t\), then using notations of \eqref{eqdef:weight-grad}, we have:
    \begin{enumerate}[(a)]
        \item \(\dbrack{-\nabla_{w_1}L(W^{(t)}, E^{(t)}), v_1} = (1\pm\widetilde{O}(\frac{1}{d}))\Lambda_{1,1}^{(t)}\)
        \item \(\dbrack{-\nabla_{w_2}L(W^{(t)}, E^{(t)}), v_1} = (1\pm O(\frac{1}{\sqrt{d}}))\Lambda_{2,1}^{(t)} + \Gamma_{2,1}^{(t)} \leq  (1\pm O(\frac{1}{\sqrt{d}}))\Lambda_{2,1}^{(t)} \pm \frac{(B_{2,1}^{(t)})^2}{(B_{1,1}^{(t)})^2}E_{1,2}^{(t)}\Lambda_{1,1}^{(t)}\) 
    \end{enumerate}
\end{lemma}

\begin{proof}
    From \eqref{eqdef:weight-grad}, we write down the gradient formula for \(B_{j,1}^{(t)}\) as follows:
    \begin{align*}
        \dbrack{- \nabla_{w_j} L_{\D}(W^{(t)}, E^{(t)}), v_1}  &= \Lambda_{j,1}^{(t)} + \Gamma_{j,1}^{(t)} - \Upsilon_{j,1}^{(t)} 
    \end{align*}
    where (ignoring the superscript \(^{(t)}\) for the RHS)
    \begin{align*}
        \Lambda_{j,1}^{(t)} & = C_0\Phi_j \alpha_{1}^6B_{j,1}^5H_{j,2} \\
        \Gamma_{j,1}^{(t)} &= C_0\Phi_{3-j} E_{3-j,j} \alpha_{1}^6B_{3-j,1}^3B_{j,1}^2H_{3-j,2} \\
        \Upsilon_{j,1}^{(t)} & = C_0\alpha_{2}^6\left(\Phi_jB_{j,2}^3B_{j,1}^2K_{j,1} + \Phi_{3-j}E_{3-j,j}B_{3-j,2}^3B_{j,1}^2K_{3-j,1}\right)
    \end{align*}
    We first prove (a), and we deal with each term individually:\\
    \textbf{Comparing \(\Lambda_{1,1}^{(t)}\) and \(\Gamma_{1,1}^{(t)}\):} When \(t\leq T_{1,1}\), we have from \myref[a]{lem:phase1-variables}{Lemma} that 
    \begin{displaymath}
        \Phi_1^{(t)}H_{1,2}^{(t)} = \frac{1}{C_2\Ecal_1^{(t)}}(1\pm\frac{1}{\polylog(d)}) = \frac{1}{C_2\Ecal_2^{(t)}}(1\pm\frac{1}{\polylog(d)}) = \Phi_2^{(t)}H_{2,2}^{(t)}(1\pm\frac{1}{\polylog(d)})
    \end{displaymath}
    Further, by \myref[b,c,d]{induct:phase-1}{Induction} and our definition of stage 1, we know \(E_{1,2}^{(t)} \leq \widetilde{O}(\frac{1}{d}) \). Now from \myref[b]{induct:phase-1}{Induction} that \(B_{2,1}^{(t)}\leq \widetilde{O}(\frac{1}{\sqrt{d}})\), together we have
    \begin{align*}
        \Gamma_{1,1}^{(t)} = C_0\alpha_1^6E_{2,1}^{(t)} \Phi_2^{(t)} H_{2,2}^{(t)}(B_{2,1}^{(t)})^3(B_{1,1}^{(t)})^2 \leq \widetilde{O}(\frac{1}{d})C_0\alpha_1^6\Phi_1^{(t)}H_{1,2}^{(t)}(B_{1,1}^{(t)})^5 = \widetilde{O}(\frac{\Lambda_{1,1}^{(t)}}{d})
    \end{align*}
    When \(t\in[T_{1,1}, T_1]\), by \myref[b]{lem:phase1-variables}{Lemma}
    we have 
    \begin{align*}
        \Phi_1^{(t)}H_{1,2}^{(t)} \geq \Omega(\frac{C_2}{(C_1\alpha_1^6 (B_{1,1}^{(t)})^6 + O(C_2))^2}) \geq \omega(\frac{1}{d^{0.1}}),\quad \text{and} \quad E_{2,1}^{(t)}\Phi_2^{(t)}H_{2,2}^{(t)} \leq \widetilde{O}(\frac{1}{d})
    \end{align*} 
    Now from our definition of stage 2, it holds that \(B_{1,1}^{(t)} \geq \Omega(\frac{1}{\alpha_1})\) while \(B_{2,1}^{(t)} \leq \widetilde{O}(\frac{1}{\sqrt{d}})\) by \myref[b]{induct:phase-1}{Induction}, which gives 
    \begin{align*}
        \Gamma_{1,1}^{(t)} = C_0\alpha_1^6E_{2,1}^{(t)} \Phi_2^{(t)} H_{2,2}^{(t)}(B_{2,1}^{(t)})^3(B_{1,1}^{(t)})^2  \leq \widetilde{O}(\frac{1}{d})C_0\alpha_1^6\Phi_1^{(t)}H_{1,2}^{(t)}(B_{1,1}^{(t)})^5 = \widetilde{O}(\frac{\Lambda_{1,1}^{(t)}}{d})
    \end{align*}
    \textbf{Comparing \(\Lambda_{1,1}^{(t)}\) and \(\Upsilon_{1,1}^{(t)}\):} Now consider \(\Upsilon_{1,1}^{(t)}\), by \myref{lem:phase1-variables}{Lemma}, we can follow the same analysis as above to get 
    \begin{align*}
        \Phi_j^{(t)} K_{j,\ell}^{(t)}\leq  \widetilde{O}(\frac{\alpha_1^{O(1)}}{d^{3/2}})\Phi_1^{(t)}H_{1,2}^{(t)} \tag*{for any \((j, \ell) \in [2]\times [2]\)}
    \end{align*}
    Combined with \(E_{2,1}^{(t)}\leq o(1)\), we can derive
    \begin{align*}
        \Upsilon_{1,1}^{(t)} &= C_0\alpha_2^6\left(\Phi_1^{(t)}K_{1,1}^{(t)}(B_{1,2}^{(t)})^3(B_{1,1}^{(t)})^2 + E_{1,2}^{(t)}\Phi_2^{(t)}K_{2,1}^{(t)}(B_{2,2}^{(t)})^3(B_{1,1}^{(t)})^2\right)\\
        &\leq \widetilde{O}(\frac{\alpha_1^{O(1)}\alpha_2^6}{d^{3/2}})C_0\alpha_1^6\Phi_1^{(t)}H_{1,2}^{(t)}(B_{1,1}^{(t)})^5 \\
        &= \widetilde{O}(\frac{\Lambda_{1,1}^{(t)}}{d^{3/2-o(1)}}) \tag{since \(C_1 = \widetilde{O}(1)\) and \(\alpha_1,\alpha_2 = d^{o(1)}\)}
    \end{align*}
        \item \textbf{Comparing \(\Lambda_{2,1}^{(t)}\) and \(\Upsilon_{2,1}^{(t)}\):} Till now (a) is proved, we can deal with (b) by only comparing \(\Lambda_{2,1}^{(t)}\) with \(\Upsilon_{2,1}^{(t)}\). Similar to the above arguments, we have by \myref[b]{induct:phase-1}{Induction} we know \(K_{j,1}^{(t)} = \widetilde{O}(\frac{C_1\alpha_1^6}{d^{3/2}}), \forall j \in [2]\), and thus
    \begin{align*}
        \Phi_j^{(t)}K_{j,\ell}^{(t)} \leq \widetilde{O}(\frac{\alpha_1^{6}}{d^{3/2}}) \Phi_2^{(t)}H_{2,2}^{(t)} \tag*{for any \((j,\ell)\in[2]\times[2]\)} 
    \end{align*}
    By \myref[e]{induct:phase-1}{Induction} we know \(E_{1,2}^{(t)} \leq \widetilde{O}(\varrho+\frac{1}{\sqrt{d}})\). Also, note that from \myref[b]{induct:phase-1}{Induction} we have \(\widetilde{O}((B_{1,2}^{(t)})^3/d)\leq \widetilde{O}((B_{2,1}^{(t)})^5) \), and thus
    \begin{align*}
        E_{1,2}^{(t)}\Phi_1^{(t)}K_{1,1}^{(t)}(B_{1,2}^{(t)})^3(B_{2,1}^{(t)})^2 \leq \widetilde{O}(\varrho+\frac{1}{\sqrt{d}})\widetilde{O}(\frac{\alpha_1^6}{d^{5/2}})\Phi_2^{(t)}H_{2,2}^{(t)} \widetilde{O}(B_{1,2}^{(t)})^3 \leq O(\frac{1}{d^{3/2}})\Phi_2^{(t)}H_{2,2}^{(t)}(B_{2,1}^{(t)})^5
    \end{align*}
    So together we have
    \begin{align*}
        |\Upsilon_{2,1}^{(t)}| &= |C_0\alpha_2^6\left(\Phi_2^{(t)}K_{2,1}^{(t)}(B_{2,2}^{(t)})^3(B_{2,1}^{(t)})^2 + E_{2,1}^{(t)}\Phi_1^{(t)}K_{1,1}^{(t)}(B_{1,2}^{(t)})^3(B_{2,1}^{(t)})^2\right)| \\
        &\leq O(\frac{1}{d^{3/2}})C_0\alpha_1^6\Phi_2^{(t)}H_{2,2}^{(t)}|(B_{2,1}^{(t)})^5| \\
        & = O(\frac{1}{d^{3/2}})|\Lambda_{2,1}^{(t)}|
    \end{align*}
    \textbf{Comparing \(\Gamma_{2,1}^{(t)}\) with \(\Lambda_{1,1}^{(t)}\):} It suffices to notice that
    \begin{align*}
        |\Gamma_{2,1}^{(t)}| \leq |E_{1,2}^{(t)}|C_0\alpha_1^6\Phi_1^{(t)}H_{1,2}^{(t)}|B_{1,1}^{(t)}|^3(B_{2,1}^{(t)})^2 = \frac{(B_{2,1}^{(t)})^2}{(B_{1,1}^{(t)})^2}|E_{1,2}^{(t)}||\Lambda_{1,1}^{(t)}|
    \end{align*}
    Combining the bounds for \(\Lambda_{2,1}^{(t)}\) and \(\Gamma_{2,1}^{(t)}\), we obtain the proof of (b).
\end{proof}

Then we can also calculate the gradients of feature \(v_2\) in this phase.

\begin{lemma}[learning feature \(v_2\) in phase I]\label{lem:learning-v2-phase1}
    For each \(t\leq T_1\), if Induction~\ref{induct:phase-1} holds at iteration \(t\), then using notations of \eqref{eqdef:weight-grad}, we have for each \(j\in[2]\):
    \begin{align}\label{eqdef:learning-v2-phase1}
        \dbrack{-\nabla_{w_j}L(W^{(t)}, E^{(t)}), v_2} = \left(1 \pm \widetilde{O}(\alpha_1^{6})(E_{3-j,j}^{(t)} + (B_{j,1}^{(t)})^3)\right)\Lambda_{j,2}^{(t)} 
    \end{align}
\end{lemma}

\begin{proof}
    Again as in the proof of \myref{lem:learning-v1-phase1}{Lemma}, we expand the notations: (ignoring the superscript \(^{(t)}\) for the RHS)
    \begin{align*}
        \Lambda_{j,2}^{(t)} & = C_0\alpha_{2}^6\Phi_j H_{j,1}B_{j,2}^5 \\
        \Gamma_{j,2}^{(t)} &= C_0\alpha_{2}^6\Phi_{j} E_{3-j,j} B_{3-j,2}^3B_{j,2}^2H_{3-j,1} \\
        \Upsilon_{j,2}^{(t)} & = C_0\alpha_{1}^6\left(\Phi_jB_{j,1}^3B_{j,2}^2K_{j,2} + \Phi_{3-j}E_{3-j,j}B_{3-j,1}^3B_{j,2}^2K_{3-j,2}\right)
    \end{align*}
    We first compare \(\Lambda_{j,2}^{(t)}\) and \(\Gamma_{j,2}^{(t)}\) as follows:  \myref{lem:phase1-variables}{Lemma} we have
    \begin{itemize}
        \item \(B_{3-j,2}^{(t)} \leq \widetilde{O}(B_{j,2}^{(t)})\) by \myref[b]{induct:phase-1}{Induction};
        \item From \myref[a,b]{lem:phase1-variables}{Lemma} we can have \(\Phi_{3-j}^{(t)} \leq \widetilde{O}(\alpha_1^{O(1)})\Phi_{j}^{(t)}, \forall j\in[2]\).
    \end{itemize} 
    Together they imply:
    \begin{align}
        C_0\alpha_2^6E_{3-j,j}^{(t)} (B_{3-j,2}^{(t)})^3(B_{j,2}^{(t)})^2\Phi_{3-j}^{(t)}H_{3-j,1}^{(t)} &\leq \widetilde{O} (\alpha_1^{O(1)}E_{3-j,j}^{(t)})C_0\alpha_{2}^6\Phi_j^{(t)} H_{j,2}^{(t)}(B_{j,2}^{(t)})^5 \nonumber\\
        & = \widetilde{O} (\alpha_1^{O(1)}E_{j,3-j}^{(t)})\Lambda_{j,2}^{(t)}\label{eqdef:lem-learning-v2-phase1-1}
    \end{align}
    Now we turn to compare \(\Lambda_{j,2}^{(t)}\) with \(\Upsilon_{j,2}^{(t)}\). We split \(\Upsilon_{j,2}^{(t)}\) into two terms \(\Upsilon_{j,2,1}^{(t)}, \Upsilon_{j,2,2}^{(t)}\)
    \begin{align*}
        \Upsilon_{j,2,1}^{(t)} = C_0\alpha_{1}^6\Phi_j^{(t)}(B_{j,1}^{(t)})^3(B_{j,2}^{(t)})^2K_{j,2}^{(t)},\quad \Upsilon_{j,2,2}^{(t)} = C_0\alpha_{1}^6\Phi_{3-j}^{(t)}E_{3-j,j}^{(t)}(B_{3-j,1}^{(t)})^3(B_{j,2}^{(t)})^2K_{3-j,2}^{(t)}
    \end{align*}
    For \(\Upsilon_{j,2,1}^{(t)} \), we can calculate
    \begin{align}
        \Upsilon_{j,2,1}^{(t)} &= C_0\alpha_{1}^6\Phi_j^{(t)}(B_{j,1}^{(t)})^3(B_{j,2}^{(t)})^2K_{j,2}^{(t)}\nonumber\\
        & \leq \widetilde{O}(\frac{C_1\alpha_2^6}{d^{3/2}})(B_{j,1}^{(t)})^3\cdot C_0\alpha_1^6\Phi_j^{(t)}H_{j,1}^{(t)}(B_{j,2}^{(t)})^2 \tag{\(K_{j,\ell}^{(t)}\leq \widetilde{O}(\frac{C_1\alpha_{\ell}^6}{d^{3/2}})\) from \myref[d]{lem:phase1-variables}{Lemma}}\nonumber\\
        &\leq \widetilde{O}(\alpha_1^6(B_{j,1}^{(t)})^3)C_0\alpha_2^6\Phi_j^{(t)}H_{j,1}^{(t)}(B_{j,2}^{(t)})^5 \tag{\(\widetilde{O}(\frac{C_1}{d^{3/2}})\leq \widetilde{O}((B_{j,2}^{(t)})^3)\) from \myref[b]{induct:phase-1}{Induction}}\nonumber\\
        & = \widetilde{O}(\alpha_1^6(B_{j,1}^{(t)})^3)\Lambda_{j,2}^{(t)} \label{eqdef:lem-learning-v2-phase1-2}
    \end{align}
    And for \(\Upsilon_{j,2,1}^{(t)} \), we use \myref[b]{induct:phase-1}{Induction} and \myref[d]{lem:phase1-variables}{Lemma} again to get
    \begin{align*}
        (B_{3-j,1}^{(t)})^3(B_{3-j,2}^{(t)})^2K_{3-j,2}^{(t)} \leq \widetilde{O}(C_1\alpha_2^6(B_{j,2}^{(t)})^5)
    \end{align*}
    and thus combined with \(\Phi_{3-j}^{(t)}\leq \widetilde{O}(\alpha_1^{6})\Phi_{j}^{(t)}, \forall j\in[2]\) from \myref[a,b]{lem:phase1-variables}{Lemma}, we can derive
    \begin{align}
        \Upsilon_{j,2,2}^{(t)} &= C_0\alpha_{1}^6\Phi_{3-j}^{(t)}E_{j,3-j}^{(t)}(B_{3-j,1}^{(t)})^3(B_{j,2}^{(t)})^2K_{3-j,2}^{(t)}\nonumber\\
        &\leq \widetilde{O} (\alpha_1^{6}E_{3-j,j}^{(t)})C_0\alpha_2^6\Phi_j^{(t)} H_{j,1}^{(t)}(B_{j,2}^{(t)})^5 \nonumber\\
        & = \widetilde{O} (\alpha_1^{6}E_{3-j,j}^{(t)})\Lambda_{j,2}^{(t)} \label{eqdef:lem-learning-v2-phase1-3}
    \end{align}
    Now combine the results of \eqref{eqdef:lem-learning-v2-phase1-1}, \eqref{eqdef:lem-learning-v2-phase1-2} and \eqref{eqdef:lem-learning-v2-phase1-3} finishes the proof of \eqref{eqdef:learning-v2-phase1}.
\end{proof}

\begin{lemma}[learning prediction head \(E_{1,2}, E_{2,1}\) in phase I]\label{lem:learning-pred-head}
    If \myref{induct:phase-1}{Induction} holds at iteration \(t \leq T_1\), then we have 
    \begin{enumerate}[(a)]
        \item \(-\nabla_{E_{1,2}}L(W^{(t)},E^{(t)}) = O(\Lambda_{1,1}^{(t)}B_{1,1}^{(t)})\left(  -O(E_{1,2}^{(t)}) +\widetilde{O}(\frac{(B_{1,2}^{(t)})^3}{(B_{1,1}^{(t)})^3}) + O(R_{1,2}^{(t)}) \right)\);
        \item \(-\nabla_{E_{2,1}}L(W^{(t)},E^{(t)}) = \widetilde{O}(\frac{(B_{1,2}^{(t)})^3}{(B_{1,1}^{(t)})^2})\Lambda_{1,1}^{(t)}  + \sum_{\ell\in[2]}C_2\Lambda_{2,\ell}^{(t)}B_{2,\ell}^{(t)}\left(  -O(E_{2,1}^{(t)}) + O(R_{1,2}^{(t)}) \right)\)
    \end{enumerate}
    
\end{lemma}

\begin{proof}
    We first write down the gradient for \(E_{j,3-j}^{(t)}\): (ignoring the time superscript \(^{(t)}\))
    \begin{align*}
        -\nabla_{E_{j,3-j}}L(W,E) &= \sum_{\ell \in [2]}C_0\Phi_j\alpha_{\ell}^6B_{j,\ell}^3(B_{3-j,\ell}^3H_{j,3-\ell} - B_{3-j,3-\ell}^3K_{j,3-\ell}) - \sum_{\ell \in [2]}\Sigma_{j,\ell}\nabla_{E_{j,3-j}}\Ecal_{j,3-j}
    \end{align*}
    where \(\nabla_{E_{j,3-j}}\Ecal_{j,3-j} = \E\left[2\dbrack{w_j,\xi_p}^3\dbrack{w_{3-j},\xi_p}^3 + 2E_{j,3-j}\dbrack{w_{3-j},\xi_p}^6\right]\). Thus we have 
    \begin{align*}
        \nabla_{E_{j,3-j}}\Ecal_{j,3-j}^{(t)} = O(1)E_{j,3-j}^{(t)} + O(R_{1,2}^{(t)})
    \end{align*}
    and by \myref{claim:Ecal-phase1}{Claim} and \myref[a,b]{lem:phase1-variables}{Lemma}
    \begin{align*}
        \Sigma_{j,\ell}^{(t)} = O(\Lambda_{1,1}^{(t)}B_{1,1}^{(t)})\frac{(B_{j,\ell}^{(t)})^6 + E_{j,3-j}^{(t)}(B_{3-j,\ell}^{(t)})^3(B_{j,\ell}^{(t)})^3}{(B_{1,1}^{(t)})^6}\frac{\Phi_j^{(t)}}{\Phi_1^{(t)}} \leq  O(\Lambda_{1,1}^{(t)}B_{1,1}^{(t)})
    \end{align*}
    Now let us look at \(\nabla_{E_{1,2}}L(W^{(t)},E^{(t)})\), first we consider the term
    \begin{align*}
        \sum_{\ell \in [2]}C_0\Phi_1^{(t)}\alpha_{\ell}^6(B_{1,\ell}^{(t)})^3((B_{2,\ell}^{(t)})^3H_{1,3-\ell}^{(t)} - (B_{2,3-\ell}^{(t)})^3K_{1,3-\ell}^{(t)})
    \end{align*}
    Using \myref{lem:phase1-variables}{Lemma} and \myref[b,c]{induct:phase-1}{Induction}, we know
    \begin{itemize}
        \item \(H_{1,1}^{(t)} \leq \widetilde{O}(H_{1,2}^{(t)})\) at \(t \leq T_{1,1}\) and \(H_{1,1}^{(t)} \leq \widetilde{O}(\alpha_1^6 H_{1,2}^{(t)})\) for \(t \in [T_{1,1},T_1]\);
        \item \(B_{2,1}^{(t)}, B_{1,2}^{(t)}, B_{2,2}^{(t)} \leq \widetilde{O}(B_{2,1}^{(t)}) \leq \widetilde{O}(B_{1,1}^{(t)})\);
        \item \(K_{1,3-\ell}^{(t)} \leq \widetilde{O}(\alpha_1^6/d^{3/2})\).
    \end{itemize}
    It can be computed that 
    \begin{align*}
        C_0\Phi_1^{(t)}\alpha_{2}^6(B_{1,2}^{(t)})^3(B_{2,2}^{(t)})^3H_{1,1}^{(t)} &\leq  \widetilde{O}(1)\left(\frac{B_{2,1}^{(t)}}{B_{1,1}^{(t)}}\right)^3C_0\Phi_1^{(t)}\alpha_{1}^3(B_{1,1}^{(t)})^6H_{1,2}^{(t)} \\
        \sum_{\ell\in[2]}\left|C_0\Phi_1^{(t)}\alpha_{\ell}^6(B_{1,\ell}^{(t)})^3(B_{2,\ell}^{(t)})^3K_{1,3-\ell}^{(t)}\right| &\leq \widetilde{O}(\frac{\alpha_1^6}{d^{3/2}})\frac{(B_{2,1}^{(t)})^3}{(B_{1,1}^{(t)})^3}C_0\Phi_1^{(t)}\alpha_{1}^6(B_{1,1}^{(t)})^6H_{1,2}^{(t)}
    \end{align*}
    Now we turn to \(\nabla_{E_{2,1}}L(W^{(t)},E^{(t)})\), similarly we have 
    \begin{align*}
        C_0\Phi_2^{(t)}\alpha_{1}^6(B_{2,1}^{(t)})^3(B_{1,1}^{(t)})^3H_{2,2}^{(t)} &\leq  \widetilde{O}(1)\left(\frac{B_{2,1}^{(t)}}{B_{1,1}^{(t)}}\right)^3C_0\Phi_1^{(t)}\alpha_{1}^6(B_{1,1}^{(t)})^6H_{1,2}^{(t)} 
    \end{align*}
    and since \(H_{2,1}^{(t)} \leq O(C_2) = O(H_{1,2}^{(t)})\) by \myref[c]{lem:phase1-variables}{Lemma}, we can go through the same arguments again to obtain
    \begin{align*}
        \left|C_0\Phi_2^{(t)}\alpha_{2}^6(B_{1,2}^{(t)})^3(B_{2,2}^{(t)})^3H_{2,1}^{(t)}\right| &\leq \widetilde{O}(1)\left(\frac{B_{1,2}^{(t)}}{B_{1,1}^{(t)}}\right)^3C_0\Phi_1^{(t)}\alpha_{1}^6(B_{1,1}^{(t)})^6 H_{1,2}^{(t)} \\
        \left|C_0\Phi_2^{(t)}\alpha_{2}^6(B_{1,2}^{(t)})^3 (B_{2,1}^{(t)})^3K_{2,1}^{(t)}\right| &\leq \widetilde{O}(\frac{\alpha_1^6}{d^{3/2}})\left(\frac{B_{1,2}^{(t)}}{B_{1,1}^{(t)}}\right)^3 C_0\Phi_1^{(t)}\alpha_{1}^6(B_{1,1}^{(t)})^6 H_{1,2}^{(t)}
    \end{align*}
    Now the proof is complete.
\end{proof}

Also, we will need the following lemma controlling gradient bounds for the noise term.

\begin{lemma}[update of \(R_{1,2}^{(t)}\) in phase I] \label{lem:learning-R12-phase1}
    Suppose \myref{induct:phase-1}{Induction} holds at iteration \(t \leq T_1\), then we have
    \begin{enumerate}[(a)]
        \item \(|\dbrack{-\nabla_{w_1}L(W^{(t)},E^{(t)}), \Pi_{V^{\perp}}w_{2}^{(t)}}| \leq \widetilde{O}(\frac{1}{\sqrt{d}} + \varrho)\Lambda_{1,1}^{(t)}B_{1,1}^{(t)} \)
        \item \(|\dbrack{-\nabla_{w_2}L(W^{(t)},E^{(t)}), \Pi_{V^{\perp}}w_{1}^{(t)}}| \leq \widetilde{O}(\frac{1}{\sqrt{d}} + \varrho)\Lambda_{1,1}^{(t)}B_{1,1}^{(t)} \)
    \end{enumerate}
\end{lemma}

\begin{proof}
    \textbf{Proof of (a):} Firstly, by \myref[a]{claim:noise-phase1}{Claim}, we can directly write
    \begin{align}
        &\quad \dbrack{\nabla_{w_1}L(W^{(t)},E^{(t)}), \Pi_{V^{\perp}}w_2^{(t)}} = -\sum_{j,\ell}\Sigma_{j,\ell}^{(t)}\dbrack{\nabla_{w_1}\Ecal_{j,3-j}^{(t)},w_2^{(t)}} \nonumber \\
        & = - \Lambda_{1,1}^{(t)}B_{1,1}^{(t)}\sum_{(j,\ell)\in[2]^2}\frac{(B_{j,\ell}^{(t)})^6 + E_{j,3-j}^{(t)}(B_{3-j,\ell}^{(t)})^3(B_{j,\ell}^{(t)})^3}{(B_{1,1}^{(t)})^6} \frac{\Phi_j^{(t)}}{\Phi_1^{(t)}} \dbrack{\nabla_{w_1}\Ecal_{j,3-j}^{(t)},w_1^{(t)}} \label{eqdef:noise-phase1-00}
    \end{align}
    Now we discuss each summand respectively: for \((j,\ell) = (1,1)\), we have 
    \begin{align}\label{eqdef:noise-phase1-0}
        \frac{(B_{j,\ell}^{(t)})^6 + E_{j,3-j}^{(t)}(B_{3-j,\ell}^{(t)})^3(B_{j,\ell}^{(t)})^3}{(B_{1,1}^{(t)})^6} = 1 + E_{1,2}^{(t)}\frac{(B_{2,1}^{(t)})^3}{(B_{1,1}^{(t)})^3} = 1 + o(\frac{1}{d^{3/2}(B_{1,1}^{(t)})^3})
    \end{align}
    where the last one is due to \myref[d]{induct:phase-1}{Induction}.
    And for \(\ell = 2\), we can see from \myref[b and d]{induct:phase-1}{Induction}, that \(\max_{(j,\ell)\neq (1,1)}|B_{j,\ell}^{(t)}| = \widetilde{O}(\frac{1}{\sqrt{d}})\) and \(E_{j,3-j}^{(t)}\leq o(1)\) to give
    \begin{align*}
        \frac{(B_{j,2}^{(t)})^6 + E_{j,3-j}^{(t)}(B_{3-j,2}^{(t)})^3(B_{j,2}^{(t)})^3}{(B_{1,1}^{(t)})^6}\frac{\Phi_j^{(t)}}{\Phi_1^{(t)}} &\leq \widetilde{O}(\frac{1}{d^3})\frac{1}{(B_{1,1}^{(t)})^6}\frac{\Phi_j^{(t)}}{\Phi_1^{(t)}} 
    \end{align*}
    On one hand, when \(t \leq T_{1,1}\), we have \(\alpha_{\ell}B_{j,\ell}^{(t)}\leq O(1)\) for all \((j,\ell)\in[2]^2\), so \myref[a]{lem:phase1-variables}{Lemma} applies for both \(\Phi_j^{(t)}\) and results in \(\Phi_2^{(t)}/\Phi_1^{(t)} \leq O(1)\). We can also apply \myref[c]{induct:phase-1}{Induction} to have \(B_{j,2}^{(t)}/B_{1,1}^{(t)}\leq \widetilde{O}(1)\). On the other hand, when \(t \in [T_{1,1}, T_1]\), we have by \myref[b]{induct:phase-1}{Induction} and \myref[a,b]{lem:phase1-variables}{Lemma} that \(\Phi_2^{(t)}/\Phi_1^{(t)}\leq \widetilde{O}(\alpha_1^{O(1)}) = d^{o(1)}\), but now \(B_{1,1}^{(t)} = d^{-o(1)} \gg \widetilde{O}(d^{-1/2})\), therefore 
    \begin{align*}
        \widetilde{O}(\frac{1}{d^3})\frac{1}{(B_{1,1}^{(t)})^6}\frac{\Phi_2^{(t)}}{\Phi_1^{(t)}} \leq \widetilde{O}(\frac{1}{d^{3/2}})\frac{1}{(B_{1,1}^{(t)})^3}
    \end{align*}
    So together, they imply
    \begin{align}\label{eqdef:noise-phase1-1}
        \frac{(B_{j,2}^{(t)})^6 + E_{j,3-j}^{(t)}(B_{3-j,2}^{(t)})^3(B_{j,2}^{(t)})^3}{(B_{1,1}^{(t)})^6}\frac{\Phi_j^{(t)}}{\Phi_1^{(t)}}  \leq \widetilde{O}(\frac{1}{d^{3/2}(B_{1,1}^{(t)})^3})
    \end{align}
    and similarly, we have 
    \begin{align}\label{eqdef:noise-phase1-2}
        \frac{(B_{2,1}^{(t)})^6 + E_{2,1}^{(t)}(B_{1,1}^{(t)})^3(B_{2,1}^{(t)})^3}{(B_{1,1}^{(t)})^6} \frac{\Phi_2^{(t)}}{\Phi_1^{(t)}} \leq \widetilde{O}(\frac{1}{d^{3/2}(B_{1,1}^{(t)})^3})
    \end{align}
    Next we turn to \(\dbrack{\nabla_{w_1}\Ecal_{j,3-j}^{(t)},w_2^{(t)}} \). When \(j=1\), we can apply \myref[d]{claim:noise-phase1}{Claim} to get 
    \begin{align}\label{eqdef:noise-phase1-3}
        \dbrack{\nabla_{w_1}\Ecal_{1,2}^{(t)},w_2^{(t)}} = O(R_{1,2}^{(t)} + \varrho) + O(E_{1,2}^{(t)}) = O(\varrho + \frac{1}{\sqrt{d}}) + O(E_{1,2}^{(t)}) \leq O(\varrho + \frac{1}{\sqrt{d}})
    \end{align}
    and when \(j=2\), we can apply \myref[e]{claim:noise-phase1}{Claim} to get 
    \begin{align}\label{eqdef:noise-phase1-4}
        \dbrack{\nabla_{w_1}\Ecal_{2,1}^{(t)},w_2^{(t)}} = -(E_{2,1}^{(t)})^2 O(R_{1,2}^{(t)} + \varrho) + O(E_{2,1}) = \widetilde{O}(\frac{1}{d^2})(\varrho + \frac{1}{\sqrt{d}}) + O(\frac{1}{d})
    \end{align}
    Combining \eqref{eqdef:noise-phase1-00}, \eqref{eqdef:noise-phase1-0}, \eqref{eqdef:noise-phase1-1}, \eqref{eqdef:noise-phase1-2}, \eqref{eqdef:noise-phase1-3}, and \eqref{eqdef:noise-phase1-4} completes the proof of (a).\\
    \newline
    \textbf{Proof of (b):} The \(\Sigma_{j,\ell}^{(t)}\) part is the same as in the proof of (a), so we only deal with \(\dbrack{\nabla_{w_2}\Ecal_{1,2}^{(t)},w_1^{(t)}}\) and \(\dbrack{\nabla_{w_2}\Ecal_{2,1}^{(t)},w_1^{(t)}}\) here. For \(\dbrack{\nabla_{w_2}\Ecal_{2,1}^{(t)},w_1^{(t)}}\), we apply \myref[d]{claim:noise-phase1}{Claim} to get 
    \begin{align}\label{eqdef:noise-phase1-5}
        \dbrack{\nabla_{w_2}\Ecal_{2,1}^{(t)},w_1^{(t)}} = O(R_{1,2}^{(t)} + \varrho) + O(1)E_{1,2}^{(t)}
    \end{align}
    and for \(\dbrack{\nabla_{w_2}\Ecal_{1,2}^{(t)},w_1^{(t)}}\), we have 
    \begin{align}\label{eqdef:noise-phase1-6}
        \dbrack{\nabla_{w_2}\Ecal_{1,2}^{(t)},w_1^{(t)}} = O(R_{1,2}^{(t)} + \varrho)(E_{2,1}^{(t)})^2 + O(1)E_{2,1}^{(t)}
    \end{align}
    Inserting \eqref{eqdef:noise-phase1-0}, \eqref{eqdef:noise-phase1-1}, \eqref{eqdef:noise-phase1-2} and \eqref{eqdef:noise-phase1-5}, \eqref{eqdef:noise-phase1-6} into the expression of \(\dbrack{-\nabla_{w_2}L(W^{(t)},E^{(t)}), \Pi_{V^{\perp}}w_{1}^{(t)}}\) finishes the proof of (b).
\end{proof}

\subsection{At the End of Phase I}
\begin{lemma}[Phase I]\label{lem:phase-1}
    Suppose \(\eta \leq \frac{1}{\poly(d)}\) is sufficiently small, then \myref{induct:phase-1}{Induction} holds for at least all \(t \leq T_1 = O(\frac{d^2}{\eta})\), and at iteration \(t = T_1\), we have
    \begin{enumerate}[(a)]
        \item \(B_{1,1}^{(T_1)} = \Omega(1)\);
        \item \(\|w_j^{(T_1)}\|_2 = 1 \pm \widetilde{O}(\varrho+\frac{1}{\sqrt{d}})\);
        \item \(B_{2,1}^{(T_1)} = \widetilde{\Theta}(\frac{1}{\sqrt{d}})\) and \(B_{j,2}^{(T_1)} = B_{j,2}^{(0)}(1\pm o(1))\) for \(j\in[2]\);
        \item \(E_{2,1}^{(T_1)}= \widetilde{O}(\frac{\eta_E/\eta}{d})\) and \(E_{1,2}^{(T_1)}\leq \widetilde{O}(\varrho + \frac{1}{\sqrt{d}})\);
        \item \(R_{1,1}^{(T_1)}, R_2^{(T_1)} = \Theta(1)\) and \(R_{1,2}^{(T_1)} = \widetilde{O}(\varrho + \frac{1}{\sqrt{d}})\).
    \end{enumerate}
\end{lemma}

\begin{proof}
    We begin by first prove the existence of \(T_1 := \min\{t: B_{1,1}^{(t)} \geq 0.01\} = O(\frac{d^2}{\eta})\) if \myref{induct:phase-1}{Induction} holds whenever \(B_{1,1}^{(t)} \leq 0.01\), then we will turn back to prove \myref{induct:phase-1}{Induction} holds throughout \(t \leq T_1\).  We split the analysis into two stages:\\
    \textbf{Proof of \(T_1 \leq O(\frac{d^2}{\eta})\):}  By \myref[a]{lem:learning-v1-phase1}{Lemma} we can write down the update of \(B_{1,1}^{(t)}\) as 
    \begin{align}\label{eqdef:thm-phase1-1}
        B_{1,1}^{(t+1)} = B_{1,1}^{(t)} + \eta(1 \pm \widetilde{O}(\frac{1}{d}))\Lambda_{1,1}^{(t)} = B_{1,1}^{(t)} + \eta(1 \pm \widetilde{O}(\frac{1}{d})) \Phi_1^{(t)}C_0\alpha_1^6H_{1,2}^{(t)} (B_{1,1}^{(t)})^5
    \end{align}
    When \(\alpha_1B_{1,1}^{(t)} \leq O(1)\), by \myref[a,c]{lem:phase1-variables}{Lemma} we have \(\Phi_1^{(t)} = \Theta(\frac{1}{C_2^2})\) and \(H_{1,2}^{(t)} = \Omega(C_2)\), this means we can lower bound the update as
    \begin{align*}
        B_{1,1}^{(t+1)} \geq B_{1,1}^{(t)} + \Omega(\frac{\eta C_0\alpha_1^6}{C_2})(B_{1,1}^{(t)})^5
    \end{align*} 
    since \(\frac{C_0\alpha_1^6}{C_2}\) is a constant, we know there exist some \(t' \geq 0\) such that \(B_{1,1}^{(t')}\geq \Omega(\frac{1}{\alpha_1})\). Also recall that \(T_{1,1}:= \min\{t: B_{1,1}^{(t)}\geq \Omega(\frac{1}{\alpha_1})\}\). So by \myref{lem:TPM}{Lemma}, where \(\eta = \frac{1}{\poly(d)}, C_t = \Omega(\frac{C_0\alpha_1^6}{C_2})\) \(\delta = \frac{1}{\polylog(d)}\) and \(A = \Omega(\frac{1}{\alpha_1}), \log(A/B_{1,1}^{(0)}) = \widetilde{O}(1)\), we have
    \begin{align*}
        T_{1,1} = O(\frac{C_2}{\eta C_0\alpha_1^6})\sum_{x_t\leq O(\frac{1}{\alpha_1})}\eta C_t \leq O(\frac{C_2}{\eta C_0\alpha_1^6})\left(O(1) + \frac{\widetilde{O}(\eta)}{B_{1,1}^{(0)}}\right)\frac{1}{(B_{1,1}^{(0)})^4} \leq \widetilde{O}(\frac{1}{\eta \alpha_1^6 (B_{1,1}^{(0)})^4})
    \end{align*}
    Since \((B_{1,1}^{(0)})^4 \geq \widetilde{\Omega}(\frac{1}{d^2})\) from our initialization, we have \(T_{1,1} \leq O(\frac{d^2}{\eta})\) and thus \(T_{1,1}\) exists. Now we consider when \(B_{1,1}^{(t)} \geq \Omega(\frac{1}{\alpha_1})\). Now by \myref[b,c]{lem:phase1-variables}{Lemma}, we have \(\Phi_1^{(t)} \geq \Omega((C_2 + \alpha_1^6)^{-2})\), which gives an update:
    \begin{align*}
        B_{1,1}^{(t+1)} \geq B_{1,1}^{(t)} + \Omega(\frac{\eta C_0\alpha_1^6}{(C_2 + \alpha_1^6)^2 })(B_{1,1}^{(t)})^5
    \end{align*}
    so again by \myref{lem:TPM}{Lemma}, choosing \(C_t = \Omega(\frac{C_0\alpha_1^6}{(C_2 + \alpha_1^6)^2 })\), 
    \begin{align*}
        T_{1} = \frac{O((C_2+\alpha_1^6)^2)}{\eta C_0\alpha_1^6}\sum_{x_t\in [\Omega(\frac{1}{\alpha_1}), 0.01] }\eta C_t \leq \left(O(1) + \frac{\widetilde{O}(\eta)}{B_{1,1}^{(T_{1,1})}}\right)\frac{\widetilde{O}(\alpha_1^{12})}{(B_{1,1}^{(T_{1,1})})^4} \leq \widetilde{O}(\frac{\alpha_1^6}{\eta  (B_{1,1}^{(T_{1,1})})^4}) \leq O(\frac{\alpha_1^6}{\eta})
    \end{align*}
    where \(O(\frac{\alpha_1^6}{\eta}) \ll O(\frac{d^2}{\eta})\), so we have proved that \(T_1\) exist.
    Now we begin to prove that \myref{induct:phase-1}{Induction} holds for all \(t \leq T_1\). \\
    \newline
    \textbf{Proof of \myref{induct:phase-1}{Induction}:} We first prove (b)--(d), and then come back to prove (a) and (d). At \(t = 0\), we know all induction holds from \myref{property-init}{Properties}. Now we suppose \myref{induct:phase-1}{Induction} holds for all iterations \(\leq t-1\) and prove it holds at \(t\). \\
    \newline
    \textbf{The growth of \(B_{2,1}^{(t)}\):} Applying \myref{lem:learning-v1-phase1}{Lemma}, we have for \(t \leq T_{1,1}\)
    \begin{align*}
        B_{1,1}^{(t+1)} &\geq B_{1,1}^{(t)} + \eta(1 - \widetilde{O}(\frac{1}{d})) \Lambda_{1,1}^{(t)} \\
        B_{2,1}^{(t+1)}&\leq B_{2,1}^{(t)} + \eta(1 + O(\frac{1}{\sqrt{d}}))\Lambda_{2,1}^{(t)} + \eta \frac{(B_{2,1}^{(t)})^2}{(B_{1,1}^{(t)})^2}E_{1,2}^{(t)}\Lambda_{1,1}^{(t)}
    \end{align*}
    For some \(t'_{1} := \min\{t: B_{1,1}^{(t)}\geq \frac{\Omega(1)}{d^{0.49}}\}\), we have \(E_{1,2}^{(t)} \leq \widetilde{O}(B_{1,1}^{(t)}\varrho)\lesssim \frac{1}{d^{0.49}}\) during \(t \leq t'_{1}\), and
    \begin{displaymath}
        \frac{(B_{2,1}^{(t)})^2}{(B_{1,1}^{(t)})^2}E_{1,2}^{(t)}\Lambda_{1,1}^{(t)} \lesssim \frac{(B_{2,1}^{(t)})^2}{d^{0.49}(B_{1,1}^{(t)})^2}\Lambda_{1,1}^{(t)} \leq \widetilde{O}(\frac{1}{d^{0.49}})\Lambda_{2,1}^{(t)}
    \end{displaymath}
    which allow us to give an upper bound to \(B_{2,1}^{(t+1)}\) as
    \begin{align*}
        B_{2,1}^{(t+1)}&\leq (1 + O(\frac{1}{\sqrt{d}})) \Lambda_{2,1}^{(t)} + \widetilde{O}(\frac{1}{d^{0.49}})\Lambda_{2,1}^{(t)} \\
        &\leq (1 + \widetilde{O}(\frac{1}{d^{0.49}})) \Phi_2^{(t)}C_0\alpha_1^6C_2 \Ecal_2^{(t)}(1 + \frac{1}{\polylog(d)}) (B_{2,1}^{(t)})^5 \tag{when \(t\leq t'_1\)} 
    \end{align*}
    Since we also have 
    \begin{align*}
        B_{1,1}^{(t+1)} \geq (1 - \widetilde{O}(\frac{1}{d})) \Lambda_{1,1}^{(t)} \geq (1 - \widetilde{O}(\frac{1}{d})) \Phi_1^{(t)}C_0\alpha_1^6\Ecal_1^{(t)}(1 - \frac{1}{\polylog(d)}) (B_{1,1}^{(t)})^5
    \end{align*}
    Since \(B_{1,1}^{(0)} \geq B_{2,1}^{(0)}(1 + \Omega(\frac{1}{\log d}))\), we can now apply \myref{coro:TPM}{Corollary} to the two sequence \(B_{1,1}^{(t+1)}\) and \(B_{2,1}^{(t+1)}\), where \(S_t = \frac{\Phi_1^{(t)}\Ecal_1^{(t)}}{\Phi_2^{(t)}\Ecal_2^{(t)}}(1 + \frac{1}{\polylog(d)})\) to get 
    \begin{align*}
        B_{1,1}^{(t'_1)} \geq \frac{1}{d^{0.499}}\quad \text{while}\quad B_{2,1}^{(t'_1)}\leq \widetilde{O}(\frac{1}{\sqrt{d}})
    \end{align*}
    Note that here the update of \(B_{2,1}^{(t)}\) at every step satisfies \(\sign(B_{2,1}^{(t+1)} - B_{2,1}^{(t)}) =  \sign(B_{2,1}^{(t)})\) which implies \(  B_{2,1}^{(t'_1)} = \widetilde{\Theta}(\frac{1}{\sqrt{d}})\). Now for every \(T\in [t'_1,T_{1}]\), we can apply \myref{lem:TPM-degree}{Lemma} to get that 
    \begin{align*}
        \sum_{t\in [t'_1, T]}\eta\frac{(B_{2,1}^{(t)})^2}{(B_{1,1}^{(t)})^2}E_{1,2}^{(t)}\Lambda_{1,1}^{(t)} \leq \widetilde{O}(\varrho+\frac{1}{\sqrt{d}}) O(\frac{1}{B_{1,1}^{(t'_1)}})\max_{t\leq T}\{(B_{2,1}^{(t)})^2\} \leq O(\frac{1}{d^{0.5+\Omega(1)}})
    \end{align*}
    Suppose we have proved that \(B_{2,1}^{(t)} \leq \widetilde{O}(\frac{1}{\sqrt{d}})\) for each \(t \leq T\), we define a new sequence 
    \begin{align*}
        \widetilde{B}_{2,1}^{(t+1)} &= \widetilde{B}_{2,1}^{(t)} + \eta(1 + \widetilde{O}(\frac{1}{d^{0.49}})) \Phi_2^{(t)}C_0\alpha_1^6C_2 \Ecal_2^{(t)}(1 + \frac{1}{\polylog(d)}) (\widetilde{B}_{2,1}^{(t)})^5 ,\\
        \text{where }\widetilde{B}_{2,1}^{(t'_1)} &= B_{2,1}^{(t'_1)} + \sum_{t\in [t'_1, T]}\eta\frac{(B_{2,1}^{(t)})^2}{(B_{1,1}^{(t)})^2}E_{1,2}^{(t)}\Lambda_{1,1}^{(t)} = (1\pm o(1))\widetilde{B}_{2,1}^{(t'_1)}
    \end{align*} 
    It can be directly seen that \(|\widetilde{B}_{2,1}^{(t)} - \widetilde{B}_{2,1}^{(0)}| \geq |B_{2,1}^{(t)} - B_{2,1}^{(0)}|\) for all \(t \in [t'_1,T]\). Notice that now \(\widetilde{B}_{2,1}^{(t'_1)} \leq d^{\Omega(1)}B_{1,1}^{(t'_1)}\), we can now apply \myref{coro:TPM}{Corollary} again to get 
    \begin{align*}
        |B_{2,1}^{(T)} - B_{2,1}^{(0)}| \leq |\widetilde{B}_{2,1}^{(T)} - \widetilde{B}_{2,1}^{(0)}| \leq \frac{1}{\sqrt{d}\polylog(d)} \tag{for every \(T\leq T_{1,1}\)}
    \end{align*}
    Now we deal with \(t \in [T_{1,1}, T_1]\). During this stage, we can directly apply \myref{coro:TPM}{Corollary} to \(\widetilde{B}_{2,1}^{(t)}\) and \(B_{1,1}^{(t)}\), where \(S_t = \frac{\Phi_1^{(t)}H_{1,2}^{(t)}}{\Phi_2^{(t)}H_{2,2}^{(t)}} \leq O(\alpha_1^{O(1)})\), to get that 
    \begin{align*}
        |B_{2,1}^{(T)} - B_{2,1}^{(0)}| \leq |\widetilde{B}_{2,1}^{(T)} - \widetilde{B}_{2,1}^{(0)}| \leq \frac{1}{\sqrt{d}\polylog(d)} \tag{for every \(T\leq T_{1}\)}
    \end{align*}
    And thus by \myref{property-init}{Lemma}, we have \(B_{2,1}^{(T)} = B_{2,1}^{(0)}(1 \pm o(1))\).\\                
    \newline
    \textbf{The growth of \(B_{1,2}^{(t)}\) and \(B_{2,2}^{(t)}\):} By \myref{lem:learning-v2-phase1}{Lemma}, we can write down the update as 
    \begin{align*}
        B_{j,2}^{(t+1)} = B_{j,2}^{(t)} + \eta \left(1 \pm \widetilde{O}(\alpha_1^{6})(E_{3-j,j}^{(t)} + (B_{j,1}^{(t)})^3)\right)\Lambda_{j,2}^{(t)}
    \end{align*}
    Since \(B_{2,1}^{(t)}\leq \widetilde{O}(\frac{1}{\sqrt{d}})\) and \(E_{1,2}^{(t)} \leq\widetilde{O}(\varrho+\frac{1}{\sqrt{d}})B_{1,1}^{(t)}, E_{2,1 }^{(t)} \leq\widetilde{O}(\frac{1}{d})\) because we chose \(\eta_E \leq \eta\), we only need to care about \((B_{1,1}^{(t)})^3\) in the update expression. Now define \(t'_2 := \min\{t: B_{1,1}^{(t)}\geq \Omega(\frac{1}{\alpha_1^2})\}\), we have 
    \begin{itemize}
        \item For \(t \leq t'_2\), by \myref{coro:TPM}{Corollary} and setting \(x_t = B_{1,1}^{(t)}\), \(C_t = (1 - \widetilde{O}(\frac{1}{d})) \Phi_1^{(t)}C_0\alpha_1^6H_{1,2}^{(t)}\), \(S_t = O(\frac{\alpha_2^6\Phi_j^{(t)}H_{j,1}^{(t)}}{\alpha_1^6\Phi_1^{(t)}H_{1,2}^{(t)}}) \leq \widetilde{O}(\frac{\alpha_2^6}{\alpha_1^6} )\ll \frac{1}{\polylog (d)}\) (by \myref[a,c]{lem:phase1-variables}{Lemma}), we have \(|B_{j,2}^{(t)} - B_{j,2}^{(0)}| \leq O(\frac{\alpha_2^6}{\alpha_1^6}\frac{1}{\sqrt{d}}) \lesssim \frac{1}{\sqrt{d}\polylog(d)}\) for all \(t \leq t'_2\), which implies \(B_{j,2}^{(t'_2)} = B_{j,2}^{(0)} \pm \frac{1}{\sqrt{d}\polylog(d)} \in [\Omega(\frac{1}{\sqrt{d}\log d}), O(\frac{\sqrt{\log d}}{\sqrt{d}})]\) by \myref{property-init}{Lemma}.
        \item For \(t \in [t'_2, T_1]\), we can use \myref{coro:TPM}{Corollary} again and let \(x_t = B_{1,1}^{(t)}\), we know \(B_{1,1}^{(t'_2)} \geq d^{\Omega(1)}B_{2,1}^{(t'_2)}\). Setting \(C_t = (1 - \widetilde{O}(\frac{1}{d})) \Phi_1^{(t)}C_0\alpha_1^6H_{1,2}^{(t)}\), \(S_t = O((1 + \alpha_1^6)\frac{\alpha_2^6\Phi_j^{(t)}H_{j,1}^{(t)}}{\alpha_1^6\Phi_1^{(t)}H_{1,2}^{(t)}}) \leq O(\alpha^{O(1)})\), we can have \(|B_{j,2}^{(t)} - B_{j,2}^{(t'_2)}| \lesssim \frac{1}{\sqrt{d}\polylog(d)}\), which implies \(B_{j,2}^{(t)} \in [\Omega(\frac{1}{\sqrt{d}\log d}), O(\frac{\sqrt{\log d}}{\sqrt{d}})]\) for all \(t \in [t'_2, T_1]\).
    \end{itemize}
    This proves \myref[b]{induct:phase-1}{Induction}. Indeed, simple calculations also proves \myref[c]{induct:phase-1}{Induction}, since the update of \(B_{1,1}^{(t)}\) is always larger than others' during \(t \leq T_1\).\\
    \newline
    \textbf{For \myref[d]{induct:phase-1}{Induction}:} From \myref{lem:learning-pred-head}{Lemma}, we can write down the update 
    \begin{align*}
        -\nabla_{E_{1,2}}L(W^{(t)},E^{(t)}) &= O(\Lambda_{1,1}^{(t)}B_{1,1}^{(t)})\left(  -C_1E_{1,2}^{(t)} +\widetilde{O}(\frac{(B_{1,2}^{(t)})^3}{(B_{1,1}^{(t)})^3}) + C_2(R_{1,2}^{(t)}+ \varrho) \right) 
    \end{align*}
    for some constants \(C_1, C_2 = \Theta(1)\). Applying \myref{lem:TPM-degree}{Lemma} to \(O(\Lambda_{1,1}^{(t)}B_{1,1}^{(t)})\frac{(B_{1,2}^{(t)})^3}{(B_{1,1}^{(t)})^3}\), we can obtain 
    \begin{align*}
        \sum_{t\leq T}O(\eta_E \Lambda_{1,1}^{(t)}B_{1,1}^{(t)})\frac{(B_{1,2}^{(t)})^3}{(B_{1,1}^{(t)})^3} = \frac{\eta_E}{\eta}\sum_{t\leq T}O(\eta \Lambda_{1,1}^{(t)})\frac{(B_{1,2}^{(t)})^3}{(B_{1,1}^{(t)})^2} \leq \widetilde{O}(\frac{\eta_E/\eta}{d^{3/2}})\frac{1}{B_{1,1}^{(0)}} \leq \widetilde{O}(\frac{\eta_E/\eta}{d}) 
    \end{align*}
    So here it suffices to notice that whenever \(|E_{1,2}^{(t)}| < 2\frac{C_2}{C_1}(R_{1,2}^{(t)}+ \varrho)\) (which is obviously satisified at \(t=0\)), we would have 
    \begin{align*}
        O(\Lambda_{1,1}^{(t)}B_{1,1}^{(t)})\left(  -O(E_{1,2}^{(t)}) + C_2(R_{1,2}^{(t)}+ \varrho) \right)  =  - O(\Lambda_{1,1}^{(t)}B_{1,1}^{(t)})\widetilde{O}(R_{1,2}^{(t)}+ \varrho) \leq O(\Lambda_{1,1}^{(t)}B_{1,1}^{(t)})\widetilde{O}(\varrho + \frac{1}{\sqrt{d}})
    \end{align*}
    In that case, we will always have (since \(E_{1,2}^{(0)} = 0\))
    \begin{align*}
        E_{1,2}^{(t+1)} &\leq \left|\sum_{t\leq T}\widetilde{O}(\eta_E \Lambda_{1,1}^{(t)}B_{1,1}^{(t)})\frac{(B_{1,2}^{(t)})^3}{(B_{1,1}^{(t)})^3}\right| + \sum_{s\leq t}O(\eta_E \Lambda_{1,1}^{(s)}B_{1,1}^{(s)})(R_{1,2}^{(s)}+ \varrho) \leq\widetilde{O}(\varrho+\frac{1}{\sqrt{d}} )\frac{\eta_E}{\eta}B_{1,1}^{(t+1)}
    \end{align*}
    Similarly for \(\nabla_{E_{2,1}}L(W^{(t)},E^{(t)})\), we can write down 
    \begin{align*}
        -\nabla_{E_{2,1}}L(W^{(t)},E^{(t)}) = \widetilde{O}(\frac{(B_{1,2}^{(t)})^3}{(B_{1,1}^{(t)})^2})\Lambda_{1,1}^{(t)}  + \sum_{\ell\in[2]}C_2\Lambda_{2,\ell}^{(t)}B_{2,\ell}^{(t)}\left(  -O(E_{2,1}^{(t)}R_2^{(t)}) + O(R_{1,2}^{(t)}) \right)
    \end{align*}
    by \myref{lem:TPM-degree}{Lemma}, we have 
    \begin{align*}
        \sum_{t \leq T_1}\eta_E \widetilde{O}(\frac{(B_{1,2}^{(t)})^3}{(B_{1,1}^{(t)})^2})\Lambda_{1,1}^{(t)} \leq \widetilde{O}(\frac{\eta_E/\eta}{d})
    \end{align*}
    and since from previous comparison results we know that 
    \begin{align*}
        \sum_{t \leq T_1}\sum_{\ell\in[2]}\eta_E C_2\Lambda_{2,\ell}^{(t)}B_{2,\ell}^{(t)} = \frac{\eta_E}{\eta}\sum_{t \leq T_1}\sum_{\ell\in[2]}\eta C_2\Lambda_{2,\ell}^{(t)}B_{2,\ell}^{(t)} \leq \widetilde{O}(\frac{\eta_E/\eta}{d})
    \end{align*}
    we can then prove the claim.\\
    \newline
    \textbf{For \myref[a]{induct:phase-1}{Induction}:} We can write down the update of \(\|w_j^{(t)}\|_2^2\) as follows:
    \begin{align*}
        \|w_j^{(t+1)}\|_2^2 &= \|w_j^{(t)} - \eta \nabla_{w_j}L(W^{(t)},E^{(t)})\|_2^2 \\
        & = \|w_j^{(t)}\|_2^2 - \eta \dbrack{\nabla_{w_j}L(W^{(t)},E^{(t)}), w_j^{(t)}} + \eta^2\|\nabla_{w_j}L(W^{(t)},E^{(t)})\|_2^2
    \end{align*}
    from \eqref{eqdef:weight-grad} and \myref[a,b,c]{induct:phase-1}{Induction} at iteration \(t\) and our assumption on \(\xi_p\), we know 
    \begin{displaymath}
        \|\nabla_{w_j}L(W^{(t)},E^{(t)})\|_2^2 \leq \widetilde{O}(d)
    \end{displaymath}
    which allow us to choose \(\eta \leq \frac{1}{\poly(d)}\) to be small enough so that \(\eta d T_1 \leq \frac{1}{\eta \poly(d)} \). Then by \myref[b]{lem:bn-grad}{Lemma}, we have 
    \begin{align*}
        \|w_j^{(t+1)}\|_2^2 &= \|w_j^{(0)}\|_2^2  \pm \eta\sum_{s\leq t}|\dbrack{\nabla_{w_j}L(W^{(s)},E^{(s)}), w_j^{(s)}}| \pm \frac{1}{\poly(d)} \\
        &\leq  \|w_j^{(0)}\|_2^2 \pm \eta \sum_{s\leq t} \widetilde{O}(\varrho + \frac{1}{\sqrt{d}} )|\Lambda_{1,1}^{(s)}|\sum_{j\in[2]}|E_{j,3-j}^{(s)}| \pm \frac{1}{\poly(d)}
    \end{align*}
    Since from the above analysis of the update of \(B_{1,1}^{(t)}\), we know \(\sum_{t\leq T_1}\Lambda_{1,1}^{(t)} \leq O(1)\). Moreover, we also know that \(|B_{1,1}^{(t)}|\) is increasing and \(\sign(\Lambda_{1,1}^{(t)}) = \sign(\Lambda_{1,1}^{(s)})\) for any \(s,t \leq T_1\). Thus they imply \(\sum_{s\leq t} |\Lambda_{1,1}^{(s)}| = |\sum_{s\leq t}\Lambda_{1,1}^{(s)}| = O(1)\), which can be combine with \myref[d]{induct:phase-1}{Induction} to prove the claim.\\
    \newline
    \textbf{Proof of \myref[e]{induct:phase-1}{Induction}:}
    We can write down the update of \(R_{1,2}^{(t)} = \dbrack{\Pi_{V^{\perp}} w_1^{(t)},w_2^{(t)}}\) as follows
    \begin{align*}
        \dbrack{\Pi_{V^{\perp}} w_1^{(t+1)}, w_2^{(t+1)}} & = \dbrack{\Pi_{V^{\perp}} w_1^{(t)} -\Pi_{V^{\perp}}\eta\nabla_{w_1}L(W^{(t)},E^{(t)}), \Pi_{V^{\perp}}w_2^{(t)} - \Pi_{V^{\perp}}\eta\nabla_{w_2}L(W^{(t)},E^{(t)})} \\
        & = R_{1,2}^{(t)} - \eta\dbrack{\nabla_{w_1}L(W^{(t)},E^{(t)}),\Pi_{V^{\perp}} w_2^{(t)}} - \eta\dbrack{\nabla_{w_2}L(W^{(t)},E^{(t)}),\Pi_{V^{\perp}} w_1^{(t)}}\\
        &\quad + \eta^2\dbrack{\Pi_{V^{\perp}}\nabla_{w_1}L(W^{(t)},E^{(t)}),\Pi_{V^{\perp}}\nabla_{w_2}L(W^{(t)},E^{(t)})}
    \end{align*}
    By Cauchy-Schwarz inequality and the same analysis above we have 
    \begin{align*}
        |\dbrack{\Pi_{V^{\perp}}\nabla_{w_1}L(W^{(t)},E^{(t)}),\Pi_{V^{\perp}}\nabla_{w_2}L(W^{(t)},E^{(t)})}| &\leq \|\nabla_{w_1}L(W^{(t)},E^{(t)})\|_2\|\nabla_{w_2}L(W^{(t)},E^{(t)})\|_2 \\
        &\leq \widetilde{O}(d)
    \end{align*}
    so by our choice of \(\eta\)
    \begin{align*}
        \sum_{t\leq T_1}\eta^2|\dbrack{\Pi_{V^{\perp}}\nabla_{w_1}L(W^{(t)},E^{(t)}),\Pi_{V^{\perp}}\nabla_{w_2}L(W^{(t)},E^{(t)})}| \leq \frac{1}{\poly(d)}
    \end{align*}
    and by \myref{lem:learning-R12-phase1}{Lemma} we have 
    \begin{align*}
        \left|- \eta\dbrack{\nabla_{w_1}L(W^{(t)},E^{(t)}),\Pi_{V^{\perp}} w_2^{(t)}} - \eta\dbrack{\nabla_{w_2}L(W^{(t)},E^{(t)}),\Pi_{V^{\perp}} w_1^{(t)}}\right| \leq  \eta\widetilde{O}( \Lambda_{1,1}^{(t)}B_{1,1}^{(t)})(\varrho+\frac{1}{\sqrt{d}})
    \end{align*}
    which implies 
    \begin{align*}
        |\dbrack{\Pi_{V^{\perp}} w_1^{(t+1)}, w_2^{(t+1)}}| &\leq |\dbrack{\Pi_{V^{\perp}} w_1^{(0)}, w_2^{(0)}}| +\sum_{s\leq t}\sum_{j\in[2]} \eta |\dbrack{\nabla_{w_j}L(W^{(s)},E^{(s)}),\Pi_{V^{\perp}} w_{3-j}^{(s)}}| + \frac{1}{\poly(d)} \\
        & \leq \widetilde{O}(\frac{1}{\sqrt{d}}) +\sum_{s\leq t}\eta\widetilde{O}( \Lambda_{1,1}^{(s)}B_{1,1}^{(s)}) + \frac{1}{\poly(d)} \\
        &\leq \widetilde{O}(\frac{1}{\sqrt{d}}) + \widetilde{O}(\varrho+\frac{1}{\sqrt{d}})B_{1,1}^{(t+1)} \\
        &\leq \widetilde{O}(\varrho+\frac{1}{\sqrt{d}})
    \end{align*}
    which completes the proof of \myref{induct:phase-1}{Induction}. As for (a) -- (e) of \myref{lem:phase-1}{Lemma}, they are just direct corrolary of our induction at \(t = T_1\). 
\end{proof}

\section{Phase II: The Substitution Effect of Prediction Head}
In this phase, As \(B_{1,1}^{(t)}\) is learned to become very large (\(B_{1,1}^{(t)} \gtrsim \|w_1^{(t)}\|_2\)). The focus now shift to grow \(E_{2,1}^{(t)}\), because we want \(C_1\alpha_1^6((B_{2,1}^{(t)})^3+E_{2,1}^{(t)}(B_{1,1}^{(t)})^3)^2\) in \(H_{2,1}^{(t)}\) to dominate \(\Ecal_{2,1}^{(t)}\). We can write down the gradient of \(E_{2,1}^{(t)}\) as
\begin{align*}
    -\nabla_{E_{2,1}} L(W^{(t)},E^{(t)}) &= \sum_{\ell \in [2]}C_0\Phi_2^{(t)}\alpha_{\ell}^6(B_{2,\ell}^{(t)})^3((B_{1,\ell}^{(t)})^3H_{2,3-\ell}^{(t)} - (B_{2,3-\ell}^{(t)})^3K_{2,3-\ell}^{(t)})  - \sum_{\ell \in [2]}\Sigma_{2,\ell}^{(t)} \nabla_{E_{2,1}} \Ecal_{2,1}^{(t)} 
\end{align*}
Now let us define
\begin{align}\label{eqdef:T_2}
    T_{2} := \min\{t: R_2^{(t)} < \frac{1}{\log d}|E_{1,2}^{(t)}|\}
\end{align}
We will prove that \(E_{2,1}^{(T_2)}\) reaches at most \(O(\sqrt{\eta_E/\eta})\) and the following induction hypothesis holds throughout \(t\in[T_1,T_2]\). In this phase, the learning of \(E_{2,1}^{(t)}\) is much faster than the growth of the first feature \(v_1\) such that \(T_2 - T_1 = o(T_1/\sqrt{d})\), which is due to the acceleration effects brought by \(B_{1,1}^{(t)} = \Omega(1)\) during this phase.

\subsection{Induction in Phase II}
We will be based on the following induction hypothesis during phase II.

\begin{induct}[Phase II]\label{induct:phase-2}
    When \(t \in [T_1,T_2]\), we hypothesize the followings would hold
    \begin{enumerate}[(a)]
        \item \(B_{1,1}^{(t)} =\Theta(1)\), \(B_{j,\ell}^{(t)} = B_{j,\ell}^{(T_1)}(1\pm o(1))  = \widetilde{\Theta}(\frac{1}{\sqrt{d}})\) for \((j,\ell)\neq (1,1)\) and \(\sign(B_{j,\ell}^{(t)}) = \sign(B_{j,\ell}^{(T_1)})\);
        \item \(|R_{1,2}^{(t)}| = \widetilde{O}(\varrho+\frac{1}{\sqrt{d}})\alpha_1^{O(1)}[R_{1}^{(t)}]^{1/2}[R_{2}^{(t)}]^{1/2}\);
        \item \(R_1^{(t)}\in [\Omega(\frac{1}{d^{3/4}\alpha_1^2}), O(1)]\), \(R_{2}^{(t)} \in [\Omega(\frac{1}{\log d}\sqrt{\eta_E/\eta}), O(1)]\);
        \item \(E_{1,2}^{(t)}\leq \widetilde{O}(\varrho + \frac{1}{\sqrt{d}})[R_1^{(t)}]^{3/2}\) and \(E_{2,1}^{(t)} \leq O(\sqrt{\eta_E/\eta})\).
    \end{enumerate}
\end{induct}

Under \myref{induct:phase-2}{Induction}, we have some results as direct corollary.

\begin{claim}\label{claim:Ecal-phase2}
    At each iteration \(t \in [T_1,T_2]\), if \myref{induct:phase-1}{Induction} holds, then
    \begin{enumerate}[(a)]
        \item \(\Ecal_j^{(t)} = \Theta(C_2[R_j^{(t)}]^3) \);
        \item \(\Ecal_{j,3-j}^{(t)} = \Ecal_j^{(t)} \pm \widetilde{O}(E_{j,3-j}^{(t)}(\varrho+\frac{1}{\sqrt{d}})[R_1^{(t)}]^{3/2}[R_2^{(t)}]^{3/2}) + O((E_{j,3-j}^{(t)})^2[R_{3-j}^{(t)}]^{3}) \) for each \(j\in[2]\);
    \end{enumerate}
\end{claim}

\begin{proof}
    It is trivial to derive (a) from the expression of \(\Ecal_j^{(t)}\) and our assumption of \(\xi_p\). For (b) it suffices to directly calculate the expression of \(\Ecal_{j,3-j}^{(t)}\) along with \myref[b]{induct:phase-2}{Induction}.
\end{proof}

\begin{lemma}[variables control in phase II]\label{lem:phase2-variables}
    In Phase II \((t\in[T_1, T_2])\), if \myref{induct:phase-2}{Induction} holds, then 
    \begin{enumerate}[(a)]
        \item \(\Phi_1^{(t)} = \widetilde{\Theta}(\frac{1}{\alpha_1^{12}})\), \(\Phi_2^{(t)} = \Theta((C_2[R_2^{(t)}]^3 + C_1\alpha_1^6(E_{2,1}^{(t)})^2)^{-2})\);
        \item \(K_{1,\ell}^{(t)} = \widetilde{O}(\alpha_{\ell}^6/d^{3/2})\), \(K_{2,\ell}^{(t)} = \widetilde{O}(E_{2,1}^{(t)}\alpha_\ell^6/d^{3/2} + \alpha_{\ell}^6/d^3)\)
        \item \(H_{1,1}^{(t)} = \Theta(C_1\alpha_1^6)\), \(H_{1,2}^{(t)} = \widetilde{O}([R_1^{(t)}]^3)\), \(H_{2,2}^{(t)}= \Theta(C_2[R_2^{(t)}]^3)\), \(H_{2,1}^{(t)}= \Theta(C_2[R_2^{(t)}]^3+C_1\alpha_1^6(E_{2,1}^{(t)})^2)\).
    \end{enumerate}
\end{lemma}

\begin{proof}
    The proof of (a) directly follows from \myref[a,c]{induct:phase-2}{Induction} and \myref{claim:Ecal-phase2}{Claim}. The proof of (b) follows directly from the expression of \(K_{j,\ell}\) and \myref[a,d]{induct:phase-2}{Induction}. The proof of (c) is also similar.
\end{proof}

\subsection{Gradient Lemmas for Phase II}

\begin{lemma}[learning prediction head \(E_{1,2}, E_{2,1}\) in phase II]\label{lem:learning-pred-head-phase2}
    If \myref{induct:phase-2}{Induction} holds at iteration \(t \in [T_1,T_2]\), then we have 
    \begin{align*}
        (a) \quad -\nabla_{E_{1,2}}L(W^{(t)},E^{(t)}) &= (1 \pm \widetilde{O}(\frac{\alpha_1^{O(1)}}{d^{3/2}}))\Sigma_{1,1}^{(t)} (-2E_{1,2}^{(t)}[R_{2}^{(t)}]^{3} \pm O(\overline{R}_{1,2}^{(t)} + \varrho)[R_{1}^{(t)}]^{3/2}[R_{2}^{(t)}]^{3/2}) \\
        &\qquad \pm \Sigma_{1,1}^{(t)} \widetilde{O}(\frac{\eta_E/\eta}{\sqrt{d}})\max\{[R_1^{(t)}]^3,\frac{\alpha_1^{O(1)}}{d^{5/2}}\},\\
        (b)\quad -\nabla_{E_{2,1}}L(W^{(t)},E^{(t)}) &= (1 \pm \widetilde{O}(\frac{\alpha_1^{O(1)}}{d^{3/2}})) C_0\Phi_2^{(t)} \alpha_{1}^6(B_{2,1}^{(t)})^3(B_{1,1}^{(t)})^3H_{2,2}^{(t)}  \\
        &\quad \pm O(\Sigma_{2,1}^{(t)})(|E_{2,1}^{(t)}|[R_{1}^{(t)}]^{3} \pm O(\overline{R}_{1,2}^{(t)} + \varrho)[R_{1}^{(t)}]^{3/2}[R_{2}^{(t)}]^{3/2})
    \end{align*}
    
\end{lemma}

\begin{proof}
    We first write down the gradient for \(E_{j,3-j}^{(t)}\): (ignoring the time superscript \(^{(t)}\))
    \begin{align*}
        -\nabla_{E_{j,3-j}}L(W,E) &= \sum_{\ell \in [2]}C_0\Phi_j\alpha_{\ell}^6B_{j,\ell}^3(B_{3-j,\ell}^3H_{j,3-\ell} - B_{3-j,3-\ell}^3K_{j,3-\ell}) - \sum_{\ell \in [2]}\Sigma_{j,\ell}\nabla_{E_{j,3-j}}\Ecal_{j,3-j}
    \end{align*}
    where \(\nabla_{E_{j,3-j}}\Ecal_{j,3-j} = \E\left[2\dbrack{w_j,\xi_p}^3\dbrack{w_{3-j},\xi_p}^3 + 2E_{j,3-j}\dbrack{w_{3-j},\xi_p}^6\right]\). Thus we have 
    \begin{align*}
        \nabla_{E_{j,3-j}}\Ecal_{j,3-j}^{(t)} = 2E_{j,3-j}^{(t)}[R_{3-j}^{(t)}]^{3} \pm O(\overline{R}_{1,2}^{(t)} + \varrho)[R_{1}^{(t)}]^{3/2}[R_{2}^{(t)}]^{3/2}
    \end{align*}
    and by \myref{claim:noise}{Claim} and \myref[a]{induct:phase-2}{Induction}, if \((j,\ell)\neq (1,1)\)
    \begin{align*}
        \Sigma_{j,\ell}^{(t)} = O(\Sigma_{1,1}^{(t)})\frac{(B_{j,\ell}^{(t)})^6 + E_{j,3-j}^{(t)}(B_{3-j,\ell}^{(t)})^3(B_{j,\ell}^{(t)})^3}{(B_{1,1}^{(t)})^6}\frac{\Phi_j^{(t)}}{\Phi_1^{(t)}} \leq  o(\frac{1}{d^{3/2}})\Sigma_{1,1}^{(t)}\frac{\Phi_j^{(t)}}{\Phi_1^{(t)}}
    \end{align*}
    Therefore for \(j=1\):
    \begin{align*}
        \sum_{\ell \in [2]}\Sigma_{1,\ell}^{(t)}\nabla_{E_{1,2}}\Ecal_{1,2}^{(t)} = (1 \pm \widetilde{O}(\frac{\alpha_1^{O(1)}}{d^{3/2}}))\Sigma_{1,1}^{(t)}\nabla_{E_{1,2}}\Ecal_{1,2}^{(t)}
    \end{align*}
    Now by \myref[a,c]{induct:phase-2}{Induction} and \myref[b,c]{lem:phase2-variables}{Lemma} we have \((B_{1,\ell}^{(t)})^3H_{1,3-\ell}^{(t)} \leq \max\{\Theta(C_2[R_1^{(t)}]^3), \widetilde{O}(\frac{\alpha_1^{6} }{d^{3/2}})\}\), which leads to the bounds
    \begin{align*}
        |(B_{1,\ell}^{(t)})^3(B_{2,\ell}^{(t)})^3H_{1,3-\ell}^{(t)}| \leq \widetilde{O}(\frac{1}{d^{3/2}}) \max\{[R_1^{(t)}]^3,\frac{\alpha_1^6}{d^{3}}\}, \qquad |(B_{1,\ell}^{(t)})^3(B_{2,3-\ell}^{(t)})^3K_{1,3-\ell}^{(t)}|\leq \widetilde{O}(\frac{1}{d^{3}}) 
    \end{align*}
    which implies 
    \begin{align*}
        \left|\sum_{\ell \in [2]}C_0\Phi_1^{(t)} \alpha_{\ell}^6(B_{1,\ell}^{(t)})^3((B_{2,\ell}^{(t)})^3H_{1,3-\ell}^{(t)} - (B_{2,3-\ell}^{(t)})^3K_{1,3-\ell}^{(t)}) \right|\lesssim \widetilde{O}(\frac{\eta_E/\eta}{\sqrt{d}})\Sigma_{1,1}^{(t)}\max\{[R_1^{(t)}]^3,\frac{\alpha_1^{O(1)}}{d^{5/2}}\}
    \end{align*}
    Combining above together, we have 
    \begin{align*}
        &-\nabla_{E_{1,2}}L(W^{(t)},E^{(t)}) \\
        = \ & (1 + o(\frac{1}{d^{3/2}}))\Sigma_{1,1}^{(t)} (-2E_{1,2}^{(t)}[R_{2}^{(t)}]^{3} \pm O(\overline{R}_{1,2}^{(t)} + \varrho)[R_{1}^{(t)}]^{3/2}[R_{2}^{(t)}]^{3/2} \pm \widetilde{O}(\frac{\eta_E/\eta}{\sqrt{d}})\max\{[R_1^{(t)}]^3,\frac{\alpha_1^{O(1)}}{d^{5/2}}\} )
    \end{align*}
    For \(-\nabla_{E_{2,1}}L(W^{(t)},E^{(t)})\), the expression is slightly different, we first observe that by \myref[a]{induct:phase-2}{Induction} 
    \begin{align*}
        \Delta_{2,2}^{(t)} \leq \widetilde{O}(\frac{1}{d^{3/2}})\Delta_{2,1}^{(t)}
    \end{align*}
    Meanwhile, by \myref[a]{induct:phase-2}{Induction} and \myref[b,c]{lem:phase2-variables}{Lemma} , we have 
    \begin{align*}
        \Xi_2^{(t)} \leq \widetilde{O}(\frac{\alpha_1^{O(1)}}{d^3})C_0C_2\Phi_2^{(t)}[R_2^{(t)}]^3,\quad \text{}
    \end{align*}
    Moreover, we can also calculate \(\Sigma_{2,1}^{(t)} = C_0C_2\alpha_1^6E_{2,1}^{(t)}\Phi_2^{(t)}(B_{2,1}^{(t)})^3) =  \widetilde{O}(\frac{\alpha_1^6}{d^{3/2}}) \Phi_2^{(t)}\), \(\Sigma_{2,2}^{(t)} = \widetilde{O}(\frac{\alpha_2^6}{d^3})\Phi_2^{(t)}\), which gives 
    \begin{align*}
        \sum_{\ell\in[2]}\Sigma_{2,\ell}^{(t)}\nabla_{E_{2,1}}\Ecal_{2,1}^{(t)} =  \Sigma_{2,1}^{(t)}(-\Theta(E_{2,1}^{(t)})[R_{1}^{(t)}]^{3} \pm O(\overline{R}_{1,2}^{(t)} + \varrho)[R_{1}^{(t)}]^{3/2}[R_{2}^{(t)}]^{3/2})
    \end{align*}
    Now we combine the above results and get
    \begin{align*}
        -\nabla_{E_{2,1}}L(W^{(t)},E^{(t)}) &= (1 \pm \widetilde{O}(\frac{\alpha_1^{O(1)}}{d^{3/2}})) C_0\Phi_2^{(t)} \alpha_{1}^6(B_{2,1}^{(t)})^3(B_{1,1}^{(t)})^3\Ecal_{2,1}^{(t)}  \\
        &\quad \pm O(\Sigma_{2,1}^{(t)})(|E_{2,1}^{(t)}|[R_{1}^{(t)}]^{3} \pm O(\overline{R}_{1,2}^{(t)} + \varrho)[R_{1}^{(t)}]^{3/2}[R_{2}^{(t)}]^{3/2})
    \end{align*}
\end{proof}

\begin{lemma}[reducing noise in phase II]\label{lem:reduce-noise-phase2}
    Suppose \myref{induct:phase-2}{Induction} holds at \(t\in [T_1,T_2]\), then
    \begin{enumerate}[(a)]
        \item \(\dbrack{-\nabla_{w_1}L(W^{(t)},E^{(t)}), \Pi_{V^{\perp}}w_1^{(t)}}  =  \Sigma_{1,1}^{(t)}\Theta( - [R_{1}^{(t)}]^3 \pm \widetilde{O}(|E_{1,2}^{(t)}| + \frac{|E_{2,1}^{(t)}|^2}{d^{3/2}})(\overline{R}_{1,2}^{(t)} + \varrho)[R_{1}^{(t)}]^{3/2}[R_{2}^{(t)}]^{3/2})\);
        \item \(\dbrack{-\nabla_{w_1}L(W^{(t)},E^{(t)}), \Pi_{V^{\perp}}w_{2}^{(t)}} = \Sigma_{1,1}^{(t)}((-\Theta(\overline{R}_{1,2}^{(t)}) + O(\varrho) )[R_{1}^{(t)}]^{5/2}[R_{2}^{(t)}]^{1/2} + \widetilde{O}(|E_{1,2}^{(t)}| + \frac{|E_{2,1}^{(t)}|^2}{d^{3/2}})R_1^{(t)}[R_{2}^{(t)}]^2)\)
    \end{enumerate}
    And furthermore 
    \begin{align*}
        (c)\quad\dbrack{-\nabla_{w_2}L(W^{(t)},E^{(t)}), \Pi_{V^{\perp}}w_2^{(t)}} &= - \Theta([R_{2}^{(t)}]^3) \Big(\Sigma_{1,1}^{(t)} \Theta((E_{1,2}^{(t)})^2) + \sum_{\ell\in[2]}\Sigma_{2,\ell}^{(t)} \Big) \\
        &\quad \pm  O\Big(\sum_{j,\ell}\Sigma_{j,\ell}^{(t)}E_{j,3-j}^{(t)} (\overline{R}_{1,2}^{(t)} + \varrho)[R_{1}^{(t)}]^{3/2}[R_{2}^{(t)}]^{3/2}\Big);\\
        (d)\quad \dbrack{-\nabla_{w_2}L(W^{(t)},E^{(t)}), \Pi_{V^{\perp}}w_{1}^{(t)}} &= \Big(\Sigma_{1,1}^{(t)} \Theta((E_{1,2}^{(t)})^2) + \sum_{\ell\in[2]}\Sigma_{2,\ell}^{(t)} \Big)(-\Theta(\overline{R}_{1,2}^{(t)}) \pm O(\varrho) )[R_{2}^{(t)}]^{5/2}[R_{1}^{(t)}]^{1/2} \\
        &\quad + O\Big(\sum_{j,\ell}\Sigma_{j,\ell}^{(t)}E_{j,3-j}^{(t)}R_{2}^{(t)}[R_{1}^{(t)}]^{2}\Big)
    \end{align*}
\end{lemma}

\begin{proof}
    The proof can be obtained directly from some calculation using \myref{claim:noise}{Claim} as follows:
    \newline 
    \textbf{Proof of (a):} From \eqref{eqdef:weight-grad}, we can obtain that 
    \begin{align*}
        \dbrack{-\nabla_{w_1}L(W^{(t)},E^{(t)}), \Pi_{V^{\perp}}w_1^{(t)}} = -\sum_{j,\ell}\Sigma_{j,\ell}^{(t)}\dbrack{\nabla_{w_1}\Ecal_{j,3-j}^{(t)},w_1^{(t)}}
    \end{align*} 
    Now from \myref[a]{claim:noise}{Claim} and \myref[a]{induct:phase-2}{Induction}, we know \((B_{j,\ell}^{(t)})^3 \leq \widetilde{O}(\frac{1}{d^{3/2}})\) and the following 
    \begin{align*}
        \Sigma_{j,\ell}^{(t)} = O(\Sigma_{1,1}^{(t)})\frac{(B_{j,\ell}^{(t)})^6 + E_{j,3-j}^{(t)}(B_{3-j,\ell}^{(t)})^3(B_{j,\ell}^{(t)})^3}{(B_{1,1}^{(t)})^6}\frac{\Phi_j^{(t)}}{\Phi_1^{(t)}} \leq  \widetilde{O}(\frac{E_{j,3-j}^{(t)}}{d^{3/2}})\Sigma_{1,1}^{(t)}\frac{\Phi_j^{(t)}}{\Phi_1^{(t)}} \tag*{for any \((j,\ell)\neq (1,1)\)}
    \end{align*}
    From \myref[a,c]{induct:phase-2}{Induction}, we know \(((B_{2,\ell}^{(t)})^3 + E_{2,1}^{(t)}(B_{1,\ell}^{(t)})^3 )^2 \leq \widetilde{O}(\frac{1}{d^{3/2}})E_{2,1}^{(t)}\) and \( R_2^{(t)} = \Theta(1)\), which by \myref[a,b]{claim:Ecal-phase2}{Claim} and \myref[a]{lem:phase2-variables}{Lemma} gives \(\Phi_2^{(t)}/\Phi_1^{(t)} \leq \widetilde{O}(\alpha_1^{O(1)})\). Combine the bounds above, we can obtain \(\Sigma_{j,\ell}^{(t)} = \widetilde{O}(E_{j,3-j}^{(t)}/d^{3/2})\Sigma_{1,1}^{(t)} \). We can then directly apply \myref{claim:noise}{Claim} to prove \myref[a]{lem:reduce-noise-phase2}{Lemma} as follows 
    \begin{align*}
        &\quad \dbrack{-\nabla_{w_1}L(W^{(t)},E^{(t)}), \Pi_{V^{\perp}}w_1^{(t)}} \\
        & = (1\pm \widetilde{O}(E_{1,2}^{(t)}))\Sigma_{1,1}^{(t)}\Big( - \Theta([R_{1}^{(t)}]^3) \pm O(E_{1,2}^{(t)})(\overline{R}_{1,2}^{(t)} + \varrho)[R_{1}^{(t)}]^{3/2}[R_{2}^{(t)}]^{3/2}\Big) \\
        &\quad  +\widetilde{O}(E_{2,1}^{(t)}/d^{3/2})\Sigma_{1,1}^{(t)}\Big(-\Theta((E_{2,1}^{(t)})^2)[R_{1}^{(t)}]^3 \pm O(E_{2,1}^{(t)})(\overline{R}_{1,2}^{(t)} + \varrho)[R_{1}^{(t)}]^{3/2}[R_{2}^{(t)}]^{3/2}\Big) \\
        &= \Theta(\Sigma_{1,1}^{(t)})\Big(-[R_{1}^{(t)}]^3 \pm \widetilde{O}(|E_{1,2}^{(t)}| + \frac{|E_{2,1}^{(t)}|^2}{d^{3/2}})(\overline{R}_{1,2}^{(t)} + \varrho)[R_{1}^{(t)}]^{3/2}[R_{2}^{(t)}]^{3/2}\Big) \tag{Since \(|E_{1,2}^{(t)}| \leq d^{-\Omega(1)}\) by \myref[c,d]{induct:phase-2}{Induction}}
    \end{align*}
    \newline
    \textbf{Proof of (b):} For \myref[b]{lem:reduce-noise-phase2}{Lemma}, we can use the same analysis for \(\Sigma_{1,1}^{(t)}\) above and \myref[(d,e)]{claim:noise}{Claim} to get (again we have used  \(\Sigma_{j,\ell}^{(t)} = \widetilde{O}(E_{j,3-j}^{(t)})\Sigma_{1,1}^{(t)} = o(\Sigma_{1,1}^{(t)})\))
    \begin{align*}
        &\quad\dbrack{-\nabla_{w_1}L(W^{(t)},E^{(t)}), \Pi_{V^{\perp}}w_2   ^{(t)}} \\
        & = (1\pm \widetilde{O}(E_{1,2}^{(t)}))\Sigma_{1,1}^{(t)}\Big((-\Theta(\overline{R}_{1,2}^{(t)}) \pm O(\varrho) )[R_{1}^{(t)}]^{5/2}[R_{2}^{(t)}]^{1/2} + E_{1,2}^{(t)}R_1^{(t)}[R_{2}^{(t)}]^2\Big) \\
        &\quad + \widetilde{O}(E_{2,1}^{(t)}/d^{3/2})\Sigma_{1,1}^{(t)}\Big((-\Theta(\overline{R}_{1,2}^{(t)}) + O(\varrho) )(E_{2,1}^{(t)})^2[R_{1}^{(t)}]^{5/2}[R_{2}^{(t)}]^{1/2} + E_{2,1}^{(t)}R_1^{(t)}[R_{2}^{(t)}]^2\Big)\\
        & = \Sigma_{1,1}^{(t)}((-\Theta(\overline{R}_{1,2}^{(t)}) + O(\varrho) )[R_{1}^{(t)}]^{5/2}[R_{2}^{(t)}]^{1/2} + \widetilde{O}(|E_{1,2}^{(t)}| + \frac{|E_{2,1}^{(t)}|^2}{d^{3/2}})R_1^{(t)}[R_{2}^{(t)}]^2)
    \end{align*}
    \textbf{Proof of (c):} Similarly to the proof of (a), we can also expand as follows 
    \begin{align*}
        &\quad \dbrack{-\nabla_{w_2}L(W^{(t)},E^{(t)}), \Pi_{V^{\perp}}w_2^{(t)}} \\
        & = (1\pm O(E_{1,2}^{(t)}))\Sigma_{1,1}^{(t)}\Big(-[R_{2}^{(t)}]^3\Theta((E_{1,2}^{(t)})^2) \pm O(E_{1,2}^{(t)})(\overline{R}_{1,2}^{(t)} + \varrho)[R_{1}^{(t)}]^{3/2}[R_{2}^{(t)}]^{3/2}\Big) \\
        &\quad - \sum_{\ell\in[2]}\Sigma_{2,\ell}^{(t)}\Big([R_{2}^{(t)}]^3 \pm O(E_{2,1}^{(t)})(\overline{R}_{1,2}^{(t)} + \varrho)[R_{1}^{(t)}]^{3/2}[R_{2}^{(t)}]^{3/2}\Big)\\
        &= - [R_{2}^{(t)}]^3 \Big(\Sigma_{1,1}^{(t)} \Theta((E_{1,2}^{(t)})^2) + \sum_{\ell\in[2]}\Sigma_{2,\ell}^{(t)} \Big)  \pm  O\Big(\sum_{j,\ell}\Sigma_{j,\ell}^{(t)}E_{j,3-j}^{(t)} (\overline{R}_{1,2}^{(t)} + \varrho)[R_{1}^{(t)}]^{3/2}[R_{2}^{(t)}]^{3/2} \Big)
    \end{align*}
    \textbf{Proof of (d):} Similarly, we can calculate (again by \(\Sigma_{j,\ell}^{(t)} = \widetilde{O}(E_{j,3-j}^{(t)})\Sigma_{1,1}^{(t)} = o(\Sigma_{1,1}^{(t)})\))
    \begin{align*}
        &\quad \dbrack{-\nabla_{w_2}L(W^{(t)},E^{(t)}), \Pi_{V^{\perp}}w_1^{(t)}} \\
        & = \sum_{\ell\in[2]}\Sigma_{1,\ell}^{(t)}\Big((-\Theta(\overline{R}_{1,2}^{(t)}) \pm O(\varrho) )(E_{1,2}^{(t)})^2[R_{2}^{(t)}]^{5/2}[R_{1}^{(t)}]^{1/2} + E_{1,2}^{(t)}R_2^{(t)}[R_{1}^{(t)}]^2\Big) \\
        &\quad +  \sum_{\ell\in[2]}\Sigma_{2,\ell}^{(t)}\Big((-\Theta(\overline{R}_{1,2}^{(t)}) \pm O(\varrho) )[R_{2}^{(t)}]^{5/2}[R_{1}^{(t)}]^{1/2} + E_{2,1}^{(t)}R_2^{(t)}[R_{1}^{(t)}]^2\Big) \\
        & = (1\pm \widetilde{O}(E_{1,2}^{(t)}) )\Sigma_{1,1}^{(t)}\Big((-\Theta(\overline{R}_{1,2}^{(t)}) \pm O(\varrho) )(E_{1,2}^{(t)})^2[R_{2}^{(t)}]^{5/2}[R_{1}^{(t)}]^{1/2} + E_{1,2}^{(t)}R_2^{(t)}[R_{1}^{(t)}]^2\Big) \\
        &\quad +  \sum_{\ell\in[2]}\Sigma_{2,\ell}^{(t)}\Big((-\Theta(\overline{R}_{1,2}^{(t)}) \pm O(\varrho) )[R_{2}^{(t)}]^{5/2}[R_{1}^{(t)}]^{1/2} + E_{1,2}^{(t)}R_2^{(t)}[R_{1}^{(t)}]^2\Big) \\
        &= \Big(\Sigma_{1,1}^{(t)} \Theta((E_{1,2}^{(t)})^2) + \sum_{\ell\in[2]}\Sigma_{2,\ell}^{(t)} \Big)(-\Theta(\overline{R}_{1,2}^{(t)}) \pm O(\varrho) )[R_{2}^{(t)}]^{5/2}[R_{1}^{(t)}]^{1/2}  + O\Big(\sum_{j,\ell}\Sigma_{j,\ell}^{(t)}E_{j,3-j}^{(t)}R_{2}^{(t)}[R_{1}^{(t)}]^{2}\Big)
    \end{align*}
    which completes the proof.
\end{proof}

\begin{lemma}[learning feature \(v_2\) in phase II]\label{lem:learning-v2-phase2}
    For each \(t\in [T_1,T_2]\), if Induction~\ref{induct:phase-2} holds at iteration \(t\), then we have for each \(j\in[2]\):
    \begin{align*}
        |\dbrack{-\nabla_{w_j}L(W^{(t)}, E^{(t)}), v_2}| \leq \widetilde{O}(\frac{\alpha_2^6\alpha_1^{6}}{d^{5/2}})\Big(\Phi_j^{(t)}(|E_{j,3-j}^{(t)}| + [R_j^{(t)}]^3)+ \Phi_{3-j}^{(t)}(|E_{3-j,j}^{(t)} |[R_{3-j}^{(t)}]^3 + \frac{|E_{3-j,j}^{(t)}|^2}{d^{3/2}})\Big) 
    \end{align*}
\end{lemma}

\begin{proof}
    Again as in the proof of \myref{lem:learning-v1-phase1}{Lemma}, we expand the notations: (ignoring the superscript \(^{(t)}\) for the RHS)
    \begin{align}
        &\dbrack{-\nabla_{w_j}L(W^{(t)}, E^{(t)}), v_2}  = \Lambda_{j,2}^{(t)} + \Gamma_{j,2}^{(t)} - \Upsilon_{j,2}^{(t)} \label{eqdef:learning-v2-phase2-1} 
    \end{align}
    where
    \begin{align*}
        \Lambda_{j,2}^{(t)} & = C_0\alpha_{2}^6\Phi_j^{(t)} H_{j,1}^{(t)}(B_{j,2}^{(t)})^5 \\
        \Gamma_{j,2}^{(t)} &= C_0\alpha_{2}^6\Phi_{3-j}^{(t)} E_{3-j,j}^{(t)} (B_{3-j,2}^{(t)})^3(B_{j,2}^{(t)})^2H_{3-j,1}^{(t)} \\
        \Upsilon_{j,2}^{(t)} & = C_0\alpha_{1}^6\left(\Phi_j^{(t)}(B_{j,1}^{(t)})^3(B_{j,2}^{(t)})^2K_{j,2}^{(t)} + \Phi_{3-j}^{(t)}E_{3-j,j}^{(t)}(B_{3-j,1}^{(t)})^3(B_{j,2}^{(t)})^2K_{3-j,2}^{(t)}\right)
    \end{align*}
    Now we further write \(\Upsilon_{j,2}^{(t)} = \Upsilon_{j,2,1}^{(t)} + \Upsilon_{j,2,2}^{(t)}\), where 
    \begin{align*}
        \Upsilon_{j,2,1}^{(t)} & = C_0\alpha_{1}^6\Phi_j^{(t)}(B_{j,1}^{(t)})^3(B_{j,2}^{(t)})^2K_{j,2}^{(t)}, \qquad \Upsilon_{j,2,2}^{(t)} = \Phi_{3-j}^{(t)}E_{3-j,j}^{(t)}(B_{3-j,1}^{(t)})^3(B_{3-j,2}^{(t)})^2K_{3-j,2}^{(t)}
    \end{align*}
    According to \eqref{eqdef:learning-v2-phase2-1}, we can first compute 
    \begin{align*}
        \Lambda_{j,2}^{(t)} - \Upsilon_{j,2,1}^{(t)} & = C_0\alpha_{2}^6\Phi_j^{(t)} (B_{j,2}^{(t)})^5H_{j,1}^{(t)} - C_0\alpha_{1}^6\Phi_j^{(t)}(B_{j,1}^{(t)})^3(B_{j,2}^{(t)})^2K_{j,2}^{(t)} \\
        & = C_0\alpha_{2}^6\Phi_j^{(t)} (B_{j,2}^{(t)})^5 \left(C_1 \alpha_{1}^6((B_{j,1}^{(t)})^3 + E_{j,3-j}^{(t)} (B_{3-j,1}^{(t)})^3)^2  + C_2\Ecal_{j,3-j}^{(t)} \right) \\
        &\quad - C_0\alpha_{1}^6\Phi_j^{(t)}(B_{j,1}^{(t)})^3(B_{j,2}^{(t)})^2C_1\alpha_{2}^6((B_{j,2}^{(t)})^3 + E_{j,3-j}^{(t)}(B_{3-j,2}^{(t)})^3)((B_{j,1}^{(t)})^3 + E_{j,3-j}^{(t)}(B_{3-j,1}^{(t)})^3) \\
        & = C_0\alpha_{2}^6C_1\alpha_{1}^6\Phi_j^{(t)}(B_{j,2}^{(t)})^5 \left(  E_{j,3-j}^{(t)} (B_{3-j,1}^{(t)})^3(B_{j,1}^{(t)} )^3 + (E_{j,3-j}^{(t)})^2 (B_{3-j,1}^{(t)})^6\right) \\
        &\quad - C_0\alpha_{2}^6C_1\alpha_{1}^6\Phi_j^{(t)}(B_{j,2}^{(t)})^2 (B_{3-j,2}^{(t)})^3 E_{j,3-j}^{(t)}\left( (B_{j,1}^{(t)})^6 + E_{j,3-j}^{(t)}(B_{3-j,1}^{(t)})^3(B_{j,1}^{(t)})^3\right)\\
        &\quad + C_0\alpha_{2}^6\Phi_j^{(t)} (B_{j,2}^{(t)})^5C_2\Ecal_{j,3-j}^{(t)}
    \end{align*}
    Then we can apply \myref[a,c,d]{induct:phase-2}{Induction}, \myref[a,b]{claim:Ecal-phase2}{Claim} and \myref[a,c]{lem:phase2-variables}{Lemma} to get 
    \begin{align*}
        |\Lambda_{j,2}^{(t)} - \Gamma_{j,2,1}^{(t)}| \leq \widetilde{O}(\frac{\alpha_2^6}{\alpha_1^{6}d^{5/2}})\Phi_j^{(t)}(|E_{j,3-j}^{(t)}| + [R_j^{(t)}]^3)
    \end{align*}
    where the last inequality is due to \myref[a,c]{lem:phase2-variables}{Lemma}. Similarly, we can also compute for \(\Gamma_{j,2}^{(t)} - \Upsilon_{j,2,2}^{(t)}\):
    \begin{align*}
        |\Gamma_{j,2}^{(t)} - \Upsilon_{j,2,2}^{(t)}| &\leq \left| C_0\alpha_{2}^6\Phi_{3-j}^{(t)} E_{3-j,j}^{(t)} (B_{3-j,2}^{(t)})^3(B_{j,2}^{(t)})^2H_{3-j,1}^{(t)}\right| \\
        &\quad + \left|C_0\alpha_{1}^6\Phi_{3-j}^{(t)}E_{3-j,j}^{(t)}(B_{3-j,1}^{(t)})^3(B_{ j,2}^{(t)})^2K_{3-j,2}^{(t)} \right| \\
        & \leq \widetilde{O}(\frac{\alpha_1^{6}\alpha_2^6}{d^{5/2}})\Phi_{3-j}^{(t)} |E_{3-j,j}^{(t)}|([R_{3-j}^{(t)}]^3 + \frac{|E_{3-j,j}^{(t)}|}{d^{3/2}})
    \end{align*}
    This completes the proof
\end{proof}

\begin{lemma}[learning feature \(v_1\) in Phase II]\label{lem:learning-v1-phase2}
    For each \(t\in [T_1,T_2]\), if Induction~\ref{induct:phase-2} holds at iteration \(t\), then we have:
    \begin{enumerate}[(a)]
        \item \(\dbrack{-\nabla_{w_1}L(W^{(t)},E^{(t)}),v_1} = \Theta(\Sigma_{1,1}^{(t)})[R_1^{(t)}]^3 +\Gamma_{1,1}^{(t)} \pm \widetilde{O}(\alpha_1^{O(1)}/d^{5/2}) \);
        \item \(\dbrack{-\nabla_{w_2}L(W^{(t)},E^{(t)}),v_1} =  \widetilde{O}(\alpha_1^{O(1)}/d^{5/2}) + \widetilde{O}(\frac{\alpha_1^{6}}{d})E_{1,2}^{(t)}\Phi_1^{(t)}[R_1^{(t)}]^3  \)
    \end{enumerate}
\end{lemma}

\begin{proof}
    As in the proof of \myref{lem:learning-v2-phase2}{Lemma}, we expand the gradient terms:
    \begin{align}
        &\dbrack{-\nabla_{w_j}L(W^{(t)}, E^{(t)}), v_1}  = \Lambda_{j,2}^{(t)} + \Gamma_{j,2}^{(t)} - \Upsilon_{j,2}^{(t)} \label{eqdef:learning-v1-phase2-1} 
    \end{align}
    where 
    \begin{align*}
        \Lambda_{j,1}^{(t)} & = C_0\alpha_{1}^6\Phi_j^{(t)} H_{j,2}^{(t)}(B_{j,1}^{(t)})^5 \\
        \Gamma_{j,1}^{(t)} &= C_0\alpha_{1}^6\Phi_{3-j}^{(t)} E_{3-j,j}^{(t)} (B_{3-j,1}^{(t)})^3(B_{j,1}^{(t)})^2H_{3-j,2}^{(t)} \\
        \Upsilon_{j,1}^{(t)} & = C_0\alpha_{1}^6\left(\Phi_j^{(t)}(B_{j,2}^{(t)})^3(B_{j,1}^{(t)})^2K_{j,1}^{(t)} + \Phi_{3-j}^{(t)}E_{3-j,j}^{(t)}(B_{3-j,2}^{(t)})^3(B_{j,1}^{(t)})^2K_{3-j,1}^{(t)}\right)
    \end{align*}
    Indeed, when \(j=1\), by \myref[a]{induct:phase-2}{Induction} and \myref[a,c]{lem:phase2-variables}{Lemma}, we can compute 
    \begin{align*}
        \Lambda_{1,1}^{(t)} &= C_0\alpha_{1}^6\Phi_{1}^{(t)}(B_{1,1}^{(t)})^5H_{1,2}^{(t)} = \Theta(\Sigma_{1,1}^{(t)})[R_1^{(t)}]^3
    \end{align*}
    and with additionally \myref[b]{lem:phase2-variables}{Lemma}, we also have
    \begin{align*}
        |\Upsilon_{1,1}^{(t)}| & = \left|C_0\alpha_{1}^6\left(\Phi_1^{(t)}(B_{1,2}^{(t)})^3(B_{1,1}^{(t)})^2K_{1,1}^{(t)} + \Phi_{2}^{(t)}E_{2,j}^{(t)}(B_{2,2}^{(t)})^3(B_{1,1}^{(t)})^2K_{2,1}^{(t)}\right)\right| \leq \widetilde{O}(\frac{\alpha_1^{O(1)}}{d^{5/2}})
    \end{align*}
    which gives the proof of (a). For (b), we can also apply \myref[a]{induct:phase-2}{Induction} and \myref[a,c]{lem:phase2-variables}{Lemma} to get 
    \begin{align*}
        \Lambda_{2,1}^{(t)} & = C_0\alpha_{1}^6\Phi_2^{(t)} H_{2,2}^{(t)}(B_{2,1}^{(t)})^5 \leq \widetilde{O}(\alpha_1^{O(1)}/d^{5/2})\\
        \Gamma_{2,1}^{(t)} &= C_0\alpha_{1}^6\Phi_{1}^{(t)} E_{1,2}^{(t)} (B_{1,1}^{(t)})^3(B_{2,1}^{(t)})^2H_{1,2}^{(t)} \leq \widetilde{O}(\frac{1}{d})E_{1,2}^{(t)}\Phi_1^{(t)}\frac{[R_1^{(t)}]^3}{\alpha_1^{6}} \\
        \Upsilon_{2,1}^{(t)} & = C_0\alpha_{1}^6\left(\Phi_2^{(t)}(B_{2,2}^{(t)})^3(B_{2,1}^{(t)})^2K_{2,1}^{(t)} + \Phi_{1}^{(t)}E_{1,2}^{(t)}(B_{1,2}^{(t)})^3(B_{2,1}^{(t)})^2K_{1,1}^{(t)}\right) \leq \widetilde{O}(\frac{\alpha_1^6}{d^{4}})
    \end{align*}
    this finishes the proof.
\end{proof}

\subsection{At the End of Phase II}

Now we shall present the main theorem of this section, which gives the result of prediction head \(E_{2,1}^{(t)}\) growth after the feature \(v_1\) is learned in the first stage. 

\begin{lemma}[Phase II]\label{lem:phase-2}
    Suppose \(\eta = \frac{1}{\poly(d)}\) is sufficiently small, then \myref{induct:phase-2}{Induction} holds for all iteration \(t \in [T_1,T_2]\), and at iteration \(t= T_2\), the followings holds:
    \begin{enumerate}[(a)]
        \item \(B_{1,1}^{(T_2)} = \Theta(1)\), \(B_{j,\ell}^{(T_2)} = B_{j,\ell}^{(T_1)}(1\pm o(1)) = \widetilde{\Theta}(\frac{1}{\sqrt{d}})\) for \((j,\ell)\neq (1,1)\)
        \item \(R_1^{(T_2)}\leq \widetilde{O}(\frac{1}{d^{3/4}})\), \(R_2^{(T_2)} = \Theta(\sqrt{\eta_E/\eta})\), and \(\overline{R}_{1,2}^{(T_2)}\leq \widetilde{O}(\varrho+\frac{1}{\sqrt{d}})\);
        \item \(|E_{1,2}^{(T_2)}| = \widetilde{O}(\varrho+\frac{1}{\sqrt{d}})[R_1^{(t)}]^{3/2}[R_2^{(t)}]^{3/2}\) and \(|E_{2,1}^{(T_2)}| = \Theta(\sqrt{\eta_E/\eta})\)
    \end{enumerate}
    Where the part of learning \(E_{2,1}^{(t)}\) is what we called \textit{substitution effect}. One can easily verify that \(|E_{2,1}^{(t)} f_1(X^{(1)})| \gg |f_2(X^{(1)})| \) when \(X\) is equipped with feature \(v_1\), as stated in \myref{lem:3-substitute}{Lemma}.
\end{lemma}

\begin{proof}
    We first will prove \myref{induct:phase-2}{Induction} holds for all iteration \(t\in [T_1, T_{2}]\). We shall first prove that if \myref{induct:phase-2}{Induction} continues to hold when \(R_2^{(t)}\geq |E_{2,1}^{(t)} |\), we shall have \([R_1^{(t)}]\) decreasing at an exponential rate.
    \newline
    \textbf{Proof of the decrease of \(R_1^{(t)}\):} Firstly, we write down the update of \(R_1^{(t)}\) using \myref[a]{lem:reduce-noise-phase2}{Lemma}:
    \begin{align*}
        R_{1}^{(t+1)} = R_{1}^{(t)} + \eta \Sigma_{1,1}^{(t)}\Theta( - [R_{1}^{(t)}]^3 \pm O(|E_{1,2}^{(t)}| + \frac{|E_{2,1}^{(t)}|^2}{d^{3/2}})(\overline{R}_{1,2}^{(t)} + \varrho)[R_{1}^{(t)}]^{3/2}[R_{2}^{(t)}]^{3/2})
    \end{align*}
    from the expression of \(\Sigma_{1,1}^{(t)}\) in \eqref{eqdef:weight-grad}, and by \myref[a]{induct:phase-2}{Induction} and \myref[a,c]{lem:phase2-variables}{Lemma}, we can compute
    \begin{align*}
        \Sigma_{1,1}^{(t)} = \Theta(C_0C_2\Phi_1^{(t)}) = \Theta(\frac{C_0C_2}{\alpha_1^{12}})
    \end{align*}
    Moreover, from \myref[c]{induct:phase-2}{Induction} we know that 
    \begin{align*}
        (|E_{1,2}^{(t)}| + \frac{|E_{2,1}^{(t)}|^2}{d^{3/2}})[R_{1}^{(t)}]^{3/2}[R_{2}^{(t)}]^{3/2} &\leq (\widetilde{\Theta}(\frac{1}{d^{3/2}})+ \widetilde{O}(\varrho + \frac{1}{\sqrt{d}})[R_1^{(t)}]^{3/2})[R_{1}^{(t)}]^{3/2}[R_{2}^{(t)}]^{3/2} \\
        &\leq (\widetilde{\Theta}(\frac{1}{d^{3/2}}) + \widetilde{O}(\varrho + \frac{1}{\sqrt{d}})[R_1^{(t)}]^{3/2})[R_{1}^{(t)}]^{3/2}
    \end{align*}
    Therefore whenever \(R_1^{(t)}\geq \frac{\alpha_1^{18}}{d^{3/4}}\) (which \(t\leq T_2\) suffices), we shall have always have 
    \begin{align*}
        (\overline{R}_{1,2}^{(t)} + \varrho)(\widetilde{\Theta}(\frac{1}{d^{3/2}}) + \widetilde{O}(\varrho + \frac{1}{\sqrt{d}})[R_1^{(t)}]^{3/2})[R_{1}^{(t)}]^{3/2} \leq o([R_1^{(t)}]^3)
    \end{align*}
    which implies, if we set \(T'_2 := \min\{t:R_1^{(t)}\geq \frac{1}{d^{3/4}\alpha_1^2}\}\), then for all \(t \in [T_1, T'_2]\), we will have
    \begin{align}
        R_{1}^{(t+1)} &= R_{1}^{(t)} + \eta \Sigma_{1,1}^{(t)}\Theta( - [R_{1}^{(t)}]^3 \pm O(|E_{1,2}^{(t)}| + \frac{|E_{2,1}^{(t)}|^2}{d^{3/2}})(\overline{R}_{1,2}^{(t)} + \varrho)[R_{1}^{(t)}]^{3/2}[R_{2}^{(t)}]^{3/2}) \nonumber\\
        &= R_{1}^{(t)} - \Theta(\eta\Sigma_{1,1}^{(t)})[R_{1}^{(t)}]^3 \label{eqdef:lem:phase-2-R1-update} \\
        &\leq R_{1}^{(t)} (1 - \Theta(\frac{\eta C_0C_2}{\alpha_1^{12}}) \frac{1}{d^{3/2}\alpha_1^2} )\tag{since \(R_1^{(t)}\geq \frac{1}{d^{3/4}}\)}\nonumber
    \end{align}
    From the last inequality we know that after \(T_2 = T_1 + \widetilde{\Theta}(\frac{d^{1.5}}{\eta \alpha_1^{\Omega(1)}})\), we shall have \( R_1^{(t)} \leq O(\frac{\alpha_1^{O(1)}}{d^{3/4}})\). Moreover, suppose \(T'_2 < T_2\), (which just mean \(R_1^{(s)}\leq O(\frac{1}{d^{3/4}\alpha_1^2})\) for some iteration \(s\in[T_1,T_2]\)) we also have 
    \begin{align*}
        R_{1}^{(t+1)} &= R_{1}^{(t)} - \Theta(\eta\Sigma_{1,1}^{(t)})[R_{1}^{(t)}]^3 \\
        & \geq R_{1}^{(t)}(1 - \Theta(\frac{\eta C_0C_2}{\alpha_1^{14}}) \frac{1}{d^{3/2}} ) 
    \end{align*}
    So when \(T_2 \leq T_1 +\widetilde{O}(\frac{d^{1.5}\alpha_1^{12}}{\eta})\) iterations, we will have \(R_1^{(t)} \geq R_1^{(s)}(1 - \Theta(\frac{\eta C_0C_2}{d^{3/2}\alpha_1^{14}}) )^{T_2-T_1} \geq \Omega(R_1^{(t)}) \) for all \(t \in [s,T_2]\), which means we have a lower bound \(R_1^{(t)}\geq \frac{1}{d^{3/4}\alpha_1^2}\) throughout \(t\in[T_1,T_2]\). This proves \myref[a]{lem:phase-2}{Lemma} and also our induction on \(R_1^{(t)}\).\\
    \newline 
    \textbf{Proof of induction for \(E_{1,2}^{(t)}\):} By \myref[a]{lem:learning-pred-head-phase2}{Lemma}, we can write 
    \begin{align*}
        -\nabla_{E_{1,2}}L(W^{(t)},E^{(t)}) &= (1 + \widetilde{O}(\frac{\alpha_1^{O(1)}}{d^{3/2}}))\Sigma_{1,1}^{(t)} (-2E_{1,2}^{(t)}[R_{2}^{(t)}]^{3} \pm O(\overline{R}_{1,2}^{(t)} + \varrho)[R_{1}^{(t)}]^{3/2}[R_{2}^{(t)}]^{3/2}) \\
        &\qquad \pm \Sigma_{1,1}^{(t)} \widetilde{O}(\frac{\eta_E/\eta}{\sqrt{d}})\max\{[R_1^{(t)}]^3,\frac{\alpha_1^{O(1)}}{d^{5/2}}\}  \\
        & = -\Theta(\Sigma_{1,1}^{(t)}[R_2^{(t)}]^3)E_{1,2}^{(t)} \pm O(\Sigma_{1,1}^{(t)})\Big((\overline{R}_{1,2}^{(t)} + \varrho)[R_{1}^{(t)}]^{3/2}[R_{2}^{(t)}]^{3/2} + \widetilde{O}(\frac{\eta_E/\eta}{\sqrt{d}})[R_1^{(t)}]^3 \Big)
    \end{align*}
    Since again from \myref[b,c]{induct:phase-2}{Induction} that \(\overline{R}_{1,2}^{(t)}\leq \widetilde{O}(\varrho+\frac{1}{\sqrt{d}}), R_1^{(t)}= O(1), R_2^{(t)} \in [\sqrt{\eta_E/\eta},O(1)]\), we can obtain the update of \(E_{1,2}^{(t)}\) as 
    \begin{align}
        E_{1,2}^{(t+1)} &= E_{1,2}^{(t)}(1 - \Theta(\eta_E\Sigma_{1,1}^{(t)}[R_2^{(t)}]^3)) \pm \widetilde{O}(\eta_E\Sigma_{1,1}^{(t)})\Big((\varrho+\frac{1}{\sqrt{d}})[R_{1}^{(t)}]^{3/2}[R_{2}^{(t)}]^{3/2} + \widetilde{O}(\frac{\eta_E/\eta}{\sqrt{d}})[R_1^{(t)}]^3 \Big) \nonumber\\
        & = E_{1,2}^{(t)}(1 - \Theta(\eta_E\Sigma_{1,1}^{(t)}[R_2^{(t)}]^3)) \pm \widetilde{O}(\varrho+\frac{1}{\sqrt{d}})\eta_E\Sigma_{1,1}^{(t)}[R_1^{(t)}]^{3/2} \nonumber\\
        & = E_{1,2}^{(t)}(1 - \Theta(\eta_E\Sigma_{1,1}^{(t)}[R_2^{(t)}]^3)) \pm \eta_E\Sigma_{1,1}^{(t)}J_{1,2}^{(t)} \nonumber
    \end{align}
    where \(J_{1,2}^{(t)} = \widetilde{C}(\varrho+\frac{1}{\sqrt{d}})[R_1^{(t)}]^{3/2}>0\) and \(\widetilde{C} = \widetilde{\Theta}(1)\) is larger than the hidden constant (including the \(\polylog (d)\) factors) of \(E_{2,1}^{(T_1)} \leq \widetilde{O}(\varrho+\frac{1}{\sqrt{d}})\) in \myref[d]{lem:phase-1}{Lemma}. And then we can compute
    \begin{align*}
        J_{1,2}^{(t+1)}&= \widetilde{C} (\varrho+\frac{1}{\sqrt{d}})[R_1^{(t+1)}]^{3/2} \\
        &= \widetilde{C}(\varrho+\frac{1}{\sqrt{d}})[R_{1}^{(t)}]^{3/2} (1 - \Theta(\eta\Sigma_{1,1}^{(t)})[R_1^{(t)}]^{2})^{3/2} \tag{due to calculations in \eqref{eqdef:lem:phase-2-R1-update}}\\
        &= J_{1,2}^{(t)}(1 - \Theta(\eta^{3/2}(\Sigma_{1,1}^{(t)})^{3/2})[R_1^{(t)}]^{3}) \tag{because \(\eta\Sigma_{1,1}^{(t)} = \frac{\alpha_1^{O(1)}}{\poly(d)}\) is very small}
    \end{align*}
    Now by \myref[d]{lem:phase-1}{Lemma}, we know \(|E_{1,2}^{(T_1)}|\leq J_{1,2}^{(T_1)}\); then we begin our induction that \(|E_{1,2}^{(t)}| < (\log\log d)J_{1,2}^{(t)}\) at for all iterations \(t \in [T_1,T_2]\). Now assume we have \(|E_{1,2}^{(t)}| = \frac{1}{2}(\log\log d)J_{1,2}^{(t)}\)\footnote{If we want \(|E_{1,2}^{(t)}| > (\log\log d)J_{1,2}^{(t)}\), then as long as \(\eta = \frac{1}{\poly(d)}\) is small enough, we can always assume to have found some iteration \(t'\in(T_1, t]\) such that  \(|E_{1,2}^{(t')}| = \frac{1}{2}(\log\log d)J_{1,2}^{(t)}\), and we set \(t = t'\) and start our argument from that iteration.}, from above calculations it holds that \(|E_{1,2}^{(t+1)}| = |E_{1,2}^{(t)}|(1 - \Theta(\eta\Sigma_{1,1}^{(t)}[R_1^{(t)}]^{3}))\). Then we would have
    \begin{align*}
        \frac{J_{1,2}^{(t+1)}}{J_{1,2}^{(t)}} \geq (1 - \Theta(\eta^{3/2}(\Sigma_{1,1}^{(t)})^{3/2})[R_1^{(t)}]^{3}) \geq (1 - \Theta(\eta_E\Sigma_{1,1}^{(t)}[R_2^{(t)}]^3)) \geq \frac{|E_{1,2}^{(t+1)}|}{|E_{1,2}^{(t)}|} \tag{because of the range of \(R_1^{(t)}\) and \(R_2^{(t)}\)}
    \end{align*}
    This proved that \(|E_{1,2}^{(t+1)}| \lesssim \log\log d\cdot J_{1,2}^{(t+1)} \leq \widetilde{O}(\varrho+\frac{1}{\sqrt{d}})[R_{1}^{(t+1)}]^{3/2}\) and also the induction can go on until \(t = T_2\).\\
    \newline
    \textbf{Proof of the growth of \(E_{2,1}^{(t)}\) and \(T_2 \leq T_1+ O(\frac{d^{1.5}}{\eta \alpha_1^4})\):} According to \myref[b]{lem:learning-pred-head-phase2}{Lemma}, we can write down the update of \(E_{2,1}^{(t)}\) as
    \begin{align*}
        -\nabla_{E_{2,1}}L(W^{(t)},E^{(t)}) &= (1 \pm O(\frac{\alpha_1^{O(1)}}{d^{3/2}})) \Delta_{2,1}^{(t)}  \\
        &\quad \pm O(\Sigma_{2,1}^{(t)})(|E_{2,1}^{(t)}|[R_{1}^{(t)}]^{3} \pm O(\overline{R}_{1,2}^{(t)} + \varrho)[R_{1}^{(t)}]^{3/2}[R_{2}^{(t)}]^{3/2})
    \end{align*}
    Then, from \myref[a,c]{lem:phase2-variables}{Lemma} and \myref{induct:phase-2}{Induction}, we have 
    \begin{align*}
        &\quad O(\Sigma_{2,1}^{(t)})(|E_{2,1}^{(t)}|[R_{1}^{(t)}]^{3} \pm O(\overline{R}_{1,2}^{(t)} + \varrho)[R_{1}^{(t)}]^{3/2}[R_{2}^{(t)}]^{3/2}) \leq   O(\frac{\polylog(d)}{d^{3/2}\alpha_1^2})\Phi_2^{(t)} \leq O(\frac{1}{d^{3/2}\alpha_1})\Phi_2^{(t)}
    \end{align*}
    and also 
    \begin{align*}
        \left|(1 \pm \widetilde{O}(\frac{\alpha_1^6}{d^{0.3}})) C_0\Phi_2^{(t)} \alpha_{1}^6(B_{2,1}^{(t)})^3(B_{1,1}^{(t)})^3H_{2,2}^{(t)}\right| \geq \widetilde{\Theta}(\frac{\alpha_1^6}{d^{3/2}})\Phi_2^{(t)}
    \end{align*}
    Now by \myref[a]{lem:phase2-variables}{Lemma} and \myref[a]{induct:phase-2}{Induction}, it allow us to simplify the update to
    \begin{align*}
        E_{2,1}^{(t+1)} &= E_{2,1}^{(t)}  -\eta_E \nabla_{E_{2,1}}L(W^{(t)},E^{(t)}) \\
        &= E_{2,1}^{(t)} +  (1 \pm \frac{1}{\alpha_1^{\Omega(1)}})\eta_E C_0C_2\alpha_1^6 \Phi_2^{(t)}(B_{2,1}^{(t)})^3(B_{1,1}^{(t)})^3\Ecal_{2,1}^{(t)}  \\
        &\geq E_{2,1}^{(t)} + \eta_E \widetilde{\Theta}(\frac{1}{d^{3/2}\alpha_1^{6}})\sign(B_{1,1}^{(t)})\sign(B_{2,1}^{(t)}) \tag{by \myref{induct:phase-2}{Induction} and \myref{claim:Ecal-phase2}{Claim}}
    \end{align*}
    Now since \(\sign(B_{j,1}^{(t)}) = \sign(B_{j,1}^{(T_1)})\), we know there is an iteration \(T'_{2,1} \leq T_1 + O(\frac{d^{1/2}\alpha_1^{O(1)}}{\eta})\) such that for all \(t \in [T'_{2,1},T_2]\), it holds
    \begin{align*}
        |E_{2,1}^{(t)}| &= \left| E_{2,1}^{(T_1)} + \sum_{t \in [T_1,T'_{2,1}]}\Theta(\eta_E C_0C_2\alpha_1^6 ) \Phi_2^{(t)}(B_{2,1}^{(t)})^3(B_{1,1}^{(t)})^3[R_2^{(t)}]^3\right| \\
        &= \left||E_{2,1}^{(T_1)}| \pm  \sum_{s \in [T_1,T'_{2,1}]}\eta_E \widetilde{\Theta}(\frac{1}{d^{3/2}\alpha_{1}^{O(1)}})\right| \\
        &\in \left[ 2|E_{2,1}^{(T_1)}|, \widetilde{O}(\frac{\alpha_{1}^{O(1)}}{d})\right]
    \end{align*}
    and thus \(\sign(E_{2,1}^{(t)}) = \prod_{j\in[2]}\sign(B_{j,1}^{(t)})\) and \(|E_{2,1}^{(t)}|\) will be increasing during \(t \in [T'_{2,1},T_2]\). Thus as long as \(R_2^{(t)}\geq |E_{2,1}^{(t)}|\) continues to hold, after at most \(\widetilde{\Theta}(\frac{d^{1.5}}{\eta \alpha_1^6}) \) iterations starting from \(T_1\), we shall have \(|E_{2,1}^{(t)}|\geq \Omega(\sqrt{\eta_E/\eta})\). 

    However, in order to actually prove \(|E_{2,1}^{(T_2)}| = \Theta(\sqrt{\eta_E/\eta})\), we will need to ensure that (1) there exist some constant \(C=\Omega(\sqrt{\eta_E/\eta})\) such that \(|E_{2,1}^{(t)}| > C\) while \(R_2^{(s)} \geq \frac{1}{\log d}|E_{2,1}^{(t)}|\) for all \(s \in [T_1,t]\); (2) we shall have a upper bound \(|E_{2,1}^{(t)}| < O(\sqrt{\eta_E/\eta})\). They will be done below.\\
    \newline
    \textbf{Proof of \(E_{2,1}^{(T_2)} = \Theta(\sqrt{\eta_E/\eta})\) and \(T_2 = T_1+\widetilde{O}(\frac{d^{3/2}\alpha_1^{O(1)}}{\eta})\):} In fact, \myref[c]{induct:phase-2}{Induction} are already proved since we have already calculated the dynamics of \(R_1^{(t)}\) and its upper bound and lower bound. In this part we are going to prove \(T_2 = T_1 + \widetilde{\Theta}(\frac{d^{1.5}\alpha_1^{12}}{\eta})\) (which means that \(R_2^{(t)} \leq |E_{2,1}|\) can be achieved in \(\widetilde{O}(\frac{d^{3/2}\alpha_1^{12}}{\eta})\) many iterations). From \myref[c]{lem:reduce-noise-phase2}{Lemma}, we can write down the update for \(R_{2}^{(t)}\) as 
    \begin{align*}
        R_2^{(t+1)} &= R_2^{(t)} - 2\eta \dbrack{\nabla_{w_2}L(W^{(t)},E^{(t)}), \Pi_{V^{\perp}}w_2^{(t)}} + \eta^2 \|\Pi_{V^{\perp}}\nabla_{w_2}L(W^{(t)},E^{(t)})\|_2^2  \\
        &= R_2^{(t)} - \eta\Theta([R_{2}^{(t)}]^3) \Big(\Sigma_{1,1}^{(t)} \Theta((E_{1,2}^{(t)})^2) + \sum_{\ell\in[2]}\Sigma_{2,\ell}^{(t)} \Big) \\
        &\quad \pm  \eta O\Big(\sum_{j,\ell}\Sigma_{j,\ell}^{(t)}E_{j,3-j}^{(t)} (\overline{R}_{1,2}^{(t)} + \varrho)[R_{1}^{(t)}]^{3/2}[R_{2}^{(t)}]^{3/2}\Big) + \frac{\eta}{\poly(d)}
    \end{align*}
    where we have used the fact that \(\|\Pi_{V^{\perp}}\nabla_{w_2}L(W^{(t)},E^{(t)})\|_2^2 \leq \widetilde{O}(d^2)\) from our assumption on the noise \(\xi_p\) and a simple bound for \(\Sigma_{j,\ell}^{(t)}\) as we have done before. Next we can resort to \myref[d]{induct:phase-2}{Induction} that \(|E_{1,2}^{(t)}| \leq \widetilde{O}(\varrho+\frac{1}{\sqrt{d}})[R_1^{(t)}]^{3/2}\) to derive
    \begin{align*}
        \sum_{s\in [T_1,t]}\eta\Sigma_{1,1}^{(s)} \Theta((E_{1,2}^{(s)})^2) &\leq \sum_{s\in [T_1,t]}\widetilde{O}(\varrho^2+\frac{1}{d})\eta \Sigma_{1,1}^{(s)}[R_1^{(s)}]^{3} \\
        &\leq \widetilde{O}(\varrho^2+\frac{1}{d}) = o(1) 
    \end{align*}
    which is because \(\sum_{t\in[T_1,T_2]}\Theta(\eta\Sigma_{1,1}^{(t)})[R_1^{(t)}]^3 \leq O(1)\) and \(\Sigma_{1,1}^{(t)}>0\) as we have calculated in the proof of \myref[a]{induct:phase-2}{Induction} above. Similarly, we can also bound
    \begin{align*}
        &\sum_{s\in [T_1,t]}\Sigma_{1,\ell}^{(s)}|E_{1,2}^{(s)}|(|\overline{R}_{1,2}^{(s)}| + \varrho) [R_{1}^{(s)}]^{3/2}[R_{2}^{(s)}]^{3/2}\leq \sum_{s\in [T_1,t]}\widetilde{O}(\varrho^2+\frac{1}{d})\eta\Sigma_{1,\ell}^{(s)}[R_1^{(s)}]^{3} \leq \widetilde{O}(\varrho+\frac{1}{\sqrt{d}}) = o(1)
    \end{align*}
    Moreover, because \(T_2 \leq T_1 + \widetilde{O}(\frac{d^{3/2}\alpha_1^{12}}{\eta})\) and \(|E_{2,1}^{(t)}|\leq O(1)\), \(\Phi_2^{(t)} \leq \alpha_1^{O(1)}\) from \myref{induct:phase-2}{Induction}, we have for each \(t \leq T_2\):
    \begin{align*}
        \sum_{s\in [T_1,t]}\eta  \Sigma_{2,\ell}^{(s)}|E_{2,1}^{(s)}|(|\overline{R}_{1,2}^{(s)}| + \varrho) [R_{1}^{(s)}]^{3/2}[R_{2}^{(s)}]^{3/2}&\leq \widetilde{O}(\frac{ |E_{2,1}^{(s)}|^2}{d^{3/2}}) \sum_{s\in [T_1,t]}\eta\Phi_2^{(s)}\widetilde{O}(\varrho+\frac{1}{\sqrt{d}}) \\
        &\leq \widetilde{O}(\frac{\eta }{d^{3/2}}) \cdot\widetilde{O}(\varrho+\frac{1}{\sqrt{d}})\cdot \widetilde{O}(\frac{d^{3/2}\alpha_1^{12}}{\eta}) \\
        &\leq \widetilde{O}(\varrho+\frac{1}{\sqrt{d}})\alpha_1^{O(1)} = o(1)
    \end{align*}
    Thus combining all the bounds above, we have proved that for each \(t \in [T_1,T_2]\), it holds 
    \begin{align}
        R_2^{(t)} &= R_2^{(T_1)}  - \sum_{s\in [T_1,t]}\Theta(\eta\Sigma_{2,1}^{(t)})[R_{2}^{(t)}]^3 \pm o(1) \nonumber \\
        &= R_2^{(T_1)} -\sum_{s\in [T_1,t]}\Theta(\eta C_0C_2) E_{2,1}^{(t)} \alpha_{1}^6\Phi_2^{(t)}(B_{2,1}^{(t)})^3(B_{1,1}^{(t)})^3[R_{2}^{(t)}]^3\pm o(1) \\
        & = R_2^{(T_1)} -\sum_{s\in [T_1,t]} \eta E_{2,1}^{(t)}\widetilde{\Theta}(\frac{1}{d^{3/2}})\Phi_2^{(t)}[R_{2}^{(t)}]^3\cdot\sign(E_{2,1}^{(t)})\cdot \sign(B_{2,1}^{(T_1)})\cdot\sign(B_{1,1}^{(T_1)})\pm o(1) 
    \end{align}
    where the last equality is because \(\sign(B_{j,\ell}^{(t)}) \equiv\sign(B_{j,\ell}^{(T_1)})\) by \myref[a]{induct:phase-2}{Induction}.
    Now from what we have proved above on the growth of \(E_{2,1}^{(t)}\) that \(\sign(E_{2,1}^{(t)}) = \sign(B_{1,1}^{(t)}B_{2,1}^{(t)}) \equiv \sign(B_{1,1}^{(T_1)}B_{2,1}^{(T_1)}) \) throughout the rest of phase II (which is just \(t\in[T'_{2,1},T_2]\)). Recall that 
    \begin{align*}
        R_2^{(T'_{2,1})} = R_2^{(T_1)} \pm o(1),\quad \text{and}\quad E_{2,1}^{(t)} - E_{2,1}^{(T'_{2,1})} = \sum_{s \in [T'_{2,1},t]}\Theta(\eta_E C_0C_2) \Phi_2^{(s)}(B_{2,1}^{(s)})^3(B_{1,1}^{(s)})^3 
    \end{align*}
    The above arguments imply for \(t\in[T'_{2,1},T_2]\):
    \begin{align*}
        R_2^{(t+1)} & = R_2^{(T_1)} - \sum_{s\in [T'_{2,1},t]}\Theta(\eta C_0C_2)E_{2,1}^{(s)} \Phi_2^{(s)}(B_{2,1}^{(s)})^3(B_{1,1}^{(s)})^3[R_{2}^{(t)}]^3 \pm o(1) \\
        &= R_2^{(T_1)} - \Theta(\frac{\eta}{\eta_E}|E_{2,1}^{(t)}|^2)- o(1) 
    \end{align*}
    Now we can confirm 
    \begin{enumerate}[(1)]
        \item there exist a constant \(C = \Theta(\sqrt{\eta_E/\eta})\) such that \(E_{2,1}^{(t)} = C\) if \(R_2^{(t)}\) falls below \( \frac{1}{\log d}|E_{2,1}^{(t)}|\);
        \item \(T_2 = T_1 + \widetilde{\Theta}(\frac{d^{3/2}\alpha_1^{12}}{\eta})\) due to the growth \(|E_{2,1}^{(t+1)}| = |E_{2,1}^{(t)}| + \eta_E \widetilde{\Theta}(\frac{1}{d^{3/2}\alpha_1^{12}\sqrt{\eta_E/\eta}})\) for \(t\in[T'_{2,1},T_2]\).
    \end{enumerate}
    which are the desired results.
    \\
    \newline
    \textbf{Proof of \myref[a]{induct:phase-2}{Induction}:} We first obtain from \myref[a]{lem:learning-v1-phase2}{Lemma} that the update of \(B_{1,1}^{(t)}\) can be written as
    \begin{align*}
        B_{1,1}^{(t+1)} = B_{1,1}^{(t)} + \eta \left(\Theta(\Sigma_{1,1}^{(t)})\sign(B_{1,1}^{(t)}) [R_1^{(t)}]^3 + \Gamma_{1,1}^{(t)} \pm \widetilde{O}(\alpha_1^{O(1)}/d^{5/2})\right)
    \end{align*}
    Now by what we have calculated above in \eqref{eqdef:lem:phase-2-R1-update}, the total decrease of \(R_1^{(t)}\) is (since \(R_1^{(t)}\) is monotone in this phase)
    \begin{align*}
        \sum_{t\in[T_1,T_2]}\Theta(\eta\Sigma_{1,1}^{(t)})[R_1^{(t)}]^3 \leq O(R_1^{(T_1)} - R_1^{(T_2)})  \leq O(1)
    \end{align*}
    And also since \(T_2 \leq T_1 + \widetilde{\Theta}(\frac{d^{3/2}\alpha_1^{12}}{\eta})\), we can bound
    \begin{align*}
        &\sum_{t\in[T_1,T_2]}\widetilde{O}(\alpha_1^{6}/d^{5/2}) \leq \widetilde{O}(\alpha_1^{O(1)}/d^{5/2})\cdot\widetilde{O}(\frac{d^{3/2}}{\eta \alpha_1^6}) \leq \widetilde{O}(\alpha_1^{O(1)}/d)
    \end{align*}
    Now we consider how the \(\Gamma_{1,1}^{(t)}\) term accumulates
    \begin{align*}
        \sum_{t\in[T_1,T_2]}\eta\Gamma_{1,1}^{(t)} & = \Bigg(\sum_{t\in[T_1,T'_{2,1}]}+\sum_{t\in[T'_{2,1},T_2]}\Bigg) \eta C_0\alpha_1^6 E_{2,1}^{(t)}\Phi_2^{(t)}(B_{2,1}^{(t)})^3(B_{1,1}^{(t)})^2H_{2,2}^{(t)} \\
        & \stackrel{\text{\ding{172}}}= \widetilde{O}(\frac{\alpha_1^{12}}{d}) + \sum_{t\in[T'_{2,1},T_2]} O\left(\eta C_0\alpha_1^6 \Phi_2^{(t)}|B_{2,1}^{(t)}|^3|B_{1,1}^{(t)}|^3H_{2,2}^{(t)} \right)\sign(B_{1,1}^{(t)})  \\
        & = \pm o(1) + O(1)\sign(B_{1,1}^{(t)})
    \end{align*}
    where in \ding{172} we have used \(|E_{2,1}^{(t)}|\leq O(1)\leq O(B_{1,1}^{(t)})\) and \(\sign(E_{2,1}^{(t)}) = \prod_{j\in[2]}\sign(B_{j,1}^{(t)})\) when \(t \in [T'_{2,1},T_2]\). These calculations tell us \(B_{1,1}^{(t)} = B_{1,1}^{(T_1)} + O(1)\sign(B_{1,1}^{(T_1)}) \pm O(\frac{1}{\alpha_1}) = \Theta(1)\) for all iterations \(t \in [T_1,T_2]\). Similarly from \myref[b]{lem:learning-v1-phase2}{Lemma}, for \(B_{2,1}^{(t)}\) we can also write 
    \begin{align*}
        B_{2,1}^{(T+1)} = B_{2,1}^{(t)} + \eta \widetilde{O}(\alpha_1^{O(1)}/d^{5/2}) + \widetilde{O}(\frac{\alpha_1^{6}}{d})E_{2,1}^{(t)}\Phi_1^{(t)}[R_1^{(t)}]^3  
    \end{align*}
    From similar calculations, it holds \(B_{2,1}^{(t)} = B_{2,1}^{(T_1)} \pm \widetilde{O}(\alpha_1^{O(1)}/d)\), which proves that \(B_{2,1}^{(t)} = B_{2,1}^{(T_1)}(1\pm o(1))\)  when \(t \in [T_1,T_2]\). Now we turn to feature \(v_2\). By \myref{lem:learning-v2-phase2}{Lemma} we have for \(j\in[2]\):
    \begin{align*}
        |\dbrack{-\nabla_{w_j}L(W^{(t)}, E^{(t)}), v_2}| &\leq \widetilde{O}(\frac{\alpha_2^6\alpha_1^{6}}{d^{5/2}})\Big(\Phi_j^{(t)}(|E_{j,3-j}^{(t)}| + [R_j^{(t)}]^3)+ \Phi_{3-j}^{(t)}(|E_{3-j,j}^{(t)} |[R_{3-j}^{(t)}]^3 + \frac{|E_{3-j,j}^{(t)}|^2}{d^{3/2}})\Big) \\
        \leq \widetilde{O}(\frac{\alpha_2^6\alpha_1^{6}}{d^{5/2}})
    \end{align*}
    where the last inequality is from \myref[a]{lem:phase2-variables}{Lemma} and \myref[c,d]{induct:phase-2}{Induction}. Thus when \(t \leq T_2 = T_1 + \widetilde{O}(\frac{d^{3/2}\alpha_1^{12}}{\eta})\) we would have 
    \begin{align*}
        B_{j,2}^{(t)} = B_{j,2}^{(T_1)} \pm \widetilde{O}(\frac{\alpha_1^{O(1)}}{d}) = B_{j,2}^{(T_1)}(1 \pm o(1)) \tag*{since \(B_{j,2}^{(T_1)} = \widetilde{\Theta}(\frac{1}{\sqrt{d}})\) by \myref[c]{lem:phase-1}{Lemma}}
    \end{align*}
    Together they proved \myref[a]{induct:phase-2}{Induction} and \myref[a]{lem:phase-2}{Lemma}. Moreover, we have also \\
    \newline
    \textbf{Proof of \myref[b]{induct:phase-2}{Induction}:} Firstly, we write down the update of \(R_{1,2}^{(t)}\) using \myref[b,d]{lem:reduce-noise-phase2}{Lemma} as follows:
    \begin{align*}
        R_{1,2}^{(t+1)} &= R_{1,2}^{(t)} - \eta\dbrack{\nabla_{w_1}L(W^{(t)},E^{(t)}),\Pi_{V^{\perp}} w_2^{(t)}} - \eta\dbrack{\nabla_{w_2}L(W^{(t)},E^{(t)}),\Pi_{V^{\perp}} w_1^{(t)}}\\
        &\quad + \eta^2\dbrack{\Pi_{V^{\perp}}\nabla_{w_1}L(W^{(t)},E^{(t)}),\Pi_{V^{\perp}}\nabla_{w_2}L(W^{(t)},E^{(t)})} \\
        & = R_{1,2}^{(t)} + \eta\Sigma_{1,1}^{(t)}((-\Theta(\overline{R}_{1,2}^{(t)}) \pm O(\varrho) )[R_{1}^{(t)}]^{5/2}[R_{2}^{(t)}]^{1/2} + \widetilde{O}(|E_{1,2}^{(t)}| + \frac{|E_{2,1}^{(t)}|^2}{d^{3/2}})R_1^{(t)}[R_{2}^{(t)}]^2) \\
        &\quad + \eta\Big(\Sigma_{1,1}^{(t)} \Theta((E_{1,2}^{(t)})^2) + \sum_{\ell\in[2]}\Sigma_{2,\ell}^{(t)} \Big)(-\Theta(\overline{R}_{1,2}^{(t)}) \pm O(\varrho) )[R_{2}^{(t)}]^{5/2}[R_{1}^{(t)}]^{1/2} \\
        &\quad + O\Big(\sum_{j,\ell}\eta\Sigma_{j,\ell}^{(t)}E_{j,3-j}^{(t)}R_{2}^{(t)}[R_{1}^{(t)}]^{2}\Big) + \frac{\eta}{\poly(d)} 
    \end{align*}
    where in the last inequality we have used 
    \begin{align*}
        & |\dbrack{\Pi_{V^{\perp}}\nabla_{w_1}L(W^{(t)},E^{(t)}),\Pi_{V^{\perp}}\nabla_{w_2}L(W^{(t)},E^{(t)})}| \\
        \leq \ & \|\Pi_{V^{\perp}}\nabla_{w_1}L(W^{(t)},E^{(t)})\|_2\|\Pi_{V^{\perp}}\nabla_{w_2}L(W^{(t)},E^{(t)})\|_2 \leq  \widetilde{O}(d)
    \end{align*}
    Now from \myref[c,d]{induct:phase-2}{Induction} that \(R_2^{(t)} = \Theta(1)\) and \(|E_{1,2}^{(t)} |\leq \widetilde{O}(\varrho+\frac{1}{\sqrt{d}})[R_1^{(t)}]^{3/2}\), \(|E_{2,1}^{(t)}|\leq O(\sqrt{\eta_E/\eta})\), we can further obtain \(|\Sigma_{2,2}^{(t)}| = \widetilde{O}(\frac{\alpha_1^{O(1)}}{d^{3/2}})|\Sigma_{2,1}^{(t)}|\), and the bound 
    \begin{align*}
        R_{1,2}^{(t+1)} &= R_{1,2}^{(t)}\Big(1 - \Theta(\eta \Sigma_{1,1}^{(t)})[R_{1}^{(t)}]^{2} - \Theta( \eta (\Sigma_{1,1}^{(t)} (E_{1,2}^{(t)})^2 + \Sigma_{2,1}^{(t)}) )[R_{2}^{(t)}]^{2}\Big) \\
        & \quad \pm \eta O(\varrho)[R_{2}^{(t)}]^{1/2}[R_{1}^{(t)}]^{1/2} \left(O( \Sigma_{1,1}^{(t)})[R_{1}^{(t)}]^{2}+ \Big(\Sigma_{1,1}^{(t)} \Theta((E_{1,2}^{(t)})^2) + \Sigma_{2,1}^{(t)} \Big)[R_{2}^{(t)}]^{2} \right)
    \end{align*}
    Notice here that there exist a constant \(C = \Theta(1)\), whenever \(|R_{1,2}^{(t)}| \geq C(\varrho+\frac{1}{\sqrt{d}})[R_{2}^{(t)}]^{1/2}[R_{1}^{(t)}]^{1/2}\), it will holds 
    \begin{align*}
        R_{1,2}^{(t+1)} &= R_{1,2}^{(t)}\Big(1 - \Theta(\eta \Sigma_{1,1}^{(t)}[R_{1}^{(t)}]^{2}) - \Theta( \eta (\Sigma_{1,1}^{(t)} (E_{1,2}^{(t)})^2 + \Sigma_{2,1}^{(t)}) )[R_{2}^{(t)}]^{2}\Big) \\
        & = R_{1,2}^{(t)}\Big(1 - \Theta(\eta \Sigma_{1,1}^{(t)}[R_{1}^{(t)}]^{2}) - \Theta( \eta (\Sigma_{1,1}^{(t)} (E_{1,2}^{(t)})^2 + \frac{\alpha_1^6}{d^{3/2}}\Sigma_{2,1}^{(t)}) )[R_{2}^{(t)}]^{2}\Big)
    \end{align*}
    Thus we can go through the same analysis as in the proof of induction for \(E_{1,2}^{(t)}\) to derive that 
    \begin{align*}
        |R_{1,2}^{(t)}| \leq \widetilde{O}(\varrho+\frac{1}{\sqrt{d}})[R_{2}^{(t)}]^{1/2}[R_{1}^{(t)}]^{1/2}
    \end{align*}
    which is the desired result. Note that at the end of phase II
    \begin{align*}
        \text{\myref[a]{induct:phase-2}{Induction}}\quad &\implies \quad \text{\myref[a]{lem:phase-2}{Lemma}}\\
        \text{\myref[b,c]{induct:phase-2}{Induction}}\quad &\implies \quad \text{\myref[b]{lem:phase-2}{Lemma}}\\
        \text{\myref[d]{induct:phase-2}{Induction}}\quad &\implies \quad \text{\myref[c]{lem:phase-2}{Lemma}}
    \end{align*}
    We now complete the proof of \myref{lem:phase-2}{Lemma}.
\end{proof}

\section{Phase III: The Acceleration Effect of Prediction Head}\label{sec:appendix-phase-3}
We shall prove in this section that the growth of \(E_{2,1}^{(t)}\) in the previous phase creates an acceleration effect to the growth of \(B_{2,2}^{(t)}\), which will finally outrun the growth of \(B_{2,1}^{(t)}\) to win the lottery. We define 
\begin{align}\label{eqdef:T3}
    T_{3}:=\min\Big\{t : |B_{2,2}^{(t)}| \geq \frac{1}{2}\min\{|B_{1,1}^{(t)}|,\sqrt{\frac{\eta}{\eta_E}}|E_{2,1}^{(t)} | \} \Big\}
\end{align}
and we call iterations \(t \in [T_2, T_3]\) as the phase III of training and \(t\geq T_3\) as the end phase of training.

\subsection{Induction in Phase III}
\begin{induct}[Phase III]\label{induct:phase-3}
    During \(t \in [T_2,T_3]\), we hypothesize the following conditions holds.
    \begin{enumerate}[(a)]
        \item \(|B_{1,1}^{(t)}| = \Theta(1)\), \(B_{2,1}^{(t)} = B_{2,1}^{(T_2)}(1\pm o(1))\), \(B_{1,2}^{(t)}= B_{1,2}^{(T_2)}(1\pm o(1))\), \(|B_{2,2}^{(t)}| \in [|B_{2,2}^{(T_2)}|, O(1)]\);
        \item \(|E_{2,1}^{(t)}| = \Theta(\sqrt{\eta_E/\eta})\), \(\sign(E_{2,1}^{(t)}) = \sign(E_{2,1}^{(T_2)})\) and \(|E_{1,2}^{(t)}| \leq \widetilde{O}(\varrho+\frac{1}{\sqrt{d}})[R_1^{(t)}]^{3/2}[R_2^{(t)}]^{3/2}\);
        \item \(R_1^{(t)} \in [\Omega(\frac{1}{d}), O(\frac{d^{o(1)}}{d^{3/4}})] \), \([R_2^{(t)}] \in [\frac{1}{\sqrt{d}}, O(\frac{1}{\log d}\sqrt{\eta_E/\eta})] \).
    \end{enumerate}
\end{induct}

As usual, before we prove the induction, we need to derive some useful claims. But firstly we shall give a much cleaner form of \(\nabla_{E_{j,3-j}}L(W^{(t)},E^{(t)})\) to help us understand the learning process of phase III and the end phase.

\begin{fact}\label{fact:pred-head-grad}
    Let us write
    \begin{align*}
        \Xi_j^{(t)} & = C_0C_1\alpha_1^6\alpha_2^6\Phi_j^{(t)} \Big((B_{1,1}^{(t)})^6(B_{2,2}^{(t)})^6 + (B_{2,1}^{(t)})^6(B_{1,2}^{(t)})^6\Big) \\
        \Delta_{j,\ell}^{(t)} & = C_0\Phi_j^{(t)}\alpha_{\ell}^6(B_{j,\ell}^{(t)})^3(B_{3-j,\ell}^{(t)})^3C_2\Ecal_{j,3-j}^{(t)}
    \end{align*}
    Then the gradient of \(E_{j,3-j}^{(t)}\) can be written as 
    \begin{align*}
        -\nabla_{E_{j,3-j}}L(W^{(t)},E^{(t)}) = - \Xi_j^{(t)}E_{j,3-j}^{(t)} + \sum_{\ell\in[2]}\Delta_{j,\ell}^{(t)} - \sum_{\ell \in [2]}\Sigma_{j,\ell}^{(t)}\nabla_{E_{j,3-j}}\Ecal_{j,3-j}^{(t)}
    \end{align*}
\end{fact}
\begin{proof}
    By expanding the gradients of \(E_{j,3-j}^{(t)}\), we can verify by checking each monomial of polynomials of \(B_{j,\ell}\) to obtain the first term, and leave the \(\Ecal_{j,3-j}^{(t)}\) part for the second term.
\end{proof}

\begin{lemma}[variables control at phase III]\label{lem:phase3-variables} 
    For \(t \in [T_2,T_3]\), if \myref{induct:phase-3}{Induction} holds at iteration \(t\), then we have 
    \begin{enumerate}[(a)]
        \item \(\Phi_1^{(t)} = \widetilde{\Theta}(\frac{1}{\alpha_1^{12}})\), \([Q_2^{(t)}]^{-2} = \Theta(C_2[R_2^{(t)}]^3+C_1\alpha_2^6(B_{2,2}^{(t)})^6)\), \(U_2^{(t)} = \Theta(C_1(\alpha_1^6(E_{2,1}^{(t)})^2 + \alpha_2^6(B_{2,2}^{(t)})^6 )) \);
        \item \(H_{1,1}^{(t)} = \Theta(C_1\alpha_1^6) \), \(H_{1,2}^{(t)} \leq O(C_2[R_1^{(t)}]^3) + \widetilde{O}(\frac{\alpha_2^6}{d^3})\);
        \item \(H_{2,1}^{(t)} = \Theta(C_1\alpha_1^6 (E_{2,1}^{(t)})^2) \), \(H_{2,2}^{(t)} = \Theta(C_2[R_2^{(t)}]^3)\);
        \item \(\Sigma_{1,2}^{(t)} \leq \widetilde{O}(\frac{|E_{1,2}^{(t)}|}{d^{3/2}})\Sigma_{1,1}^{(t)} \);
        \item \(\Ecal_{j,3-j}^{(t)} = (1\pm o(1))\Ecal_j^{(t)} = O(C_2[R_j^{(t)}]^3) \)
    \end{enumerate}
\end{lemma}

\begin{proof}
    Assuming \myref{induct:phase-3}{Induction} holds at \(t \in [T_2,T_3]\), we can recall the expression of these variables and prove their bounds directly. The bounds for \(\Phi_1\) and \(H_{1,1}\) comes from \(|B_{1,1}^{(t)}| = \Theta(1)\) and \(|B_{1,2}^{(t)}|, |E_{1,2}^{(t)}| = o(1)\). The bounds for \(Q_2, U_2\) comes from our definition of \(T_3\) in \eqref{eqdef:T3}. The rest of the claims can be derived by similar arguments using \myref{induct:phase-3}{Induction}.
\end{proof}

\subsection{Gradient Lemmas for Phase III}

In this subsection, we would give some gradient lemmas concerning the dynamics of our network in Phase III.

\begin{lemma}[learning feature \(v_2\) in phase III]\label{lem:learning-v2-phase3}
    For each \(t\in [T_2,T_3]\), if Induction~\ref{induct:phase-3} holds at iteration \(t\), then we have:
    \begin{enumerate}[(a)]
        \item \(\dbrack{-\nabla_{w_1}L(W^{(t)}, E^{(t)}), v_2} =  \Theta(\frac{(B_{1,2}^{(t)})^2}{(B_{2,2}^{(t)})^2})E_{2,1}^{(t)} \Lambda_{2,2}^{(t)}  \pm \widetilde{O}(\frac{\alpha_1^{O(1)}}{d^{4}})|E_{2,1}^{(t)}|^2\Phi_2^{(t)} \pm \widetilde{O}(\frac{\alpha_1^{O(1)}}{d^{5/2}}) \);
        \item \(\dbrack{-\nabla_{w_2}L(W^{(t)}, E^{(t)}), v_2} = (1\pm \widetilde{O}(\frac{1}{d}))\Lambda_{2,2}^{(t)}\)
    \end{enumerate}
\end{lemma}

\begin{proof}
    Since \(\dbrack{-\nabla_{w_j}L(W^{(t)},E^{(t)}), v_2} = \Lambda_{j,2}^{(t)} + \Gamma_{j,2}^{(t)} - \Upsilon_{j,2}^{(t)}\), let us write down the definition of \(\Lambda_{j,2}^{(t)}, \Gamma_{j,2}^{(t)}, \Upsilon_{j,2}^{(t)}\) respectively:
    \begin{align*}
        \Lambda_{j,2}^{(t)} & = C_0\alpha_{2}^6\Phi_j^{(t)} H_{j,1}^{(t)}(B_{j,2}^{(t)})^5 \\
        \Gamma_{j,2}^{(t)} &= C_0\alpha_{2}^6\Phi_{3-j}^{(t)} E_{3-j,j}^{(t)} (B_{3-j,2}^{(t)})^3(B_{j,2}^{(t)})^2H_{3-j,1}^{(t)} \\
        \Upsilon_{j,2}^{(t)} & = C_0\alpha_{1}^6\left(\Phi_j^{(t)}(B_{j,1}^{(t)})^3(B_{j,2}^{(t)})^2K_{j,2}^{(t)} + \Phi_{3-j}^{(t)}E_{3-j,j}^{(t)}(B_{3-j,1}^{(t)})^3(B_{j,2}^{(t)})^2K_{3-j,2}^{(t)}\right)
    \end{align*}
    Again we decompose \(\Upsilon_{j,2}^{(t)} = \Upsilon_{j,2,1}^{(t)} + \Upsilon_{j,2,2}^{(t)}\) as in the proof of \myref{lem:learning-v2-phase2}{Lemma}, where 
    \begin{align*}
        \Upsilon_{j,2,1}^{(t)} & = C_0\alpha_{1}^6\Phi_j^{(t)}(B_{j,1}^{(t)})^3(B_{j,2}^{(t)})^2K_{j,2}^{(t)}, \qquad \Upsilon_{j,2,2}^{(t)} = \Phi_{3-j}^{(t)}E_{3-j,j}^{(t)}(B_{3-j,1}^{(t)})^3(B_{3-j,2}^{(t)})^2K_{3-j,2}^{(t)}
    \end{align*}
    This gives 
    \begin{align*}
        \Lambda_{j,2}^{(t)} - \Upsilon_{j,2,1}^{(t)} & = C_0\alpha_{2}^6\Phi_j^{(t)} (B_{j,2}^{(t)})^5H_{j,1}^{(t)} - C_0\alpha_{1}^6\Phi_j^{(t)}(B_{j,1}^{(t)})^3(B_{j,2}^{(t)})^2K_{j,2}^{(t)} \\ 
        & = C_0\alpha_{2}^6C_1\alpha_{1}^6\Phi_j^{(t)}(B_{j,2}^{(t)})^5 \left(  E_{j,3-j}^{(t)} (B_{3-j,1}^{(t)})^3(B_{j,1}^{(t)} )^3 + (E_{j,3-j}^{(t)})^2 (B_{3-j,1}^{(t)})^6\right) \\
        &\quad - C_0\alpha_{2}^6C_1\alpha_{1}^6\Phi_j^{(t)}(B_{j,2}^{(t)})^2 (B_{3-j,2}^{(t)})^3 E_{j,3-j}^{(t)}\left( (B_{j,1}^{(t)})^6 + E_{j,3-j}^{(t)}(B_{3-j,1}^{(t)})^3(B_{j,1}^{(t)})^3\right)\\
        &\quad + C_0\alpha_{2}^6\Phi_j^{(t)} (B_{j,2}^{(t)})^5C_2\Ecal_{j,3-j}^{(t)}
    \end{align*}
    When \(j=1\), from \myref{induct:phase-3}{Induction} and \myref[a]{lem:phase3-variables}{Lemma} (which gives \(\Phi_1^{(t)}\leq \alpha_1^{O(1)}\Phi_2^{(t)}\)), we can crudely obtain
    \begin{align*}
        &\left|C_0\alpha_{2}^6C_1\alpha_{1}^6\Phi_1^{(t)}(B_{1,2}^{(t)})^5 \left(  E_{1,2}^{(t)} (B_{2,1}^{(t)})^3(B_{1,1}^{(t)} )^3 + (E_{1,2}^{(t)})^2 (B_{2,1}^{(t)})^6\right)\right| \leq \widetilde{O}(\frac{\alpha_1^{O(1)}}{d^{4}})\Phi_1^{(t)} |E_{1,2}^{(t)} |\\
        & \left|C_0\alpha_{2}^6C_1\alpha_{1}^6\Phi_1^{(t)}(B_{1,2}^{(t)})^2 (B_{2,2}^{(t)})^3 E_{1,2}^{(t)}\left( (B_{1,1}^{(t)})^6 + E_{1,2}^{(t)}(B_{2,1}^{(t)})^3(B_{1,1}^{(t)})^3\right)\right| \leq \widetilde{O}(\frac{\alpha_1^{O(1)}}{d})\Lambda_{2,2}^{(t)}|E_{1,2}^{(t)} | \\
        &\left|C_0\alpha_{2}^6\Phi_1^{(t)} (B_{1,2}^{(t)})^5C_2\Ecal_{1,2}^{(t)}\right| = \widetilde{O}(\frac{\alpha_1^{6}}{d^{5/2}})\Sigma_{1,1}^{(t)}[R_1^{(t)}]^3
    \end{align*}
    So we have
    \begin{align*}
        \Lambda_{1,2}^{(t)} - \Upsilon_{1,2,1}^{(t)} & = \widetilde{O}(\frac{\alpha_1^{6}}{d^{5/2}})\Sigma_{1,1}^{(t)}[R_1^{(t)}]^3 \pm \widetilde{O}(\frac{\alpha_1^{O(1)}}{d})\Lambda_{2,2}^{(t)}|E_{1,2}^{(t)} |
    \end{align*}
    When \(j=2\), we can also derive using \myref{lem:phase3-variables}{Lemma} about \(H_{2,1}^{(t)}\) and \myref{induct:phase-3}{Induction} about \(B_{2,1}^{(t)}\) and some rearrangement to obtain
    \begin{align*}
        &C_0\alpha_{2}^6\Phi_2^{(t)}(B_{2,2}^{(t)})^5 \left[C_1\alpha_{1}^6\left(  E_{2,1}^{(t)} (B_{1,1}^{(t)})^3(B_{2,1}^{(t)} )^3 + (E_{2,1}^{(t)})^2 (B_{1,1}^{(t)})^6\right) + C_2\Ecal_{2,1}^{(t)}\right] = (1\pm \widetilde{O}(\frac{1}{d}))\Lambda_{2,2}^{(t)} \\
        &\left| C_0\alpha_{2}^6C_1\alpha_{1}^6\Phi_2^{(t)}(B_{2,2}^{(t)})^2 (B_{1,2}^{(t)})^3 E_{2,1}^{(t)}\left( (B_{2,1}^{(t)})^6 + E_{2,1}^{(t)}(B_{1,1}^{(t)})^3(B_{2,1}^{(t)})^3\right)\right| \leq \widetilde{O}(\frac{\alpha_1^{O(1)}}{d^{3}})|E_{2,1}^{(t)}|\Phi_2^{(t)}
    \end{align*}
    which leads to the approximation
    \begin{align*}
        \Lambda_{2,2}^{(t)} - \Upsilon_{1,2,2}^{(t)} & = (1\pm \widetilde{O}(\frac{1}{d}))\Lambda_{2,2}^{(t)} \pm \widetilde{O}(\frac{\alpha_1^{O(1)}}{d^{3}})|E_{2,1}^{(t)}|\Phi_2^{(t)}
    \end{align*}
    Similarly, we can also calculate 
    \begin{align*}
        \Gamma_{j,2}^{(t)} - \Upsilon_{j,2,2}^{(t)} &= C_0\alpha_{2}^6\Phi_{3-j}^{(t)} E_{3-j,j}^{(t)} (B_{3-j,2}^{(t)})^3(B_{j,2}^{(t)})^2H_{3-j,1}^{(t)} - C_0\alpha_{1}^6\Phi_{3-j}^{(t)}E_{3-j,j}^{(t)}(B_{3-j,1}^{(t)})^3(B_{j,2}^{(t)})^2K_{3-j,2}^{(t)} \\
        & = C_0\alpha_{2}^6C_1\alpha_{1}^6\Phi_{3-j}^{(t)}(B_{3-j,2}^{(t)})^3(B_{j,2}^{(t)})^2 E_{3-j,j}^{(t)}\left(  E_{3-j,j}^{(t)} (B_{j,1}^{(t)})^3(B_{3-j,1}^{(t)} )^3 + (E_{3-j,j}^{(t)})^2 (B_{j,1}^{(t)})^6\right) \\
        &\quad - C_0\alpha_{2}^6C_1\alpha_{1}^6\Phi_{3-j}^{(t)}(B_{j,2}^{(t)})^5  (E_{3-j,j}^{(t)})^2\left( (B_{3-j,1}^{(t)})^6 + E_{3-j,j}^{(t)}(B_{j,1}^{(t)})^3(B_{3-j,1}^{(t)})^3\right)\\
        &\quad + C_0\alpha_{2}^6\Phi_{3-j}^{(t)}E_{3-j,j}^{(t)} (B_{3-j,2}^{(t)})^3(B_{j,2}^{(t)})^2C_2\Ecal_{3-j,j}^{(t)}
    \end{align*}
    When \(j=1\), following similar procedure as above, we can apply \myref{induct:phase-3}{Induction} and \myref{lem:phase3-variables}{Lemma} to give
    \begin{align*}
        \Gamma_{1,2}^{(t)} - \Upsilon_{1,2,2}^{(t)} & = \Theta(\frac{(B_{1,2}^{(t)})^2}{(B_{2,2}^{(t)})^2})E_{2,1}^{(t)} \Lambda_{2,2}^{(t)}  \pm \widetilde{O}(\frac{\alpha_1^{O(1)}}{d^{4}})|E_{2,1}^{(t)}|^2\Phi_2^{(t)} 
    \end{align*}
    Note that the first term on the RHS dominates the term \(\pm \widetilde{O}(\frac{\alpha_1^{O(1)}}{d})\Lambda_{2,2}^{(t)}|E_{1,2}^{(t)} |\) in the approximation for \(\Lambda_{1,2}^{(t)} - \Upsilon_{1,2,1}^{(t)}\) due to \myref[a,b]{induct:phase-3}{Induction}. When \(j=2\), since \(\Phi_1^{(t)}\leq \widetilde{\Theta}(\frac{1}{\alpha_1^{12}})\leq \alpha_1^{O(1)}\Phi_2^{(t)}H_{2,1}^{(t)}\) in this phase and \(|B_{1,1}^{(t)} |=O(1)\), we can derive
    \begin{align*}
        |\Gamma_{2,2}^{(t)} - \Upsilon_{2,2,2}^{(t)}| & \leq  \widetilde{\Theta}(\frac{\alpha_1^{O(1)}}{d^{3}})(E_{1,2}^{(t)})^2\Phi_1^{(t)} +  \alpha_1^{O(1)}(E_{1,2}^{(t)})^2\Lambda_{2,2}^{(t)}
    \end{align*}
    It can be seen that \((E_{1,2}^{(t)})^2\Phi_1^{(t)} \leq (E_{2,1}^{(t)})^2\Phi_2^{(t)}\) by \myref{induct:phase-3}{Induction} and \myref{lem:phase3-variables}{Lemma}. And by similar arguments we can have \( (1\pm \widetilde{O}(\frac{1}{d}))\Lambda_{2,2}^{(t)} \geq \frac{1}{d^{\Omega(1)}}\widetilde{O}(\frac{\alpha_1^{O(1)}}{d^{3}})|E_{2,1}^{(t)}|\Phi_2^{(t)}\). Combining all the results above, we can finish the proof.
\end{proof}

\begin{lemma}[learning feature \(v_1\) in Phase III]\label{lem:learning-v1-phase3}
    For each \(t\in [T_2,T_3]\), if Induction~\ref{induct:phase-3} holds at iteration \(t\), then we have: (recall that \(\Delta\)-notation is from \myref{fact:pred-head-grad}{Fact} )
    \begin{enumerate}[(a)]
        \item \(\dbrack{-\nabla_{w_1}L(W^{(t)},E^{(t)}),v_1} = \Theta(\Sigma_{1,1}^{(t)}[R_1^{(t)}]^3) \pm O(\frac{(B_{1,2}^{(t)})^3}{(B_{2,2}^{(t)})^3} + \frac{1}{\sqrt{d}} )\alpha_1^{O(1)}\Lambda_{2,2}^{(t)} + \frac{E_{2,1}^{(t)}}{B_{1,1}^{(t)}}\Delta_{2,1}^{(t)} - \frac{B_{2,2}^{(t)} }{B_{1,1}^{(t)}}\Lambda_{2,2}^{(t)} \);
        \item \(\dbrack{-\nabla_{w_2}L(W^{(t)},E^{(t)}),v_1} = \widetilde{O}(\frac{\alpha_1^{O(1)}}{d^{5/2}})\Phi_2^{(t)}[R_2^{(t)}]^3  \pm \widetilde{O}(\frac{\alpha_1^{O(1)}}{d})\Lambda_{2,2}^{(t)} \pm \widetilde{O}(\frac{\alpha_1^{O(1)}}{d^3})  \)
    \end{enumerate}
\end{lemma}

\begin{proof}
    Recall that \(\dbrack{-\nabla_{w_j}L(W^{(t)},E^{(t)}),v_1} = \Lambda_{j,1}^{(t)} + \Gamma_{j,1}^{(t)} - \Upsilon_{j,1}^{(t)}\). Similar to the proof of \myref{lem:learning-v2-phase3}{Lemma}, we can decompose \(\Upsilon_{j,1}^{(t)} = \Upsilon_{j,1,1}^{(t)} + \Upsilon_{j,1,2}^{(t)}\) and do similar calculations:
    \begin{align*}
        \Lambda_{j,1}^{(t)} - \Upsilon_{j,1,1}^{(t)} & = C_0C_1\alpha_{1}^6\alpha_{2}^6\Phi_j^{(t)}(B_{j,1}^{(t)})^5 \left(  E_{j,3-j}^{(t)} (B_{3-j,2}^{(t)})^3(B_{j,2}^{(t)} )^3 + (E_{j,3-j}^{(t)})^2 (B_{3-j,2}^{(t)})^6\right) \\
        &\quad - C_0C_1\alpha_{1}^6\alpha_{2}^6\Phi_j^{(t)}(B_{j,1}^{(t)})^2 (B_{3-j,1}^{(t)})^3 E_{j,3-j}^{(t)}\left( (B_{j,2}^{(t)})^6 + E_{j,3-j}^{(t)}(B_{3-j,2}^{(t)})^3(B_{j,2}^{(t)})^3\right)\\
        &\quad + C_0\alpha_{1}^6\Phi_j^{(t)} (B_{j,1}^{(t)})^5C_2\Ecal_{j,3-j}^{(t)}
    \end{align*}
    When \(j=1\), from \myref{induct:phase-3}{Induction} and \myref[a]{lem:phase3-variables}{Lemma} we know \(\Phi_1^{(t)}\leq \alpha_1^{(O(1))}\) during \(t\in[T_2,T_3]\), which allow us to derive 
    \begin{align*}
        &\quad C_0C_1\alpha_{1}^6\alpha_{2}^6\Phi_1^{(t)}(B_{1,1}^{(t)})^5 \left(  E_{1,2}^{(t)} (B_{2,2}^{(t)})^3(B_{1,2}^{(t)} )^3 + (E_{1,2}^{(t)})^2 (B_{2,2}^{(t)})^6\right) \\
        & \leq \widetilde{O}(\Sigma_{1,1}^{(t)}(E_{1,2}^{(t)})^2) +  C_0C_1\alpha_{1}^6\alpha_{2}^6(B_{1,1}^{(t)})^5E_{1,2}^{(t)} (B_{2,2}^{(t)})^3(B_{1,2}^{(t)} )^3\\ 
        &\leq O(\frac{(B_{1,2}^{(t)})^3}{(B_{2,2}^{(t)})^3})\alpha_1^{O(1)}\Lambda_{2,2}^{(t)}|E_{1,2}^{(t)}| +\Theta(\Sigma_{1,1}^{(t)}[R_1^{(t)}]^3)
    \end{align*}
    And 
    \begin{align*}
        \left|C_0C_1\alpha_{1}^6\alpha_{2}^6\Phi_1^{(t)}(B_{1,1}^{(t)})^2 (B_{2,1}^{(t)})^3 E_{1,2}^{(t)}\left( (B_{1,2}^{(t)})^6 + E_{1,2}^{(t)}(B_{2,2}^{(t)})^3(B_{1,2}^{(t)})^3\right) \right| \leq \widetilde{O}(\frac{1}{d^{3/2}})|E_{1,2}^{(t)}|\Lambda_{2,2}^{(t)}
    \end{align*}
    which can be summarized as
    \begin{align*}
        \Lambda_{1,1}^{(t)} - \Upsilon_{1,1,1}^{(t)} &= \Theta(\Sigma_{1,1}^{(t)}[R_1^{(t)}]^3) \pm O(\frac{(B_{1,2}^{(t)})^3}{(B_{2,2}^{(t)})^3} + (B_{2,1}^{(t)})^3 + \frac{1}{\sqrt{d}} )|E_{1,2}^{(t)}|\alpha_1^{O(1)}\Lambda_{2,2}^{(t)}
    \end{align*}
    A similar calculation also gives 
    \begin{align*}
        \Lambda_{2,1}^{(t)} - \Upsilon_{2,1,1}^{(t)} & = \widetilde{O}(\frac{\alpha_1^{O(1)}}{d^{5/2}})\Phi_2^{(t)}[R_2^{(t)}]^3  \pm \widetilde{O}(\frac{\alpha_1^{O(1)}}{d^4})\Phi_2^{(t)}|E_{2,1}^{(t)} | \pm \widetilde{O}(\frac{\alpha_1^{O(1)}}{d})\Lambda_{2,2}^{(t)}B_{2,2}^{(t)}
    \end{align*}
    Now we turn to the other terms in the gradient, from similar calculations in the proof of \myref{lem:learning-v2-phase2}{Lemma}, we have 
    \begin{align*}
        \Gamma_{j,1}^{(t)} - \Upsilon_{j,1,2}^{(t)} &= C_0\alpha_{2}^6C_1\alpha_{1}^6\Phi_{3-j}^{(t)}(B_{3-j,1}^{(t)})^3(B_{j,1}^{(t)})^2 E_{3-j,j}^{(t)}\left(  E_{3-j,j}^{(t)} (B_{j,2}^{(t)})^3(B_{3-j,2}^{(t)} )^3 + (E_{3-j,j}^{(t)})^2 (B_{j,2}^{(t)})^6\right) \\
        &\quad - C_0\alpha_{2}^6C_1\alpha_{1}^6\Phi_{3-j}^{(t)}(B_{j,1}^{(t)})^5  (E_{3-j,j}^{(t)})^2\left( (B_{3-j,2}^{(t)})^6 + E_{3-j,j}^{(t)}(B_{j,2}^{(t)})^3(B_{3-j,2}^{(t)})^3\right)\\
        &\quad + C_0\alpha_{2}^6\Phi_{3-j}^{(t)}E_{3-j,j}^{(t)} (B_{3-j,1}^{(t)})^3(B_{j,1}^{(t)})^2C_2\Ecal_{3-j,j}^{(t)}
    \end{align*}
    which also similarly gives
    \begin{align*}
        \Gamma_{1,1}^{(t)} - \Upsilon_{1,1,2}^{(t)} & = \frac{E_{2,1}^{(t)}}{B_{1,1}^{(t)}}\Delta_{2,1}^{(t)} - \frac{B_{2,2}^{(t)} }{B_{1,1}^{(t)}}\Lambda_{2,2}^{(t)} \pm \widetilde{O}(\frac{\alpha_1^{O(1)}}{d^{3/2}})\Lambda_{2,2}^{(t)}
    \end{align*}
    and 
    \begin{align*}
        |\Gamma_{2,1}^{(t)} - \Upsilon_{2,1,2}^{(t)}| \leq \widetilde{O}(\frac{\alpha_1^{O(1)}}{d})\Phi_1^{(t)}((E_{1,2}^{(t)})^2 + |E_{1,2}^{(t)}|[R_1^{(t)}]^3) \leq \widetilde{O}(\frac{\alpha_1^{O(1)}}{d^3})
    \end{align*}
    which finishes the proof.
\end{proof}

\begin{lemma}[reducing noise in phase III]\label{lem:reduce-noise-phase3}
    Suppose \myref{induct:phase-3}{Induction} holds at \(t\in [T_2,T_3]\), then we have 
    \begin{align*}
        (a) \quad\dbrack{-\nabla_{w_1}L(W^{(t)},E^{(t)}), \Pi_{V^{\perp}}w_1^{(t)}} &= - \Theta([R_{1}^{(t)}]^3) \Big(\Sigma_{1,1}^{(t)}  + \sum_{\ell\in[2]}\Sigma_{2,\ell}^{(t)}(E_{2,1}^{(t)})^2 \Big) \\
        &\quad \pm  O\Big(\sum_{j,\ell}\Sigma_{j,\ell}^{(t)}E_{j,3-j}^{(t)} (\overline{R}_{1,2}^{(t)} + \varrho)[R_{1}^{(t)}]^{3/2}[R_{2}^{(t)}]^{3/2}\Big);\\
        (b) \quad \dbrack{-\nabla_{w_1}L(W^{(t)},E^{(t)}), \Pi_{V^{\perp}}w_{2}^{(t)}} &= \Big(\Sigma_{1,1}^{(t)} + \sum_{\ell\in[2]}\Sigma_{2,\ell}^{(t)}(E_{2,1}^{(t)})^2 \Big)(-\Theta(\overline{R}_{1,2}^{(t)}) \pm O(\varrho) )[R_{1}^{(t)}]^{5/2}[R_{2}^{(t)}]^{1/2} \\
        &\quad + O\Big(\sum_{(j,\ell)\neq (1,2)}\Sigma_{j,\ell}^{(t)}E_{j,3-j}^{(t)}R_{1}^{(t)}[R_{2}^{(t)}]^{2}\Big)\\
        (c) \quad\dbrack{-\nabla_{w_2}L(W^{(t)},E^{(t)}), \Pi_{V^{\perp}}w_2^{(t)}} &= - \Theta([R_{2}^{(t)}]^3) \Big(\sum_{\ell\in[2]}\Sigma_{1,1}^{(t)} \Theta((E_{1,2}^{(t)})^2) + \sum_{\ell\in[2]}\Sigma_{2,\ell}^{(t)} \Big) \\
        &\quad \pm  O\Big(\sum_{j,\ell}\Sigma_{j,\ell}^{(t)}E_{j,3-j}^{(t)} (\overline{R}_{1,2}^{(t)} + \varrho)[R_{1}^{(t)}]^{3/2}[R_{2}^{(t)}]^{3/2}\Big);\\
        (d) \quad \dbrack{-\nabla_{w_2}L(W^{(t)},E^{(t)}), \Pi_{V^{\perp}}w_{1}^{(t)}} &= \Big(\Sigma_{1,1}^{(t)} \Theta((E_{1,2}^{(t)})^2) + \sum_{\ell\in[2]}\Sigma_{2,\ell}^{(t)} \Big)(-\Theta(\overline{R}_{1,2}^{(t)}) \pm O(\varrho) )[R_{2}^{(t)}]^{5/2}[R_{1}^{(t)}]^{1/2} \\
        &\quad + O\Big(\sum_{(j,\ell)\neq (1,2)}\Sigma_{j,\ell}^{(t)}E_{j,3-j}^{(t)}R_{2}^{(t)}[R_{1}^{(t)}]^{2}\Big)
    \end{align*}
\end{lemma}

\begin{proof}
    The proof of \myref{lem:reduce-noise-phase3}{Lemma} is very similar to \myref{lem:reduce-noise-phase2}{Lemma}, but we write it down to stress some minor differences. As in \eqref{eqdef:weight-grad}, we first write down
    \begin{align*}
        \dbrack{-\nabla_{w_1}L(W^{(t)},E^{(t)}), \Pi_{V^{\perp}}w_1^{(t)}} = -\sum_{j,\ell}\Sigma_{j,\ell}^{(t)}\dbrack{\nabla_{w_1}\Ecal_{j,3-j}^{(t)},w_1^{(t)}}
    \end{align*} 
    \textbf{Proof of (a):} Combine the bounds above, we can obtain for each \(j\in[2]\): \(\Sigma_{1,2}^{(t)} = \widetilde{O}(E_{1,2}^{(t)}/d^{3/2})\Sigma_{1,1}^{(t)} \). We can then directly apply \myref{claim:noise}{Claim} to prove \myref[a]{lem:reduce-noise-phase3}{Lemma} as follows 
    \begin{align*}
        &\quad \dbrack{-\nabla_{w_1}L(W^{(t)},E^{(t)}), \Pi_{V^{\perp}}w_1^{(t)}} \\
        & = (1\pm \widetilde{O}(E_{1,2}^{(t)}/d^{3/2}))\Sigma_{1,1}^{(t)}\Big( - \Theta([R_{1}^{(t)}]^3) \pm O(E_{1,2}^{(t)})(\overline{R}_{1,2}^{(t)} + \varrho)[R_{1}^{(t)}]^{3/2}[R_{2}^{(t)}]^{3/2}\Big) \\
        &\quad  +(\Sigma_{2,1}^{(t)}+ \Sigma_{2,2}^{(t)})\Big(-\Theta((E_{2,1}^{(t)})^2)[R_{1}^{(t)}]^3 \pm O(E_{2,1}^{(t)})(\overline{R}_{1,2}^{(t)} + \varrho)[R_{1}^{(t)}]^{3/2}[R_{2}^{(t)}]^{3/2}\Big) \\
        &= -\Theta(\Sigma_{1,1}^{(t)} + \Sigma_{2,1}^{(t)}+ \Sigma_{2,2}^{(t)})[R_{1}^{(t)}]^3  \pm O(\sum_{j,\ell}\Sigma_{j,\ell}^{(t)}E_{j,3-j}^{(t)}(\overline{R}_{1,2}^{(t)} + \varrho)[R_{1}^{(t)}]^{3/2}[R_{2}^{(t)}]^{3/2}\Big) \tag{Since \(|E_{1,2}^{(t)}| \leq d^{-\Omega(1)}\) by \myref{induct:phase-3}{Induction}}
    \end{align*}
    \textbf{Proof of (b):} For \myref[b]{lem:reduce-noise-phase2}{Lemma}, we can use the same analysis for \(\Sigma_{1,1}^{(t)}\) above and \myref[d,e]{claim:noise}{Claim} to get (again we have used  \(\Sigma_{1,2}^{(t)} = \widetilde{O}(E_{1,2}^{(t)}/d^{3/2})\Sigma_{1,1}^{(t)} \))
    \begin{align*}
        &\quad\dbrack{-\nabla_{w_1}L(W^{(t)},E^{(t)}), \Pi_{V^{\perp}}w_2   ^{(t)}} \\
        & = (1\pm \widetilde{O}(E_{1,2}^{(t)}/d^{3/2}))\Sigma_{1,1}^{(t)}\Big((-\Theta(\overline{R}_{1,2}^{(t)}) \pm O(\varrho) )[R_{1}^{(t)}]^{5/2}[R_{2}^{(t)}]^{1/2} + E_{1,2}^{(t)}R_1^{(t)}[R_{2}^{(t)}]^2\Big) \\
        &\quad + \Theta( \Sigma_{2,1}^{(t)}+ \Sigma_{2,2}^{(t)})\Big((-\Theta(\overline{R}_{1,2}^{(t)}) + O(\varrho) )(E_{2,1}^{(t)})^2[R_{1}^{(t)}]^{5/2}[R_{2}^{(t)}]^{1/2} + E_{2,1}^{(t)}R_1^{(t)}[R_{2}^{(t)}]^2\Big)\\
        & = \Big(\Sigma_{1,1}^{(t)} + \sum_{\ell\in[2]}\Sigma_{2,\ell}^{(t)}(E_{2,1}^{(t)})^2 \Big)((-\Theta(\overline{R}_{1,2}^{(t)}) + O(\varrho) )[R_{1}^{(t)}]^{5/2}[R_{2}^{(t)}]^{1/2}) \\
        &\quad + O\Big(\sum_{(j,\ell)\neq (2,1)}\Sigma_{j,\ell}^{(t)}E_{j,3-j}^{(t)}R_{1}^{(t)}[R_{2}^{(t)}]^{2}\Big)
    \end{align*}
    \textbf{Proof of (c):} Similarly to the proof of (a), we can also expand as follows 
    \begin{align*}
        &\quad \dbrack{-\nabla_{w_2}L(W^{(t)},E^{(t)}), \Pi_{V^{\perp}}w_2^{(t)}} \\
        & = (1\pm \widetilde{O}(E_{1,2}^{(t)}/d^{3/2}))\Sigma_{1,1}^{(t)}\Big(-[R_{2}^{(t)}]^3\Theta((E_{1,2}^{(t)})^2) \pm O(E_{1,2}^{(t)})(\overline{R}_{1,2}^{(t)} + \varrho)[R_{1}^{(t)}]^{3/2}[R_{2}^{(t)}]^{3/2}\Big) \\
        &\quad - \sum_{\ell\in[2]}\Sigma_{2,\ell}^{(t)}\Big([R_{2}^{(t)}]^3 \pm O(E_{2,1}^{(t)})(\overline{R}_{1,2}^{(t)} + \varrho)[R_{1}^{(t)}]^{3/2}[R_{2}^{(t)}]^{3/2}\Big)\\
        &= - \Theta([R_{2}^{(t)}]^3)\Big(\Sigma_{1,1}^{(t)} \Theta((E_{1,2}^{(t)})^2) + \sum_{\ell\in[2]}\Sigma_{2,\ell}^{(t)} \Big)  \pm  O\Big(\sum_{j,\ell}\Sigma_{j,\ell}^{(t)}E_{j,3-j}^{(t)} (\overline{R}_{1,2}^{(t)} + \varrho)[R_{1}^{(t)}]^{3/2}[R_{2}^{(t)}]^{3/2} \Big)
    \end{align*}
    \textbf{Proof of (d):} Similarly, we can calculate 
    \begin{align*}
        &\quad \dbrack{-\nabla_{w_2}L(W^{(t)},E^{(t)}), \Pi_{V^{\perp}}w_1^{(t)}} \\
        & = (1\pm \widetilde{O}(E_{1,2}^{(t)}/d^{3/2}) )\Sigma_{1,1}^{(t)}\Big((-\Theta(\overline{R}_{1,2}^{(t)}) \pm O(\varrho) )(E_{1,2}^{(t)})^2[R_{2}^{(t)}]^{5/2}[R_{1}^{(t)}]^{1/2} + E_{1,2}^{(t)}R_2^{(t)}[R_{1}^{(t)}]^2\Big) \\
        &\quad +  \sum_{\ell\in[2]}\Sigma_{2,\ell}^{(t)}\Big((-\Theta(\overline{R}_{1,2}^{(t)}) \pm O(\varrho) )[R_{2}^{(t)}]^{5/2}[R_{1}^{(t)}]^{1/2} + E_{1,2}^{(t)}R_2^{(t)}[R_{1}^{(t)}]^2\Big) \\
        &= \Big(\Sigma_{1,1}^{(t)} \Theta((E_{1,2}^{(t)})^2) + \sum_{\ell\in[2]}\Sigma_{2,\ell}^{(t)} \Big)(-\Theta(\overline{R}_{1,2}^{(t)}) \pm O(\varrho) )[R_{2}^{(t)}]^{5/2}[R_{1}^{(t)}]^{1/2}  \\
        &\quad + O\Big(\sum_{(j,\ell)\neq (2,1)}\Sigma_{j,\ell}^{(t)}E_{j,3-j}^{(t)}R_{2}^{(t)}[R_{1}^{(t)}]^{2}\Big)
    \end{align*}
    which completes the proof.
\end{proof}

\begin{lemma}[learning the prediction head in phase III]\label{lem:learning-pred-head-phase3}
    If \myref{induct:phase-3}{Induction} holds at iteration \(t \in [T_2,T_3]\), then using the notations from \myref{fact:pred-head-grad}{Fact}, we have 
    \begin{align*}
        -\nabla_{E_{j,3-j}}L(W^{(t)},E^{(t)}) & = \Theta(\sum_{\ell\in[2]}\Sigma_{j,\ell}^{(t)}) (-E_{j,3-j}^{(t)}[R_{3-j}^{(t)}]^{3} \pm O(\overline{R}_{1,2}^{(t)} + \varrho)[R_{1}^{(t)}]^{3/2}[R_{2}^{(t)}]^{3/2}) \\
        &\quad  -\Xi_j^{(t)}E_{j,3-j}^{(t)} +\sum_{\ell\in[2]}\Delta_{j,\ell}^{(t)} 
    \end{align*}
\end{lemma}

\begin{proof}
    By \myref{fact:pred-head-grad}{Fact}, we only need to bound the last term \(\sum_{\ell \in [2]}\Sigma_{j,\ell}^{(t)}\nabla_{E_{1,2}}\Ecal_{j,3-j}^{(t)}\), which can be directly obtained from applying \myref{claim:noise}{Claim}.
\end{proof}

\subsection{At the End of Phase III}

In order to argue that \(B_{2,2}^{(T_2)} = \Omega(1) \) at the end of phase III, we need to define some auxiliary notions. Recall that \(T_3\) is defined in \eqref{eqdef:T3}, and now we further define 
\begin{align}\label{eqdef:T3.1-and-T3.2}
    T_{3,1} &:= \min\{t: C_1\alpha_2^6(B_{2,2}^{(t)})^6 \geq C_2[R_2^{(t)}]^3 \},\qquad T_{3,2}^{(t)} = \min\big\{t: |B_{2,2}^{(t)}|\geq \frac{1}{3}\min\{|E_{2,1}^{(t)}|,|B_{1,1}^{(t)}|\} \big\}
\end{align}
It can be observed that if \myref{induct:phase-3}{Induction} holds for \(t \in [T_2,T_3]\) and our learning rate \(\eta\) is small enough, we shall have \(T_2 < T_{3,1} \leq T_{3,2} < T_{3}\). Now we are ready to present the main lemma we want to prove in this phase.

\begin{lemma}[Phase III]\label{lem:phase-3}
    Let \(T_3\) be defined as in \eqref{eqdef:T3}. Suppose \(\eta = \frac{1}{\poly(d)}\) is sufficiently small, then \myref{induct:phase-3}{Induction} holds for all iteration \(t \in [T_2,T_3]\), and at iteration \(t= T_3\), the followings holds:
    \begin{enumerate}[(a)]
        \item \(|B_{1,1}^{(T_3)}| = \Theta(1)\), \(|B_{2,2}^{(T_3)}| = \Theta(1)\), \(B_{j,\ell}^{(T_3)} = B_{j,\ell}^{(T_2)}(1 \pm o(1))\) for \(j\neq \ell\);
        \item \(R_1^{(T_3)} = \widetilde{O}(\frac{1}{d^{3/4}})\), \(R_2^{(T_3)} \in [\widetilde{O}(\frac{1}{d^{1/2}}),\widetilde{O}(\frac{1}{d^{1/4}})]\), and \(\overline{R}_{1,2}^{(T_3)}\leq \widetilde{O}(\varrho+\frac{1}{\sqrt{d}})\);
        \item \(|E_{2,1}^{(T_2)}| = \Theta(\sqrt{\eta_E/\eta})\) and \(|E_{1,2}^{(T_2)}| = \widetilde{O}(\varrho + \frac{1}{\sqrt{d}})[R_1^{(t)}]^{3/2}[R_2^{(t)}]^{3/2} = \widetilde{O}(\frac{1}{d})\).
    \end{enumerate}
    Moreover, \(|B_{2,2}^{(t)}|\) is increasing and \(R_2^{(t)}\) is decreasing. The part of learning \(|B_{2,2}^{(t)}|\) till \(\Omega(1)\) and keeping \(B_{2,1}^{(t)}\) close to its initialization is what's been accelerated by the prediction head \(E_{2,1}^{(t)}\).
\end{lemma}
The proof of \myref{lem:phase-3}{Lemma} will be proven after we have proven \myref{induct:phase-3}{Induction}, which will again be proven after some intermediate results are proven.

\begin{lemma}[The growth of \(B_{2,2}^{(t)}\) before \(T_{3,1}\)]\label{lem:thm-phase-3-stage1}
    Let \(T_{3,1}\) be defined as in \eqref{eqdef:T3.1-and-T3.2}. If \myref{induct:phase-3}{Induction} holds for \(t \in [T_2,T_{3,1}]\), then we have \(R_2^{(T_{3,1})} \leq \frac{\alpha_1^{12}}{d^{1/4}}\) and \(B_{2,2}^{(T_{3,1})} \in [\frac{1}{d^{1/4}},O(\frac{\alpha_1^{O(1)}}{d^{1/4}})]\) and \(T_{3,1}\leq T_2+\widetilde{O}(\frac{d^{1.625}\alpha_1^{O(1)}}{\eta})\).
\end{lemma}

\begin{proof}
    Firstly by \myref[b]{lem:reduce-noise-phase3}{Lemma} , we can write down the update of \(R_2^{(t)} \): (as in \myref{lem:phase-2}{Lemma}) 
    \begin{align*}
        R_2^{(t+1)} &=  R_2^{(t)} - \eta \Theta([R_{2}^{(t)}]^3)\Big(\Sigma_{1,1}^{(t)} \Theta((E_{1,2}^{(t)})^2) + \sum_{\ell\in[2]}\Sigma_{2,\ell}^{(t)} \Big)  \\
        &\quad \pm  O\Big(\sum_{j,\ell}\eta \Sigma_{j,\ell}^{(t)}E_{j,3-j}^{(t)} (\overline{R}_{1,2}^{(t)} + \varrho)[R_{1}^{(t)}]^{3/2}[R_{2}^{(t)}]^{3/2} \Big) \pm \frac{\eta}{\poly(d)}
    \end{align*}
    Next, by \myref{claim:noise}{Claim} and \myref[a]{lem:phase3-variables}{Lemma} combined with \myref[a,b]{induct:phase-3}{Induction}, we have \(\widetilde{O}(\frac{|E_{2,1}^{(t)}|}{d^{3/2}})\Sigma_{1,1}^{(t)}\frac{\Phi_1^{(t)} }{\Phi_2^{(t)}} \leq \widetilde{O}(\Sigma_{2,1}^{(t)}) \), which leads to the bound
    \begin{align*}
        \eta\Sigma_{1,1}^{(t)} \Theta((E_{1,2}^{(t)})^2) &\leq \widetilde{O}(\varrho^2+\frac{1}{d})\alpha_1^{O(1)}\eta \Sigma_{1,1}^{(t)}[R_1^{(t)}]^{3}[R_2^{(t)}]^{3} \leq  O(\frac{1}{d^{9/4}})\eta\Sigma_{1,1}^{(t)}[R_2^{(t)}]^{3} \leq O(\frac{\alpha_1^{O(1)}}{d^{3/4}})\eta\Sigma_{2,1}^{(t)}[R_2^{(t)}]^3
    \end{align*}
    Similarly, we can bound the following term
    \begin{align*}
        \sum_{\ell\in[2]} \eta\Sigma_{1,\ell}^{(t)}|E_{1,2}^{(t)}|(|\overline{R}_{1,2}^{(t)}| + \varrho) [R_{1}^{(t)}]^{3/2}[R_{2}^{(t)}]^{3/2}&\leq \widetilde{O}(\varrho^2+\frac{1}{d})\alpha_1^{O(1)}\sum_{\ell\in[2]}\eta\Sigma_{1,\ell}^{(t)}[R_1^{(t)}]^{3}[R_2^{(t)}]^{3} \\
        &\leq \widetilde{O}(\varrho^2+\frac{1}{d})\alpha_1^{O(1)}\frac{1}{d^{9/4}}\sum_{\ell\in[2]}\eta\Sigma_{1,\ell}^{(t)}[R_2^{(t)}]^{3} \\
        &\leq \widetilde{O}(\frac{\alpha_1^{O(1)} }{d^{3/4}})\eta\Sigma_{2,1}^{(t)}[R_2^{(t)}]^{3}
    \end{align*}
    Moreover, from \myref[c]{induct:phase-3}{Induction} that \(R_2^{(t)} \geq R_1^{(t)}\), we can also calculate for each \(t \in [T_2, T_{3,1}]\):
    \begin{align*}
        \eta  \Sigma_{2,\ell}^{(s)}|E_{2,1}^{(t)}|(|\overline{R}_{1,2}^{(t)}| + \varrho) [R_{1}^{(t)}]^{3/2}[R_{2}^{(t)}]^{3/2}&\leq \widetilde{O}(\varrho+\frac{1}{\sqrt{d}})\alpha_1^{O(1)} \eta\Sigma_{2,\ell}^{(t)} [R_2^{(t)}]^3 
    \end{align*}
    Thus by combining the results above, we have the update of \(R_2^{(t)}\) at \(t\in [T_2,T_{3}]\) as follows:
    \begin{align}
        R_2^{(t+1)} &=  R_2^{(t)} - \eta \Theta([R_{2}^{(t)}]^3)\Big(\Sigma_{1,1}^{(t)} \Theta((E_{1,2}^{(t)})^2) + \sum_{\ell\in[2]}\Sigma_{2,\ell}^{(t)} \Big) \nonumber\\
        & = R_2^{(t)} - \eta(\Sigma_{2,1}^{(t)}+\Sigma_{2,2}^{(t)} ) [R_{2}^{(t)}]^3 \label{eqdef:phase3-R2-update}
    \end{align}
    which implies that \(R_2^{(t)} \) is decreasing throughout phase III. From \myref[a]{lem:phase3-variables}{Lemma} and \myref[b]{induct:phase-3}{Induction}, we know that for \(t\in [T_2, T_{3,1}]\):
    \begin{align*}
        \Phi_2^{(t)} = Q_2^{(t)}/[U_2^{(t)}]^{3/2} = \Theta(\frac{1}{\sqrt{C_2[R_2^{(t)}]^3} (C_1\alpha_1^6(E_{2,1}^{(t)} )^2 )^{3/2}})
    \end{align*}
    which implies (also using a bit of \myref{claim:noise}{Claim} and \myref[a]{induct:phase-3}{Induction})
    \begin{align*}
        \Sigma_{2,1}^{(t)}[R_2^{(t)}]^3 &= (1 \pm \widetilde{O}(\frac{1}{d^{3/2}}))E_{2,1}^{(t)}\Delta_{2,1}^{(t)} \\
        &= (1 \pm \widetilde{O}(\frac{1}{d^{3/2}}))(1 \pm \widetilde{O}(\frac{1}{d^{3/2}}))C_0C_2\alpha_1^6\Phi_2^{(t)}E_{2,1}^{(t)}(B_{1,1}^{(t)})^3(B_{2,1}^{(t)})^3[R_2^{(t)}]^3 \\
        & = \Theta(\frac{C_2^{1/2}[R_2^{(t)}]^{3/2}}{(U_2^{(t)} )^{3/2}})C_0\alpha_1^6E_{2,1}^{(t)}(B_{1,1}^{(t)})^3(B_{2,1}^{(t)})^3 \\
        & = \Theta(\frac{C_0C_2^{1/2}|B_{2,1}^{(T_2)}|^3}{C_1^{3/2}\alpha_1^{3}|E_{2,1}^{(T_2)}|})[R_2^{(t)}]^{3/2}\tag{because \(B_{2,1}^{(t)} = B_{2,1}^{(T_2)}(1\pm o(1))\), \(B_{1,1}^{(t)} = \Theta(B_{1,1}^{(T_2)})\) and \(E_{2,1}^{(t)} = \Theta(E_{2,1}^{(T_2)})\sign(B_{1,1}^{(T_2)}B_{2,1}^{(T_2)})\)}
    \end{align*}
    And for \(\Sigma_{2,2}^{(t)} \), from some simple calcualtions (using \myref{claim:noise}{Claim}), we have 
    \begin{itemize}
        \item when \(|B_{2,2}^{(t)}|\leq \frac{\alpha_1}{\alpha_2}\sqrt{|B_{2,1}^{(T_2)}|}\), we would have \(\Sigma_{2,2}^{(t)} \leq O(\Sigma_{2,1}^{(t)})\);
        \item otherwise, we have \(\Sigma_{2,1}^{(t)} +\Sigma_{2,2}^{(t)} = \Theta(\Sigma_{2,2}^{(t)}) \).
    \end{itemize}
    So by \eqref{eqdef:phase3-R2-update}, we know \(R_2\) is decreasing for \(t\in[T_2,T_{3,1}]\) by at least
    \begin{align}\label{eqdef:phase3-R2-update-2}
        R_2^{(t+1)} \leq R_2^{(t)} - \eta\Theta(\frac{C_0C_2^{1/2}|B_{2,1}^{(T_2)}|^3}{C_1^{3/2}\alpha_1^{3}|E_{2,1}^{(T_2)}|})[R_2^{(t)}]^{3/2} \leq R_2^{(t)} ( 1 - \eta \zeta [R_2^{(t)}]^{1/2}) 
    \end{align}
    where \(\zeta := \Theta(\frac{C_0C_2^{1/2}|B_{2,1}^{(T_2)}|^3}{C_1^{3/2}\alpha_1^{3}|E_{2,1}^{(T_2)}|}) = \widetilde{\Theta}(\frac{\sqrt{\eta/\eta_E}}{d^{3/2}\alpha_1^3})\). By this update, we can prove \(T_{3,1} \leq T_2+ O(\frac{d^{3/2+1/8}\alpha_1^{O(1)}}{\eta})\). In order to do that, we can first see that for some \(t'_{3,1} \in [T_2+ \widetilde{\Theta}(\frac{d^{3/2}\alpha_1^2\sqrt{\eta_E/\eta}}{\eta}),T_2+ \widetilde{\Theta}(\frac{d^{3/2}\alpha_1^4\sqrt{\eta_E/\eta}}{\eta})]\), we shall have \(R_2^{(t'_{3,1})} \leq d^{-1/4}\). Indeed, suppose otherwise \(R_2^{(t'_{3,1}-1)}\geq d^{-1/4} \), then \eqref{eqdef:phase3-R2-update-2} implies
    \begin{align*}
        R_2^{(t'_{3,1})}  &\leq R_2^{(t'_{3,1}-1)} ( 1 - \eta\zeta [R_2^{(t'_{3,1}-1)}]^{1/2}) \leq R_2^{(t'_{3,1}-1)} ( 1 - \eta\zeta \frac{1}{d^{1/8}}) \\
        &\leq R_2^{(T_2)}\left( 1 -  \Theta(\frac{C_0C_2^{1/2}\sqrt{\eta/\eta_E}}{C_1^{3/2}d^{3/2}\alpha_1^3})\frac{\eta}{d^{1/8}}\right)^{t'_{3,1}-T_2-1} \\
        &\leq O(\sqrt{\eta_E/\eta}) \left( 1 -  \Theta(\frac{C_0C_2^{1/2}\sqrt{\eta/\eta_E}}{C_1^{3/2}d^{3/2}\alpha_1^3})\frac{\eta}{d^{1/8}}\right)^{t'_{3,1}-T_2-1} 
    \end{align*}
    which means there must exist an iteration \(t'_{3,1} \in [T_2+ \widetilde{\Theta}(\frac{d^{3/2}\alpha_1^2\sqrt{\eta_E/\eta}}{\eta}),T_2+ \widetilde{\Theta}(\frac{d^{3/2}\alpha_1^4\sqrt{\eta_E/\eta}}{\eta})]\) such that \(R_2^{(t'_{3,1}-1)} \geq d^{-1/4}\) (so the above update bound is still valid when the RHS is for \(t\leq t'_{3,1}-1\)) and \(R_2^{(t'_{3,1})} < d^{-1/4}\). Next we need to prove that at \(t  = t'_{3,1}\), it holds \(C_1\alpha_2^6(B_{2,2}^{(t)})^6 \geq C_2[R_2^{(t)}]^3 \). Let us discuss several possible cases:
    \begin{enumerate}
        \item Suppose \(|B_{2,2}^{(t'_{3,1})}|\geq \frac{\alpha_1}{\alpha_2}|B_{2,1}^{(T_1)}|^{1/2} \geq \Theta(\frac{1}{d^{1/4}})\) (by \myref[a]{induct:phase-3}{Induction} and \myref{lem:phase-3}{Lemma}), then we already have \(C_1\alpha_2^6(B_{2,2}^{(t'_{3,1})})^6 \geq C_2[R_2^{(t'_{3,1})}]^3\) and \(T_{3,1} \leq t'_{3,1}\);
        \item Suppose otherwise \(|B_{2,2}^{(t'_{3,1})}|\leq \frac{\alpha_1}{\alpha_2}|B_{2,1}^{(T_1)}|^{1/2}\), then we shall have \(\Sigma_{2,2}^{(t)}\leq O(\Sigma_{2,1}^{(t)}) \). So the update of \(R_2^{(t)}\) during \(t\in [T_2,T_{3,1}]\) can be written as 
        \begin{align*}
            R_2^{(t+1)} = R_2^{(t)} - \Theta(\eta\Sigma_{2,1}^{(t)})[R_2^{(t)}]^{3} = R_2^{(t)} ( 1 - \Theta(\eta \zeta) [R_2^{(t)}]^{1/2})
        \end{align*}
        Let \(t'_{3,2} = \min\{t: R_2^{(t)}\leq 2d^{-1/4}\}\) be an iteration between \(T_2\) and \(t'_{3,1}\), we shall have 
        \begin{align*}
            &\sum_{t\in [t'_{3,2},t'_{3,1}]} \eta \zeta [R_2^{(t)}]^{3/2} = \Theta(R_2^{(t'_{3,2})} - R_2^{(t'_{3,1})}) = \Theta(\frac{1}{d^{1/4}})\quad \text{and}\quad R_2^{(t)} \in [0.99\frac{1}{d^{1/4}}, 2.01\frac{1}{d^{1/4}}]
        \end{align*}
        which also implies \(t'_{3,1} - t'_{3,2} = \Theta(\frac{d^{1/8}}{\eta\zeta}) = \widetilde{\Theta}(\frac{d^{3/2+1/8}\alpha_1^3\sqrt{\eta_E/\eta}}{\eta}) \).
        In this case, let us look at the update of \(B_{2,2}^{(t)}\) at \(t \in [T_2,T_{3}]\). By \myref[2]{lem:learning-v2-phase3}{Lemma}, we have 
        \begin{align*}
            B_{2,2}^{(t+1)} = B_{2,2}^{(t)} + \eta (1\pm \widetilde{O}(\frac{1}{d}))\Lambda_{2,2}^{(t)} 
        \end{align*}
        It is not hard to see \(|B_{2,2}^{(t)}|\) is monotonically increasing. Also by \myref[a]{induct:phase-3}{Induction} and \myref[a]{lem:phase3-variables}{Lemma}, if we sum together the update between \(t'_{3,2}\) and \(t'_{3,1}\) as follows: (suppose the sign of \(B_{2,2}^{(t'_{3,2})} \) is positive for now, the negative case can be similarly dealt with)
        \begin{align*}
            B_{2,2}^{(t'_{3,2})} + \sum_{t \in [t'_{3,2},t'_{3,1}]}\eta(1\pm \widetilde{O}(\frac{1}{d}))\Lambda_{2,2}^{(t)}  &= \sum_{t \in [t'_{3,2},t'_{3,1}]}\Theta( \frac{\eta C_0C_1\alpha_1^6\alpha_2^6(E_{2,1}^{(T_2)})^2}{\sqrt{C_2[R_2^{(t)}]^3}(C_1\alpha_1^6(E_{2,1}^{(T_2)})^2)^{3/2}}) (B_{2,2}^{(t)})^5 \\
            &\geq B_{2,2}^{(t'_{3,2})} + (B_{2,2}^{(T_2)})^4\sum_{t \in [t'_{3,2},t'_{3,1}]}\Theta( \frac{\eta C_0\alpha_1^3\alpha_2^6B_{2,2}^{(t)}}{C_1^{1/2}C_2^{1/2}[R_2^{(t)}]^{3/2}|E_{2,1}^{(T_2)}|}) \\
            & \geq B_{2,2}^{(t'_{3,2})} \prod_{t=t'_{3,2}}^{t'_{3,1}} \Big(1 + \eta\widetilde{\Theta}(\frac{\alpha_1^3\alpha_2^6 }{d^{3/2+1/8}\sqrt{\eta_E/\eta}})\Big) \\
            & \geq \widetilde{\Theta}(\frac{1}{\sqrt{d}}) \Big(1 + \eta\widetilde{\Theta}(\frac{\alpha_1^3\alpha_2^6 }{d^{3/2+1/8}\sqrt{\eta_E/\eta}})\Big)^{\widetilde{\Theta}(d^{3/2}\alpha_1^3\sqrt{\eta_E/\eta}/\eta)} \\
            &\geq \Omega(e^{\alpha_1})
        \end{align*}
        which is a contradiction to our assumption \(|B_{2,2}^{(t'_{3,1})}| \leq \frac{\alpha_1}{\alpha_2}|B_{2,1}^{(T_1)}|^{1/2}\). Since \(|B_{2,2}^{(t)}| \) is monotonically increasing, we know there must exist some iteration \(t\leq t'_{3,1}\) such that \(|B_{2,2}^{(t)}| \geq \frac{\alpha_1}{\alpha_2}|B_{2,1}^{(T_1)}|^{1/2}\), which means \(T_{3,1} \leq t'_{3,1}\).
    \end{enumerate}
    Thus we proved the bound of \(T_{3,1} \leq T_2 + \widetilde{\Theta}(\frac{d^{3/2}\alpha_1^{O(1)}}{\eta})\). 
    
    Using similar arguments, we can prove that \(R_2^{(T_{3,1})}\leq \frac{\alpha_1^{O(1)}}{d^{1/4}}\). Indeed, we can set \(T_{3,3}:= \min\{t: |B_{2,2}^{(t'_{3,1})}| \geq \frac{\alpha_1}{\alpha_2}|B_{2,1}^{(T_1)}|^{1/2}\}\). From our arguments in this proof, we know \(\Sigma_{2,2}^{(t)} \leq O(\Sigma_{2,2}^{(t)})\) for \(t\leq T_{3,3}\). Now we can further choose \(t'_{3,3} = \min\{t:R_2^{(t)}\leq a \}\) for some \(a = \frac{\alpha_1^{12}}{d^{1/4}}\) to be some iteration with \(R_2^{(t)} \geq a\) for \(t \in [T_2, t'_{3,3}]\) and \(t'_{3,3} - T_2 = \Theta(\frac{\sqrt{a}\log d}{\eta\zeta})\). Now we can work out the update of \(B_{2,2}^{(t)}\) during \(t\in [T_2,t'_{3,3}]\) again to see that \(|B_{2,2}^{(t'_{3,3})}| \leq B_{2,2}^{(T_2)}\Big(1 + \eta\widetilde{\Theta}(\frac{\alpha_1^3\alpha_2^6 }{d^{2}a^{3/2}\sqrt{\eta_E/\eta}})\Big)^{\frac{\sqrt{a}}{\eta\zeta}} \leq \widetilde{O}(\frac{1}{\sqrt{d}}) \). This would prove that \(t'_{3,3}\leq T_{3,3}\) and \(R_{2}^{(T_{3,3})}\leq \frac{\alpha_1^{O(1)}}{d^{1/4}}\). So we also have \(|B_{2,2}^{(T_{3,3})}|\leq \frac{\alpha_1^{O(1)}}{d^{1/4}} \) because of the definition of \(T_{3,1}\). But since \(T_{3,3}\geq T_{3,1}\) by our arguments above and the fact that \(|B_{2,2}^{(t)}|\) is increasing, we shall have \(|B_{2,2}^{(T_{3,1})}|\in [\frac{1}{d^{1/4}}, \frac{\alpha_1^{O(1)}}{d^{1/4}}] \).
\end{proof}

Now we proceed to characterize the learning of \(B_{2,2}^{(t)}\) during \(t\in [T_{3,1},T_{3,2}]\).

\begin{lemma}[The growth of \(B_{2,2}^{(t)}\) until \(T_{3}\)]\label{lem:thm-phase-3-stage2}
    Let \(T_{3,1}, T_{3,2}\) be defined as in \eqref{eqdef:T3.1-and-T3.2}. If \myref{induct:phase-3}{Induction} holds true for all \(t \in [T_2,T_{3}] \), then we have \(T_{3,2} = T_{3,1} + \widetilde{O}(\frac{d^{1/4}\alpha_1^{O(1)}}{\eta})\) and \(T_{3}\leq T_{3,2} + \widetilde{O}(\frac{\alpha_1^{O(1)}}{\eta})\).
\end{lemma}

\begin{proof}
    We first calculate the bound for \(T_{3,2}\). After \(T_{3,1}\), since \(|B_{2,2}^{(t)}|\) is increasing while \(R_2^{(t)}\) is decreasing by \myref{induct:phase-3}{Induction}. So by \myref[a]{lem:phase3-variables}{Lemma}, we have 
    \begin{align*}
        [Q_2^{(t)}]^{-2} =\Theta(C_1\alpha_2^6(B_{2,2}^{(t)})^6 ),\quad \Phi_2^{(t)} = Q_2^{(t)}/[U_2^{(t)}]^{3/2} = \Theta((C_1^{3/2}\alpha_2^3\alpha_1^{9}|B_{2,2}^{(t)})|^3|E_{2,1}^{(t)}|^3)^{-1})
    \end{align*}
    So according to \myref{lem:learning-v2-phase3}{Lemma}, we would have for all \(t\in [T_{3,1},T_{3,2})\):
    \begin{align*}
        \dbrack{-\nabla_{w_2}L(W^{(t)},E^{(t)}),v_2} = (1 \pm o(1))\Lambda_{2,2}^{(t)} = \Theta(\frac{1}{C_1^{3/2}\alpha_1^{9}|E_{2,1}^{(T_2)}|^3})(B_{2,2}^{(t)})^2\sign(B_{2,2}^{(t)})
    \end{align*}
    where we have used \((E_{2,1}^{(t)})^3 = \Theta((E_{2,1}^{(T_2)})^3)\) from \myref[a]{induct:phase-3}{Induction}. So when \(t \in [T_{3,1},T_{3,2}]\), we can write down the explicit form of \(\Lambda_{2,2}^{(t)}\) and use \myref[d]{lem:phase3-variables}{Lemma} to derive
    \begin{align*}
        |B_{2,2}^{(t+1)}| &= |B_{2,2}^{(t)}| + \eta \Theta(\frac{C_1\alpha_1^6|E_{2,1}^{(T_2)}|^2}{C_1^{3/2}\alpha_1^{9}|E_{2,1}^{(T_2)}|^3})(B_{2,2}^{(t)})^2 \\
        & \geq |B_{2,2}^{(t)}|\left(1 + \Theta(\frac{1}{C_1\alpha_1^{O(1)}})|B_{2,2}^{(T_{3,1})}|\right) \\
        &\geq |B_{2,2}^{(t)}|\left(1 + \Theta(\frac{1}{C_1\alpha_1^{O(1)}})\frac{1}{d^{1/4}}\right)
    \end{align*}
    Thus after \(\widetilde{O}(\frac{d^{1/4}\alpha^{O(1)}}{\eta})\) many iterations, we would have \(|B_{2,2}^{(t)}|\geq \frac{1}{3}\min\{|E_{2,1}^{(t)}|,|B_{1,1}^{(t)}|\}\). Now let us deal with the growth of \(|B_{2,2}^{(t)}|\) at \(t\in [T_{3,2},T_{3,3}]\). During this stage, since \(B_{2,2}^{(t)}\) is still increasing and \(|E_{2,1}^{(t)}|  = |E_{2,1}^{(T_2)}|\) by \myref{induct:phase-3}{Induction}, we have from \myref[a]{lem:phase3-variables}{Lemma} that 
    \begin{align*}
        \Phi_2^{(t)} = Q_2^{(t)}/[U_2^{(t)}]^{3/2} = \Theta(\frac{1}{C_1^2\alpha_2^{12}(B_{2,2}^{(t)})^{12} }) \geq \Theta(\frac{1}{C_1^2\alpha_1^{O(1)}})
    \end{align*}
    And we can redo the calcualtions as above to get \(T_{3}\leq T_{3,2} + \widetilde{O}(\frac{\alpha_1^{O(1)}}{\eta})\) since \(\sqrt{\eta/\eta_E}|E_{2,1}^{(t)} |\) and \(|B_{1,1}^{(t)}|\) are both \(\Theta(1)\) according to \myref[a,b]{induct:phase-3}{Induction}
\end{proof}

\paragraph{Proving The Main Lemma.} Now we finally begin to prove \myref{lem:phase-3}{Lemma}.

\begin{proof}[Proof of \myref{lem:phase-3}{Lemma}]
    We start with proving \myref{induct:phase-3}{Induction}. \\
    \newline
    \textbf{Proof of \myref[a]{induct:phase-3}{Induction}:} From \myref{lem:learning-v1-phase3}{Lemma}, we know the update of \(B_{1,1}^{(t)}\) can be written as 
    \begin{align*}
        B_{1,1}^{(t+1)} &= B_{1,1}^{(t)} + \Theta(\eta\Sigma_{1,1}^{(t)}[R_1^{(t)}]^3) \pm \eta O(\frac{(B_{1,2}^{(t)})^3}{(B_{2,2}^{(t)})^3} + \frac{1}{\sqrt{d}} )\alpha_1^{O(1)}\Lambda_{2,2}^{(t)} + \frac{E_{2,1}^{(t)}}{B_{1,1}^{(t)}}\eta\Delta_{2,1}^{(t)} - \frac{B_{2,2}^{(t)} }{B_{1,1}^{(t)}}\eta\Lambda_{2,2}^{(t)} 
    \end{align*}
    Since from \myref{lem:thm-phase-3-stage1}{Lemma} and \myref{lem:thm-phase-3-stage2}{Lemma}, we know \(T_3 \leq \widetilde{O}(\frac{d^{1.625}\alpha_1^{O(1)}}{\eta})\) and from \myref{claim:noise}{Claim} and \myref[a,c]{induct:phase-3}{Induction} we have \(\Sigma_{1,1}^{(t)}[R_1^{(t)}]^3\leq \widetilde{O}(\frac{\alpha_1^{O(1)}}{d^{2.25}}) \), we shall have 
    \begin{align*}
        \sum_{s\in [T_2,t)} \Theta(\eta\Sigma_{1,1}^{(s)}[R_1^{(s)}]^3) \leq \widetilde{O}(\frac{d^{1.625}\alpha_1^{O(1)}}{\eta}) \widetilde{O}(\frac{\eta\alpha_1^{O(1)}}{d^{2.25}}) \leq \frac{1}{\sqrt{d}} = o(1)
    \end{align*}
    Further more, by applying \myref{lem:TPM-degree}{Lemma} to \(x_t = B_{2,2}^{(t)}\) with \(q' = q-2\), and notice that \(\sign(B_{j,2}^{(t)}) = \sign(B_{j,2}^{(T_2)})\) for all \(t\in[T_2,T_3]\), we also have 
    \begin{align*}
        \left|\sum_{s\in [T_2,t)} O(\frac{(B_{1,2}^{(s)})^3}{(B_{2,2}^{(s)})^3})\alpha_1^{O(1)}\eta\Lambda_{2,2}^{(s)}\right| \leq \widetilde{O}(\frac{\alpha_1^{O(1)}}{\sqrt{d}})
    \end{align*}
    Now we turn to the last two terms. We first see that from the expression \eqref{eqdef:phase3-R2-update} of \(R_2^{(t)}\)'s update, we have that (note that \(\sign(E_{2,1}^{(t)}\Delta_{2,1}^{(t)}) = 1\))
    \begin{align*}
        \sum_{s\in [T_2,t)}\frac{E_{2,1}^{(s)}}{|B_{1,1}^{(s)}|}\eta\Delta_{2,1}^{(s)} = \sum_{s\in [T_2,t)}\frac{1}{|B_{1,1}^{(s)}|}\Theta(\eta\Sigma_{2,1}^{(s)}[R_2^{(s)}]^3) = \Theta(\frac{\sqrt{\eta_E/\eta}}{ |B_{1,1}^{(T_2)}|}) = \Theta(\sqrt{\eta_E/\eta})
    \end{align*}
    where we have used the fact that \(\Sigma_{2,1}^{(t)}[R_2^{(t)}]^3 = (1 \pm O(\frac{1}{d}))E_{2,1}^{(t)}\Delta_{2,1}^{(t)}\) and \(\sum_{s\in [T_2,t)}\eta \Sigma_{2,1}^{(s)}[R_2^{(s)}]^3 \lesssim R_2^{(T_2)}\) from \eqref{eqdef:phase3-R2-update} (which holds for all \(t\in[T_2,T_3]\)). And also, the analysis above shows that 
    \begin{align*}
        |B_{1,1}^{(t)}| = |B_{1,1}^{(T_2)}| + O(\sqrt{\eta_E/\eta})  - \sum_{s\in [T_2,t]}\frac{B_{2,2}^{(s)} }{B_{1,1}^{(s)}}\eta\Lambda_{2,2}^{(s)}
    \end{align*}
    for all \(t\in[T_2,T_3]\), which means that either \(\sum_{s\in [T_2,t]}\frac{B_{2,2}^{(s)} }{|B_{1,1}^{(s)}|}\eta\Lambda_{2,2}^{(s)} \leq \sum_{s\in [T_2,t)}\frac{E_{2,1}^{(s)}}{|B_{1,1}^{(s)}|}\eta\Delta_{2,1}^{(s)}\) and we have \(|B_{1,1}^{(t)}|\geq |B_{1,1}^{(T_2)}|\) holds throughout \(t\in[T_2,T_3]\), or that \(\sum_{s\in [T_2,t]}\frac{B_{2,2}^{(s)} }{|B_{1,1}^{(s)}|}\eta\Lambda_{2,2}^{(s)} \geq \Omega(\sqrt{\eta_E/\eta})\), in which case we would have \(|B_{1,1}^{(t)}|\) to be actually decreasing (as \(B_{2,2}^{(t)}\) is increasing). Now that since \(B_{1,1}^{(T_2)} = \Theta(1)\), we can easily see by our definition of \(T_3\) and the monotonicity of \(B_{1,1}^{(t)}\) after going below \(B_{1,1}^{(T_2)} -\Omega(\sqrt{\eta_E/\eta})\) that \(B_{1,1}^{(t)} \geq 0.49B_{1,1}^{(T_2)} = \Omega(1)\) for all \(t\in[T_2,T_3]\).

    Next let us look at the change of \(B_{2,1}^{(t)}\). From \myref{lem:learning-v1-phase3}{Lemma}, we can write down the update of \(B_{2,1}^{(t)}\):
    \begin{align*}
        B_{2,1}^{(t+1)} = B_{2,1}^{(t)} + \widetilde{O}(\frac{\alpha_1^{O(1)}}{d^{5/2}})\eta \Phi_2^{(t)}[R_2^{(t)}]^3  \pm \widetilde{O}(\frac{\alpha_1^{O(1)}}{d})\eta\Lambda_{2,2}^{(t)} \pm \widetilde{O}(\frac{\eta\alpha_1^{O(1)}}{d^3}) 
    \end{align*}
    For the first term, according to \myref{lem:thm-phase-3-stage1}{Lemma} and \myref{lem:thm-phase-3-stage2}{Lemma} and \(R_2^{(t)}\leq O(\sqrt{\eta_E/\eta}) = o(1)\) for all \(t\in[T_2,T_3]\) by \myref[c]{induct:phase-3}{Induction}, we have \(\Phi_2^{(t)}[R_2^{(t)}]^3 \leq \alpha_1^{O(1)}\) for all \(t\in[T_2,T_3]\) and 
    \begin{align*}
        \sum_{s\in [T_2,t]}\widetilde{O}(\frac{\alpha_1^{O(1)}}{d^{5/2}})\eta \Phi_2^{(s)}[R_2^{(s)}]^3 \leq \widetilde{O}(\frac{d^{1.625}\alpha_1^{O(1)}}{\eta})\eta \widetilde{O}(\frac{\alpha_1^{O(1)}}{d^{5/2}}) \leq \widetilde{O}(\frac{\alpha_1^{O(1)}}{d^{7/8}})
    \end{align*}
    And similarly as in the proof of induction for \(B_{1,1}^{(t)}\), we have 
    \begin{align*}
        \sum_{s\in [T_2,t]}\widetilde{O}(\frac{\alpha_1^{O(1)}}{d})\eta\Lambda_{2,2}^{(s)} \leq \widetilde{O}(\frac{\alpha_1^{O(1)}}{d}),\quad \sum_{s\in [T_2,t]}\widetilde{O}(\frac{\eta\alpha_1^{O(1)}}{d^3}) \leq \widetilde{O}(\frac{\alpha_1^{O(1)}}{d})
    \end{align*}
    which proved the induction for \(B_{2,1}^{(t)}\) since \(|B_{2,1}^{(T_2)}| = \widetilde{\Theta}(\frac{1}{\sqrt{d}})\).\\
    Next we go on for the induction of \(B_{1,2}^{(t)}\), we write down its update:
    \begin{align*}
        B_{1,2}^{(t+1)} & = B_{1,2}^{(t)} + \Theta(\frac{(B_{1,2}^{(t)})^2}{(B_{2,2}^{(t)})^2})E_{2,1}^{(t)} \eta \Lambda_{2,2}^{(t)}  \pm \eta \widetilde{O}(\frac{\alpha_1^{O(1)}}{d^{4}})|E_{2,1}^{(t)}|^2\Phi_2^{(t)} \pm \eta \widetilde{O}(\frac{\alpha_1^{O(1)}}{d^{5/2}})
    \end{align*}
    By \myref{lem:thm-phase-3-stage1}{Lemma} and \myref{lem:thm-phase-3-stage2}{Lemma}, we have for any \(t \in [T_2,T_3]\)
    \begin{align*}
        \sum_{s\in [T_2,t]} \eta \widetilde{O}(\frac{\alpha_1^{O(1)}}{d^{5/2}}) \leq \frac{1}{\sqrt{d}\polylog (d)}
    \end{align*}
    and also
    \begin{align*}
        \sum_{s\in [T_2,t]}\eta \widetilde{O}(\frac{\alpha_1^{O(1)}}{d^{4}})|E_{2,1}^{(t)}|^2\Phi_2^{(t)} &\leq \left(\sum_{s\in [T_2,T_{3,1}]}+ \sum_{s\in [T_{3,1},T_{3}]}\right)\eta \widetilde{O}(\frac{\alpha_1^{O(1)}}{d^{4}})|E_{2,1}^{(t)}|^2\Phi_2^{(t)} \\
        &\leq \eta\widetilde{O}(\frac{\alpha_1^{O(1)}}{d^{4}})\cdot (T_{3,1}-T_2)\cdot O(\alpha_1^{O(1)}d^{3/8}) + \eta\widetilde{O}(\frac{\alpha_1^{O(1)}}{d^{4}})(T_{3} - T_{3,1}) \\
        &\leq \widetilde{O}(\frac{\alpha_1^{O(1)}}{d^{2}})
    \end{align*}
    Now we consider the term \(\Theta(\frac{(B_{1,2}^{(t)})^2}{(B_{2,2}^{(t)})^2})E_{2,1}^{(t)} \eta \Lambda_{2,2}^{(t)}\), we have by \myref[a]{induct:phase-3}{Induction} that 
    \begin{align*}
        \left|\sum_{s\in [T_2,t]}\Theta(\frac{(B_{1,2}^{(t)})^2}{(B_{2,2}^{(t)})^2})E_{2,1}^{(t)} \eta \Lambda_{2,2}^{(t)}\right| \leq O(\sqrt{\eta_E/\eta}(B_{1,2}^{(T_2)})^2) \sum_{s\in [T_2,t]} \eta \frac{|\Lambda_{2,2}^{(t)}|}{(B_{2,2}^{(t)})^2}
    \end{align*}
    where we have used our induction hypothesis that \(B_{1,2}^{(t)} =  B_{1,2}^{(T_2)}(1\pm o(1))\). Using \myref{lem:TPM-degree}{Lemma} by setting \(x_t = B_{2,2}^{(t)}\), \(q'=3\), and \(A = \Theta(1) \geq d^{\Omega(1)}B_{2,2}^{(T_2)}\), it holds that 
    \begin{align*}
        \left|\sum_{s\in [T_2,t]}\Theta(\frac{(B_{1,2}^{(t)})^2}{(B_{2,2}^{(t)})^2})E_{2,1}^{(t)} \eta \Lambda_{2,2}^{(t)}\right| \leq O(\sqrt{\eta_E/\eta})\frac{(B_{1,2}^{(T_2)})^2}{ |B_{2,2}^{(T_2)}| } \leq O(\sqrt{\eta_E/\eta})\frac{(B_{1,2}^{(0)})^2}{ |B_{2,2}^{(0)}| }  \leq \frac{1}{\sqrt{d}\polylog(d)}
    \end{align*}
    where in the second inequality we have used \myref[c]{lem:phase-1}{Lemma}, \myref[a]{lem:phase-2}{Lemma} and \myref{property-init}{Lemma}, and in the last our choice of \(\eta_E/\eta \leq \frac{1}{\polylog(d)}\). This ensures the induction can go on until \(t=T_3\). And we finished our proof of \myref[a]{induct:phase-3}{Induction}.
    \\
    \newline
    \textbf{Proof of \myref[b]{induct:phase-3}{Induction}:} Let us write down the update of \(E_{1,2}^{(t)}\) using \myref{lem:learning-pred-head-phase3}{Lemma}:
    \begin{align*}
        E_{1,2}^{(t+1)} &= E_{1,2}^{(t)}(1 -\eta_E \Xi_1^{(t)}) +  \sum_{\ell\in[2]}\Theta(\eta_E\Sigma_{1,\ell}^{(t)}) (-E_{1,2}^{(t)}[R_{2}^{(t)}]^{3} \pm O(\overline{R}_{1,2}^{(t)} + \varrho)[R_{1}^{(t)}]^{3/2}[R_{2}^{(t)}]^{3/2}) + \sum_{\ell\in[2]}\eta_E \Delta_{1,\ell}^{(t)} \\
        & = E_{1,2}^{(t)}(1 -\eta_E \Xi_1^{(t)} - \sum_{\ell\in[2]} \Theta(\eta_E\Sigma_{1,\ell}^{(t)})[R_{2}^{(t)}]^{3} ) + \widetilde{O}(\frac{\eta_E}{d^{3/2}})\Phi_1^{(t)}[R_1^{(t)}]^3 \\
        &\quad \pm \widetilde{O}(\varrho + \frac{1}{\sqrt{d}})\sum_{\ell\in[2]}\eta_E\Sigma_{1,\ell}^{(t)}[R_{1}^{(t)}]^{3/2}[R_{2}^{(t)}]^{3/2} \\
        & = E_{1,2}^{(t)}(1 -\eta_E \Xi_1^{(t)} -  \Theta(\eta_E\Sigma_{1,1}^{(t)})[R_{2}^{(t)}]^{3} ) \pm \widetilde{O}(\varrho + \frac{1}{\sqrt{d}})\eta_E\Sigma_{1,1}^{(t)}[R_{1}^{(t)}]^{3/2}[R_{2}^{(t)}]^{3/2}
    \end{align*}
    where in the last inequality we have used \(R_2^{(t)} \geq R_1^{(t)}\) from \myref[c]{induct:phase-3}{Induction} and \(\Sigma_{1,1}^{(t)} \geq \Omega(\Phi_1^{(t)})\), \(\Sigma_{2,1}^{(t)} \leq \widetilde{O}(\frac{1}{d^{3/2}})\Sigma_{1,1}^{(t)}\) from \myref{claim:noise}{Claim} and \myref[a]{induct:phase-3}{Induction}. Now we can use the same analysis in the proof of \myref{lem:phase-2}{Lemma} on \(E_{1,2}^{(t)}\) to prove the desired claim, which we do not repeat here. 
    
    As for \(E_{2,1}^{(t)}\), we can obtain similar expressions:
    \begin{align*}
        E_{2,1}^{(t+1)}& = E_{2,1}^{(t)}(1 -\eta_E \Xi_2^{(t)} -  \sum_{\ell\in[2]}\Theta(\eta_E\Sigma_{2,\ell}^{(t)})[R_{1}^{(t)}]^{3} ) \\
        &\quad \pm \widetilde{O}(\varrho + \frac{1}{\sqrt{d}})\sum_{\ell\in[2]}\Theta(\eta_E\Sigma_{2,\ell}^{(t)})[R_{1}^{(t)}]^{3/2}[R_{2}^{(t)}]^{3/2} + \sum_{\ell\in[2]}\eta_E \Delta_{2,\ell}^{(t)} 
    \end{align*}
    Now we can obtain bounds for each terms as 
    \begin{align*}
        \sum_{s\in [T_2,t]}\sum_{\ell\in[2]}\Theta(\eta_E\Sigma_{2,\ell}^{(s)})[R_{1}^{(s)}]^{3} \leq \widetilde{O}(\frac{\eta_E\alpha_1^{O(1)}}{d^2}) \cdot \widetilde{O}(\frac{d^{1.625}\alpha_1^{O(1)}}{\eta}) \leq \frac{1}{d^{3/4}}
    \end{align*}
    and by \eqref{eqdef:phase3-R2-update} in \myref{lem:thm-phase-3-stage1}{Lemma}, we also have for any \(t\in [T_2,T_3]\)
    \begin{align*}
        \sum_{s\in [T_2,t]}\widetilde{O}(\varrho + \frac{1}{\sqrt{d}})\sum_{\ell\in[2]}\Theta(\eta_E\Sigma_{2,\ell}^{(s)})[R_{1}^{(s)}]^{3/2}[R_{2}^{(s)}]^{3/2} &\leq \widetilde{O}(\varrho + \frac{1}{\sqrt{d}})\sum_{s\in [T_2,t]}\sum_{\ell\in[2]}\Theta(\eta_E\Sigma_{2,\ell}^{(s)})[R_{2}^{(s)}]^{3} \\
        &\leq \widetilde{O}(\varrho + \frac{1}{\sqrt{d}})R_2^{(T_2)} \\
        &\leq \widetilde{O}(\varrho + \frac{1}{\sqrt{d}})
    \end{align*}
    And also by using our induction and by \eqref{eqdef:phase3-R2-update} in \myref{lem:thm-phase-3-stage1}{Lemma}:
    \begin{align*}
        \sum_{s\in [T_2,t]}\sum_{\ell\in[2]}\eta_E \Delta_{2,\ell}^{(s)} \leq \sum_{s\in [T_2,t]} \frac{\eta_E/\eta}{|E_{2,1}^{(t)}|}\Theta(\eta\Sigma_{2,1}^{(s)} + \eta\Sigma_{2,2}^{(s)})[R_{2}^{(s)}]^{3} \leq \frac{\eta_E/\eta}{|E_{2,1}^{(T_2)}|} R_{2}^{(T_2)} \leq O(\frac{\eta_E/\eta}{\log d}) = o(\sqrt{\eta_E/\eta})
    \end{align*}
    Finally, we can calculate 
    \begin{align*}
        \sum_{s\in [T_2,t]} \eta_E \Xi_2^{(t)}E_{2,1}^{(t)} =\sum_{s\in [T_2,t]} \frac{\eta_E}{\eta}\frac{B_{2,2}^{(t)}}{E_{2,1}^{(t)}}\eta\Lambda_{2,2}^{(t)}
    \end{align*}
    By resorting to the defintion of \(T_3\) and go through similar analysis as for the induction of \(B_{1,1}^{(t)}\), we can obtain that \(|E_{2,1}^{(t)}|\) is either above \(|E_{2,1}^{(T_2)}|(1+o(1))\) or is decreasing and always above \(\frac{1}{2}|E_{2,1}^{(T_2)}|\). This proves \myref[b]{induct:phase-3}{Induction}.
    \\
    \newline
    \textbf{Proof of \myref[c]{induct:phase-3}{Induction}:} The proof of induction of \(R_2^{(t)}\) is half done in \myref{lem:thm-phase-3-stage1}{Lemma}, we only need to complete the part when \(t\in[T_{3,1},T_3]\), since by \eqref{eqdef:phase3-R2-update}, we always have  \(R_2^{(t)}\) to be decreasing by
    \begin{align*}
        R_2^{(t+1)} = R_2^{(t)}(1 - \sum_{\ell\in[2]}\Theta(\eta\Sigma_{2,\ell}^{(s)})[R_{2}^{(t)}]^{2})
    \end{align*}
    And when \(t\in [T_{3,1},T_3]\), we have 
    \begin{align*}
        \sum_{\ell\in[2]}\Theta(\eta\Sigma_{2,\ell}^{(s)} \leq \widetilde{O}(\eta d^{3/8+o(1)})
    \end{align*}
    So if we suppose \(R_2^{(T_3)}\leq \frac{1}{\sqrt{d}}\), we shall have for \(T_{3} - T_{3,1} = O(d^{1/4+o(1)}/\eta)\) many iterations that 
    \begin{align*}
        R_2^{(t+1)} \geq R_2^{(T_{3,1})} (1 - \frac{\eta}{d^{5/8}} )^{T_3-T_{3,1}} \geq \Omega(R_2^{(T_{3,1})}) \geq \frac{1}{d^{1/4}} \tag{by \myref{lem:thm-phase-3-stage1}{Lemma}}
    \end{align*}
    So it negates our supposition, which completes the proof of the induction for \(R_2^{(t)}\) in \(t\in [T_2,T_3]\). 

    Now we turn to the proof of induction for \(R_1^{(t)}\), we write down its update: (as in \myref{lem:phase-2}{Lemma})
    \begin{align*}
        R_1^{(t+1)} &= R_1^{(t)} - \Theta(\eta [R_{1}^{(t)}]^3) \Big(\Sigma_{1,1}^{(t)}  + \sum_{\ell\in[2]}\Sigma_{2,\ell}^{(t)}(E_{2,1}^{(t)})^2 \Big) \\
        &\quad \pm  O\Big(\sum_{j,\ell}\eta\Sigma_{j,\ell}^{(t)}E_{j,3-j}^{(t)} (\overline{R}_{1,2}^{(t)} + \varrho)[R_{1}^{(t)}]^{3/2}[R_{2}^{(t)}]^{3/2}\Big) \pm \frac{\eta}{\poly(d)}
    \end{align*}
    It is straightforward to derive
    \begin{align*}
        \sum_{\ell\in[2]}\Sigma_{1,\ell}^{(t)}|E_{1,2}^{(t)}| (\overline{R}_{1,2}^{(t)} + \varrho)[R_{1}^{(t)}]^{3/2}[R_{2}^{(t)}]^{3/2} \leq \widetilde{O}(\varrho+\frac{1}{\sqrt{d}})^2\sum_{\ell\in[2]}\Sigma_{1,\ell}^{(t)} [R_{1}^{(t)}]^{3}[R_{2}^{(t)}]^{3}
    \end{align*}
    and when \(t\in [T_{2},T_{3,1}]\):
    \begin{align*}
        \sum_{s\in [T_2,t]}\sum_{\ell\in[2]}\eta \Sigma_{2,\ell}^{(s)}|E_{2,1}^{(s)}| (\overline{R}_{1,2}^{(s)} + \varrho)[R_{1}^{(s)}]^{3/2}[R_{2}^{(s)}]^{3/2} &\leq \widetilde{O}(\varrho+\frac{1}{\sqrt{d}})\frac{d^{o(1)}d^{3/8}}{d^{9/8}}\sum_{s\in [T_2,t]}\sum_{\ell\in[2]}\eta\Sigma_{2,\ell}^{(s)}[R_{2}^{(s)}]^{3} \\
        &\leq o(\frac{d^{o(1)}}{d^{3/4}})
    \end{align*}
    and when \(t\in [T_{3,1},T_{3}]\):
    \begin{align*}
        \sum_{s\in [T_2,t]}\sum_{\ell\in[2]}\eta \Sigma_{2,\ell}^{(s)}|E_{2,1}^{(s)}| (\overline{R}_{1,2}^{(s)} + \varrho)[R_{1}^{(s)}]^{3/2}[R_{2}^{(s)}]^{3/2} \leq \widetilde{O}(\varrho+\frac{1}{\sqrt{d}})\frac{d^{o(1)}d^{3/8}}{d^{9/8}} \eta \widetilde{O}(\frac{d^{1/4+o(1)}}{\eta}) \leq O(\frac{1}{d})
    \end{align*}
    So these combined with \myref{lem:phase-2}{Lemma} proved that \(R_1^{(t)}\leq O(\frac{d^{o(1)}}{d^{3/4}})\) for all \(t\in[T_2,T_3]\). We can go through some similar analysis about \(R_2^{(t)}\) to get that \(R_1^{(t)}\geq \frac{1}{d}\) for all \(t\in[T_2,T_3]\). 

    Finally we begin to prove the induction of \(\overline{R}_{1,2}^{(t)}\). Similarly as in the proof of \myref{lem:phase-2}{Lemma}, we first write down  
    \begin{align*}
        R_{1,2}^{(t+1)} &= R_{1,2}^{(t)} - \eta\dbrack{\nabla_{w_1}L(W^{(t)},E^{(t)}),\Pi_{V^{\perp}} w_2^{(t)}} - \eta\dbrack{\nabla_{w_2}L(W^{(t)},E^{(t)}),\Pi_{V^{\perp}} w_1^{(t)}}\\
        &\quad + \eta^2\dbrack{\Pi_{V^{\perp}}\nabla_{w_1}L(W^{(t)},E^{(t)}),\Pi_{V^{\perp}}\nabla_{w_2}L(W^{(t)},E^{(t)})} \\ 
        & = R_{1,2}^{(t)} + \eta\Big(\Sigma_{1,1}^{(t)} + \sum_{\ell\in[2]}\Sigma_{2,\ell}^{(t)}(E_{2,1}^{(t)})^2 \Big)(-\Theta(\overline{R}_{1,2}^{(t)}) \pm O(\varrho) )[R_{1}^{(t)}]^{5/2}[R_{2}^{(t)}]^{1/2} \\
        &\quad + \eta \Big(\Sigma_{1,1}^{(t)} \Theta(E_{1,2}^{(t)})^2 + \sum_{\ell\in[2]}\Sigma_{2,\ell}^{(t)} \Big)(-\Theta(\overline{R}_{1,2}^{(t)}) \pm O(\varrho) )[R_{2}^{(t)}]^{5/2}[R_{1}^{(t)}]^{1/2} \\
        &\quad + O\Big(\sum_{(j,\ell)\neq (1,2)}\eta\Sigma_{j,\ell}^{(t)}E_{j,3-j}^{(t)}(R_{1}^{(t)}[R_{2}^{(t)}]^{2}+R_{2}^{(t)}[R_{1}^{(t)}]^{2})\Big) \pm \frac{\eta}{\poly(d)} 
    \end{align*}
    Note that since \(|E_{1,2}^{(t)}| \leq \widetilde{O}(\varrho+\frac{1}{\sqrt{d}}) [R_{2}^{(t)}]^{3/2}[R_{1}^{(t)}]^{3/2}\) and \(R_1^{(t)} \leq O(\frac{1}{d^{3/4 }})\), it holds
    \begin{align*}
        \sum_{(j,\ell)\neq (1,2)}\eta\Sigma_{j,\ell}^{(t)}|E_{j,3-j}^{(t)}|R_{2}^{(t)}[R_{1}^{(t)}]^{2} &\leq \sum_{(j,\ell)\neq (1,2)}\eta\Sigma_{j,\ell}^{(t)}|E_{j,3-j}^{(t)}|R_{1}^{(t)}[R_{2}^{(t)}]^{2} \\
        &\leq o\left(\Sigma_{1,1}^{(t)}[R_1^{(t)}]^2 + \sum_{\ell\in[2]}\Sigma_{2,\ell}^{(t)}[R_2^{(t)}]^2\right)\widetilde{O}(\varrho+\frac{1}{\sqrt{d}})[R_{2}^{(t)}]^{1/2}[R_{1}^{(t)}]^{1/2}
    \end{align*}
    so the update becomes 
    \begin{align*}
        R_{1,2}^{(t+1)} &= R_{1,2}^{(t)}\left(1 - \eta\Theta\Big(\Sigma_{1,1}^{(t)} + \sum_{\ell\in[2]}\Sigma_{2,\ell}^{(t)}(E_{2,1}^{(t)})^2 \Big)[R_1^{(t)}]^2 -\eta \Theta\Big(\Sigma_{1,1}^{(t)} (E_{1,2}^{(t)})^2 + \sum_{\ell\in[2]}\Sigma_{2,\ell}^{(t)} \Big)[R_2^{(t)}]^2 \right) \\
        &\quad \pm \eta\widetilde{O}(\varrho+\frac{1}{\sqrt{d}})[R_1^{(t)}]^{1/2}[R_2^{(t)}]^{1/2} \Theta\Big(\Sigma_{1,1}^{(t)} + \sum_{\ell\in[2]}\Sigma_{2,\ell}^{(t)}(E_{2,1}^{(t)})^2 \Big)[R_1^{(t)}]^2 \\
        &\quad \pm \eta\widetilde{O}(\varrho+\frac{1}{\sqrt{d}})[R_1^{(t)}]^{1/2}[R_2^{(t)}]^{1/2} \Theta\Big(\Sigma_{1,1}^{(t)}(E_{1,2}^{(t)})^2 + \sum_{\ell\in[2]}\Sigma_{2,\ell}^{(t)} \Big)[R_2^{(t)}]^2
    \end{align*}
    Now we can use the same arguments as in the proof of \(\overline{R}_{1,2}^{(t)}\) in \myref{lem:phase-2}{Lemma} to conclude.
    \\
    \newline
    \textbf{Proof of \myref[a,b,c]{lem:phase-3}{Lemma}:} Indeed, at the end of phase III:
    \begin{align*}
        \text{\myref[a]{induct:phase-3}{Induction}}\quad &\implies \quad \text{\myref[a]{lem:phase-3}{Lemma}}\\
        \text{\myref[b]{induct:phase-3}{Induction}}\quad &\implies \quad \text{\myref[c]{lem:phase-3}{Lemma}}\\
        \text{\myref[c]{induct:phase-3}{Induction}}\quad &\implies \quad \text{\myref[b]{lem:phase-3}{Lemma}}
    \end{align*}
    Now we have completed the whole proof.
\end{proof}

\section{The End Phase: Convergence}

When we arrive at \(t=T_3\), we have already obtained the representation we want for the encoder network \(f(X)\), where \(v_1\) and \(v_2\) are satisfactorily learned by different neurons. In the last phase, we prove that such features are the solutions that the algorithm are converging to, which gives a stronger guarantee than just accidentally finding the solution at some intermediate steps.

To prove the convergence, we need to ensure all the good properties that we got through the training still holds. Fortunately, mosts of \myref{induct:phase-3}{Induction} still hold, as we summarized below:

\begin{induct}\label{induct:end-phase}
    At the end phase, i.e. when \(t\in[T_3,T]\), \myref[a]{induct:phase-3}{Induction} continues to hold except that \(|B_{2,2}^{(t)}| = \Theta(1)\), \myref[b]{induct:phase-3}{Induction} will hold except that for \(|E_{2,1}^{(t)}|\) only the upper bound still holds, and the upper bounds in \myref[c]{induct:phase-3}{Induction} still hold while the lower bounds for \(R_1^{(t)},R_2^{(t)}\) is \(1/\poly(d)\). Moreover, there is a constant \(C=O(1)\) such that when \(t  \geq T_{3}+ \frac{\alpha_1^{C}}{\eta}\), we would have \(|E_{2,1}^{(t)}| \leq \widetilde{O}(\varrho+\frac{1}{\sqrt{d}})[R_1^{(t)}]^{3/2}[R_2^{(t)}]^{3/2}\).
\end{induct}

Now we present the main theorem of the paper, which we shall prove in this section.

\begin{theorem}[End phase: convergence]\label{thm:end-phase}
    For some \(T_4 = T_3 + \frac{d^{2+o(1)}}{\eta} \) and \(T = \poly(d)/\eta\), we have for all \(t\in [T_4,T]\) that \myref{induct:end-phase}{Induction} holds true and:
    \begin{enumerate}[(a)]
        \item Successful learning of both \(v_1, v_2\): \(|B_{1,1}^{(t)}|,|B_{2,2}^{(t)}|  = \Theta(1)\) while \(|B_{2,1}^{(t)}|,|B_{1,2}^{(t)}| =\widetilde{O}(\frac{1}{\sqrt{d}}) \).
        \item Successful denoising at the end: \(R_j^{(t)} \leq R_j^{(T_3)}(1 - \widetilde{\Theta}(\frac{1}{\alpha_j^6})[R_j^{(t)}]^2)\) for all \(j\in[2]\).
        \item Prediction head is close to the identity: \(|E_{j,3-j}^{(t)}| \leq \widetilde{O}(\varrho+\frac{1}{\sqrt{d}})[R_1^{(t)}]^{3/2}[R_1^{(t)}]^{3/2} \) for all \(j\in[2]\);
    \end{enumerate}
    In fact, (b) and (c) also imply for some sufficiently large \(t = \poly(d)/\eta\), it holds \(R_j^{(t)} \leq \frac{1}{\poly(d)}\) and \(|E_{j,3-j}^{(t)}| \leq \frac{1}{\poly(d)}\) for all \(j\in[2]\).
\end{theorem}

And we have a simple corollary for the objective convergence.

\begin{corollary}[objective convergence, with prediction head]\label{coro:obj-converge-with-pred}
    Let \(\mathsf{OPT}\) denote the global minimum of the population objective \eqref{eqdef:loss-obj-expression}. It is easy to derive that \(\mathsf{OPT} = 2 - 2\frac{C_0}{C_1} = \Theta(\frac{1}{\log d})\). We have for some sufficiently large \(t \geq \poly(d)/\eta\):
    \begin{align*}
        L(W^{(t)}, E^{(t)}) \leq \mathsf{OPT} + \frac{1}{\poly(d)}
    \end{align*}
\end{corollary}

Now we need to establish some auxiliary lemmas:

\begin{lemma}\label{lem:reducing-noise-end-phase}
    For some \(t\in [T_3, \poly(d)/\eta]\), if \myref{induct:end-phase}{Induction} holds from \(T_3\) to \(t\), we have \myref{lem:reduce-noise-phase3}{Lemma} holds at \(t\).
\end{lemma}

\begin{proof}
    Simple from similar calculations in the proof of \myref{lem:reduce-noise-phase3}{Lemma} .
\end{proof}

\begin{lemma}
    For some \(t\in [T_3, \poly(d)/\eta]\), if \myref{induct:end-phase}{Induction} holds from \(T_3\) to \(t\), we have for each \(j\in[2]\) that 
    \begin{align*}
        \sum_{s\in [T_3,t]}\sum_{\ell\in[2]}\eta\Sigma_{j,\ell}^{(s)}[R_{j}^{(s)}]^3 \leq O(R_{j}^{(T_3)}),\quad \forall j\in[2]
    \end{align*}
\end{lemma}

\begin{proof}
    Notice that when \myref{induct:end-phase}{Induction} holds, we always have 
    \begin{align*}
        \sum_{\ell\in[2]}(\Sigma_{j,\ell}^{(t)} + \Sigma_{3-j,\ell}^{(t)} (E_{3-j,j}^{(t)})^2 ) = (1 \pm o(1))\sum_{\ell\in[2]}\Sigma_{j,\ell}^{(t)}
    \end{align*}
    we can use \myref{lem:reducing-noise-end-phase}{Lemma} to obtain the update of \(R_2^{(t)}\) as in the calculations when we obtained \eqref{eqdef:phase3-R2-update}:
    \begin{align*}
        R_2^{(t)} = R_2^{(T_3)} - \sum_{s\in[T_3,t)}\sum_{\ell\in[2]}\Theta(\eta\Sigma_{2,\ell}^{(s)}) [R_{2}^{(s)}]^3 
    \end{align*}
    which means that \(R_2^{(t)}\) is decreasing from \(T_3\) to \(t\). Summing up the update, the part of \(R_2^{(t)}\) is solved. For the part of \(R_1^{(t)}\), we separately discuss when \(|E_{2,1}^{(t)}|\) is larger than or smaller than \(\widetilde{O}(\varrho+\frac{1}{\sqrt{d}})[R_1^{(t)}]^{3/2}[R_2^{(t)}]^{3/2}\). 
    When the former happens, which we know from \myref{induct:end-phase}{Induction} that it cannot last until some \(t'_{4} = T_3+\frac{\alpha_1^{O(1)}}{\eta}\) many iterations, we have for \(t\in[T_3,t'_{4}]\)
    \begin{align*}
        \sum_{s\in[T_3,t)}\sum_{(j,\ell)\in[2]^2}\eta\Sigma_{j,\ell}^{(s)}|E_{j,3-j}^{(s)}| (\overline{R}_{1,2}^{(s)} + \varrho)[R_{1}^{(s)}]^{3/2}[R_{2}^{(s)}]^{3/2} &\leq \widetilde{O}(\varrho+\frac{1}{\sqrt{d}})\frac{\alpha_1^{O(1)}}{d}R_1^{(T_3)} \leq \frac{1}{d}R_1^{(T_3)}
    \end{align*}
    Now for \(t\geq t'_{4}\) we can simply go through similar calculations as in the proof of \myref[c]{induct:phase-3}{Induction} to obtain 
    \begin{align*}
        \sum_{s\in[t'_4,t)}\sum_{(j,\ell)\in[2]^2}\eta\Sigma_{j,\ell}^{(s)}|E_{j,3-j}^{(s)}| (\overline{R}_{1,2}^{(s)} + \varrho)[R_{1}^{(s)}]^{3/2}[R_{2}^{(s)}]^{3/2} &\leq \sum_{s\in[t'_4,t)}\widetilde{O}(\varrho+\frac{1}{\sqrt{d}})^2\sum_{(j,\ell)\in[2]^2}\eta\Sigma_{j,\ell}^{(s)} [R_{1}^{(s)}]^{3}[R_{2}^{(s)}]^{3} \\
        & \leq \widetilde{O}(\varrho+\frac{1}{\sqrt{d}})^2R_2^{(T_3)}\max_{s\in[t'_4,t)}[R_1^{(s)}]^3 \\
        &\leq \frac{1}{d} R_1^{(T_3)}
    \end{align*}
    So by applying \myref[a]{lem:reducing-noise-end-phase}{Lemma} and \myref{lem:reduce-noise-phase3}{Lemma}, we have 
    \begin{align*}
        R_1^{(t)} = (1\pm o(1))R_1^{(T_3)} - \sum_{s\in[T_3,t)}\sum_{\ell\in[2]}\Theta(\eta\Sigma_{j,\ell}^{(s)}) [R_{1}^{(s)}]^3
    \end{align*}
    which proves the claim.
\end{proof}

\begin{lemma}\label{lem:Xi-end-phase}
    For some \(t\in [T_3, \poly(d)/\eta]\), if \myref{induct:end-phase}{Induction} holds from \(T_3\) to \(t\). Then we have \(|E_{j,3-j}^{(t)}|\) is decreasing until \(|E_{j,3-j}^{(t)}| \leq O(\overline{R}_{1,2}^{(t)} + \varrho)[R_{1}^{(t)}]^{3/2}[R_{2}^{(t)}]^{3/2} + \widetilde{O}(\frac{1}{d^{3/2}})[R_j^{(t)}]^3\). Moreover, we have for each \(t\in [T_3,T]\) that 
    \begin{align*}
        \left|\sum_{s\in [T_3,t]}\eta_E\Xi_j^{(t)} E_{j,3-j}^{(s)}\right| \leq |E_{j,3-j}^{(T_3)}| + \widetilde{O}(\varrho+\frac{1}{\sqrt{d}}) \leq O(\sqrt{\eta_E/\eta})
    \end{align*}
\end{lemma}

\begin{proof}
    We can go through the same calculations in the proof of \myref[b]{induct:phase-3}{Induction} (using \myref{fact:pred-head-grad}{Fact}) to obtain 
    \begin{align*}
        E_{j,3-j}^{(t+1)} &= E_{j,3-j}^{(t)}(1 -\eta_E \Xi_j^{(t)}) + \sum_{\ell\in[2]}\eta_E \Delta_{j,\ell}^{(t)} \\
        &\quad +  \sum_{\ell\in[2]}\Theta(\eta_E\Sigma_{j,\ell}^{(t)}) (-E_{j,3-j}^{(t)}[R_{3-j}^{(t)}]^{3} \pm O(\overline{R}_{1,2}^{(t)} + \varrho)[R_{1}^{(t)}]^{3/2}[R_{2}^{(t)}]^{3/2}) \\
        & = E_{j,3-j}^{(t)}(1 -\eta_E \Xi_j^{(t)} - \eta_E\Theta(\Sigma_{j,j}^{(t)}[R_{3-j}^{(t)}]^{3}))+ \widetilde{O}(\frac{1}{d^{3/2}})\sum_{\ell\in[2]}\eta_E \Sigma_{j,\ell}^{(t)}[R_j^{(t)}]^3 \\
        &\quad \pm O(\eta_E\Sigma_{j,j}^{(t)}) ( \overline{R}_{1,2}^{(t)} + \varrho)[R_{1}^{(t)}]^{3/2}[R_{2}^{(t)}]^{3/2}
    \end{align*}
    where we have used in the second equality that \(\sum_{\ell\in[2]}\Delta_{j,\ell}^{(t)} \leq \widetilde{O}(\frac{1}{d^{3/2}})\sum_{\ell\in[2]}\Sigma_{j,\ell}^{(t)}[R_j^{(t)}]^3\) and also \(\Sigma_{j,3-j}^{(t)} \leq O(\frac{1}{d^{3/2}})\Sigma_{j,j}^{(t)}\) for both \(j\in[2]\) when \myref{induct:end-phase}{Induction} holds. Note that from above calculations, there exist a constant C such that if \(|E_{j,3-j}^{(t)}| \geq C( \overline{R}_{1,2}^{(t)} + \varrho)[R_{1}^{(t)}]^{3/2}[R_{2}^{(t)}]^{3/2} + \sum_{\ell\in[2]}\eta_E \Delta_{j,\ell}^{(t)}\), we have \(|E_{2,1}^{(t)}|\) to be decreasing. Now it suffices to observe that:
    \begin{align*}
        \sum_{s\in [T_3,t]} O(\eta_E\Sigma_{j,j}^{(t)}) ( \overline{R}_{1,2}^{(t)} + \varrho)[R_{1}^{(t)}]^{3/2}[R_{2}^{(t)}]^{3/2} &\leq \sum_{s\in [T_3,t]} O(\eta_E\Sigma_{1,1}^{(t)}+\eta_E\Sigma_{2,2}^{(t)}) ( \overline{R}_{1,2}^{(t)} + \varrho)([R_{1}^{(t)}]^{3}+[R_{2}^{(t)}]^{3}) \\
        &\leq \widetilde{O}(\varrho+\frac{1}{\sqrt{d}})
    \end{align*}
    which is from \myref{induct:end-phase}{Induction}, \myref[c]{induct:phase-3}{Induction} and \myref{lem:reducing-noise-end-phase}{Lemma}. Also note that \(\Sigma_{j,j}^{(t)}[R_{3-j}^{(t)}]^3 \leq O(\frac{d^{o(1)}}{d^{3/4}})\Xi_j^{(t)}\) at this stage, we have 
    \begin{align*}
        E_{3-j,j}^{(t)} = E_{j,3-j}^{(T_3)} - \sum_{s\in[T_3,t)}\Xi_j^{(s)}E_{j,3-j}^{(s)}  + \widetilde{O}(\varrho+\frac{1}{\sqrt{d}})
    \end{align*}
    Recalling the expression of \(\Xi_j^{(t)}\) finishes the proof.
\end{proof}

\begin{lemma}\label{lem:comparison-E21-B11}
    Recall \(T_2\) defined in \eqref{eqdef:T_2} and \(T_3\) defined in \eqref{eqdef:T3}, we have
    \begin{align*}
        \sqrt{\eta/\eta_E}\max_{t\leq T_3} |E_{2,1}^{(t)}| \leq \sum_{t\leq T_2} \frac{\eta\Sigma_{1,1}^{(t)}}{|B_{1,1}^{(t)}|}\Ecal_{1,2}^{(t)} + \frac{1}{\alpha_1^{\Omega(1)}}
    \end{align*}
\end{lemma}
To prove this lemma, we need a simple claim.
\begin{claim}\label{claim:growth}
    If \(\{x_t\}_{t<T}, x_t\geq 0\) is an increasing sequence and \(C=\Theta(1)\) is a constant such that \(x_{t+1} - x_t\leq O(\eta)\) and \(\sum_{t< T}x_t(x_{t+1}-x_t) = C\), then for each \(\delta \in (\frac{1}{d}, 1)\) it holds \(|x_T - \sqrt{C}| \leq O(\delta^2+x_0^2+O(\frac{\log d}{d}))\).
\end{claim}

\begin{proof}
    Indeed, for every \(g \in 0, 1, \dots \), we define \(\mathcal{T}_g := \min\{t: x_t \geq  (1+\delta)^g x_0\}\). and define \(b := \min\{g: ((1+\delta)^gx_0)^2 \geq  C-\delta^2\}\). Now for any \(g < b\), we have
    \begin{align*}
        \sum_{t\in [\mathcal{T}_g,\mathcal{T}_{g+1}]}x_t(x_{t+1}-x_t) &\geq x_{\mathcal{T}_g} (x_{\mathcal{T}_{g+1}}-x_{\mathcal{T}_g}) \geq (1+\delta)^{g}\delta(1+\delta)^{g-1}x_0^2 - \frac{1}{d} = \delta(1+\delta)^{2g-1}x_0^2- \frac{1}{d}
    \end{align*}
    By our definition of \(\mathcal{T}_g\), we can further get 
    \begin{align*}
        C = \sum_{t<T}x_t(x_{t+1}-x_t) = \sum_{g=1}^b \sum_{t\in [\mathcal{T}_g,\mathcal{T}_{g+1}]}x_t(x_{t+1}-x_t) \geq (1+\delta)^{2b}x_0^2 - x_0^2 - \frac{b}{d} \geq C - \delta^2-x_0^2- \frac{b}{d}
    \end{align*}
    And also we have \(C\leq (\max_{t\leq T}x_t)\sum_{t<T}(x_{t+1}-x_t) = x_T^2\), so we have \(|x_T^2 - C| \leq \delta^2+x_0^2 + \frac{b}{d}\), where \(b = O(\log(C)/\log(1+\delta))\leq O(\log d)\), which proves the claim.
\end{proof}

\begin{proof}[Proof of \myref{lem:comparison-E21-B11}{Lemma}]
    From the proof of \myref{lem:phase-2}{Lemma} and \myref{lem:phase-3}{Lemma} we know that 
    \begin{align*}
        \max_{t\leq T_3}|E_{2,1}^{(t)}|\leq \sum_{t\leq T_3} (1\pm\frac{1}{\alpha_1^{\Omega(1)}})\eta_E|\Delta_{2,1}^{(t)}| + \widetilde{O}(\varrho+\frac{1}{\sqrt{d}}) 
    \end{align*}
    And since from the proof of \myref{lem:phase-2}{Lemma} we know that 
    \begin{align*}
        R_2^{(T_3)} &= R_2^{(0)} - \sum_{t\leq T_3} (1\pm \widetilde{O}(\frac{1}{d^{3/2}})) \eta \Sigma_{2,1}^{(t)}\Ecal_{2,1}^{(t)} \pm \widetilde{O}(\varrho+\frac{1}{\sqrt{d}}) \\
        &= (1\pm \widetilde{O}(\frac{1}{d^{3/2}}))\sum_{t\leq T_3} E_{2,1}^{(t)} \Delta_{2,1}^{(t)}\pm \widetilde{O}(\varrho+\frac{1}{\sqrt{d}})
    \end{align*}
    We can define some alternative variables \(\widetilde{E}_{2,1}^{(t)}\) updated as \(\widetilde{E}_{2,1}^{(t+1)} = \widetilde{E}_{2,1}^{(t)} + \eta_E\Delta_{2,1}^{(t)}\) and \(\widetilde{R}_2^{(t+1)} = \widetilde{R}_2^{(t)} - \widetilde{E}_{2,1}^{(t)} \Delta_{2,1}^{(t)}\). It is easy to see that \(|E_{2,1}^{(t)} - \widetilde{E}_{2,1}^{(t)}| \leq \frac{1}{\alpha_1^{\Omega(1)}}\max_{t\leq T_3}|E_{2,1}^{(t)}|\). From above calculations, we know \(\frac{\eta}{\eta_E}\sum_{t\in [T_1,T_{3}]}\widetilde{E}_{2,1}^{(t)}(\widetilde{E}_{2,1}^{(t+1)}-\widetilde{E}_{2,1}^{(t)}) = \widetilde{R}_2^{(T_1)} \pm \widetilde{O}(\varrho+\frac{1}{\sqrt{d}}) + O(\frac{1}{d^{1/4}})\), which by \myref{claim:growth}{Claim} implies that 
    \begin{align*}
        \sqrt{\eta/\eta_E}|\widetilde{E}_{2,1}^{(T_{3})}| = \sqrt{\widetilde{R}_2^{(T_1)}} \pm O(\frac{1}{d^{1/4}}) = \sqrt{2} \pm \widetilde{O}(\varrho+\frac{1}{\sqrt{d}}) \pm O(\frac{1}{d^{1/4}})
    \end{align*}
    And when we turn back, we shall have \(\sqrt{\eta/\eta_E}\max_{t\leq T_3}|E_{2,1}^{(t)}| \leq  \sqrt{2} + \frac{1}{\alpha_1^{\Omega(1)}}\). Now we can use similar techniques on \(B_{1,1}^{(t)}\) and \(R_1^{(t)}\). Indeed, from \eqref{eqdef:lem:phase-2-R1-update} and similar arguments in phase I, we know for all \(t\in[T_1,T_2]\)
    \begin{align}\label{eqdef:R1-update-phase2-end-phase}
        R_1^{(t+1)} &= R_1^{(0)} - \sum_{s\leq t} (1\pm \widetilde{O}(\frac{1}{d^{3/2}}))\eta\Sigma_{1,1}^{(s)}\Ecal_{1,2}^{(s)} \pm \widetilde{O}(\varrho+\frac{1}{\sqrt{d}}) \\
        R_1^{(t+1)} & \leq R_1^{(t)}(1 - \widetilde{O}(\frac{\eta}{\alpha_1^6})[R_1^{(t)}]^2)\nonumber
    \end{align} 
    So one can obtain that at some iteration \(t' = T_1 + O(\frac{d\alpha_1^{O(1)}}{\eta})\), we shall have \(R_1^{(t)}\leq O(\frac{1}{\sqrt{d}})\) for all \(t\geq t'\). Now let us consider the growth of \(B_{1,1}^{(t)}\) before \(t'\), which clearly constitutes of 
    \begin{align*}
        B_{1,1}^{(t')} &= B_{1,1}^{(T_1)} + \sum_{t\in [T_1,t')}(\Lambda_{1,1}^{(t)}+\Gamma_{1,1}^{(t)}-\Upsilon_{1,1}^{(t)}) \\
        & = B_{1,1}^{(T_1)} + \sum_{t\in [T_1,t')} \left( \frac{\eta\Sigma_{1,1}^{(t)}}{|B_{1,1}^{(t)}|}\Ecal_{1,2}^{(t)}\sign(B_{1,1}^{(t)}) + \eta\Gamma_{1,1}^{(t)}-\eta\Upsilon_{1,1}^{(t)}\right) \\
        & = B_{1,1}^{(0)} + \sum_{t< t'}  \frac{\eta\Sigma_{1,1}^{(t)}}{|B_{1,1}^{(t)}|}\Ecal_{1,2}^{(t)}\sign(B_{1,1}^{(t)}) + \sum_{t\in [T_1,t')}\eta\left(\Gamma_{1,1}^{(t)}-\Upsilon_{1,1}^{(t)}\right) + \widetilde{O}(\frac{1}{\sqrt{d}})
    \end{align*}
    where the last one comes from the proof of \myref{lem:phase-1}{Lemma}. Moreover by using the same arguments in the proof of \myref{lem:phase-2}{Lemma} we can easily prove that 
    \begin{align*}
        &\Bigg|\sum_{t\in [T_1,t')}( \Gamma_{1,1}^{(t)}-\Upsilon_{1,1}^{(t)})\Bigg| \leq \widetilde{O}(\frac{1}{\sqrt{d}})\quad \implies\quad  \sum_{t< t'}  \frac{\eta\Sigma_{1,1}^{(t)}}{|B_{1,1}^{(t)}|}\Ecal_{1,2}^{(t)} \geq |B_{1,1}^{(t')}| - |B_{1,1}^{(0)}| - \widetilde{O}(\frac{1}{\sqrt{d}})
    \end{align*}
    And for \(t\in[t',T_2]\), we also have by \eqref{eqdef:R1-update-phase2-end-phase} that 
    \begin{align*}
        \sum_{t\in [t',T_2]} \frac{\eta\Sigma_{1,1}^{(t)}}{|B_{1,1}^{(t)}|}\Ecal_{1,2}^{(t)} \leq \sum_{t\in [t',T_2)}\eta\Sigma_{1,1}^{(t)}\Ecal_{1,2}^{(t)} \leq O(\frac{1}{\sqrt{d}})
    \end{align*}
    Recall \(R_1^{(0)} = \sum_{t\in [0,t')} (1\pm \widetilde{O}(\frac{1}{d^{3/2}}))\eta\Sigma_{1,1}^{(t)}\Ecal_{1,2}^{(t)} \pm \widetilde{O}(\varrho+\frac{1}{\sqrt{d}})\) by \eqref{eqdef:R1-update-phase2-end-phase} and \(R_1^{(t)} \leq O(\frac{1}{\sqrt{d}})\) for \(t\geq t'\). Now we can finally go through the same analysis using \myref{claim:growth}{Claim} on \(B_{1,1}^{(t)}\) and \(R_1^{(t)}\) during \(t\in[0,t']\) as above to obtain that
    \begin{align*}
        \sum_{t\leq T_2} \frac{\eta\Sigma_{1,1}^{(t)}}{|B_{1,1}^{(t)}|}\Ecal_{1,2}^{(t)}\geq (1 - \widetilde{O}(\frac{1}{d^{3/2}}))\sqrt{R_1^{(0)}} - \widetilde{O}(\frac{1}{\sqrt{d}})= 1 - \widetilde{O}(\varrho+\frac{1}{\sqrt{d}})
    \end{align*}
    Combining the results, we finishes the proof.
\end{proof}

Now we are prepared to prove \myref{thm:end-phase}{Theorem}.

\subsection{Proof of Convergence}

\begin{proof}[Proof of \myref{thm:end-phase}{Theorem}] First we start with the \(B_{j,\ell}^{(t)}\)s. Indeed, we can go through similar calculations to see that all gradients \(\dbrack{-\nabla_{w_j}L(W^{(t)},E^{(t)}),v_{\ell}}\) can be decomposed into 
    \begin{align*}
        \dbrack{-\nabla_{w_j}L(W^{(t)},E^{(t)}), v_\ell} = (\Lambda_{j,\ell}^{(t)} - \Upsilon_{j,\ell,1}^{(t)})+ (\Gamma_{j,\ell}^{(t)} - \Upsilon_{j,\ell, 2}^{(t)})
    \end{align*}
    where \(\Lambda_{j,\ell}^{(t)} - \Upsilon_{j,\ell,1}^{(t)}\) and \(\Gamma_{j,\ell}^{(t)} - \Upsilon_{j,\ell, 2}^{(t)}\) can be expressed as
    \begin{align*}
        \Lambda_{j,\ell}^{(t)} - \Upsilon_{j,\ell,1}^{(t)} & = C_0\alpha_{2}^6C_1\alpha_{1}^6\Phi_j^{(t)}(B_{j,\ell}^{(t)})^5 \left(  E_{j,3-j}^{(t)} (B_{3-j,3-\ell}^{(t)})^3(B_{j,3-\ell}^{(t)} )^3 + (E_{j,3-j}^{(t)})^2 (B_{3-j,3-\ell}^{(t)})^6\right) \\
        &\quad - C_0\alpha_{2}^6C_1\alpha_{1}^6\Phi_j^{(t)}(B_{j,\ell}^{(t)})^2 (B_{3-j,\ell}^{(t)})^3 E_{j,3-j}^{(t)}\left( (B_{j,3-\ell}^{(t)})^6 + E_{j,3-j}^{(t)}(B_{3-j,3-\ell}^{(t)})^3(B_{j,3-\ell}^{(t)})^3\right)\\
        &\quad + C_0\alpha_{2}^6\Phi_j^{(t)} (B_{j,\ell}^{(t)})^5C_2\Ecal_{j,3-j}^{(t)} \\
        \Gamma_{j,\ell}^{(t)} - \Upsilon_{j,\ell,2}^{(t)} & = C_0\alpha_{2}^6C_1\alpha_{1}^6\Phi_{3-j}^{(t)}(B_{3-j,\ell}^{(t)})^3(B_{j,\ell}^{(t)})^2 E_{3-j,j}^{(t)}\left(  E_{3-j,j}^{(t)} (B_{j,3-\ell}^{(t)})^3(B_{3-j,3-\ell}^{(t)} )^3 + (E_{3-j,j}^{(t)})^2 (B_{j,3-\ell}^{(t)})^6\right) \\
        &\quad - C_0\alpha_{2}^6C_1\alpha_{1}^6\Phi_{3-j}^{(t)}(B_{j,\ell}^{(t)})^5  (E_{3-j,j}^{(t)})^2\left( (B_{3-j,3-\ell}^{(t)})^6 + E_{3-j,j}^{(t)}(B_{j,3-\ell}^{(t)})^3(B_{3-j,3-\ell}^{(t)})^3\right)\\
        &\quad + C_0\alpha_{2}^6\Phi_{3-j}^{(t)}E_{3-j,j}^{(t)} (B_{3-j,\ell}^{(t)})^3(B_{j,\ell}^{(t)})^2C_2\Ecal_{3-j,j}^{(t)}
    \end{align*}
    Firstly, for all the terms that contain factors of \((B_{j,\ell}^{(t)})^2(B_{3-j,\ell}^{(t)})^2\) (or \((B_{j,\ell}^{(t)})^2(B_{j,3-\ell}^{(t)})^2\)), we can apply \myref{lem:Xi-end-phase}{Lemma}, our \myref{induct:end-phase}{Induction} assumption and \(|E_{j,3-j}^{(t)}|\leq O(1),\forall t \in [T_3,T]\) to obtain that their (multiplicated by \(\eta\)) summation over \(t\in [T_3,T]\) is absolutely bounded by \(\widetilde{O}(\frac{1}{d})\). So we can move on to deal with all other terms. When \(j=\ell\), Using \myref{lem:Xi-end-phase}{Lemma}, we have 
    \begin{align*}
        \sum_{t\in [T_3,T]} \eta C_0\alpha_{2}^6C_1\alpha_{1}^6\Phi_j^{(t)}|B_{j,\ell}^{(t)}|^5(E_{j,3-j}^{(t)})^2 (B_{3-j,3-\ell}^{(t)})^6 &= \sum_{t\in [T_3,T]}\frac{\eta\Xi_j^{(t)}}{|B_{j,\ell}^{(t)}|}(E_{j,3-j}^{(t)})^2 \\
        &\leq \sqrt{\frac{\eta}{\eta_E}}|E_{j,3-j}^{(T_3)}| + \widetilde{O}(\varrho+\frac{1}{\sqrt{d}}) = O(1)
    \end{align*}
    And the sign of LHS is \(\sign(B_{j,\ell}^{(t)})\). Moreover, for \(j=\ell=1\), from \myref{lem:comparison-E21-B11}{Lemma} and \myref{lem:Xi-end-phase}{Lemma} we also have 
    \begin{align*}
        \sum_{t\in [T_3,T]}\eta C_0\alpha_{2}^6C_1\alpha_{1}^6\Phi_{2}^{(t)}|B_{1,1}^{(t)}|^5  (E_{2,1}^{(t)})^2 (B_{2,2}^{(t)})^6 &\leq \sqrt{\frac{\eta}{\eta_E}}\Bigg|\sum_{t\in [T_3,T]}\eta _E\Xi_j^{(t)}E_{j,3-j}^{(t)}\Bigg|\\
        &\leq \sqrt{\frac{\eta}{\eta_E}}|E_{2,1}^{(T_3)}| + \widetilde{O}(\varrho+\frac{1}{\sqrt{d}}) \\
        &\leq \sum_{t\leq T_2} \frac{\eta\Sigma_{1,1}^{(t)}}{|B_{1,1}^{(t)}|}\Ecal_{1,2}^{(t)} + \frac{1}{\alpha_1^{\Omega(1)}} 
    \end{align*}
    Since we have
    \begin{align*}
        B_{1,1}^{(T_2)} = \sum_{s\leq T_2}\frac{\eta\Sigma_{1,1}^{(t)}}{|B_{1,1}^{(t)}|}\Ecal_{1,2}^{(t)} + \sum_{s\leq T_2}\frac{\eta\Sigma_{2,1}^{(t)}}{|B_{1,1}^{(t)}|}\Ecal_{2,1}^{(t)} - \sum_{t\in [T_3,T]}\frac{\eta\Xi_j^{(t)}}{|B_{j,\ell}^{(t)}|}(E_{j,3-j}^{(t)})^2
    \end{align*}
    And since by \myref{induct:phase-2}{Induction} we have \(|B_{1,1}^{(t)}| = \Theta(1)\) during \(t\in[T_1,T_2]\) and \(\sum_{t\in[T_1,T_2]}\eta\Sigma_{2,1}^{(t)} \geq R^{(T_1)} - o(1) = \sqrt{2} - o(1)\). For all the other terms in the gradient , we can apply \myref{lem:Xi-end-phase}{Lemma}, our \myref{induct:end-phase}{Induction} assumption and \(|E_{j,3-j}^{(t)}|\leq O(1)\) so we have for \(t\in[T_3,T]\)
    \begin{align*}
        |B_{1,1}^{(t)}| &= \sum_{s\leq T_2}\frac{\eta\Sigma_{1,1}^{(t)}}{|B_{1,1}^{(t)}|}\Ecal_{1,2}^{(t)} + \sum_{s\leq T_2}\frac{\eta\Sigma_{2,1}^{(t)}}{|B_{1,1}^{(t)}|}\Ecal_{2,1}^{(t)} - \sum_{t\in [T_3,T]}\frac{\eta\Xi_j^{(t)}}{|B_{j,\ell}^{(t)}|}(E_{j,3-j}^{(t)})^2 - o(1) \\
        &\geq \sqrt{\eta/\eta_E}\max_{t\leq T_3}|E_{2,1}^{(t)}| + \sum_{s\leq T_2}\frac{\eta\Sigma_{2,1}^{(t)}}{|B_{1,1}^{(t)}|}\Ecal_{2,1}^{(t)} - \sqrt{\frac{\eta}{\eta_E}}|E_{j,3-j}^{(T_3)}| + \widetilde{O}(\varrho+\frac{1}{\sqrt{d}}) - o(1) \\
        &\geq \sum_{s\leq T_2}\frac{\eta\Sigma_{2,1}^{(t)}}{|B_{1,1}^{(t)}|}\Ecal_{2,1}^{(t)} - o(1) \geq \Omega(1)
    \end{align*}
    which also proved \(|B_{1,1}^{(t)}| = O(1)\) since all the terms on the RHS are absolutely \(O(1)\) bounded.
    Since one can see from \myref{lem:Xi-end-phase}{Lemma} that \(|E_{2,1}^{(t)}|\) is decreasing before it reaches \(\frac{1}{d})\). Moreover this proves \(\sqrt{\eta/\eta_E} |E_{2,1}^{(t)}| \leq B_{1,1}^{(t)}\) for all \(t\in [T_3,T]\), and also the fact that 
    \begin{align*}
        B_{1,1}^{(t)} \geq \Omega(1),\quad \forall t\in [T_3, T]
    \end{align*}
    
    The case of \(B_{2,2}^{(t)}\) is much more simple as \(E_{1,2}^{(t)} \leq \widetilde{O}(\frac{1}{d})\) throughout \(t\in[T_3,T]\) by \myref{lem:Xi-end-phase}{Lemma} and \myref[c]{lem:phase-3}{Lemma}, Now we can go through the similar calculations again to obtain that \(B_{2,2}^{(t)} = \Theta(1)\) for all \(t\in[T_3,T]\). When \(j\neq \ell\), all the terms calculated in the expansion of \(\Lambda_{j,\ell}^{(t)} - \Upsilon_{j,\ell,1}^{(t)}\) and \(\Gamma_{j,\ell}^{(t)} - \Upsilon_{j,\ell,2}^{(t)}\) contain factors of \((B_{2,1}^{(t)})^2 = \widetilde{O}(\frac{1}{d})\) or \((B_{1,2}^{(t)})^2 = \widetilde{O}(\frac{1}{d})\). So we can similarly use \myref{lem:Xi-end-phase}{Lemma} as before to derive  that \(B_{j,3-j}^{(t)} = B_{j,3-j}^{(T_3)}(1 \pm \widetilde{O}(\frac{\alpha_1^{O(1)}}{\sqrt{d}}))\) for all \(t\in[T_3,T]\) and \(j\in[2]\).

    As for the prediction head, the induction of \(E_{1,2}^{(t)}\) follows from exactly the same proof in \myref{lem:phase-3}{Lemma}. The part of \(E_{2,1}^{(t)}\) is half done in \myref{lem:Xi-end-phase}{Lemma}. It suffices to notice that \(\Xi_2^{(t)} = \widetilde{\Theta}(\frac{\alpha_1^6}{\alpha_2^6})\) and if \(|E_{2,1}^{(t)}| \geq C( \overline{R}_{1,2}^{(t)} + \varrho)[R_{1}^{(t)}]^{3/2}[R_{2}^{(t)}]^{3/2}\) for some \(C=O(1)\), then
    \begin{align*}
        E_{2,1}^{(t+1)} &= E_{2,1}^{(t)}(1 -\eta_E \Xi_2^{(t)} - \eta_E\Theta(\Sigma_{2,2}^{(t)}[R_{1}^{(t)}]^{3}))+ \widetilde{O}(\frac{1}{d^{3/2}})\sum_{\ell\in[2]}\eta_E \Sigma_{2,\ell}^{(t)}[R_2^{(t)}]^3 \\
        &\quad \pm O(\eta_E\Sigma_{2,2}^{(t)}) ( \overline{R}_{1,2}^{(t)} + \varrho)[R_{1}^{(t)}]^{3/2}[R_{2}^{(t)}]^{3/2} \\
        & \leq E_{2,1}^{(t)}(1 - \widetilde{\Theta}(\frac{\eta\alpha_1^6}{\alpha_2^6}))
    \end{align*}
    So after \(\frac{\alpha_1^{O(1)}}{\eta}\) many epochs will we have
    \begin{align*}
        |E_{2,1}^{(t)}| \leq (\log d)|\overline{R}_{1,2}^{(t)} + \varrho|[R_{1}^{(t)}]^{3/2}[R_{2}^{(t)}]^{3/2} \leq \widetilde{O}( \varrho + \frac{1}{\sqrt{d}})[R_{1}^{(t)}]^{3/2}[R_{2}^{(t)}]^{3/2}
    \end{align*}
    as desired. And the rest of the induction of \(E_{2,1}^{(t)}\) is the same as in the induction arguments of \(E_{1,2}^{(t)}\) in \myref{lem:phase-3}{Lemma}.

    The induction of \(R_1^{(t)}, R_2^{(t)}\) and \(R_{1,2}^{(t)}\) is exactly the same as those in the proof of \myref{lem:phase-3}{Lemma} except here we only need \(R_1^{(t)}/R_2^{(t)} \in [\frac{1}{\alpha_1^{O(1)} },\alpha_1^{O(1)}]  \) after \(T_4\). Indeed, from the update of \(R_j^{(t)}\) (which can be easily worked out), we have
    \begin{align*}
        R_j^{(t+1)} = R_j^{(t)}(1 - \Theta(\eta\Sigma_{j,j}^{(t)})[R_j^{(t)}]^2 ) = R_j^{(t)}(1 - \widetilde{\Theta}(\frac{\eta }{\alpha_j^6})[R_j^{(t)}]^2)
    \end{align*}
    Now after \(\frac{d^2\alpha_1^{O(1)}}{\eta}\) many epochs, we can obtain from similar arguments in \myref{lem:phase-3}{Lemma} that \(R_1^{(t)}/R_2^{(t)} \in [\frac{1}{\alpha_1^{O(1)} },\alpha_1^{O(1)}]  \) and \(R_j^{(t)}\leq \frac{1}{d}\). The induction can go on untill \(t=\poly(d)/\eta\). 

    For the convergence of \(B_{1,1}^{(t)}\) and \(B_{2,2}^{(t)}\) after \(t=T_4\), notice that their changes depend on \(\sum_{t\geq T_4}\frac{E_{j,3-j}^{(t)}}{B_{j,j}^{(t)}}\Xi_j^{(t)}\), which stay very small after \(T_4\), we have that \(|B_{j,j}^{(t)} - B_{j,j}^{(T_4)}| \leq o(1)\) for all \(j\in[2]\). This finishes the whole proof.
\end{proof}

\section{Learning Without Prediction Head}\label{sec:w/o-pred}

When we do not use prediction head in the network architecture, the analysis is much simpler. We can reuse most of the gradient calculations in previous sections as long as we set \(E^{(t)}\) to the identity. Note that here we allow \(m\geq 1\) to be any positive integer.

\begin{theorem}[learning without the prediction head]\label{thm:without-pred-head}
    Let \(m = o(\alpha_1/\alpha_2)\). If we keep \(E^{(t)}\equiv I_m\) during the whole training process, then for all \(t \in [\widetilde{\Omega}(\frac{d^{2}}{\eta}), \poly(d)/\eta]\), we shall have \(|B_{j,1}^{(t)}| = \Theta(1)\), \(|B_{j,2}^{(t)}| = \widetilde{O}(\frac{1}{\sqrt{d}})\) and \(R_j^{(t)} = O(\frac{1}{d^{1-o(1)}})\) for all \(j\in[m]\) with probability \(1 - o(1)\). Moreover, for a longer training time \(t = \poly(d)/\eta\), we would have \(R_j^{(t)} \leq \frac{1}{\poly(d)}\) for all \(j\in[m]\).
\end{theorem}

Moreover, it is direct to obtain a objective convergence result similar to \myref{coro:obj-converge-with-pred}{Corollary}.

\begin{corollary}[objective convergence, without prediction head]
    Let \(\mathsf{OPT}\) denote the global minimum of the population objective \eqref{eqdef:loss-obj-expression}. When trained with \(E^{(t)}\equiv I_m\), we have for some sufficiently large \(t \geq \poly(d)/\eta\):
    \begin{align*}
        L(W^{(t)}, I_m) \leq \mathsf{OPT} + \frac{1}{\poly(d)}
    \end{align*}
\end{corollary}

\begin{proof}[Proof of \myref{thm:without-pred-head}{Theorem}]
    The proof is easy to obtain since it is very similar to some proofs in previous sections, and we only sketch it here. Indeed, using the calculations in \myref{lem:learning-v1-phase3}{Lemma} and \myref{lem:learning-v2-phase3}{Lemma} and set \(E_{i,j}^{(t)}, i\neq j \in [m]\) to zero. We shall have (note that here \(\Ecal_{j,r}^{(t)} \equiv \Ecal_j^{(t)}\) for any \(r\neq j\))
    \begin{align*}
        \dbrack{-\nabla_{w_j} L(W^{(t)},E^{(t)}), v_\ell} = C_0C_2\alpha_{\ell}^6(B_{j,\ell}^{(t)})^5\Phi_j^{(t)}\Ecal_j^{(t)} = \Theta(C_0C_2\alpha_{\ell}^6\Phi_j^{(t)}(B_{j,\ell}^{(t)})^5 [R_j^{(t)}]^3)
    \end{align*}
    Now we can go through the similar induction arguments as in the proof of \myref{lem:phase-1}{Lemma} (with TPM lemma to distinguish the learning speed) to obtain that for each \(j\in[m]\):
    \begin{align*}
        |B_{j,1}^{(t)}| = \Theta(1),\quad |B_{j,2}^{(t)} | = |B_{j,2}^{(0)}|(1\pm o(1)),\quad \forall j\in [m] \tag{when \(t \geq \frac{d^2}{\eta}\)}
    \end{align*}
    When this is proven, we can also reuse the calculations as in the proof of \myref{lem:reduce-noise-phase2}{Lemma} to obtain that 
    \begin{align*}
        R_j^{(t+1)} = R_j^{(t)}(1 - \Theta(\eta\Sigma_{j,1}^{(t)} )[R_j^{(t)}]^2 ) = R_j^{(t)}(1 - \Theta(\eta C_0C_2\alpha_{1}^6\Phi_j^{(t)}(B_{j,1}^{(t)})^6[R_j^{(t)}]^2 ),\quad \forall j\in[m]
    \end{align*}
    So again after some \(t = \widetilde{O}(\frac{d^{2}}{\eta})\), we shall have \(R_j^{(t)}\leq O(\frac{d^{o(1)}}{d})\). While the decrease of \(R_j^{(t)}\) is happening, we can make induction that \(|B_{j,2}^{(t)}| = |B_{j,2}^{(0)}|(1 \pm o(1))\), since if it holds for all previous iterations before \(t\), then 
    \begin{align*}
        \sum_{s\leq t-1} \eta|\dbrack{-\nabla_{w_j} L(W^{(s)},E^{(s)}), v_2}| &= \sum_{s\leq t-1}\eta C_0\alpha_{2}^6\Phi_j^{(s)}|B_{j,2}^{(s)}|^5C_2\Ecal_j^{(s)} \\
        & \stackrel{\text{\ding{172}}}\leq \frac{1}{\polylog(d)}|B_{j,2}^{(0)}|
    \end{align*} 
    where \ding{172} is due to \myref{coro:TPM}{Corollary}, where \(x_t = |B_{j,1}^{(t)}|\) and \(y_t = |B_{j,2}^{(t)}|\) and \(S_t\leq \frac{1}{\polylog(d)}\), \(y_0\leq  O(\log d)x_0\). which finishes the proof.
\end{proof}

\section{Tensor Power Method Bounds}

In this section, we give two lemmas related to the tensor power method that can help us in previous sections' proofs. 

\begin{lemma}[TPM, adapted from \cite{allen2020towards}]\label{lem:TPM}
    Consider an increasing sequence \(x_t \geq 0\) defined by \(x_{t+1} = x_t + \eta C_t x_{t}^q\) for some integer \(q \geq 3\) and \(C_t > 0\), and suippose for some \(A > 0\) there exist \(t' \geq 0\) such that \(x_{t'}\geq A\). Then for every \(\delta > 0\), and every \(\eta \in (0,1)\):
    \begin{align*}
        \sum_{t\geq 0, x_t\leq A}\eta C_t &\geq \left(\frac{\delta(1+\delta)^{-1}}{(1+\delta)^{q-1} - 1}\left(1 - \left(\frac{(1+\delta)x_0}{A}\right)^{q-1}\right) - \frac{O(\eta A^{q})}{x_0}\frac{\log(A/x_0)}{\log(1+\delta)}\right)\cdot\frac{1}{x_0^{q-1}} \\
        \sum_{t\geq 0, x_t\leq A}\eta C_t &\leq \left(\frac{(1+\delta)^{q-1}}{q-1} + \frac{O(\eta A^{q})}{x_0}\frac{\log(A/x_0)}{\log(1+\delta)}\right)\cdot\frac{1}{x_0^{q-1}} 
    \end{align*}
\end{lemma}
This lemma has a corollary:
\begin{corollary}[TPM, from \cite{allen2020towards}]\label{coro:TPM}
    Let \(q\geq 3\) be a constant and \(x_0, y_0 = o(1)\) and \(A =O(1)\). Let \(\{x_t, y_t\}_{t\geq 0}\) be two positive sequences updated as 
    \begin{itemize}
        \item \(x_{t+1} = x_t + \eta C_t x_t^{q}\) for some \(C_t > 0\);
        \item \(y_{t+1} = y_t + \eta S_t C_t y_t^{q}\) for some \(S_t > 0\).
    \end{itemize}
    Suppose \(x_0 \geq y_0 (\max_{t: x_t \leq A} S_t)^{\frac{1}{q-1}}(1 + \frac{1}{\polylog (d)})\), then \(y_t \leq \widetilde{O}(y_0)\) for all \(t\) such that \(x_t \leq A\). Moreover, if \(x_0 \geq y_0(\max_{t:x_t\leq A}S_t)^{\frac{1}{q-1}}\log(d)\), we would have \(|y_t-y_0|\lesssim \frac{|y_0|}{\polylog(d)}\).
\end{corollary}

Moreover, we prove the following lemma for comparing the updates of different variables.

\begin{lemma}[TPM of different degrees]\label{lem:TPM-degree}
    Consider an increasing sequences \(x_t \geq 0\) defined by \[x_{t+1} = x_t + \eta C_t x_{t}^{q}\] for some integer \(q > q' \geq 3 \) and \(q'\leq q - 2\), and \(C_t > 0\), and further suppose given \(A=O(1)\), there exists \( t' \geq 0, x_{t'} \geq A\). Then for every \(\delta > 0\) and every \(\eta \in (0,1)\):
    \begin{align*}
        \sum_{t\geq 0, x_t\leq A}\eta C_t x_t^{q'} &\leq (1+\delta)^{q'}\left(O(1) + \eta b A^{q}\right)\frac{1}{x_0^{q - q' - 1}} \\
        \sum_{t\geq 0, x_t\leq A}\eta C_t x_t^{q'} &\geq (1+\delta)^{-q'}\left(\delta(1+\delta)^{-1}\frac{1 - (1+\delta)^{-b(q-q'-1)}}{1 - (1+\delta)^{-(q-q'-1)}} - \eta bA^q\right)\frac{1}{x_0^{q-q'-1}}
    \end{align*}
    where \(b = \Theta(\log(A/x_0)/\log(1+\delta))\). When \(A = x_0d^{\Theta(1)}\) , \(\eta = o(\frac{1}{A^q\delta})\) and \(q = O(1)\), then
    \begin{align*}
        \sum_{t\geq 0, x_t\leq A}\eta C_t x_t^{q'} = \Theta(\frac{1}{x_0^{q-q'-1}})
    \end{align*}
\end{lemma}

\begin{proof}
    For every \(g \in 0, 1, \dots \), we define \(\mathcal{T}_g := \min\{t: x_t \geq  (1+\delta)^gx_0\}\). and define \(b := \min\{g: (1+\delta)^g \geq  A\}\), we can write down the following two inequalities according to the update of \(x_t\):
    \begin{align*}
        \sum_{t \in [\mathcal{T}_g, \mathcal{T}_{g+1}]}\eta C_t[(1 + \delta)^gx_0]^{q} &\leq (1+ \delta)x_{\mathcal{T}_{g}} - x_{\mathcal{T}_{g}} + \eta A^{q}  \leq \delta(1+\delta)^gx_0 + \eta A^{q} \\
        \sum_{t \in [\mathcal{T}_g, \mathcal{T}_{g+1}]}\eta C_t[(1 + \delta)^{g+1}x_0]^{q} &\geq (1+ \delta)x_{\mathcal{T}_{g}} - x_{\mathcal{T}_{g}} - \eta A^{q}\geq \delta(1+\delta)^gx_0 - \eta A^{q} 
    \end{align*}
    where \(g+1 \leq b\). Dividing both sides by \([(1 + \delta)^gx_0]^{q- q'}\) in the first inequality and \([(1 + \delta)^{g+1}x_0]^{q- q'}\) in the second, we have 
    \begin{align*}
        \sum_{t \in [\mathcal{T}_g, \mathcal{T}_{g+1}]}\eta C_t[(1 + \delta)^gx_0]^{q'} &\leq \frac{\delta}{(1+\delta)^{g(q-q'-1)}}\frac{1}{x_0^{q-q'-1}} + \frac{\eta A^q}{x_0^{q - q' - 1}} \\
        \sum_{t \in [\mathcal{T}_g, \mathcal{T}_{g+1}]}\eta C_t[(1 + \delta)^{g+1}x_0]^{q'} &\geq \frac{\delta(1+\delta)^{-1}}{(1+\delta)^{(g+1)(q-q'-1)}}\frac{1}{x_0^{q-q'-1}} - \frac{\eta A^q}{x_0^{q - q' - 1}}
    \end{align*}
    Therefore if we sum over \(g = 0,\dots,b\), then 
    \begin{align*}
        \sum_{t\geq 0, x_t\leq A}\eta C_t x_t^{q'} & \leq \sum_{t\geq 0, x_t \leq A}\eta C_t[(1 + \delta)^{g+1}x_0]^{q'}\\
        & = (1+\delta)^{q'}\sum_{t\geq 0, x_t \leq A}\eta C_t[(1 + \delta)^gx_0]^{q'} \\
        &\leq (1+\delta)^{q'}\sum_{0\leq g \leq b}\left(\frac{\delta}{(1+\delta)^{g(q-q'-1)}}\frac{1}{x_0^{q-q'-1}} + \frac{\eta A^q}{x_0^{q - q' - 1}}\right) \\
        &= (1+\delta)^{q'}O\left(\frac{\delta }{(1+\delta)^{q-q'-1} - 1} + \eta b A^q \right)\frac{1}{x_0^{q-q'-1}} \\
        &\leq (1+\delta)^{q'}O\left(\frac{1}{q-q'-1} + \eta b A^q\right)\frac{1}{x_0^{q-q'-1}} 
    \end{align*}
    For the lower bound, we also have 
    \begin{align*}
        \sum_{t\geq 0, x_t\leq A}\eta C_t x_t^{q'} & \geq (1+\delta)^{-q'}\sum_{t\geq 0, x_t \leq A}\eta C_t[(1 + \delta)^{g+1}x_0]^{q'}\\
        & \geq (1+\delta)^{-q'}\sum_{0\leq g \leq b}\left(\frac{\delta(1+\delta)^{-1}}{(1+\delta)^{(g+1)(q-q'-1)}} - \eta A^q\right)\frac{1}{x_0^{q-q'-1}} \\
        &= (1+\delta)^{-q'}\left(\delta(1+\delta)^{-1}\frac{1 - (1+\delta)^{-b(q-q'-1)}}{1 - (1+\delta)^{-(q-q'-1)}} - \eta bA^q\right)\frac{1}{x_0^{q-q'-1}} \\
        & = (1+\delta)^{-q'}\left(\delta(1+\delta)^{-1}\frac{1 - (1+\delta)^{-b(q-q'-1)}}{1 - (1+\delta)^{-(q-q'-1)}} - \eta bA^q\right)\frac{1}{x_0^{q-q'-1}}
    \end{align*}
    Inserting \(b = \Theta(\log(A/x_0)/\log(1+\delta))\) proves the lower bound. For the last one we can choose \(\delta = \frac{1}{\sqrt{\log d}}\) to get:
    \begin{align*}
        b = \Theta(\polylog(d)),\quad \frac{\delta(1 - (1+\delta)^{-b(q-q'-1)})}{1 - (1+\delta)^{-(q-q'-1)}} = \Omega(1),\quad (1+\delta)^{-q'} = \Omega(1), 
    \end{align*}
    which proves the claim.
\end{proof}

\clearpage
\bibliography{contrastive}
\bibliographystyle{plainnat}

\end{document}